\documentclass[3p, 12pt]{elsarticle}

\makeatletter
\def\ps@pprintTitle{%
    \let\@oddhead\@empty
    \let\@evenhead\@empty
    \let\@evenfoot\@oddfoot
    }
\makeatother
\usepackage{tabularx}
\usepackage{amsmath,amssymb,amsbsy,amsfonts,amsthm}
\usepackage{mathtools}
\usepackage{mathrsfs}
\usepackage{bm}
\usepackage{graphicx}
\usepackage[final]{hyperref}
\usepackage[ruled,vlined]{algorithm2e}
\usepackage[shortlabels]{enumitem}
\usepackage{array}
\usepackage{multirow}
\usepackage{makecell}
\usepackage[title]{appendix}
\usepackage{empheq}
\usepackage{cleveref}
\usepackage{cases}
\usepackage{textcomp}
\usepackage{gensymb}
\usepackage{subcaption}
\hypersetup{
	colorlinks=true,       
	linkcolor=blue,        
	citecolor=blue,        
	filecolor=magenta,     
	urlcolor=blue         
}
\usepackage{float}
\usepackage{multirow}
\usepackage[normalem]{ulem}

\usepackage{tikz}
\usetikzlibrary{calc,patterns,decorations.pathreplacing,calligraphy}
\usetikzlibrary{shapes,positioning}
\tikzstyle{bag} = [align=center]

\usepackage[colorinlistoftodos,prependcaption,textsize=normalsize]{todonotes}
\usepackage{regexpatch}

\usepackage{tcolorbox} 
\usepackage{xcolor} 

\makeatletter
\xpatchcmd{\@todo}{\setkeys{todonotes}{#1}}{\setkeys{todonotes}{inline,#1}}{}{}
\makeatother

\usepackage{wrapfig}

\theoremstyle{plain}
\newtheorem{theorem}{Theorem}[section]
\newtheorem{lemma}[theorem]{Lemma}
\newtheorem{corollary}[theorem]{Corollary}
\newtheorem{proposition}[theorem]{Proposition}

\newtheorem{assumption}[theorem]{Assumption} 

\theoremstyle{definition}
\newtheorem{definition}[theorem]{Definition}

\theoremstyle{remark}
\newtheorem{remark}{Remark}[section]

\newcommand{\bC}{\mathbb{C}}
\newcommand{\bR}{\mathbb{R}}
\newcommand{\bE}{\mathbb{E}}
\newcommand{\bT}{\mathbb{T}}

\newcommand{\bN}{\mathbb{N}}
\newcommand{\bZ}{\mathbb{Z}}

\newcommand{\cA}{\mathcal{A}}

\newcommand{\cC}{\mathcal{C}}

\newcommand{\cE}{\mathcal{E}}
\newcommand{\cF}{\mathcal{F}}
\newcommand{\cG}{\mathcal{G}}
\newcommand{\cH}{\mathcal{H}}
\newcommand{\cI}{\mathcal{I}}
\newcommand{\cJ}{\mathcal{J}}
\newcommand{\cK}{\mathcal{K}}
\newcommand{\cL}{\mathcal{L}}
\newcommand{\cM}{\mathcal{M}}
\newcommand{\cN}{\mathcal{N}}
\newcommand{\cO}{\mathcal{O}}
\newcommand{\cP}{\mathcal{P}}
\newcommand{\cQ}{\mathcal{Q}}
\newcommand{\cR}{\mathcal{R}}
\newcommand{\cS}{\mathcal{S}}
\newcommand{\cT}{\mathcal{T}}

\newcommand{\cX}{\mathcal{X}}
\newcommand{\cY}{\mathcal{Y}}
\newcommand{\cZ}{\mathcal{Z}}

\newcommand{\HS}{\mathrm{HS}}
\newcommand{\Op}{\mathscr{L}}

\newcommand{\dx}{\mathrm{d}x}

\newcommand{\real}{\mathrm{Re}}
\newcommand{\imag}{\mathrm{Im}}

\newcommand{\id}{\mathrm{Id}}

\newcommand{\zlex}{\mathscr{K}_{+}}

\newcommand{\smootha}{\mathbb{A}_{\infty}}
\newcommand{\srla}{\mathbb{G}_{\infty}}

\newcommand{\ift}{\mathrm{IFT}}

\newcommand{\tcutoff}{\mathrm{cut}}
\newcommand{\tabsval}{\mathrm{abs}}
\newcommand{\tclip}{\mathrm{clip}}

\newcommand{\cutoff}{\tcutoff}
\newcommand{\fcutoff}{f_{\tcutoff}}
\newcommand{\clip}{f_{\tclip}}
\newcommand{\absval}{f_{\tabsval}}
\newcommand{\funcut}{F_{\tcutoff}}
\newcommand{\sigmoid}{\Phi}

\newcommand{\cont}{C} 
\newcommand{\const}{\mathscr{C}} 

\DeclareMathOperator*{\argmin}{arg\,min}

\newcommand{\dobs}{d_{\mathrm{obs}}}
\newcommand{\yobs}{y_{\mathrm{obs}}}
\newcommand{\noise}{\mathrm{noise}}

\Crefname{equation}{}{}

\begin{document}
\begin{frontmatter}
\title{Derivative-Informed Fourier Neural Operator: Universal Approximation and Applications to PDE-Constrained Optimization}
\author[1]{Boyuan Yao\corref{joint}}
\ead{johnyao@utexas.edu}

\author[2]{Dingcheng Luo\corref{joint}}
\ead{dingcheng.luo@qut.edu.au}

\author[3]{Lianghao Cao}
\ead{lianghao@caltech.edu}

\author[4]{Nikola Kovachki}
\ead{nkovachki@nvidia.com}

\author[5]{Thomas O'Leary-Roseberry}
\ead{oleary-roseberry.1@osu.edu}

\author[1,6]{Omar Ghattas}
\ead{omar@oden.utexas.edu}

\address[1]{Oden Institute for Computational Engineering and Sciences, The University of Texas at Austin, 201 E 24th St, Austin, TX 78712, USA.
}

\address[2]{School of Mathematical Sciences and Centre for Data Science, Queensland University of Technology, 2 George St, Brisbane City, QLD 4000, Australia.}

\address[3]{Department of Computing and Mathematical Sciences,
California Institute of Technology, 1200 E. California Blvd.,
Pasadena, CA 91125, USA.}

\address[4]{Nvidia Corporation, 2788 San Tomas Expressway, Santa Clara, CA 95051, USA.}

\address[5]{Department of Mathematics, The Ohio State University, 100 Math Tower, 231 W 18th Ave, Columbus, OH 43210, USA.
}

\address[6]{Walker Department of Mechanical Engineering, The University of Texas at Austin, 204 E. Dean Keeton St, Austin, TX 78712, USA.
}

\cortext[joint]{These two authors contributed equally to this work as joint first authors.}

\begin{abstract}
We present approximation theories and efficient training methods for derivative-informed Fourier neural operators (DIFNOs) with applications to PDE-constrained optimization. A DIFNO is an FNO trained by minimizing its prediction error jointly on output and Fr\'echet derivative samples of a high-fidelity operator (e.g., a parametric PDE solution operator). As a result, a DIFNO can closely emulate not only the high-fidelity operator's response but also its sensitivities. To motivate the use of DIFNOs instead of conventional FNOs as surrogate models, we show that accurate surrogate-driven PDE-constrained optimization requires accurate surrogate Fr\'echet derivatives. Then, we establish (i) simultaneous universal approximation of continuously differentiable operators and their Fréchet derivatives by FNOs on compact sets, and (ii) universal approximation of continuously differentiable operators by FNOs in weighted Sobolev spaces with input measures that have unbounded supports. Our theoretical results certify the capability of FNOs for accurate derivative-informed operator learning and for the solution of PDE-constrained optimization problems. Furthermore, we develop efficient training schemes that leverage dimensionality reduction and multi-resolution techniques to significantly reduce memory and computational costs in Fr\'echet derivative learning. Numerical examples on nonlinear diffusion––reaction, Helmholtz, and Navier--Stokes equations demonstrate that DIFNOs are superior in sample complexity for operator learning and solving infinite-dimensional PDE-constrained inverse problems, achieving high accuracy at low training sample sizes.
\end{abstract}
\begin{keyword}
Operator learning, PDE-constrained optimization, universal approximation, surrogate modeling, inverse problems

\vspace{-15 pt}
\hspace{1pt}
\end{keyword}
\end{frontmatter}

\section{Introduction}\label{section:intro}
We consider the approximation of operators on function spaces, particularly solution operators of parametric partial differential equations (PDEs). Let $a \in \cX$ denote the parametric input, e.g., coefficient fields and initial conditions, and $u \in \cY$ denote the PDE solution where $\cX$ and $\cY$ are normed vector spaces. A \textit{PDE solution operator} $\cG:\cX \rightarrow \cY$ is defined as
\begin{equation}\label{eq:operator_definition}
    \mathcal{G}:\cX\ni a \mapsto u\in\cY\qquad \text{such that } \qquad \cM(\mathcal{G}(a),a) = 0 \in \cY',
\end{equation}
where $\cY'$ denotes the topological dual of $\cY$, and $\cM:\cY\times\cX\to\cY'$ is the PDE residual. For PDEs that describe nonlinear multiscale systems, the numerical evaluation costs of $\cG$ at a single instance of the input $a$ can be high due to (i) a high-dimensional solution space after discretizing $\cY$ and (ii) use of iterative procedures for time stepping and residual minimization. Furthermore, rapid and repeated evaluations of $\cG$ at different inputs are necessary for solving inverse problems \cite{kaipio2005statistical, stuart2010inverse, ghattas2021learning, nelsen2025operator} and optimal control problems \cite{troltzsch2010optimal, manzoni2021optimal} that are formulated as \textit{PDE-constrained optimization} problems \cite{biegler2003large, hinze2009optimization, kouri2018optimization, antil2018brief}. The task is to find an input to the PDE model that leads to desirable features of the PDE solution through optimization. In these cases, it is advantageous to first invest in the construction of a fast-to-evaluate approximation of the operator $\cG$, denoted $\widetilde{\cG}:\cX\times \mathbb{R}^{d_{w}}\to\cY$ with weights $w\in\mathbb{R}^{d_{w}}$. We refer to $\widetilde{\cG}_w(\cdot)\coloneqq \widetilde{\cG}(\cdot, w)$ as an \textit{operator surrogate} and $\cG$ as the \textit{high-fidelity operator}. The operator surrogate $\widetilde{\cG}_w$ is constructed by minimizing a surrogate discrepancy measure (e.g., mean squared error) to find the optimal set of weights $w^{\dagger}$, and then deployed in place of the high-fidelity operator $\cG$ for procuring fast but approximate PDE solutions.

We limit our attention to a class of neural network-based operator surrogates called \textit{Fourier neural operators} (FNOs) \cite{li2021fourier}. Classical developments of neural networks have primarily focused on learning mappings between finite-dimensional Euclidean spaces or finite sets. In contrast, neural operators \cite{kovachki2023nueral,berner2025principled} are designed for function spaces and enjoy two key advantages: (i) they are discretization-agnostic, i.e., using the same weights across different discretizations of the function spaces, and (ii) they are universal approximators of nonlinear continuous operators. An FNO is a neural operator that parametrizes integral kernels via coefficients in an expansion onto a Fourier basis, thereby enabling efficient computation through Fast Fourier Transform (FFT) algorithms. There has been considerable interest in the last few years in developing approximation theory for FNOs \cite{kovachki2021universal, kim2024bounding, koshizuka2024understanding, lanthaler2025nonlocality, lanthaler2025discretization}, extending the FNO architecture \cite{wen2022ufno, you2022ifno, li2023geofno, huang2025fnm}, developing novel FNO training schemes \cite{li2024physics, li2023gino}, and applying FNOs to tackle modeling and computation challenges in science and engineering \cite{kurth2023fourcastnet, wen2023real, azizzadenesheli2024neural, bhattacharya2024learning, bhattacharya2025learning, wang2025coarse}.

The focus of this work is the integration of FNOs with \textit{derivative-informed operator learning} \cite{oleary2024dino, cao2025derivative, luo2025dis}, an emerging paradigm for supervised learning of operator surrogates. Conventional operator learning \cite{kovachki2023nueral, kovachki2024operator} minimizes the mean-squared prediction error on output samples of the high-fidelity operator:
\begin{equation}\label{eq:mse_loss}
    w^{\dagger}\in\argmin_{w} \frac{1}{N}\sum_{i=1}^N \left\|u^{(i)}- \widetilde{\cG}_w\left(a^{(i)}\right)\right\|_{\cY}^2,\quad a^{(i)}\sim \mu,\quad u^{(i)}=\cG\left(a^{(i)}\right),
\end{equation}
where $\mu$ is a data-generating probability measure on $\cX$, and $\|\cdot\|_{\cY}$ denote the norm on $\cY$. 
While this formulation is flexible and applies to black-box models, one may further exploit the structure of $\cG$ to derive formulations with superior performance; see, e.g., physics-informed operator learning \cite{li2024physics, goswami2023physics}. We are interested in derivative-informed operator learning, which exploits the structure of $\cG$ by incorporating its \textit{Fr\'echet derivative}, denoted as $D\cG$, into the operator learning formulation:
\begin{equation}\label{eq:dino_loss_samples}
\begin{gathered}
        a^{(i)}\sim\mu,\quad u^{(i)}=\cG\left(a^{(i)}\right),\quad J^{(i)} = D\cG\left(a^{(i)}\right),\\
        w^{\dagger}\in\argmin_{w}\frac{1}{N}\sum_{i=1}^N \left(\left\|u^{(i)} -\widetilde{\cG}_w\left(a^{(i)}\right)\right\|^2_\cY + \left\|J^{(i)} - D\widetilde{\cG}_w\left(a^{(i)}\right)\right\|^2_{\text{HS}(\cX, \cY)} \right),
\end{gathered}
\end{equation}
where $\|\cdot\|_{\text{HS}(\cX, \cY)}$ denotes the Hilbert--Schmidt norm for linear operators between $\cX$ and $\cY$, and $D\widetilde{\cG}_w$ denotes the Fr\'echet derivative of $\widetilde{\cG}_w$. In practice, derivative samples can be generated using solutions to the direct or adjoint sensitivity equations \cite{giles2000introduction, ghattas2021learning, cao2025derivative} or using differentiable PDE solvers \cite{griewank2008evaluating, deavilabelbuteperes2018endtoend, holl2024phiflow}, often at low incremental cost relative to generating the solutions $u^{(i)}$. A precise definition of derivative-informed operator learning considered in this work is given in \Cref{subsec:dino}.

Derivative-informed learning of neural operators (DINO) offers two key practical advantages over the conventional method. First, it consistently improves training efficiency, achieving high generalization accuracy with a significantly lower training sample size and generation cost \cite{oleary2024dino, qiu2024derivative, cao2024lazydino, cao2025derivative, luo2025efficient, go2025accurate}. Second, it directly controls errors in the Fr\'echet derivative approximation, thereby enhancing the performance in a variety of downstream tasks such as Bayesian inverse problems \cite{cao2024lazydino, cao2025derivative}, optimal design under uncertainty \cite{luo2025efficient}, and optimal experimental design \cite{go2025accurate, go2025sequential}. Existing studies on DINO mainly focused on reduced-basis neural operators, such as PCANet \cite{hesthaven2018pod, bhattacharya2021pca, luo2025dis}, DIPNet \cite{oleary2022dipnet, luo2025dis, cao2024lazydino, cao2025derivative, go2025accurate}, and DeepONet \cite{lu2021deep, qiu2024derivative}. The FNOs differ from these architectures because they do not involve linear input projection and output reconstruction steps. It has empirically been shown that they are more efficient for PDEs exhibiting hyperbolic and advection-dominated behavior \cite{lanthaler2022nonlinear}. Moreover, among alternative neural operator architectures that share this feature, such as graph-based neural operators \cite{li2020multipole,li2020gno}, FNOs offer higher computational efficiency through the FFT. It is therefore of interest to extend the DINO framework to FNOs.

\subsection{Contributions}
This work develops approximation theory and efficient training schemes for derivative-informed FNO (DIFNO) with applications to PDE-constrained optimization. The contributions of this work are listed below.

\begin{itemize}

    \item[(C1)] \textit{\Cref{prop:surrogate_optimization_error,prop:surrogate_optimization_error_strongly_convex}: Accurate surrogate-driven solutions of PDE-constrained optimization require accurate surrogate Fr\'echet derivative.} We show that errors in surrogate-driven solutions of PDE-constrained optimization are controlled by a sum of pointwise errors of the operator surrogate $\widetilde{\cG}_w$ and its Fr\'echet derivative $D\widetilde{\cG}_w$. We first establish this result for the residual errors measured through the high-fidelity objective function gradient at surrogate solutions. In regions where the high-fidelity objective is strongly convex, this result is extended to hold for the error between the high-fidelity and surrogate optimization solutions. These findings highlight that, for surrogate-driven optimization to succeed, the operator surrogate must closely emulate the high-fidelity operator's sensitivities, motivating our study of DIFNOs.
    
    \item[(C2)]\textit{\Cref{theorem:main_ua}: Simultaneous universal approximation of continuously differentiable operators and their Fr\'echet derivatives by FNOs on compact sets.}
    We show that FNOs with smooth activation functions can simultaneously approximate a continuously differentiable operator and its Fr\'echet derivative on compact input sets uniformly. Our analysis is based on the strategy presented in \cite{kovachki2021universal}, which considers the universal approximation of a continuous operator. We extend several of the approximation results in \cite{kovachki2021universal} to simultaneously consider both the approximation of an operator and its Fr\'echet derivative. This result certifies the approximation capabilities of FNOs in the setting of derivative-informed operator learning, providing a firm foundation for deriving further theoretical results. For instance, it supplies the conditions for proving that FNOs have the capacity to solve inverse problems arbitrarily well (\Cref{theorem:inverse_problem_ua}).
    
    \item[(C3)]\textit{\Cref{theorem:main_ua_mu}: Simultaneous universal approximation of continuously differentiable operators and their Fr\'echet derivatives by FNOs in weighted Sobolev space with input measures that have unbounded topological supports.}
    Using a novel cutoff argument, we extend the result of \Cref{theorem:main_ua} to the setting where the errors are measured using weighted Sobolev norms with input measures that have unbounded topological supports (e.g., Gaussian measures).
    We show that FNOs with \textit{GELU-like} activation functions (\Cref{def:activation}) can approximate continuously differentiable operators that have locally Lipschitz derivatives, under moment assumptions on the measure, operator, and local Lipschitz constants (\Cref{assumption:measurability}) similar to those in \cite{bhattacharya2021pca}.
    In particular, our cutoff argument simultaneously addresses the FNO and its Fr\'echet derivatives, whereas previous approaches (e.g., \cite{lanthaler2022error}) typically only handle the operator approximation. This result establishes a more practical foundation for DIFNOs, as it is consistent with numerical implementation, where the training and testing errors are measured using these weighted Sobolev norms for operators.
    
    \item[(C4)] \textit{Efficient methods for DIFNO training.} The computational and memory cost of DIFNO training increases rapidly with the grid size and can be prohibitive under a limited training budget. We propose two techniques for derivative compression that alleviate these costs: dimension reduction and multi-resolution. We note that dimension reduction is commonly used in derivative-informed learning of reduced-basis neural operators \cite{oleary2024dino, qiu2024derivative, luo2025dis}, but the multi-resolution technique has not been explored. These approaches eliminate the quadratic dependence of training costs on the high-fidelity grid size, resulting in vastly improved efficiency in derivative-informed FNO training, i.e., achieving the same accuracy at significantly lower cost.

    \item[(C5)] \textit{High accuracy in learning PDE solution operators and solving inverse problems.} 

    Supported by numerical results on nonlinear diffusion--reaction, Helmholtz, and Navier--Stokes PDEs, we show that derivative-informed operator learning consistently improves the accuracy of FNOs in approximating PDE solution operators compared to conventional operator learning at various training sample sizes. We also demonstrate that the high accuracy of DIFNOs translates into high accuracy in solving inverse problems, where the goal is to recover the input parameter $a$ from sparse observations of the PDE solution $u$ using a trained DIFNO in place of the PDE model. These results extend previous findings on derivative-informed learning of reduced basis neural operators \cite{oleary2024dino, cao2024lazydino, qiu2024derivative, cao2025derivative, luo2025efficient, go2025accurate} and DeepONet \cite{qiu2024derivative} to FNOs. These results also agree with our theoretical findings in more general and practical settings.

\end{itemize}

\subsection{Outline of the Paper}
The remainder of this paper is organized as follows. Section 2 introduces the mathematical framework for operator learning and PDE-constrained optimization, including an error bound for surrogate-driven optimization and an overview of the derivative-informed operator learning formulation. Section 3 defines FNOs and presents our universal approximation result on compact sets, along with its implications for inverse problems solved by FNOs. Section 4 extends the analysis to weighted Sobolev spaces. Section 5 introduces efficient training strategies. Section 6 presents numerical results for operator learning and inverse problems. Section 7 summarizes the findings and discusses directions for future work. The appendices contain proofs of our theoretical results.

\section{Operator Learning and PDE-Constrained Optimization}\label{section:background}
\subsection{Mathematical Setting and Notation}
We begin by introducing the mathematical setting and notation. Our work is focused on the approximation of nonlinear operators $\cG : \cX \rightarrow \cY$, where $\cX$ and $\cY$ are separable Hilbert spaces. 
For a Hilbert space $\cX$, let $\cX'$ denote its topological dual, and $\langle \cdot, \cdot \rangle_{\cX}$ denote the inner product on $\cX$ which induces the norm $\| \cdot \|_{\cX}$. 

\paragraph{Continuous Linear Operators}Let $\Op(\cX, \cY)$ denote the space of continuous linear operators from $\cX$ to $\cY$ with the operator norm given by
\begin{equation*}
    \|\cA\|_{\Op(\cX, \cY)} = \sup_{v \in \cX, v \neq 0} \frac{\|\cA v\|_{\cY}}{\|v\|_{\cX}}.
\end{equation*}
Additionally, let $\HS(\cX, \cY)$ denote the space of Hilbert--Schmidt operators with the Hilbert--Schmidt norm given by
\begin{equation*}
    \|\cA\|_{\HS{(\cX, \cY})} = \left(  
    \sum_{k=1}^{\infty} \|\cA \psi_k \|_{\cY}
    \right)^{1/2},
\end{equation*}
where $(\psi_k)_{k=1}^{\infty}$ is any orthonormal basis of $\cX$.
Analogously, we define the space of continuous $m$-linear operators from $(\cX)^m$ to $\cY$, 
denoted $\Op_m(\cX, \cY)$, which has the norm
\begin{equation*}
    \|\cA\|_{\Op_m(\cX, \cY)} := \sup_{\substack{v_k \in \cX, v_k \neq 0 \\ k = 1 \dots m}} \frac{\|\cA(v_1, \dots, v_m) \|_{\cY}}{\|v_1\|_{\cX} \cdots \|v_k\|_{\cX}},
\end{equation*}
where $\Op(\cX,\cY) = \Op_{1}(\cX,\cY)$.
Their Hilbert--Schmidt counterpart, $\HS_{m}(\cX, \cY)$, has the norm 
\begin{equation}\label{eq:hilbert-schmidt-multi}
    \|\cA\|_{\HS_m(\cX, \cY)} := \left(\sum_{k_1, \dots, k_m = 1}^{\infty} \|\cA(\psi_{k_1}, \dots, \psi_{k_m})\|_{\cY}^2\right)^{1/2}.
\end{equation}
\paragraph{Continuously Differentiable Operators}
For an open set $U \subset \cX$, let $\cont^{m}(U;\cY)$ denote the space of $m$-times continuously Fr\'echet differentiable operators from $U$ to $\cY$,
along with the shorthand notation $\cont(U; \cY) := \cont^{0}(U;\cY)$ and $\cont^{m}(U) := \cont^{m}(U;\bR)$. For an operator $\cG :\cX \rightarrow \cY$ that is $m$-times differentiable at $a \in \cX$,  its $m$-th derivative is denoted as $D^m\cG(a) \in \Op_m(\cX, \cY)$. 
Note that when $\cX$ and $\cY$ are finite-dimensional, $D^m \cG$ corresponds to a finite-dimensional tensor, and hence any equivalent tensor norm (e.g., Frobenius, $\ell^{\infty}$) can be used.
In this work, when not specified, 
the terms \textit{derivative} and \textit{differentiability} are assumed to be in the Fr\'echet sense.
For any compact subset $K \subset U$, we also define the $\cont^{m}(K;\cY)$ norm as 
\[
    \| \cG \|_{\cont^{m}(K;\cY)} := \max_{k \in \{ 0, \dots, m \}} \sup_{a \in K} \| D^k \cG(a) \|_{\Op_k(\cX;\cY)}.
\]

\paragraph{Hilbert Scales} For a given Hilbert space $\cX$, we are also interested in the notion of Hilbert scales generated by an unbounded, self-adjoint operator $\cA : D(\cA) \subset \cX \rightarrow \cX$.
That is, for $s \geq 0$, 
we consider the spaces $\cX_{s} := D(\cA^{s})$, where $\cX_0 = \cX$. These are separable Hilbert spaces with the inner product $\langle u, v \rangle_{\cX_{s}} = \langle \cA^{\frac{s}{2}}u, \cA^{\frac{s}{2}} v \rangle_{\cX}$.
For $s' \geq s$, we have the continuous embedding $\cX_{s'} \hookrightarrow \cX_{s}$. Much of our theory will be developed using Hilbert scales $H^{s}(\bT^d)$ as input/output spaces, i.e., fractional order Sobolev spaces on the $d$-dimensional torus. Here, by default, we assume that the function spaces are real-valued. 
First, let $L^p(\bT^d) := L^p(\bT^d; \bR)$, $p \in [1, \infty]$, denote the Lebesgue spaces on the torus $\bT^d$ identified with $[0, 2\pi]^d$. For $u \in L^2(\bT^d)$, we define its Fourier transform as
\begin{equation*}
    \hat{u}(k) := \cF(u)(k) = \frac{1}{(2\pi)^d} \int_{\bT^d} u(x) \exp(-i\langle k, x\rangle) dx, 
\end{equation*}
where $\cF : L^2(\bT^d) \rightarrow \ell^2(\bZ^d; \bC)$ and $k = (k_1, \dots, k_d) \in \bZ^d$. The inverse transform $\cF^{-1} :  \ell^2(\bZ^d; \bC) \rightarrow L^2(\bT^d)$ is given by
\begin{equation*}
    u(x) := \cF^{-1}(\hat{u})(x) = \sum_{k \in \bZ} \hat{u}(k) \exp(i \langle k , x\rangle),
\end{equation*}
where we have used $\langle \cdot, \cdot \rangle$ to denote the standard Euclidean inner product on $\bR^d$.
Note that for real-valued functions, the Fourier coefficients satisfy the constraint $\hat{u}(k) = \overline{\hat{u}(-k)}$, 
where $\overline{z}$ denotes the complex conjugate of $z \in \bC$. We then define the norm $\| \cdot \|_{H^s}$ as
\begin{equation*}
    \| u \|^2_{H^s} := (2\pi)^d \sum_{k \in \bZ^d} (1 + |k|^{2})^{s} |\hat{u}(k)|^2.
\end{equation*}
The spaces $H^s(\bT^d) := H^s(\bT^d; \bR)$ are then defined for $s \geq 0$ as
\begin{equation}\label{eq:real_sobolev}
    H^s(\bT^d) := \{ u \in L^2(\bT^d) : \|u \|_{H^s} < \infty \}.
\end{equation}
In particular, the $H^s(\bT^d)$ spaces form a Hilbert scale generated by the operator $\cA = \cI - \Delta$ over the domain $D(\cA) = H^2(\bT^d)$, where $\cI$ and $\Delta$ denote the identity and Laplacian operators, respectively.
Moreover, for $s \geq 0$, the scale can be shifted such that 
\begin{equation}\label{eq:shifted_X}
    \cX_{\delta} \coloneqq H^{s+\delta}(\bT^d)
\end{equation}
is also a Hilbert scale with $\delta \geq 0$.
Additional useful properties of the $H^s$ spaces are presented in \Cref{sec:notes_on_periodic_sobolev_spaces}.

\paragraph{Spectral Projection and Truncated Fourier Transform} As in \cite{kovachki2021universal}, given $N \in \bN$, we will denote by $L^2_N(\bT^d)$ the space of trigonometric polynomials with Fourier modes up to $|k|_{\infty} \leq N$, i.e.~$u_N \in L^2(\bT^d)$ is of the form 
\[
    u_N(x) = \sum_{|k|_{\infty} \leq N} \hat{u}(k) \exp(i \langle k, x \rangle),
\]
where $|k|_{\infty} := \max(|k_1|, \dots, |k_d|)$. We will also use $H^s_{N}(\bT^d)$ to denote $L^2_{N}(\bT^d)$ equipped with the $H^s(\bT^d)$ norm.
We denote the $L^2$-orthogonal projection onto $L^2_N(\bT^d)$ as $\cP_N : L^2(\bT^d) \rightarrow L_N^2(\bT^d)$,
which can be explicitly written as 
\begin{equation*}
    (\cP_N u)(x) = \sum_{|k|_{\infty} \leq N} \hat{u}(k) \exp(i \langle k , x \rangle), \qquad \hat{u}(k) =(\cF u)(k).
\end{equation*}
Additionally, we will define the operators $\cF_N : L^2(\bT^d) \rightarrow \bC^{K_N}$ and $\cF_{N}^{-1} : \bC^{K_N} \rightarrow L^2(\bT^d)$ 
as the truncated Fourier transforms, i.e.~
\begin{align*}
    (\cF_N u)(k) = (\cF u)(k) \quad |k|_{\infty} \leq N,\quad \text{and}\quad\cF_N^{-1}(\hat{u})(x) = \sum_{|k|_{\infty} \leq N} \hat{u}(k) \exp(i \langle k, x \rangle).
\end{align*}
Here, $K_N$ is the number of elements in the set $\{ k = (k_1, \dots, k_d) : |k|_{\infty} \leq N \}.$
We can also identify $\bC^{K_N}$ with $\bR^{2K_N}$, so that $\cF_N$ and $\cF_N^{-1}$ can be viewed as mapping to and from $\bR^{2K_N}$.

\subsection{Operator Surrogate Driven PDE-Constrained Optimization}

We motivate derivative-informed learning of operator surrogates in the setting where the surrogate substitutes for the PDE model in solving PDE-constrained optimization. Let us consider the following constrained minimization problem defined through the functional $F:\cY\times\cX\rightarrow \overline{\mathbb{R}}$:
\begin{equation} \label{eq:det_opt_constrained}
    \inf\limits_{a\in\cX} F(u, a)\quad \text{such that}\quad \cM(u,a)=0
\end{equation}
where $\cM:\cY\times \cX\to\cY'$ is the PDE residual, and $\overline{\mathbb{R}}$ is the extended real number line. This constrained optimization problem arises in inverse problems and optimal control. For example, an inverse problem often has the following form:
\begin{equation}\label{eq:inverse_problem}
    F(u, a) = \frac{1}{2}\left\|{\Gamma_{\noise}^{-1/2}}\left(\cH(u, a) - \yobs \right)\right\|^2 + \mathcal{R}(a),
\end{equation}
where $\yobs \in \bR^{\dobs}$ is the observation data, $\mathcal{H}:\cY\times \cX\to\bR^{\dobs}$ is an observation operator that maps the input and the PDE solution to the predicted observations, $\Gamma_{\mathrm{noise}} \in \bR^{\dobs \times \dobs}$ is a noise covariance matrix, and $\mathcal{R}:\cX\to\overline{\bR}$ is a regularization functional.

In the setting of operator learning, we often write the constrained optimization problem in \Cref{eq:det_opt_constrained} as an unconstrained optimization problem using the PDE solution operator $\cG$. The unconstrained problem is defined through the objective function $f:\cX\rightarrow \overline{\mathbb{R}}$:
\begin{equation} \label{eq:det_opt}
    \inf\limits_{a\in\cX} f(a),\quad f(a)\coloneqq F(\cG(a), a).
\end{equation}
When $f\in C^1(U)$, where $U\subset \cX$ is a non-empty open set, any minimizer in $U$ is a stationary point. This means that an optimization solution $a^{\star}\in U$, when it exists, satisfies
\begin{equation}\label{eq:local_minimizer}
    Df(a^\star) = 0\quad \in \cX'.
\end{equation}
The stationary point can be found via gradient-based optimization iterations starting from an initial guess $a_0\in \cX$:
\begin{equation*}
    a_{k+1} = a_k - P_k Df(a_k),\quad k\in \bN_0,
\end{equation*}
where $P_k \in \Op(\cX', \cX)$ is an operator that returns the update given the derivative and could possibly also depend on the history $\{a_i\}_{i=0}^k$. Due to the costs associated with the evaluation of $f$ and $Df$, which involve PDE solutions and sensitivity computations, it is advantageous to substitute an operator surrogate $\widetilde{\cG}_w$ for $\cG$ in the objective function, leading to the approximate optimization problem:
\begin{equation}\label{eq:fw_approx_optimization}
    \inf\limits_{a\in\cX} \widetilde{f}_w(a),\quad \widetilde{f}_w(a) = F(\widetilde{\cG}_w(a), a).
\end{equation}
Similarly, when $\widetilde{f}_w\in C^1(U)$ and there exists a minimizer $a^{\dagger}\in U$, we have
\begin{equation*}
    D\widetilde{f}_w(a^{\dagger})=0 \quad \in \cX'.
\end{equation*}

We are interested in understanding how errors in the operator surrogate approximation affect the optimization solutions via the following quantities.
\begin{enumerate}[(i)]
    \item \textit{The residual error} $\|Df(a^{\dagger})\|_{\cX'}$ of a surrogate-driven optimization solution $a^{\dagger}$ according to the high-fidelity objective function $f$.
    \item \textit{The solution error} $\|a^{\dagger}-a^{\star}\|_{\cX}$ between a solution $a^{\star}$ of the high-fidelity objective $f$ and a nearby solution $a^{\dagger}$ of surrogate objective $\widetilde{f}_w$.
\end{enumerate}
One key observation is that while controlling the pointwise errors of the operator surrogate $\widetilde{\cG}_w$ is often sufficient for controlling the error of the surrogate objective $\widetilde{f}_w$, it is insufficient for the errors of the surrogate-driven optimization solutions that are stationary points. The distinction between the two errors is crucial for neural operator surrogates, as they tend to introduce non-convexity into the objective function $\widetilde{f}_w$, leading to the discovery of potentially spurious local minimizers rather than a global one. To illustrate this point further, we provide a simple example below.

\paragraph{Example: Pointwise Convergence of the Objective is Insufficient}
Consider the sequence of one dimensional functions $f_k:[-1,1] \to \mathbb{R}$ given by
\begin{equation*}
    f_k(x) = x^2 + \frac{1}{k}\sin(k\pi x).
\end{equation*}
We have $\lim\limits_{k\rightarrow \infty} f_k(x) = x^2$ uniformly. Each $f_k$ is nonconvex, and yet the limit is strongly convex. Here, the limiting global minimizer converges; however, pre-asymptotically, any given choice of $f_k$ has many local minima, which may be severely sub-optimal. Here, the function $f_k$ is an analogue of the surrogate objectives given by a sequence of operator surrogates with decreasing pointwise errors.

\begin{figure}[h]
\center
    \begin{tikzpicture}[node distance=8cm, auto]
    \node (img1) {\includegraphics[width=6cm]{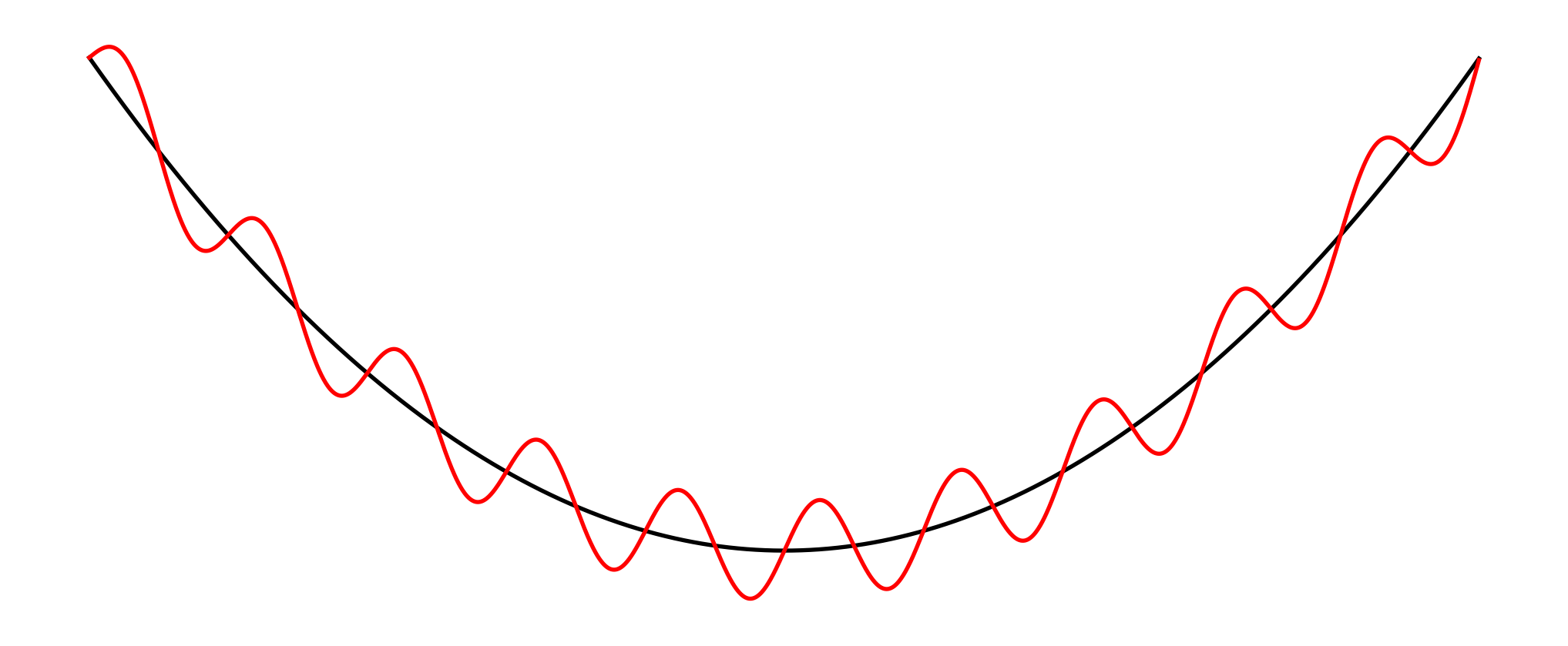}};
    
    \node (img2) [right of=img1] {\includegraphics[width=6cm]{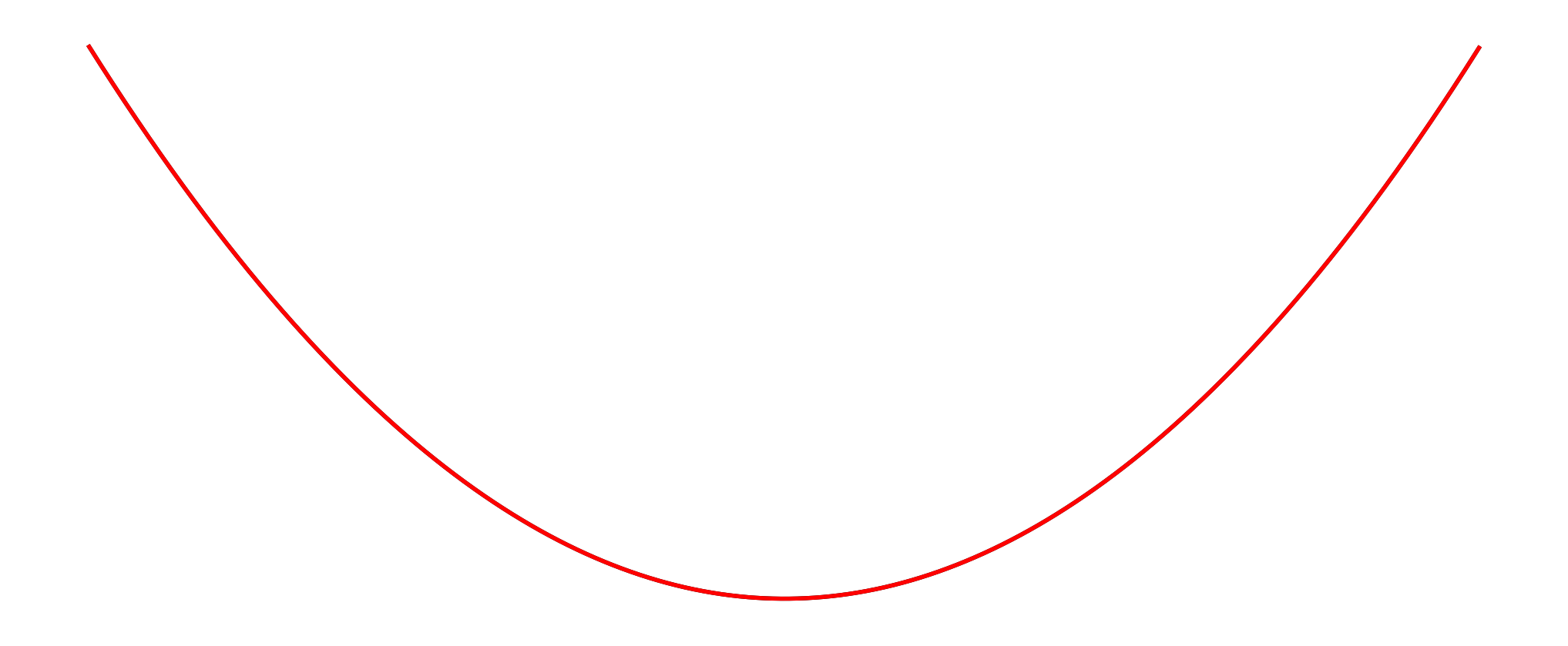}};
    
    \draw[->, thick] (img1) -- (img2);
    \end{tikzpicture}
    \caption{An illustration of a uniformly converging sequence of $f_k(x) = x^2 + \frac{1}{k}\sin(k\pi x)$, where each function is nonconvex and has many local minima which are far away from being optimal with respect to the limiting function, which is strongly convex.}
    \label{fig:local_convergence_pathology}
\end{figure}

In the following proposition, we show that controlling the residual error of the surrogate-driven optimization requires controlling the pointwise errors of the operator surrogate $\widetilde{\cG}_w$ and its Fr\'echet derivative $D\widetilde{\cG}_w$. The proof of this result is shown in \Cref{sec:proof_optimization_error}.

\begin{proposition}[Surrogate-Driven Optimization Errors]\label{prop:surrogate_optimization_error}
Let $\cX$ and $\cY$ be Banach spaces. Furthermore, let $\cG, \widetilde{\cG}_{w} \in C^1(\cX;\cY)$ be continuously differentiable operators and let $F : \cY \times \cX \rightarrow \bR$ be continuously differentiable with a Lipschitz continuous derivative such that $f, \widetilde{f}_w$ are as defined in \eqref{eq:det_opt} and \eqref{eq:fw_approx_optimization}, respectively.
Consider a stationary point $a^{\dagger}$ of $\widetilde{f}_w$, i.e., $D \widetilde{f}_{w}(a^{\dagger}) = 0$. 

Then, the residual error at $a^{\dagger}$ is bounded by
\begin{equation*}
    \| D f(a^{\dagger}) \|_{\cX'}
    \leq \const 
    \left(\|\cG(a^{\dagger}\|_{\cY} +  \| D \cG(a^{\dagger}) \|_{\Op(\cX,\cY)} + \|a^{\dagger}\|_{\cX} + 1\right) \left( \mathscr{E}_0(a^{\dagger}) + \mathscr{E}_1(a^{\dagger}) + \mathscr{E}_0(a^{\dagger}) \mathscr{E}_1(a^{\dagger}) \right),
\end{equation*}
where $\const > 0 $ is a constant depending on $F$, while
$\mathscr{E}_0(a) := \|\cG(a) - \widetilde{\cG}_{w}(a)\|_{\cY}$ and 
$\mathscr{E}_1(a) := \|D \cG(a) - D \widetilde{\cG}_{w}(a)\|_{\Op(\cX,\cY)}$ 
denote the output and derivative errors, respectively.
\end{proposition}

Additionally, we can provide more concrete results when the high-fidelity objective $f$ is strongly convex. In particular, provided that the surrogate is sufficiently accurate within a ball around the minimizer of $f$, the stationary points of $\widetilde{f}_w$ within the ball have errors of the same order.
The proof of this result is provided in \Cref{sec:proof_optimization_error_strongly_convex}.

\begin{proposition}[Surrogate-Driven Optimization Errors for Strongly Convex Objectives]\label{prop:surrogate_optimization_error_strongly_convex}
Consider the setting of \Cref{prop:surrogate_optimization_error}.
Furthermore, let $\cG$ and $F$ be such that
$f$ is twice differentiable and strongly convex, 
i.e., there exists $\lambda > 0$ such that $D^{2} f(a)(h,h) \geq \lambda \|h\|^{2}_{\cX}$ for all $a, h \in \cX$.

Consider the minimizer $a^{\star}$ of $f$ and a ball $B_R(a^{\star}) := \{\|a - a^{\star}\|_{\cX} \leq R\}$ of radius $R > 0$ around $a^{\star}$ over which $\cG$ and $D\cG$ are bounded. That is, $f(a^{\star}) = \inf_{a \in \cX}f(a)$ and there exists $M > 0$ such that 
$
    \sup_{a \in B_R(a^{\star})} \|\cG(a) \|_{\cY} < M 
$
and
$
    \sup_{a \in B_R(a^{\star})} \|D \cG(a) \|_{\Op(\cX, \cY)} \leq M.
$
Then, we have the following:
\begin{enumerate}[(i)]
    \item There exists a constant $\const = \const(\cG, F, M, R)$ such that for any stationary point $a^{\dagger} \in \cX$ of $\widetilde{f}_{w}$ that lies in $B_R(a^{\star})$ (i.e., $D \widetilde{f}_{w}(a^{\dagger}) = 0$ and $a^{\dagger} \in B_R(a^{\star})$) the error is bounded as 
    \begin{equation*}
        \|a^{\dagger} - a^{\star} \|_{\cX}
        \leq \const 
        \left(
        \mathscr{E}_0(a^{\dagger}) + \mathscr{E}_1(a^{\dagger}) + \mathscr{E}_0(a^{\dagger}) \mathscr{E}_1(a^{\dagger})
        \right).
    \end{equation*}
    \item Suppose additionally $\widetilde{\cG}_{w}$ is such that $\widetilde{f}_w$ is weakly lower semicontinuous. Then, there exists $\epsilon_0 > 0$ such that $\widetilde{f}_{w}$ has at least one local minimizer in $B_R(a^{\star})$ whenever the surrogate errors satisfy $\sup_{a \in B_R(a^{\star})} \mathscr{E}_0(a) \leq \epsilon_0$ and $\sup_{a \in B_R(a^{\star})} \mathscr{E}_1(a) \leq \epsilon_0$.
\end{enumerate}
\end{proposition}
\Cref{prop:surrogate_optimization_error_strongly_convex} shows that the accuracy of the outputs and derivatives of the surrogate is sufficient to preclude the pathological example illustrated in \Cref{fig:local_convergence_pathology}.
Therefore, the results in \Cref{prop:surrogate_optimization_error,prop:surrogate_optimization_error_strongly_convex} strongly suggest that, if the end goal of operator surrogate construction is to obtain surrogate-driven solutions of PDE-constrained optimization, one should construct the operator surrogate taking into consideration the quality of the Fr\'echet derivative approximation. This motivates two aspects of our work:

\begin{enumerate}[(i)]
    \item The operator surrogate $\widetilde{\cG}_w$ and its derivative $D\widetilde{\cG}_w$ should be  universal approximators of the true operator $\cG$ and its Fr\'echet derivative $D\cG$, \emph{simultaneously}. We prove these results for FNOs on compact sets in \Cref{theorem:main_ua} and in weighted Sobolev spaces in \Cref{theorem:main_ua_mu}.
    \item The operator learning formulation should also minimize error in the surrogate predictions of the Fr\'echet derivative. This is a key motivation for derivative-informed operator learning which is introduced in the next subsection.
\end{enumerate}

\subsection{Derivative-Informed Neural Operator Learning}\label{subsec:dino}

The operator surrogate is constructed by minimizing a surrogate discrepancy measure, denoted as $\textrm{dist}(\cdot,\cdot)$:
\begin{equation*}
    \inf_{w\in\bR^{d_w}}\textrm{dist}(\cG, \widetilde{\cG}_w).
\end{equation*}
Let $\mu$ be a data-generating probability measure on $\cX$ and $p\geq 1$ be a norm exponent. The conventional formulation uses the norm of the weighted Bochner space $L^{p}_{\mu}(\cX;\cY)$ to measure the surrogate discrepancy
\begin{equation*}
\textrm{dist}(\cG, \widetilde{\cG}_w) = \|\cG - \widetilde{\cG}_w\|^{p}_{L^p_{\mu}} \coloneqq \bE_{a\sim\mu}\left[\|\cG(a) - \widetilde{\cG}_w(a)\|^{p}_{\cY}\right].
\end{equation*}
When $p=2$ and the expectation is estimated using samples from $\mu$, we arrive at the mean square error minimization in \cref{eq:mse_loss}.

Motivated by \Cref{prop:surrogate_optimization_error,prop:surrogate_optimization_error_strongly_convex}, we use alternative surrogate discrepancy measures that penalize errors in both the operator's output and its Fr\'echet derivatives. We consider the norm of the weighted Sobolev space $W^{m,p}_{\mu, \delta}(\cX; \cY)$, with  $m \in \bN$, $p \geq 1$, and $\delta \geq 0$, defined as follows
\begin{equation}\label{eq:sobolev_norm}
    \| \cG \|_{W^{m,p}_{\mu, \delta}}^p
    = 
        \bE_{a \sim \mu} \left[ \| \cG(a) \|_{\cY}^p 
            + \sum_{1 \leq n \leq m}  \|D^{n} \cG(a)\|_{\HS(\cX_{\delta}, \cY)}^p
        \right].
\end{equation}
Note that $\cX_{\delta}$ is defined by shifting $\cX$ in the Hilbert Scale by $\delta\geq 0$ (see \cref{eq:shifted_X}) which means that the scale parameter $\delta$ controls the increment in regularity of $\cX_{\delta}$ relative to $\cX$. This weighted Sobolev norm is used to measure the surrogate discrepancy:
\begin{equation}\label{eq:dino_loss_expectation}
    \textrm{dist}(\cG, \widetilde{\cG}_w) =  
    \|\cG - \widetilde{\cG}_w \|_{W^{m,p}_{\mu, \delta}}^{p}.
\end{equation}
When $m = 1$, $p = 2$, and $\delta=0$, and the expectation is estimated using samples, we arrive at the canonical formulation of derivative-informed operator learning in \cref{eq:dino_loss_samples}. We will focus on $m=1$ and $p=2$ in this work. 

In the rest of this subsection, we discuss the use of the Hilbert--Schmidt norm and the scale parameter $\delta\geq 0$ in the derivative-informed operator learning formulation. The computational challenges of derivative-informed learning of neural operators (DINOs) are reviewed.

\paragraph{The Hilbert--Schmidt Norm}
The use of the Hilbert--Schmidt norm (as opposed to, say, the operator norm $\|\cdot\|_{\Op_n}$) for the derivative error in \cref{eq:dino_loss_expectation} is important for the definition of the Sobolev norm in infinite-dimensional input and output spaces. First, as we will see in the following section, the compactness of the derivative operator implied by the Hilbert--Schmidt norm is required for universal approximation of neural operators. Second, the Hilbert--Schmidt norm can be conveniently evaluated using \eqref{eq:hilbert-schmidt-multi}, which reverts to computing Frobenius norms of Jacobian matrices upon discretization; see \Cref{section:training}.

\paragraph{The Scale Parameter $\delta$}
For $\cG \in \cont^m(\cX; \cY)$,
we can only guarantee that $D^{n} \cG(a) \in \cL_{n}(\cX, \cY)$ for $n \leq m$ and not $D^{n}\cG(a) \in \HS_{n}(\cX, \cY)$. To address this, we relax the norm by considering a smaller space $\cX_{\delta} \hookrightarrow \cX$ within the Hilbert scale for some $\delta \geq 0$. In the context of the Sobolev spaces $\cX = H^{s}$, restricting derivative directions to $\cX_{\delta} = H^{s+\delta}$ essentially corresponds to measuring the accuracy of derivative operators acting on smoother functions, and as we will show, there always exists $\delta \geq 0$ such that $\cL_{n}(\cX, \cY) \hookrightarrow \HS_{n}(\cX_{\delta}, \cY)$. This restriction of derivative directions to $\cX_{\delta}$ is natural in the context of regularized inverse problems such as that of \eqref{eq:inverse_problem}, where regularization via norms in smoother Hilbert scales, $\cR(a) = \frac{1}{2} \|a - \bar{a}\|_{\cX_{\delta}}^2$, is often used to impose additional regularity requirements on the inverse solution. In this setting, the accuracy of derivative operators acting in directions from $\cX_{\delta}$ is of primary importance. Moreover, this choice of derivative norms is also used in Malliavin calculus to define derivatives with respect to Gaussian random fields, and plays an important role in Bayesian inverse problems with Gaussian priors; see \cite{stuart2010inverse, cotter2013mcmc, cao2025derivative}.

\paragraph{Computational Challenges of DINOs}

 While DINO offers numerous benefits over conventional operator learning, including significantly improved sample complexity and better performance in surrogate-driven optimization \cite{ cao2024lazydino, oleary2024dino, qiu2024derivative, go2025accurate, luo2025efficient, cao2025derivative}, a major challenge of DINOs stems from the increased computational and memory costs of generating and training on Fr\'echet derivative samples. For high-dimensional discretizations of the input and output spaces, dimension-reduction techniques are necessary to manage the cost of training on derivative samples, whose size scales quadratically with the grid dimensions. When using reduced-basis neural operators \cite{hesthaven2018pod, bhattacharya2021pca,  oleary2022dipnet}, these costs are mitigated by (i) directly generating reduced derivative samples on the input and output latent spaces and (ii) constraining the derivative-informed learning to a finite-dimensional mapping from the input to output latent spaces. Consequently, the costs of Fr\'echet derivative learning for reduced-basis DINOs scale with the reduced-basis dimension rather than the grid dimension. However, for general architectures such as the FNO that do not employ linear input projection and output reconstruction steps, these computational benefits are not immediately available. In this work, we discuss several techniques to mitigate the training costs for DIFNOs in \Cref{section:training}.

\section{Universal Approximation on Compact Sets}\label{section:difno_section}
In this section, we establish the simultaneous universal approximation of continuous operators and their Fr\'echet derivatives by FNOs. First, we review the FNO architecture. Then we state the main result and sketch the proof. Assumptions underlying the main results are discussed. Finally, we show that this result implies that FNO can accurately solve PDE-constrained optimization problems.

\subsection{Fourier Neural Operator}
\begin{figure}[h]
    \centering
    \includegraphics[width=0.6\linewidth]{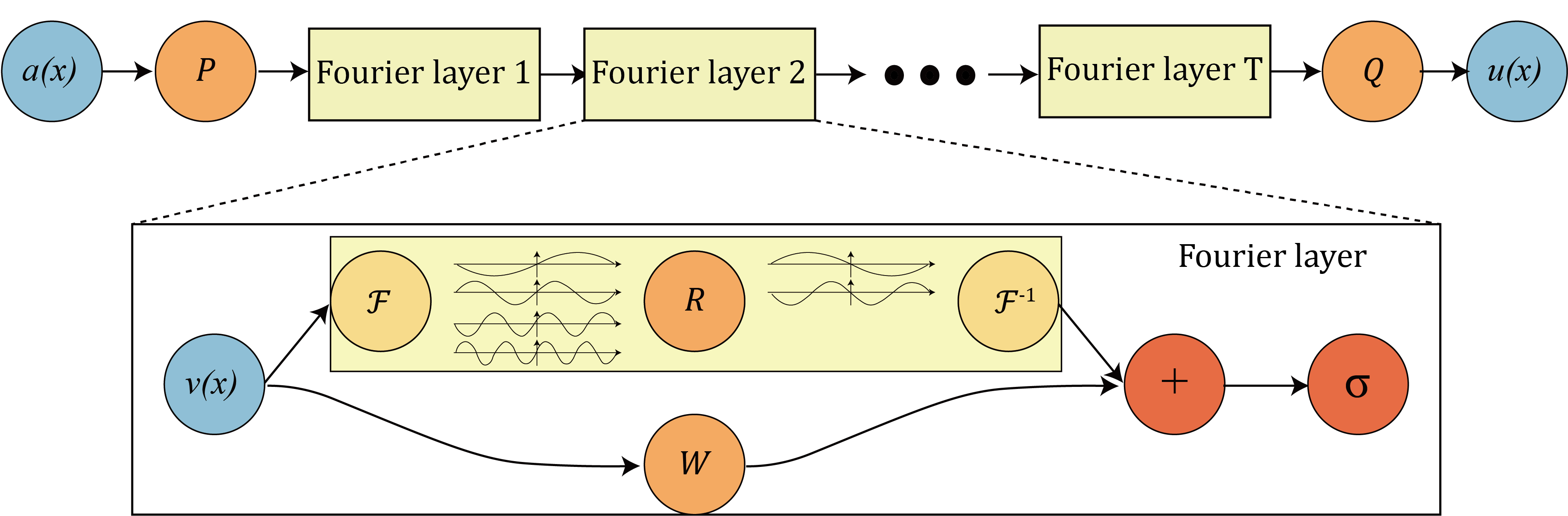}
    \caption{A schematic of the Fourier neural operator. Reprinted (adapted) with permission from \cite{li2021fourier}.}
    \label{fig:fno_schematic}
\end{figure}

We begin by reviewing the FNO architecture \cite{li2021fourier}. 
For a given spatial dimension $d \in \bN$, an FNO defines a mapping on the periodic domain $\bT^d$ of the form:
\begin{equation}
    \cN(a) = \cQ \circ \cL_{d_L} \circ \dots \circ \cL_{1} \circ \cR(a),
\end{equation}
where $\cR(a)(x) = Ra(x)$ with $R \in \bR^{d_v \times d_a}$ is referred to as the lifting layer, and $\cQ(v)(x) = Qv(x)$ with $Q \in \bR^{d_u \times d_v}$ is referred to as the projection layer. 
Given a nonlinear activation function $\sigma : \bR \rightarrow \bR$,
the nonlinear operator layers $\cL_{\ell}, \ell = 1, \dots, d_L$ are defined as
\begin{equation}
    \cL_{\ell}(v)(x) = \sigma\left( W_{\ell} v(x) + b_{\ell}(x) + \cF^{-1} \big( P_{\ell} \odot \cF(v) \big)(x) \right),
\end{equation}
where $W_{\ell} \in \bR^{d_v \times d_v}$ is a pointwise mapping, and $b_{\ell}(x): \bT^d \rightarrow \bR^{d_v}$ is a spatially varying bias, $\cF^{-1} (P_{\ell} \odot \cF(v))$ is a convolution operation. For $\hat{v}(k) = (\hat{v}_1(k), \dots, \hat{v}_{d_v}(k))$, the mapping $P_{\ell} \odot \hat{v}$ is defined as a frequency-wise matrix multiplication:
\begin{equation}
    (P_{\ell} \odot \hat{v})(k) = P_{\ell}(k) \hat{v}(k) = \left( \sum_{j=1}^{d_v} (P_{\ell}(k))_{ij} \hat{v}_j(k) \right)_{i=1}^{d_v},
\end{equation}
where $P_{\ell}(k) \in \bC^{d_v \times d_v}$ for $k \in \bZ^d$ is the representation of the convolution kernel in the frequency domain.
To ensure that the FNO is finitely parameterized, we take $P_{\ell}(k) = 0 $ for all $|k|_{\infty} > N$, 
where the cutoff frequency $N$ is an architecture hyperparameter.
\begin{remark}
    We use a notation of $\cN$ in place of $\widetilde{\cG}_w$ for an element in the class of FNOs. This emphasizes the distinction between (i) the operator learning problem involving finding an optimal set of weights $w$ in a given operator surrogate $\widetilde{\cG}_w$, where the architectural hyperparameters are fixed, and (ii) the universal approximation problem that allows the variations in the architectural hyperparameters in addition to the weights which is encapsulated into a generic notation $\cN$.
\end{remark}

\subsection{Main Result}
Our main result is the universal approximation property of FNOs in the $\cont^1$ sense over compact sets. 
To formulate this result, we will focus on FNOs with activation functions that are smooth, non-polynomial, and globally Lipschitz. 
\begin{definition}[Smooth Activation Functions]\label{def:smooth_activation}
    We say that $\sigma : \bR \rightarrow \bR$ is a smooth activation function, denoted $\sigma \in \smootha$, if $\sigma \in \cont^{\infty}(\bR)$ is non-polynomial and globally Lipschitz.
\end{definition}

The main result is then stated below, along with a sketch of the proof. A detailed proof is provided in \Cref{sec:supporting_lemmas_and_ua}, which develops several supporting lemmas before completing the proof in \Cref{sec:proof_of_main_ua}.

\begin{theorem}[Universal Approximation in $C^1$ on Compact Sets]
\label{theorem:main_ua}
     Let $s,s',\delta \geq 0$ and suppose
     \begin{enumerate}[label=(\roman*)]
         \item $\cG:\cX\to\cY$ is continuously differentiable, i.e., $\cG\in C^1(\cX;\cY)$ for $\cX := H^s(\bT^{d}; \bR^{d_a})$ and $\cY := H^{s'}(\bT^{d}; \bR^{d_u})$;
         \item $D\cG:\cX\to \HS(\cX_{\delta}, \cY)$ is well defined and continuous, i.e., $D\cG\in C(\cX;\HS(\cX_{\delta}, \cY))$, for  $\cX_{\delta} := H^{s+\delta}(\bT^d; \bR^{d_a})$.
     \end{enumerate}
     Furthermore, consider $\sigma \in \smootha$.
     Then, for any compact set $K \subset \cX$ and $\epsilon > 0$, 
     there exists a continuously differentiable FNO, denoted as $\cN\in C^1(\cX;\cY)$, with activation function $\sigma$, 
     such that 
     \begin{equation}
         \max
         \left\{ 
            \sup_{a \in K} 
            \|\cG(a) - \cN(a)\|_{\cY}, \;
            \sup_{a \in K} 
            \|D \cG(a) - D \cN(a)\|_{\HS(\cX_{\delta}, \cY)}
         \right\}
         \leq \epsilon.
     \end{equation}
     
\end{theorem}
\begin{proof}[A Sketch of the Proof] 
Our approach follows that of \cite{kovachki2021universal}. 
We first consider the approximation of $\cG$ by a projection onto the finite set of Fourier modes $|k|_{\infty} \leq N$, $\cG_N := \cP_N \circ \cG \circ \cP_N$.
We show in \Cref{lemma:truncated_operator} that for sufficiently large $N$, the operator and derivative errors, i.e., $\cG - \cG_N$ and $D \cG - D \cG_N$, 
can simultaneously be made arbitrarily small over compact sets. 

Since the truncated operator $\cG_N$ is essentially a mapping between finite-dimensional spaces, we can then approximate $\cG_N$ by an FNO with finitely many parameters (\Cref{lemma:truncated_operator}).
This approximation is constructed by further decomposing $\cG_N$ as 
\[ \cG_N = \cP_N \circ \real \cF_N^{-1} \circ \widehat{\cG}_N \circ \cF_N \circ \cP_N, \quad \widehat{\cG}_N\coloneqq \cF_N \circ \cG \circ \real \cF_N^{-1}, \]
where the Fourier space $\bR^{2K_N}$ is identified with the space of constant vector-valued functions, $L^2_0(\bT^d; \bR^{2K_N})$.
We then show that $\cF_N \circ \cP_N$ and $\cP_N \circ \real \cF_N^{-1}$ 
can all be approximated to arbitrary accuracy over compact sets (\Cref{lemma:ft_approximation} and \Cref{lemma:inverse_ft_approximation}).
On the other hand, $\widehat{\cG}_N := \cF_N \circ \cG \circ \real \cF_N^{-1}$, which corresponds to a mapping on $\bR^{2K_N}$, 
is handled using standard universal approximation theorems for functions in  $C^{1}(\bR^{2K_N}; \bR^{2K_N})$ (e.g., \cite[Theorem 4.1]{Pinkus99}), 
where the resulting neural network $\widetilde{G} \in C^{1}(\bR^{2K_N}; \bR^{2K_N})$ 
is identified with an FNO that maps between constant vector-valued functions.
Finally, 
\Cref{lemma:composition} (an extension of \cite[Lemma 47]{kovachki2021universal} to include derivative accuracy) is used to show that their composition, $\cG_N$, 
can then be arbitrarily well-approximated in terms of both the output and the derivative error by composing the approximations to $\cF_N \circ \cP_N$, $\widehat{\cG}_N$, and $\cP_N \circ \real \cF_N^{-1}$. 

\end{proof}

\subsection{Remarks on the Main Result}
Here we provide remarks on the main results in terms of (i) continuous differentiability of FNOs and (ii) compactness of the FNO derivatives.

\paragraph{On the Continuous Differentiability of the FNOs}
   \Cref{theorem:main_ua} guarantees the existence of a continuously differentiable FNO approximating $\cG$. 
    However, for arbitrary activations $\sigma \in C^{\infty}(\bR)$, values of $s$ and $s'$, widths, or depths, there is no guarantee that the FNOs are continuously differentiable or even continuous. 
    For example, consider an FNO consisting of purely local layers, $\cL_{\ell}(v) = \sigma(W_{\ell} v + b_{\ell})$.
    When $\sigma$ is nonlinear, differentiability of this layer is not guaranteed when viewed as a superposition operator from $L^2(\bT^d)$ to $L^2(\bT^d)$ (see \cite[Section 4.3.3]{troltzsch2010optimal}).
    Additional assumptions, such as smoother input spaces (e.g., large $s$) and growth conditions on $\sigma$, may be required to guarantee continuity and differentiability. Nevertheless, the approximate FNOs constructed in the proof of \Cref{theorem:main_ua} are continuously differentiable due to the initial finite-rank projection onto $L^2_N(\bT^d)$.

\paragraph{On the Compactness of the FNO Derivatives}
        Similar to the definition of $W^{k,p}_{\mu,\delta}$, we have assumed that the derivatives $D \cG(a)$ 
        are Hilbert--Schmidt operators from $ \cX_{\delta}=H^{s+\delta}(\bT^d; \bR^{d_a})$ to $\cY=H^{s'}(\bT^{d}; \bR^{d_u})$ for some $\delta \geq 0$, which is stronger than being continuous linear operators from $\cX=H^{s}(\bT^{d}; \bR^{d_a})$ to $\cY$.
        This enforces compactness of the derivative operators $D \cG(a)$ and ensures that finite-rank projections $D \cG(a)$, such as the Fourier projections $\cP_N \circ D \cG(a) \circ \cP_N$ used in the analysis, converge to $D \cG(a)$ as the rank increases.
        
        Indeed, the inherent finiteness of the FNO parameterization means that it cannot represent arbitrary continuous linear operators. 
        To see this, consider a two-layered FNO, $\cN = \cQ \circ \cL_2 \circ \cL_1 \circ \cR$ mapping from $H^s(\bT^d)$ to $H^s(\bT^d)$.
        Evaluating the derivative of this operator at some point $a \in H^s(\bT^d)$, we have 
        \[
            D \cN (a) =  \cQ 
            \sigma'(W_{2} v_1 + b_2 + \cK_{2}) ( W_2 + \cK_2) 
            \sigma'(W_{1} v_0 + b_1 + \cK_{1}) ( W_1 + \cK_1) 
            \cR, 
        \]
        where $v_0 = \cR a$, $v_1 = \cL_1(\cR(a))$, 
        $\cK_2 = \cF^{-1}  P_{2}  \odot \cF$, $\cK_1 = \cF^{-1}  P_{1} \odot \cF$, and $\sigma'(v)$ denotes the local linear operator defined as 
        \[
            \sigma'(v) w = [\sigma'(v_1(x)) w_1(x), \dots, \sigma'(v_{d_v}(x)) w_{d_v}(x)],
        \]
        when $v, w$ are vector-valued functions of size $d_v$.
        Note that $\cK_2$ and $\cK_1$ are both of finite rank due to the truncation to Fourier frequencies $|k|_{\infty} \leq N$.
        Using the shorthand 
        $\cA_{2} = \sigma'(W_2 v_1 + b_2 + \cK_2 v_1) $
        and 
        $\cA_{1} = \sigma'(W_2 v_1 + b_2 + \cK_2 v_1)$, we can write the derivative as 
        \[
        D \cN(a) = 
        \underbrace{\cQ \cA_2 W_2 \cA_1 W_1 \cR}_{\text{local multiplication}}  + 
        \underbrace{
        \cQ \cA_2 W_2 \cA_1 \cK_1 \cR  + 
        \cQ \cA_2 \cK_2 \cA_1 W_1 \cR  + 
        \cQ \cA_2 \cK_2 \cA_1 \cK_1 \cR
        }_{\text{finite rank}}.
        \]
        By induction, one can show that FNOs of arbitrary depths have derivatives $D\cN(a)$ that are a finite-rank perturbation of a local pointwise multiplication operation, which is not dense in $\Op(H^s(\bT^d), H^s(\bT^d))$.

        Nevertheless, even when the derivative operator is not Hilbert--Schmidt, 
        its restriction to smoother spaces in the Hilbert scale can be Hilbert--Schmidt, i.e., restricting the derivative operator to act on directions in $\cX_{\delta} = H^{s+\delta}(\bT^d; \bR^{d_a})$ instead of $\cX = H^s(\bT^d; \bR^{d_a})$. 
        In particular, for any $s \geq 0$, the inclusion operator $\cI : H^{s+\delta}(\bT^d) \rightarrow H^{s}(\bT^d)$ is Hilbert--Schmidt whenever $\delta > d/2$ (see \Cref{prop:inclusion_operator_is_hilbert_schmidt}). 
        Thus, for $D \cG(a) \in \Op(H^{s}(\bT^d), H^{s'}(\bT^d))$, the restriction of $D \cG(a)$ to directions in $H^{s+\delta}(\bT^d)$ is Hilbert--Schmidt, since 
        \[
            \|D \cG(a) \|_{\HS(H^{s+\delta}, H^{s'})} 
            \leq 
            \|D \cG(a) \|_{\Op(H^{s}, H^{s'})} 
            \| \cI \|_{\HS(H^{s+\delta}, H^{s})}, 
        \]
        which additionally implies that $a \mapsto D \cG(a) \in \HS(H^{s + \delta}(\bT^d; \bR^{d_a}), H^{s'}(\bT^d; \bR^{d_u}))$
        is well defined and continuous. 
        
        For this reason, the universal approximation result in \Cref{theorem:main_ua} can still be applied to operators $\cG \in \cont^{1}(\cX; \cY)$ whose derivatives are not in $\HS(\cX, \cY)$ provided that one is interested in measuring the accuracy of the derivative in $\HS(\cX_{\delta}, \cY)$ using a smoother space $\cX_{\delta}$. 
        An example of this is in regularized inverse problems, which we discuss below.

\subsection{Application to PDE-Constrained Optimization}

Under the settings of \Cref{theorem:main_ua} and \Cref{prop:surrogate_optimization_error_strongly_convex}, FNOs can be utilized to approximate PDE-constrained optimization problems to arbitrary accuracy. 
As an example, we will consider the regularized inverse problem from \eqref{eq:inverse_problem}
where the regularization $\cR(a)$ is the squared norm on a smoother space $\cX_{\delta}$ from the Hilbert scale.

\begin{theorem}\label{theorem:inverse_problem_ua}
    Let $s, s' \geq 0$ and suppose $\cG \in C^1(\cX;\cY)$ for $\cX := H^{s}(\bT^d)$ and $\cY := H^{s'}(\bT^d)$.
    For any $\delta > d/2$, consider the minimization problem over $\cX_{\delta} := H^{s+\delta}(\bT^d)$,
    \begin{equation}\label{eq:cost_function_inverse_problem}
       \min_{a \in \cX_{\delta}} F(\cG(a), a) = \frac{1}{2} \| \Gamma_{\noise}^{-1/2}(\cH  (\cG(a)) - \yobs) \|^2 + \frac{\beta}{2} \|a\|_{\cX_{\delta}}^2,
    \end{equation}
    where $\cH \in \Op(\cY, \mathbb{R}^{\dobs})$ is a continuous observation operator, $\yobs \in \mathbb{R}^{\dobs}$ is the data vector, 
    $\Gamma_{\noise} \in \mathbb{R}^{\dobs \times \dobs}$ is the noise covariance matrix, and $\beta > 0$ is the regularization constant.
    Additionally, assume that $\beta$ and $\cG$ are such that $f(a) = F(\cG(a), a)$ is strongly convex on $\cX_{\delta}$, i.e.,
    $D^2 f(a)(h, h) \geq \lambda \|h\|_{\cX_{\delta}}^2$.

    Then, given $\sigma \in \smootha$, for any $\epsilon, R > 0$, there exists an FNO $\cN \in \cont^1(\cX_{\delta};\cY)$ with activation function $\sigma$ such that 
    $f_{\cN} = F(\cN(a), a)$ has at least one stationary point $a^{\dagger}$ inside the ball $B_R(a^{\star}) := \{ \|a - a^{\star}\|_{\cX_{\delta}} \leq R\}$ and any of such stationary points $a^{\dagger} \in B_R(a^{\star})$ have their error bounded by
    \begin{equation}
        \|a^{\dagger} - a^{\star} \|_{\cX_{\delta}} \leq \epsilon,
    \end{equation}
    where $a^{\star} = \argmin_{a \in \cX_{\delta}} f(a)$ is the minimizer of $f$.
\end{theorem}

The proof of \Cref{theorem:inverse_problem_ua} is given in \Cref{subsec:proof_inverse_problem_ua}.
We note that the assumptions used here are common for regularized inverse problems involving PDEs. 
In particular, strong convexity of $f$ is guaranteed for linear inverse problems (i.e., $\cH \circ \cG$ is linear). For nonlinear inverse problems, strong convexity can often be achieved using sufficiently strong regularization (large $\beta$).

\section{Universal Approximation in Weighted Sobolev Spaces}
In this section, we extend the result on compact domains to the entire input space $\cX$ by considering the $W^{1,2}_{\mu, \delta}$ norm.
To do so, we first adopt additional moment and growth assumptions. Then we state our main results and sketch the proof.

\subsection{Assumptions and Definitions}
We consider the following assumptions for the high-fidelity operator $\cG$ following \cite{bhattacharya2021pca}.
\begin{assumption}\label{assumption:measurability}
    Let $s, s', \delta \geq 0$ and let $\cX := H^s(\bT^d; \bR^{d_a})$, $\cY := H^{s'}(\bT^d; \bR^{d_u})$, and $\cX_{\delta} := H^{s + \delta}(\bT^d; \bR^{d_a})$.
    For a Borel probability measure $\mu$ on $\cX$ and a continuously differentiable operator $\cG : \cX \rightarrow \cY$, 
    assume that $\mu$ has bounded fourth moments, i.e. $\int_{\cX} \|a\|_{\cX}^4 d \mu (a) < \infty$,
    $\cG \in L^2_{\mu}(\cX; \cY)$, and $D \cG \in L^2_{\mu}(\cX; \HS(\cX_{\delta}, \cY))$.
    
    Additionally, assume that $\cG$ and $D{\cG}$ are locally Lipschitz, i.e., 
    \begin{align}
        \| \cG(a_1) - \cG(a_2) \|_{\cY} &\leq L_0(a_1, a_2) \|a_1 - a_2\|_{\cX} \\
        \| D \cG(a_1) - D \cG(a_2) \|_{\HS(\cX_{\delta}, \cY)} &\leq L_1(a_1, a_2) \|a_1 - a_2\|_{\cX} 
    \end{align}
    for any $a_1, a_2 \in \cX$, 
    where the local Lipschitz constants $L_i: \cX \times \cX \rightarrow \bR_{+}$, $i = 1, 2$ satisfy
    \begin{enumerate}[label=(\roman*)]
        \item (Symmetry) For all $a_1, a_2 \in \cX$, $L_{i}(a_1, a_2) = L_{i}(a_2, a_1)$; \label{assumption:measure_symm}
        \item (Measurability) For all $a \in \cX$, $L_i(a, \cdot) : \cX \rightarrow \bR_{+}$ is measurable;
        \item (Non-increasing) For all $a, a_1, a_2 \in \cX$, $L_i(a, a_1) \leq L(a, a_2)$ whenever $\|a_1\|_{\cX} \leq \|a_2\|_{\cX}$;
        \item (Bounded fourth moments) $\int_{\cX} L_i(a,a)^4 d \mu(a) < \infty$.
        \label{assumptions:measure_moments}
    \end{enumerate}
\end{assumption}
The local Lipschitz constants can be understood as generalizations of the derivative norm that appears in the mean value theorem. 
For the output value, this is simply
\[
    \| \cG(a_1) - \cG(a_2) \|_{\cY} \leq \sup_{t \in [0,1]} \|D\cG( t a_1 + (1-t)a_2) \|_{\Op(\cX,\cY)} \|a_1 - a_2 \|_{\cX}.
\]
\Cref{assumption:measurability} can then be thought of as assumptions on the growth rate of the derivatives of $\cG$. 
As a simple example, if 
$\|D\cG(a)\|_{\Op(\cX,\cY)} \lesssim \|a\|_{\cX}$, then $L_0(a_1, a_2) \propto \max \{ \|a_1\|_{\cX}, \|a_2\|_{\cX} \} $
satisfies \Cref{assumption:measure_symm,assumptions:measure_moments}.
Analogous statements can be made about the derivative Lipschitz constant $L_1(a_1, a_2)$ through growth rates of the second derivatives.

We will also focus on activation functions similar to the Gaussian Error Linear Unit (GELU) in the following sense:
\begin{definition}[GELU-like Activation Functions]\label{def:activation}
    Let $\sigmoid \in \cont^{\infty}(\bR)$ be a sigmoid function satisfying the following: 
    \begin{enumerate}[label=(\roman*)]
        \item (Monotonicity) $\sigmoid(x)$ is strictly increasing over $\bR$;
        \item (Boundedness) $\sigmoid(x) \rightarrow 0 $ as $x \rightarrow -\infty$ and $\sigmoid(x) \rightarrow 1$ as $x \rightarrow \infty$;
        \item (Symmetry) $\sigmoid(x) + \sigmoid(-x) = 1$;
        \item (Bounded growth rate) $x^2\sigmoid'(x) \rightarrow 0$ as $|x| \rightarrow \infty$ \label{def:activation-point-bounded-growth}.
    \end{enumerate}
    Then, we say $\sigma$ is a GELU-like activation function, denoted $\sigma \in \srla$, if it takes the form $\sigma(x) = x \sigmoid(x)$.
    Note that \Cref{def:activation-point-bounded-growth} additionally implies that $x \sigmoid'(x) \rightarrow 0$.
\end{definition}
In particular, the GELU function uses the cumulative distribution function of the standard normal distribution as $\Phi(x)$. Other examples of activation functions in $\srla$ include the swish function, which uses the logistic sigmoid, and approximations to GELU that use $\Phi(x)$ of the form $(1 + \tanh(a_1(x + a_2 x^3)))/2$ for chosen values of $a_1, a_2 > 0$.

Activation functions $\sigma \in \srla$ satisfy typical properties of GELU. Namely, $\sigma(x) \approx x$ as $x \rightarrow \infty$ and $\sigma(x) \approx 0$ as $x \rightarrow -\infty$; this is made precise in \Cref{proposition:properties_of_gelu}. This class of activation functions is used in the universal approximation results in the next subsection.

\subsection{Main Result}
The main result is stated below, the proof of which is based on a cutoff argument that extends our approximation results over compact sets. A detailed proof is provided in \Cref{sec:unbounded_domains}, which again develops several supporting lemmas before completing the proof in \Cref{sec:proof_of_main_ua_mu}.
\begin{theorem}[Universal Approximation in $W^{1,2}_{\mu, \delta}$]
    \label{theorem:main_ua_mu}
     
    Let $s, s', \delta \geq 0$ and let $\cX = H^s(\bT^d; \bR^{d_a})$, $\cY = H^{s'}(\bT^d; \bR^{d_u})$, and $\cX_{\delta} = H^{s + \delta}(\bT^d; \bR^{d_a})$.
    Suppose $\mu$ is a Borel probability measure on $\cX$ and $\cG \in \cont^{1}(\cX;\cY)$, which together satisfy \Cref{assumption:measurability}. 
    Then, given $\sigma \in \srla$, for any $\epsilon > 0$,
    there exists an FNO $\cN \in \cont^1(\cX;\cY)$ with activation function $\sigma$ such that 
    \begin{equation}
    \|\cG - \cN\|^2_{W^{1,2}_{\mu, \delta}} = 
     \| \cG - \cN \|_{L^2_{\mu}(\cX; \cY)}^2 + \| D \cG - D \cN \|_{L^2_{\mu}(\cX; \HS(\cX_{\delta}, \cY))}^2 \leq \epsilon^2. 
    \end{equation}
\end{theorem}
\begin{proof}[A Sketch of the Proof] 
As in the proof of \Cref{theorem:main_ua}, we first approximate $\cG$ by $\cG_N$ and then approximate $\cG_N$ by an FNO $\cN$.
In \Cref{lemma:truncation_mu}, we show that under the moment and local Lipschitz continuity assumptions of \Cref{assumption:measurability}, the truncation error $\|\cG  - \cG_N\|_{W^{1,2}_{\mu, \delta}}^2$ can be made arbitrarily small for sufficiently large $N$.

Given an appropriate choice of $N$ satisfying 
$\|\cG  - \cG_N\|_{W^{1,2}_{\mu, \delta}}^2 \leq \epsilon^2/2$, we then look to approximate $\cG_N$ by $\cN$. 
Note that the $W^{1,2}_{\mu, \delta}$ approximation errors of $\cG_N$ can be decomposed by splitting the contributions on balls 
of radius $R$, $B_R:=B_R(0) = \{\|a \|_{\cX} \leq R\}$, 
and its complement, $B_R^c$. That is,
\[
    \|\cG_N - \cN \|_{L^2_{\mu}}^2 = 
    \int_{B_R} \|\cG_N(a) - \cN(a) \|_{\cY}^2  d\mu(a) + 
    \int_{B_R^c} \|\cG_N(a) - \cN(a) \|_{\cY}^2  d\mu(a)
\]
and 
\[
    \|D \cG_N - D \cN \|_{L^2_{\mu}}^2 = 
    \int_{B_R} \|D \cG_N(a) - D \cN(a) \|_{\HS(\cX_{\delta}, \cY)}^2  d\mu(a) + 
    \int_{B_R^c} \|D \cG_N(a) - D \cN(a) \|_{\HS(\cX_{\delta}, \cY)}^2  d\mu(a).
\]
From this, we observe that the $W^{1,2}_{\mu, \delta}$ result 
can be obtained by taking $R \rightarrow \infty$
provided that for any $R$, we can approximate $\cG_N$ over $B_R$ arbitrarily well in a way that its error over $B_R^{c}$ is controlled.

To this end, we use a cutoff argument along with the results of the universal approximation on compact sets to construct such an approximation of $\cG_{N}$.
In particular, \Cref{lemma:truncated_operator} can be used to construct an FNO $\cN_G$ approximating the mapping to the Fourier coefficients of $\cG_N$, i.e., $\cN_G \approx \cF_N \circ \cG_N$,
such that $\real \cF^{-1} \circ \cN_G$ approximates $\cG_N$ in $B_R$ with errors less than any $\epsilon_1 > 0$.
However, this does not necessarily control the approximation error in $B_R^{c}$.

To address this, 
we show that using activation functions $\sigma \in \srla$, 
we can construct a cutoff function
$\funcut : \cX \rightarrow \bR$ (\Cref{lemma:cutoff_l2})
such that its product with $\cN_G$ has the properties
\begin{enumerate}[(i)]
    \item for $a \in B_R$, the product $\funcut(a) \cdot {\cN}_{G}(a) \approx \cN_G(a)$,
    \item for any $a \in \cX$, the product $\funcut (a) \cdot \cN_G(a)$ and its derivative $D_a (\funcut(a) \cdot \cN_G(a)) $ both have growth rates with respect to $a \in \cX$ that are independent of $\epsilon_1$ and $R$.
\end{enumerate}
In particular, the cutoff function $\funcut$ 
can be explicitly constructed as an FNO, where the output is identified with constant functions.

We then show that the operator $\real \cF_{N}^{-1} (\funcut (a) \cdot {\cN_G}(a))$ is well-approximated by an FNO $\cN$, 
and can in turn approximate $\cG_N$ to arbitrary accuracy within $B_R$, 
while again having 
growth rates independent of 
$\epsilon_1$ and $R$
(\Cref{corollary:operator_approximation_bounded_growth}).
This allows us to construct $\cN$
such that 
$\|\cG_N - \cN\|_{W^{1,2}_{\mu, \delta}}^2 \leq \epsilon^2/2$
for any $\epsilon > 0$ by considering arbitrarily large $R$
such that 
\[
    \int_{B_R^c} \|\cG_N(a) - \cN(a) \|_{\cY}^2  d\mu(a)
    +  
    \int_{B_R^c} \|D \cG_N(a) - D \cN(a) \|_{\HS(\cX_{\delta}, \cY)}^2  d\mu(a)    
    \leq \frac{\epsilon^2}{4},
\]
and refining the approximation accuracy over $B_R$ such that 
\[
    \int_{B_R} \|\cG_N(a) - \cN(a) \|_{\cY}^2  d\mu(a) 
    +
    \int_{B_R} \|D \cG_N(a) - D \cN(a) \|_{\HS(\cX_{\delta}, \cY)}^2  d\mu(a)
    \leq \frac{\epsilon^2}{4}.
\]
Combining the approximation results for 
$\|\cG - \cG_N\|_{W^{1,2}_{\mu, \delta}}^2$
and 
$\|\cG_N - \cN\|_{W^{1,2}_{\mu, \delta}}^2$
then yields \Cref{theorem:main_ua_mu}.

\end{proof}

\section{Efficient Derivative-Informed Learning}\label{section:training}

The approximation-theoretic results above establish that for sufficiently smooth classes of operators, FNO can arbitrarily well approximate the operator and its derivatives simultaneously. In practice, however, the computational and memory costs for DIFNO training increase rapidly with the grid size. Consequently, DIFNO training on 2D and 3D geometry with dense grids can be prohibitive under a limited training budget. In this section, we address this challenge and discuss methods to mitigate these costs.

The rest of the section is organized as follows. First, we derive the form of the loss function when the input and output functions are approximated using a uniform grid. We then introduce two approaches for efficient DIFNO training. The first approach uses reduced bases of the input and output spaces to explicitly compress the derivative of the loss function. In contrast, the second approach employs mixed-resolution training using a hierarchy of grids. Lastly, we suggest additional techniques to accelerate DIFNO training.

\begin{remark}
   In this work, we focus on reducing training costs rather than those associated with generating derivative training samples. Derivative data can be generated by solving sensitivity equations (e.g. via the adjoint method) or by employing automatic differentiation with a differentiable PDE solver. In particular, the costs associated with sensitivity equations are discussed in \cite{ghattas2021learning, cao2025derivative}.
\end{remark}

\subsection{Estimating the Output Loss on a Grid}
We estimate the output loss function by discretizing the input and output spaces using a uniform grid. Let $\Lambda$ denote the set of grid point coordinates and $|\Lambda|\in\mathbb{N}$ denote the size of the grid. Let $\{\mathcal{B}_j^{\Lambda}\}_{j=1}^{|\Lambda|}$ denote nodal basis functions, each corresponding to a grid point $x_j\in\Lambda$, e.g., the piecewise linear hat functions generated by $\mathbb{P}_1$ finite element. Let $\cX^{\Lambda}\subset\cX$ and $\cY^{\Lambda}\subset\cY$ denote the finite-dimensional input and output subspaces spanned by the nodal basis functions; the following nodal interpolation properties hold
\begin{equation*}
    a(x) = \sum_{j=1}^{|\Lambda|} a(x_j)\mathcal{B}^{\Lambda}_j(x) \quad\forall a\in \cX^{\Lambda},\quad u(x) = \sum_{j=1}^{|\Lambda|} u(x_j)\mathcal{B}^{\Lambda}_j(x) \quad\forall u\in\cY^{\Lambda}.
\end{equation*}
Define the symmetric and positive definite weighting tensor $\mathbf{W}_{\cY^{\Lambda}}\in \bR^{d_{u}\times|\Lambda|\times d_{u}\times|\Lambda|}$ on the discretized output space:
\begin{equation*}
    [\mathbf{W}_{\cY^{\Lambda}}]_{ijkl} = 
    \begin{cases}
         \left\langle \mathcal{B}^{\Lambda}_j, \mathcal{B}^{\Lambda}_l\right\rangle_{\cY},& 1\leq i=k\leq d_{\cY},\\
         0, &\text{otherwise}.
    \end{cases}
\end{equation*}
Note that $\mathbf{W}_{\cY^{\Lambda}}$ is typically sparse for locally supported nodal basis functions. When the nodal basis functions are piecewise constant and $\langle\cdot,\cdot\rangle_{\cY}$ is the spatial $L^2$ inner product, the action of $\mathbf{W}_{\cY^{\Lambda}}$ divides each nodal value by the grid cell size. When the nodal basis functions are piecewise constant and $\langle\cdot,\cdot\rangle_{\cY}$ is the spatial $H^1$ inner product, the action of $\mathbf{W}_{\cY^{\Lambda}}$ is equivalent to the finite difference discretization of the differential operator $(\mathcal{I} - \Delta)$. We use this weighting tensor to compute the norm on $\cY$ for functions in $\cY^{\Lambda}$. i.e., for all $u\in \cY^{\Lambda}$ we have
\begin{equation*}
    \|u\|_{\cY} = \left\|\mathbf{W}_{\cY^{\Lambda}}^{1/2} \left(u|_{\Lambda}\right)\right\|_F,
\end{equation*}
where $\|\cdot\|_F$ denotes the Frobenius norm.

We define an input projection $\mathcal{P}_{\Lambda}:\cX\to\cX^{\Lambda}$ that maps an input to its grid representation:
\begin{equation*}
    \mathcal{P}_{\Lambda} a = \sum_{j=1}^{|\Lambda|} a(x_j)\mathcal{B}_j^{\Lambda}.
\end{equation*}
The output loss function can be estimated on the grid $\Lambda$ as follows:
\begin{equation}
    \mathbb{E}_{a\sim \mu} \Big[\left\|\cG(a) - \widetilde{\cG}_w(a)\right\|^2_{\cY}\Big]\approx \mathbb{E}_{a\sim \mu} \Big[\left\|\mathbf{W}_{\cY^\Lambda}^{1/2}\left(\cG(a)|_{\Lambda} - \widetilde{\cG}_w(\mathcal{P}_{\Lambda}a)|_{\Lambda}\right)\right\|_F^2\Big],\label{eq:discretized_output_loss}
\end{equation}
rwhere $\cG(a)|_{\Lambda}\in\bR^{d_{u}\times|\Lambda|}$ is the vector of the output function evaluated at grid points, with columns given by $[\cG(a)](x_j)$ where $x_j\in\Lambda$. Here we revert to the notation $\widetilde{\cG}_w$ in place of $\cN$ for an FNO with a fixed architecture and tunable weights $w$.

\begin{remark}
    We note that the $H^1$ Sobolev norm in the output space can be consistently evaluated via the FFT of the nodal values, but the computational cost incurred during training can be substantial. On the contrary, assembling the weighting tensor $\mathbf{W}_{\cY^{\Lambda}}$ only has to be done once, and the subsequent norm evaluations are relatively cheap. Finally, the weighting-tensor approach using $\mathbb{P}_1$ finite elements is equivalent to the commonly adopted finite-difference method (with 3-, 5-, and 7-point stencils in 1-, 2-, and 3D, respectively) for estimating the $H^1$ norm in the output loss function.
\end{remark}

\subsection{Estimating the Derivative Loss on a Grid}
We consider a matrix representation for the derivative operators $D\mathcal{G}(a)$ and $D\widetilde{\mathcal{G}}_w(a)$, which is used to estimate the derivative loss function measured in the norm of $\text{HS}(\cX_{\delta}, \cY)$.
Let $\{\psi_k\}_{k=1}^{\infty}$ and $\{\phi_j\}_{j=1}^{\infty}$ be the orthonormal bases for $\cX_{\delta}$ and $\cY$, respectively.
They define a \textit{Jacobian} $\mathbf{J}: \cX \mapsto \text{HS}(\ell^2, \ell^2)$ of the derivative operator $D\cG$, 
where $\ell^2$ denotes the set of square-summable sequences, following an isometric isomorphism between $\text{HS} (\cX_{\delta}, \cY)$ and $\text{HS} (\ell^2, \ell^2)$:
\begin{equation}
    [\mathbf{J}(a)]_{jk} := \left\langle\phi_j, D\cG(a) \psi_k\right\rangle_{\cY},\quad \|\mathbf{J}(a)\|_F = \|D\cG(a)\|_{\text{HS}(\cX_{\delta}, \cY)}, \label{eq:jacobian_isometry}
\end{equation}
where $\| \cdot \|_F$ is the Frobenius norm. We similarly define $\widetilde{\mathbf{J}}_w: \cX \mapsto \text{HS}(\ell^2, \ell^2)$ for the FNO derivative $D\widetilde{\mathcal{G}}_w(a)$. With the matrix representation, the derivative training samples can be generated as Jacobian matrices at each parameter sample, and the derivative approximation error can be estimated in the Frobenius norm of the Jacobian matrices.

\begin{table}[h]
    \centering
\scalebox{0.88}{
\begin{tikzpicture}[scale = 1.0, transform shape,  node distance=2cm, box/.style={draw, rectangle, minimum width=1cm, minimum height=1cm},
 double/.style={draw, anchor=text, rectangle split,rectangle split parts=2}]
    \node[draw, thick, rounded corners=2pt, inner sep=6pt] at (0,0) (Origin) {$\mathbb{E}_{a\sim \mu} \Big[\underbrace{\left\|\mathbf{W}_{\cY^\Lambda}^{1/2}\left(\cG(a)|_{\Lambda} - \widetilde{\cG}_w(\mathcal{P}_{\Lambda}a)|_{\Lambda}\right)\right\|_F^2}_{\displaystyle\mathclap{\text{Discretized Output Loss}}} + \underbrace{\left\| \mathbf{J}^{\Lambda}(a) - \widetilde{\mathbf{J}}^{\Lambda}_w(\mathcal{P}_{\Lambda}a)\right\|_F^2}_{\displaystyle\text{Discretized Derivative Loss}}\Big]$};

    \node at ($(Origin)+(0,1.3)$) {DIFNO Loss Function on a Grid $\Lambda$};

    \node at ($(Origin)+(-0.3,-2.8)$) {{\renewcommand{\arraystretch}{1.5}
                \begin{tabular}{|c|c|c|}
                \hline
                & Discretized Output Loss & Discretized Derivative Loss \\
                \hline
                 \makecell{Data Generation\\Costs}& 1 PDE solve &  $\mathcal{O}(|\Lambda|)$ linearized PDE solves \\ \hline
                 Training Costs& \begin{tabular}{c}
                    $\mathcal{O}(|\Lambda|)$ memory \\ $\mathcal{O}(|\Lambda| \log(|\Lambda|))$ computation
                 \end{tabular} & \begin{tabular}{c}
                    $\mathcal{O}(|\Lambda|^2)$ memory \\ $\mathcal{O}(|\Lambda|^2 \log(|\Lambda|))$ computation
                 \end{tabular} \\ \hline
                \end{tabular}}};
\end{tikzpicture}
}
\caption{The loss function and training costs of DIFNOs on a uniform grid $\Lambda$. The discretized output loss and discretized derivative loss are defined as \eqref{eq:discretized_output_loss} and \eqref{eq:discretized_derivative_loss}.
Note that the costs of direct training on the derivative loss scales quadratically or more with the grid size and are much higher than training on the output loss. Hence, a more efficient and scalable training scheme for DIFNOs is necessary.}
\label{table:cost_full}
\end{table}

When estimating the derivative loss by discretizing the input and output spaces using a uniform grid, we may represent the finite-dimensional input and output spaces $\cX_{\delta}^{\Lambda}$ and $\cY^{\Lambda}$ using orthonormal bases $\{\psi^{\Lambda}_k\}_{k=1}^{d_{a}|\Lambda|}$ and $\{\phi^{\Lambda}_i\}_{j=1}^{d_{u}|\Lambda|}$, respectively. They define the \textit{discretized Jacobian} $\mathbf{J}^\Lambda:\cX\to \mathbb{R}^{d_{a}|\Lambda|\times d_{u}|\Lambda|}$:
\begin{equation}
    [\mathbf{J}^\Lambda(a)]_{jk} := \left\langle\phi^{\Lambda}_j, D\cG(a) \psi^{\Lambda}_k\right\rangle_{\cY}, \quad \|\mathbf{J}^\Lambda(a)\|_F = \|D\cG(a)\|_{\text{HS}(\cX_{\delta}, \cY)}.\label{eq:discretized_jacobian_isometry}
\end{equation}
Similarly, we have $\widetilde{\mathbf{J}}_w^\Lambda:\cX\to \mathbb{R}^{d_{a}|\Lambda|\times d_{u}|\Lambda|}$ for the discretized Jacobian of the FNO. The derivative loss is thus given by
\begin{equation}
    \mathbb{E}_{a\sim\mu}[\|D\cG(a) - D\widetilde{\cG}_w(a)\|^2_{\text{HS}(\cX_{\delta},\cY)}]\approx \mathbb{E}_{a\sim\mu}[\|\mathbf{J}^\Lambda(a) - \widetilde{\mathbf{J}}_w^\Lambda(\mathcal{P}_{\Lambda}a)\|_F^2].\label{eq:discretized_derivative_loss}
\end{equation}

In terms of both memory and evaluation costs (see \Cref{table:cost_full}), each discretized Jacobian vector product is similar to a forward pass. Consequently, estimating the derivative-loss  requires $\cO(|\Lambda|^2 \log (|\Lambda|)$ computation and $\cO (|\Lambda|^2)$ memory costs, while each output loss estimate requires $\cO(|\Lambda| \log (|\Lambda|)$ computation and $\cO (|\Lambda|)$ memory costs. Estimating and backpropagating through the derivative loss creates a memory and computational bottleneck that must be removed for efficient DIFNO training.

\subsection{Reduced Basis Derivative-Informed Learning} \label{sec:reduced_bases}

In this subsection, we consider compressing the derivative loss using linear dimension-reduction techniques so that its memory and computational costs scale similarly to those of the output loss with respect to the grid size. Let $\cX^r_{\delta}$ and $\cY^r$ denote low-dimensional subspaces of $\cX^{\Lambda}_{\delta}$ and $\cY^{\Lambda}$ with dimension $r_{\cX_{\delta}}\ll d_{a}|\Lambda|$ and $r_{\cY}\ll d_{u}|\Lambda|$, respectively. Suppose $\cX^r_{\delta}=\text{span}(\{\psi_k\}_{k=1}^{r_{\cX}})$ and $\cY^r=\text{span}(\{\phi_j\}_{j=1}^{r_{\cY}})$, the discretized Jacobian can be decomposed as
\begin{equation*}
    \mathbf{J}(a) = \begin{bmatrix}
        \mathbf{J}^{rr}(a)& \mathbf{J}^{r_{\perp}}(a)\\
        \mathbf{J}^{\perp r}(a) & \mathbf{J}^{\perp \perp}(a)
    \end{bmatrix}, 
\end{equation*}
where $\widetilde{\mathbf{J}}^{rr}(a)\in\bR^{r_{\cY}\times r_{\cX_{\delta}}}$ is the Jacobian component in the low-dimensional subspaces $\cX^{r}_{\delta}$ and $\cY^{r}$, and the other three components have sizes that grow with the grid size. Similarly, we define the components of $\widetilde{\mathbf{J}}_{w}(a)$ as $\widetilde{\mathbf{J}}_{w}^{rr}(a)$, $\widetilde{\mathbf{J}}_{w}^{r\perp}(a)$, $\widetilde{\mathbf{J}}_{w}^{\perp r}(a)$, and $\widetilde{\mathbf{J}}_{w}^{\perp \perp}(a)$. The discretized derivative loss can be decomposed as
\begin{equation*}
    \mathbb{E}_{a\sim \mu}\left[\|\mathbf{J}^\Lambda(a) - \widetilde{\mathbf{J}}_w^\Lambda(\mathcal{P}_{\Lambda}a)\|_F^2\right] = \mathbb{E}_{a\sim \mu}\left[\sum_{j, k \in\{r, \perp\}}\|\mathbf{J}^{jk}(a) - \widetilde{\mathbf{J}}_w^{jk}(\mathcal{P}_{\Lambda}a)\|_F^2\right].
\end{equation*}
We consider approximating the derivative loss using a reduced loss in a pre-determined low-dimensional subspace:
\begin{equation}
    \mathbb{E}_{a\sim \mu}\left[\|\mathbf{J}^\Lambda(a) - \widetilde{\mathbf{J}}_w^\Lambda(\mathcal{P}_{\Lambda}a)\|_F^2\right] \approx \mathbb{E}_{a\sim \mu}\left[\|\mathbf{J}^{rr}(a) - \widetilde{\mathbf{J}}_w^{rr}(\mathcal{P}_{\Lambda}a)\|_F^2\right]. \label{eq:reduced_derivative_loss}
\end{equation}
The key to enabling reduced-basis DIFNO training is careful choice of the low-dimensional subspaces $\cX^r_{\delta}$ and $\cY^r$, i.e., their reduced basis functions and dimensions, such that the discarded Jacobian components have small contributions to the discretized derivative loss. Here, we discuss several commonly used subspace detection techniques that involve solving generalized eigenvalue problems, in which the magnitudes of the eigenvalues quantify the relevance of the corresponding eigenfunctions for reducing the Jacobian. See \cite{luo2025dis} for a detailed analysis. 

\begin{table}[h]
    \centering
\scalebox{0.88}{
\begin{tikzpicture}[scale = 1.0, transform shape,  node distance=2cm, box/.style={draw, rectangle, minimum width=1cm, minimum height=1cm},
 double/.style={draw, anchor=text, rectangle split,rectangle split parts=2}]
    \node[draw, thick, rounded corners=2pt, inner sep=6pt] at (0,0) (Origin) {$\mathbb{E}_{a\sim \mu} \Big[\underbrace{\big\|(\mathbf{W}_{\cY^{\Lambda}})^{1/2}\big(\cG(a)|_{\Lambda} - \widetilde{\cG}_w(\mathcal{P}_{\Lambda}a)|_{\Lambda}\big)\big\|_F^2}_{\displaystyle\mathclap{\text{Discretized Output Loss}}} + \underbrace{\big\|\mathbf{J}^{rr}(a)-\widetilde{\mathbf{J}}_w^{rr}(\mathcal{P}_{\Lambda}a)\big\|_F^2}_{\displaystyle\mathclap{\text{Reduced Derivative Loss}}}\Big]$};
    
    \node at ($(Origin)+(0,1.2)$) {
   Reduced-Basis DIFNO Loss Function on a Grid $\Lambda$
    };

    \node at ($(Origin)+(0,-2.8)$) {{\renewcommand{\arraystretch}{1.5}
                \begin{tabular}{|c|c|c|}
                \hline
                & Discretized Output Loss & Reduced Derivative Loss \\
                \hline
                 \makecell{Data generation\\Costs}& 1 PDE solve &  $\min \{r_{\cX_{\delta}}, r_{\cY}\}$ linearized PDE solves \\ \hline
                 Training Costs& \begin{tabular}{c}
                    $\mathcal{O}(|\Lambda|)$ memory \\ $\mathcal{O}(|\Lambda| \log(|\Lambda|))$ computation
                 \end{tabular} & \begin{tabular}{c}
                    $\mathcal{O}(\min\{r_{\cX_{\delta}}, r_{\cY}\} \times |\Lambda|)$ memory \\ $\mathcal{O}(\min\{r_{\cX_{\delta}}, r_{\cY}\} \times |\Lambda| \log(|\Lambda|))$ computation
                 \end{tabular} \\ \hline
                \end{tabular}}};

\end{tikzpicture}
}
\caption{The loss function and training costs of reduced-basis DIFNOs on a uniform grid $\Lambda$. Reducing the derivative loss using dimensionality reduction techniques eliminates the quadratic dependence of the training cost on the grid size, replacing it with a linear dependence on memory and a linearithmic dependence on computation.}
\label{table:cost_subspace}
\end{table}

We first consider sample-based dimension reduction, in particular, principal component analysis (PCA), also known as proper orthogonal decomposition (POD), for input and output reduction. Let $\overline{a}$ be the mean of $\mu$, and let $\overline{u}$ be the mean of the pushforward $\cG_\sharp \mu$, then we define the empirical covariance operators:
\begin{equation}
    \cC_{\cX} = \mathbb{E}_{a\sim \mu}\left[\langle \cdot, a - \overline{a}\rangle_{\cX} (a - \overline{a})\right], \quad \cC_{\cY} = \mathbb{E}_{a\sim \mu}\left[\langle \cdot, \mathcal{G}(a) - \overline{u}\rangle_\cY (\mathcal{G}(a) - \overline{u})\right].
\end{equation}
Consider the following eigenvalue problems for the two covariance operators:
\begin{align}
    &\cC_{\cX} \psi_k = \lambda^{\cX}_k\psi_k,\quad \left\langle\psi_j, \psi_k\right\rangle_{\cX} = \delta_{jk},\\
    &\cC_{\cY} \phi_j = \lambda^{\cY}_j\phi_j,\quad \left\langle\phi_j, \phi_k\right\rangle_{\cY} = \delta_{jk},
\end{align}
where $\delta_{jk}$ is the Kronecker delta. The eigenfunctions with leading eigenvalues are used to define the subspaces $\cX^r_{\delta}$ and $\cY^r$. Note that when $\mu$ is Gaussian, and the covariance is known, this dimension reduction technique follows the Karhunen--Lo\`eve expansion (KLE) of the input variable, which will be referred to as KLE for short. The distinction from PCA is that KLE is computed directly from the known covariance operator rather than from samples. When $\cX_{\delta}$ is the image of $\cC_{\cX}^{s}$ for some $s>0$, the $\cX_{\delta}$-orthonormal basis of the PCA or KLE subspace can be found by rescaling each eigenfunction $\psi_k$ by its eigenvalue.

Next, we consider using derivative-based dimension reduction, which involves solving the following eigenvalue problems
\begin{align} \label{eq:dis}
    \mathbb{E}_{a\sim \mu} \left[D\cG(a)^*D\cG(a) \right]\psi_k &= \lambda^{\cX}_k\psi_k, \quad\left\langle\psi_j, \psi_k\right\rangle_{\cX_{\delta}} = \delta_{jk}, \\
    \mathbb{E}_{a\sim \mu} \left[D\cG(a)D\cG(a)^*\right]\phi_j &= \lambda^{\cY}_j\phi_j, \quad\left\langle\phi_j, \phi_k\right\rangle_{\cY} = \delta_{jk}.
\end{align}
Here, the adjoint $D\cG(a)^*$ of $D\cG(a)$ is taken with an input norm in $\cX_{\delta}$ and an output norm in $\cY$. The eigenfunctions with leading eigenvalues are used to define the subspaces $\cX^r_{\delta}$ and $\cY^r$. In this way, the eigenfunctions capture the left and right singular functions of the derivative, which are most important under the input distribution, making them optimal for reducing derivative loss. We refer to this dimension reduction technique as the derivative-informed subspace (DIS).

\subsection{Mixed-Resolution Derivative-Informed Learning} \label{sec:mixed_res}

\begin{table}[h]
    \centering
\scalebox{0.88}{
\begin{tikzpicture}[scale = 1.0, transform shape, node distance=2cm, box/.style={draw, rectangle, minimum width=1cm, minimum height=1cm},
 double/.style={draw, anchor=text, rectangle split,rectangle split parts=2}]
    \node[draw, thick, rounded corners=2pt, inner sep=6pt] at (0,0) (Origin) {$\mathbb{E}_{a\sim \mu} \Big[\underbrace{\big\|(\mathbf{W}_{\cY^{\Lambda_H}})^{1/2}\big(\cG(a)|_{\Lambda_H} - \widetilde{\cG}_w(a)|_{\Lambda_H}\big)\big\|_F^2}_{\displaystyle\mathclap{\text{High-Res.\ Output Loss}}} + \underbrace{\|\mathbf{J}^{\Lambda_L}\big(a\big) - \widetilde{\mathbf{J}}^{\Lambda_L}_w\big(\mathcal{P}_{\Lambda_L}a\big)\|_{F}^2}_{\displaystyle\mathclap{\text{Low-Res.\ Derivative Loss}}}\Big]$};

    \node at ($(Origin)+(0,1.2)$) {Mixed-Resolution DIFNO Loss Function on Grids $\Lambda_H$ and $\Lambda_L$};

    \node at ($(Origin)+(-0.3,-2.7)$) {{\renewcommand{\arraystretch}{1.5}
                \begin{tabular}{|c|c|c|}
                \hline
                & High-Resolution Output Loss & Low-Resolution Derivative Loss \\
                \hline
                 \makecell{Data Generation\\Costs}& 1 PDE solve &  $\cO (|\Lambda_L|)$ linearized PDE solves \\ \hline
                 Training Costs& \begin{tabular}{c}
                    $\mathcal{O}(|\Lambda_H|)$ memory \\ $\mathcal{O}(|\Lambda_H| \log(|\Lambda_H|))$ computation
                 \end{tabular} & \begin{tabular}{c}
                    $\mathcal{O}(|\Lambda_L|^2)$ memory \\ $\mathcal{O}(|\Lambda_L|^2 \log(|\Lambda_L|))$ computation
                 \end{tabular} \\ \hline
                \end{tabular}}};

\end{tikzpicture}
}
\caption{The loss function and training costs of mixed-resolution DIFNOs on a uniform grid $\Lambda$. By estimating the derivative loss on low-resolution,
the training costs can be dramatically reduced.}

\label{table:cost_mixedres}
\end{table}

We propose a mixed-resolution derivative learning method that leverages the resolution-agnostic property of FNOs to reduce the memory and computational costs associated with derivative loss during training. Let $\Lambda_L$ and $\Lambda_H$ denote a low- and high-resolution uniform grid so that $\Lambda_L\subset\Lambda_H$. The mixed-resolution derivative-informed learning estimates the output loss and the derivative loss at high and low resolutions, respectively. In particular, the high-resolution discretized derivative loss is approximated as
\begin{equation}
    \mathbb{E}_{a\sim \mu}\left[\|\mathbf{J}^{\Lambda_H}(a) - \widetilde{\mathbf{J}}_w^{\Lambda_H}(\mathcal{P}_{\Lambda_H}a)\|_F^2\right] \approx \mathbb{E}_{a\sim \mu}\left[\|\mathbf{J}^{\Lambda_L}(a) - \widetilde{\mathbf{J}}_w^{\Lambda_L}(\mathcal{P}_{\Lambda_L}a)\|_F^2\right]. \label{eq:low_res_derivative_loss}
\end{equation}

Mixed-resolution derivative-informed learning is closely related to reduced-basis derivative-informed learning, with one important distinction. The low-resolution input and output spaces, $\cX_{\delta}^{\Lambda_L}$ and $\cY^{\Lambda_L}$, can be interpreted as the low-dimensional input and output subspaces, $\cX^r_{\delta}$ and $\cY^r$, in the reduced-basis approach. In particular, the mixed-resolution approach compresses the discretized Jacobian matrices using reduced bases that span $\cX^{\Lambda_L}_{\delta}$ and $\cY^{\Lambda_L}$. The key difference between the mixed-resolution and reduced-basis approaches is that the former evaluates the FNO derivative at low resolution, whereas the latter evaluates it at high resolution. Consequently, the memory and computational cost of the mixed-resolution approach can be much lower than those of the reduced-basis approach when the grid reduction from high to low resolution is aggressive. 

On the other hand, the derivative samples $\mathbf{J}^{\Lambda_L}(a)$ are often generated at high resolution with $a\in \cX^{\Lambda_H}$ via sensitivity solves or automatic differentiation of a numerical solver. When we evaluate FNO Jacobians at the low resolution $\widetilde{\mathbf{J}}_w^{\Lambda_L}(\mathcal{P}_{\Lambda_L}a)$, we introduce errors due to the input projection to low resolution. When the reduction from the high to low resolution is aggressive, it will introduce an irreducible gap in the derivative loss, which we found to be mainly attributed to the input projection errors, in particular, the difference between $a$ and $\mathcal{P}_{\Lambda_L}a$. This gap can be significant in some cases. Moreover, the low resolution may limit the number of Fourier modes that can be manipulated during training. To mitigate this discrepancy in practice, we introduce an intermediate resolution $\Lambda_M$ such that $\Lambda_L\subset\Lambda_M\subset \Lambda_H$ and evaluate the discretized Jacobian of the FNO at the resolution of $\Lambda_M$ to reduce this gap in the derivative loss, i.e., 
\begin{equation*}
    \mathbb{E}_{a\sim \mu}\left[\|\mathbf{J}^{\Lambda_H}(a) - \widetilde{\mathbf{J}}_w^{\Lambda_H}(\mathcal{P}_{\Lambda_H}a)\|_F^2\right] \approx \mathbb{E}_{a\sim \mu}\left[\|\mathbf{J}^{\Lambda_M}(a) - \widetilde{\mathbf{J}}_w^{\Lambda_M}(\mathcal{P}_{\Lambda_M}a)\|_F^2\right].
\end{equation*}

\subsection{Additional Computational Savings}

In addition to the above techniques for reducing DIFNO training cost, we introduce the following practical techniques that enhance the DIFNO training process.

\paragraph{Modes of Differentiation}

One can choose between the forward or reverse mode differentiation to compute the Jacobian vector product or Jacobian transport vector product of the FNO on a grid. One should choose the forward mode when $r_{\cX_{\delta}}<r_{\cY}$ or $d_a<d_u$ in the reduced-basis or mixed-resolution approach, and the reverse mode otherwise. In particular, it is shown in \cite{paszke2019pytorch} that reverse mode requires slightly less memory in PyTorch when $r_{\cX_{\delta}}=r_{\cY}$ or $d_a=d_u$,
but forward mode is more computationally efficient.

\paragraph{Gradient Accumulation}
Gradient accumulation techniques further reduce the memory burden by avoiding Jacobian action computation all at once, meaning that we form their contributions to the derivative loss via column partitions of the reduced or low-resolution Jacobian entries.

\paragraph{Multi-GPU Training}

Another mechanism to reduce the memory complexity is to reduce the batch size. This is necessary for large FNO training runs. Combining this strategy with a distributed, multi-GPU data ecosystem enables training FNOs with finer grid representations.

\section{Numerical Experiments}\label{section:numerics}

In this section, we evaluate DIFNOs in comparison to the conventional FNOs for learning solution operators of parametric PDEs. In particular, we consider the nonlinear diffusion--reaction, Helmholtz, and Navier--Stokes equations. We investigate the generalization accuracy of the FNOs and DIFNOs at varying training sample sizes. Additionally, we use FNOs and DIFNOs to solve inverse problems that, given sparse pointwise observations of the PDE solution $u$, recover the input $a$. The surrogate-driven inverse solutions are compared against those obtained using the PDE models.

\subsection{Numerical Setup}

\paragraph{Input Spaces and Data-Generating Distributions} We consider $\cX \coloneqq L^2(\Omega)$ where $\Omega$ is the domain of the input and the data-generating distribution is defined through a centered Gaussian measure with covariance $\cC_{\cX}$. We either define the measure $\mu$ directly as this Gaussian measure or as a pushforward thereof. We employ Mat\'ern covariances arising from inverses of elliptic operators \cite{villa2024note}, defined as $\cC_\cX = (\omega\cI - \rho\Delta)^{-\tau}$, where $\omega,\rho>0$ and $\tau\geq 2$ control the regularity and statistics of the input samples. We define the shifted input space $\cX_{\delta}$ as the image of $\cC_{\cX}^{1/2}$, i.e., the Cameron--Martin space of the Gaussian, with a norm given by $\|a\|_{\cX_{\delta}}\coloneqq\|\cC_{\cX}^{-1/2}a\|_{\cX}$.

\paragraph{Derivative Sample Generation}
We employ two distinct approaches for generating derivative samples.
For the Navier–Stokes equations, derivatives are computed via Automatic Differentiation (AD), 
and for the nonlinear diffusion-reaction and Helmholtz equations, we solve the forward sensitivity equations using sparse direct solvers. 
Efficiency in the latter is achieved by reusing the factorized left-hand-side matrix across multiple sensitivity equations at each input sample. 
Detailed discussions on the cost of this approach can be found in \cite{ghattas2021learning, cao2025derivative, luo2025efficient}.

\paragraph{Error Metrics} 

We define relative error for output and derivative by the following
\begin{subequations}
\begin{align}\label{eq:error}
    \cE_{\text{output}} &=  \frac{\left\|\mathbf{W}_{\cY^\Lambda}^{1/2}\left(\cG\left(a\right)\big|_{\Lambda} - \widetilde{\cG_w}\left(\mathcal{P}_{\Lambda}a\right)\big|_{\Lambda}\right)\right\|_F}
    {\left\|\mathbf{W}_{\cY^\Lambda}^{1/2} \cG\left(a\right)\big|_{\Lambda}\right\|_F}, \\
    \cE_{\text{derivative}} &= \frac{\left\| \mathbf{J}^{\Lambda}\left(a\right) - \widetilde{\mathbf{J}}^{\Lambda}_w\left(\mathcal{P}_{\Lambda}a\right)\right\|_F}
    {\| \mathbf{J}^{\Lambda}\left(a\right) \|_F}.
\end{align}
\end{subequations}
For training, we employ mean-squared relative errors using Monte Carlo samples $a^{(i)} \sim \mu$, and the reported generalization error is computed by the mean relative error.
For all examples, we used 128 samples for model validation and 128 for testing.

\paragraph{Inverse Problem}
Following the inverse problem formulation in \eqref{eq:inverse_problem}, we seek to reconstruct the input function $a$ that best fits the observation data $\yobs$ 
using a PDE model as a reference and a learned FNO as a surrogate. In all examples, we take the noise covariance to be a scalar multiple of the identity, $\Gamma_{\noise} = \gamma^2 \cI$.
We use a regularization term given by $\cR(a)=\|a\|_{\cX_{\delta}}^2/2$.
For surrogate inverse solutions, we compute the relative spatial $L^2$ error with respect to PDE solutions and report them in \Cref{fig:poisson_map_comparison}, \Cref{fig:helmholtz_map_comparison}, and \Cref{fig:ns_map_comparison}.

\paragraph{Fourier Neural Operator}
We adopt the standard FNO architecture introduced in \cite{li2021fourier}, with 4 Fourier layers, each with 32 channels. For the Helmholtz and Navier-Stokes equations, 12 Fourier modes are used per dimension, while 8 are used for the nonlinear diffusion-reaction equation. For both conventional operator learning and derivative-informed operator learning, we train FNOs with the Adam optimizer using an initial learning rate of $0.001$ for 500 epochs, reducing the learning rate by a factor of $0.1$ when training plateaus.

\begin{remark}
In our numerical examples, we refer to FNO trained via conventional operator learning in $L^2_{\mu}$ as FNO, and refer to FNO trained via derivative-informed operator learning in $H^1_{\mu}$ as DIFNO. We emphasize that the distinction is in the training formulation rather than the architecture.
\end{remark}

\paragraph{Software Implementation}

For numerical experiments, 
we employ FEniCS \cite{logg2012fenics, alnaes2015fenics} for finite element solutions of PDEs. 
The libraries hIPPYlib \cite{villa2021hippylib} and hIPPYflow \cite{oleary2021hippyflow} are used for data generation and inverse problem solutions for the nonlinear diffusion--reaction and Helmholtz equations,
while for the Navier--Stokes equation, we implement the above pipelines in PyTorch \cite{paszke2019pytorch}. 
The FNOs and their training are implemented using PyTorch. Unless otherwise specified, inexact Newton-CG, implemented in hIPPYlib, is used for solving the inverse problem, while L-BFGS, implemented in PyTorch, is used for FNO-driven solutions of the inverse problem.

\subsection{Nonlinear Diffusion--Reaction}
\paragraph{Problem Setup} We consider the following nonlinear diffusion--reaction equation
\begin{equation}\label{eq:poisson}
    \begin{aligned}
        -\nabla \cdot (\exp(a(x)) \nabla u(x)) + u(x)^3 &= f(x),  &&x \in \Omega\coloneqq(0,1)^2, \\
        u(x) &= 0, &&x\in\partial\Omega,
    \end{aligned}
\end{equation}
where $\partial\Omega$ is the boundary of the unit square domain, the source term $f$ consists of four Gaussian bumps, and the input to the PDE solution operator is the log-diffusivity field. The input distribution has a covariance of $\cC_\cX = (10\cI/3 - \Delta/30)^{-2}$. The input and output spaces are discretized using linear triangular finite elements with $64$ cells in each direction, and we use the $L^2((0,1)^2)$ norm for the output space $\cY$. We use a $65\times 65$ high-fidelity grid for the FNOs and DIFNOs. For the reduced-basis DIFNOs, we use $r_{\cX_{\delta}}=r_{\cY} = 289$. For mixed-resolution DIFNOs, we use a low-fidelity grid $17 \times 17$ and set $\Lambda_M$ to $33 \times 33$. The source term, a sample of the log-diffusivity field, and the PDE solution at the sample are visualized in \Cref{fig:poisson_source_term}.

\begin{figure}[htb]
\centering
{
\small
\begin{tabular}{c c c}
    \hspace{-0.05\linewidth}Source Term&\hspace{-0.05\linewidth} \makecell{Log-Diffusivity\\ Field (Input)}&\hspace{-0.05\linewidth} \makecell{PDE Solution\\ with Observations}\\
    \includegraphics[width=0.25\linewidth]{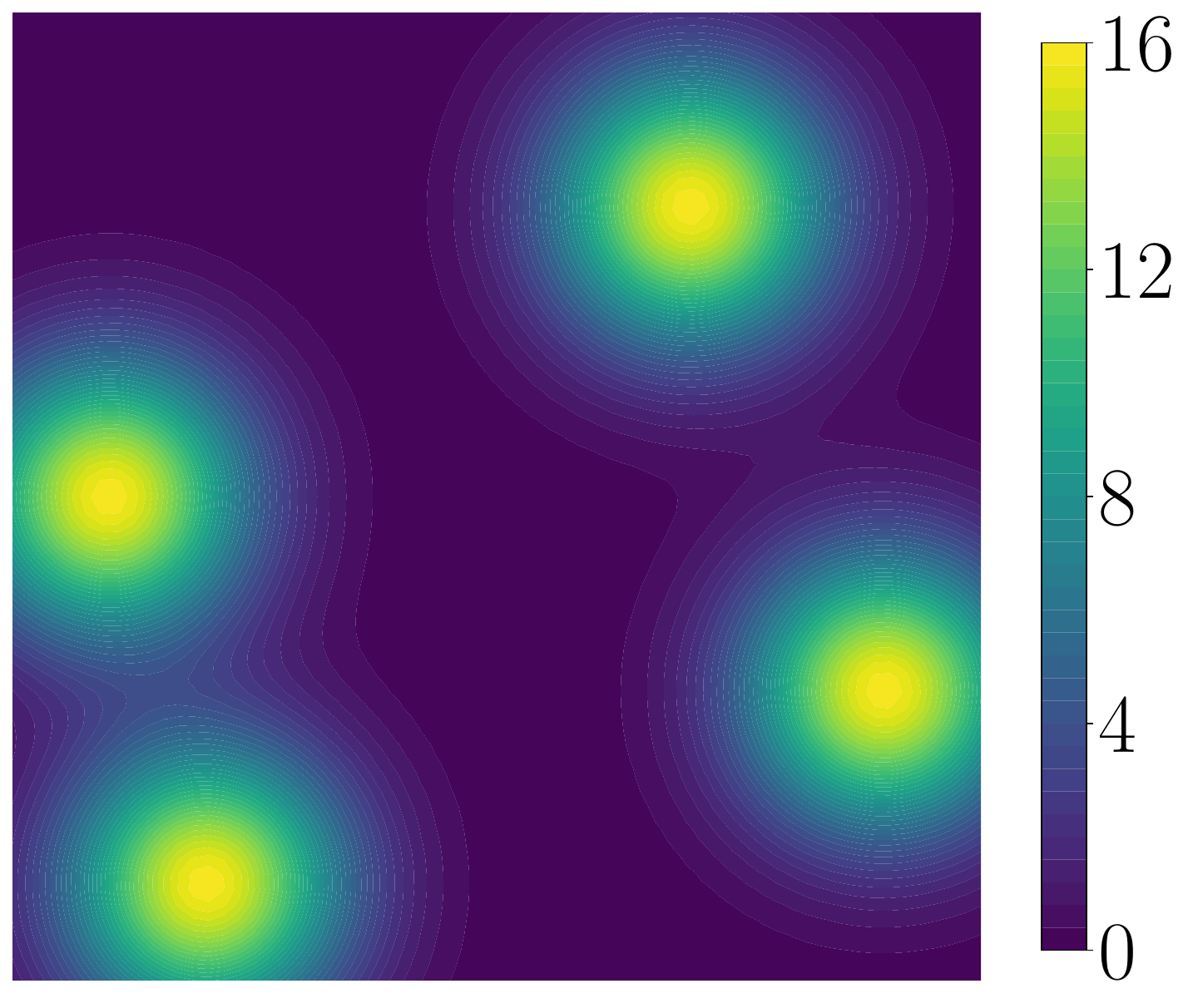}  &
    \includegraphics[width=0.25\linewidth]{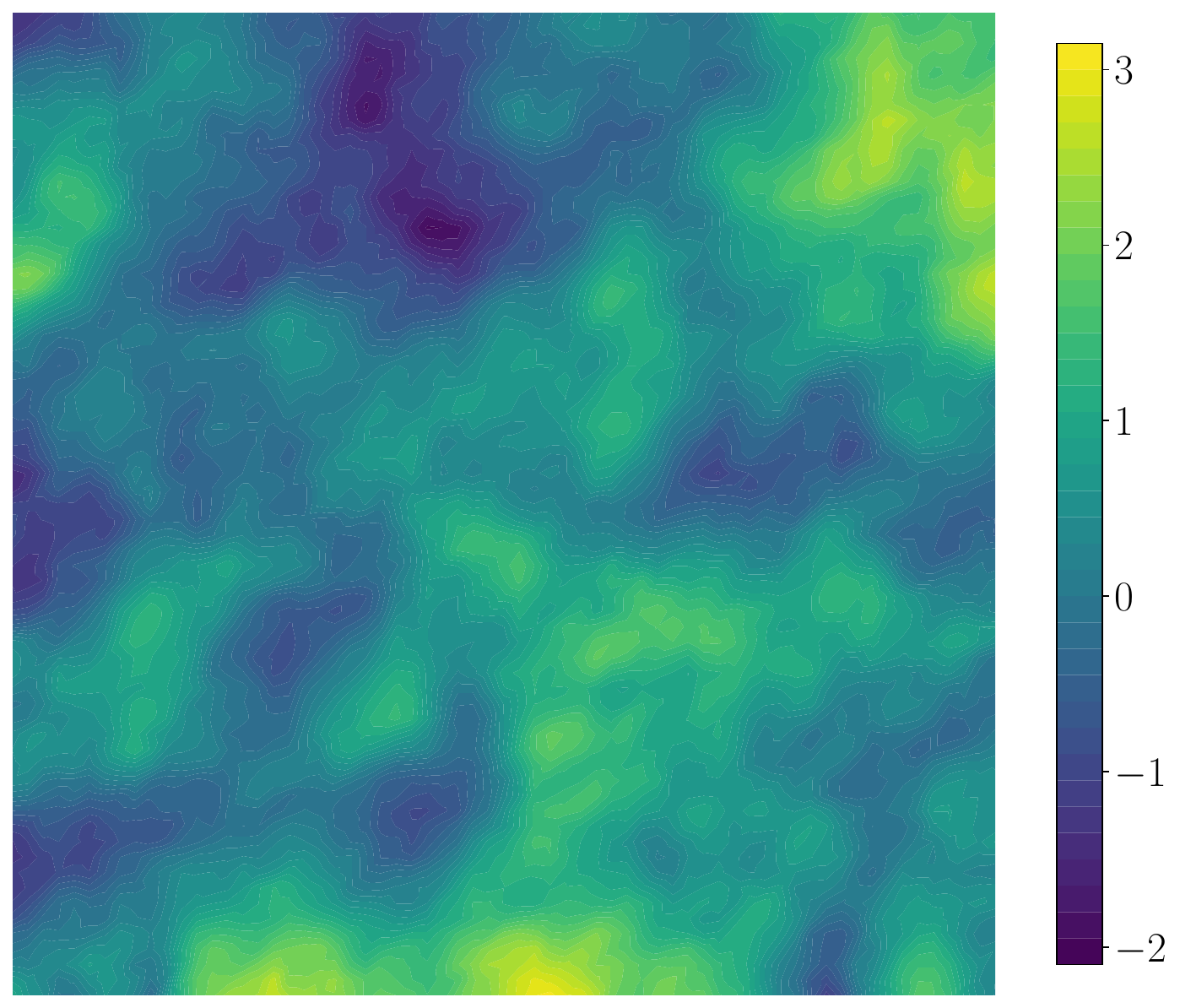} &
    \includegraphics[width=0.25\linewidth]{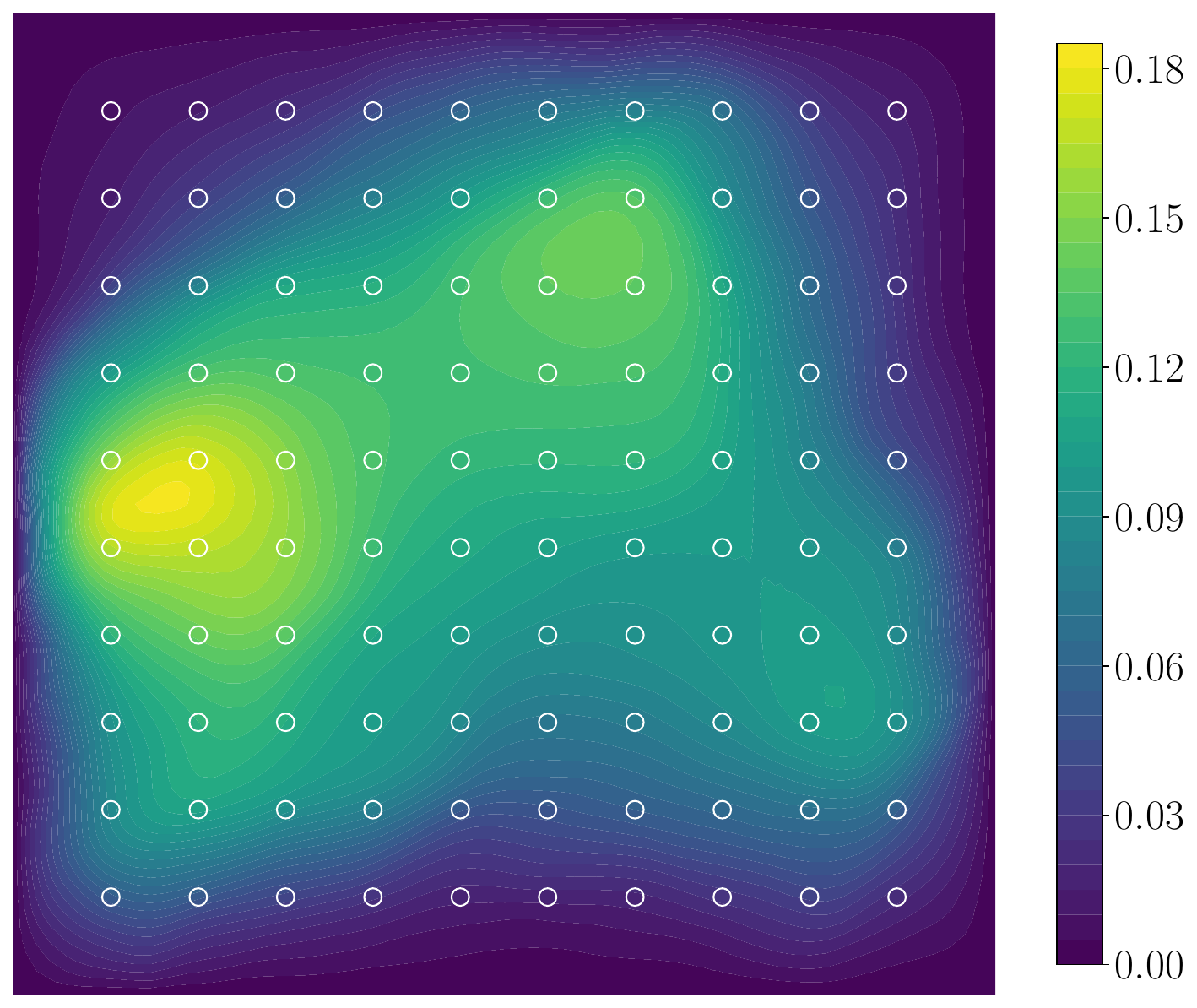}
\end{tabular}
}
    \caption{The step up of the numerical study on the nonlinear diffusion--reaction equation. (\textit{Left}) The source term consists of four Gaussian bumps. (\textit{Center}) An input sample $a^{(i)} \sim \mu$ of the log-diffusivity field. (\textit{Right}) The PDE solution at $a^{(i)}$. The circles indicate the locations of the pointwise observations of the PDE solution for the inverse problem.}
    \label{fig:poisson_source_term}
\end{figure}

\paragraph{Training Results} We train FNOs, reduced-basis DIFNOs, and mixed-resolution DIFNOs at varying training sample sizes, and their testing results are shown in \Cref{fig:poisson_train}. These results demonstrate that DIFNOs achieve substantial improvements in predicting both PDE solutions and Fr\'echet derivatives. The reduced-basis DIFNOs with two linear dimension reduction methods, i.e., (i) KLE input reduction and POD output reduction, and (ii) DIS input and output reduction, perform similarly, and both outperform the mixed-resolution DIFNOs. On the other hand, the mixed-resolution training is much faster as the derivative loss is computed on a lower-resolution grid.

\begin{figure}[htb]
    \centering
{\small
    \renewcommand{\arraystretch}{1.2} 
    \setlength{\tabcolsep}{5pt}       

    \begin{tabular}{c c c}
        & \hspace{0.03\linewidth} \makecell{Nonlinear Diffusion--Reaction\\
        PDE Solution Relative Error} & \hspace{0.03\linewidth} \makecell{Nonlinear Diffusion--Reaction\\Fr\'echet Derivative Relative Error} \\
        
        \raisebox{3em}{\rotatebox{90}{Relative Error}} & 
        \includegraphics[width=0.42\linewidth]{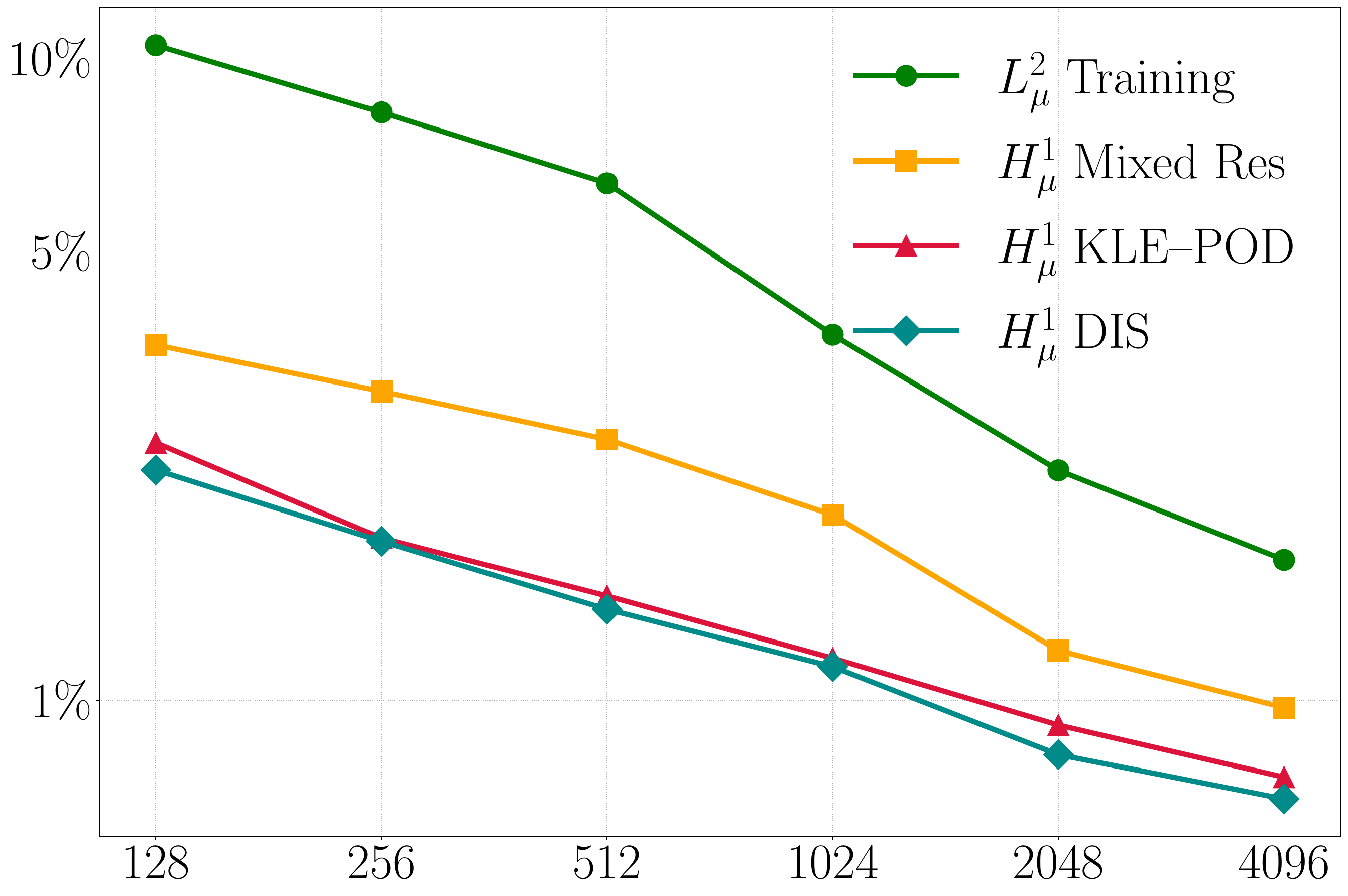} & 
        \includegraphics[width=0.42\linewidth]{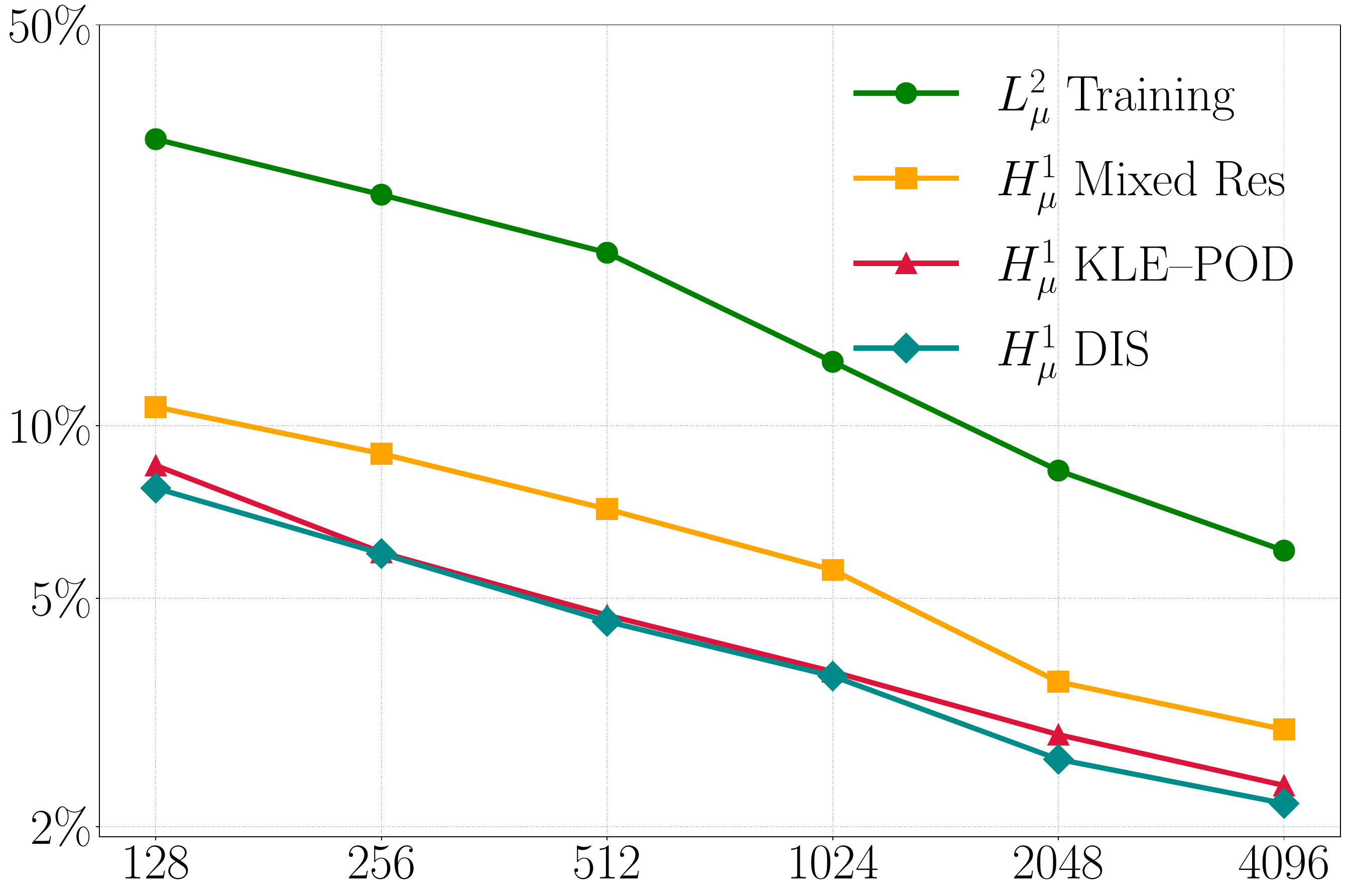} \\
        
        & \hspace{0.03\linewidth} Number of Training Samples & \hspace{0.03\linewidth} Number of Training Samples \\
    \end{tabular}
}
    \caption{The relative error in the PDE solutions and Fr\'echet derivative predictions for the nonlinear diffusion--reaction equation. We train FNOs ($L^2_{\mu}$ Training), reduced-basis DIFNOs with KLE--POD dimension reduction ($H^1_{\mu}$ KLE--POD), reduced-basis DIFNOs with DIS dimension reduction ($H^1_{\mu}$ DIS), and mixed-resolution DIFNOs ($H^1_{\mu}$ Mixed Res.) on a varying number of training samples. We use an additional 128 testing samples to evaluate the relative error.}
    \label{fig:poisson_train}
\end{figure}

\paragraph{Inverse Problem Results} We study solutions to the inverse problem obtained by the PDE, the trained FNOs, and the trained DIFNOs. The data-generating (true) log-diffusivity field consists of a highly diffusive region and a low-diffusive region embedded in a medium-diffusivity background (left of \Cref{fig:poisson_map}), as adapted from \cite{cui2016dimension, cao2025derivative}. Pointwise observations of the PDE solution at this log-diffusivity field are extracted; see locations of the observations in \Cref{fig:poisson_source_term}. We then add a $1\%$ additive white noise to obtain the observation data for the inverse problem. The inverse solutions and their spatial pointwise errors are visualized in \Cref{fig:poisson_map,fig:poisson_map_comparison}. Even with a small training sample size of $512$, the inverse solutions by the DIFNOs closely follow the reference solution, and DIFNOs outperform FNO on this task in different training sizes. These results highlight the crucial role of accurate Fr\'echet derivative approximation in surrogate-driven solutions of inverse problems. Furthermore, although the reduced-basis DIFNO achieves superior accuracy on test samples, the mixed-resolution DIFNO performs similarly to the reduced-basis DIFNO in solving the inverse problem.

\begin{figure}[htb]
    \centering
{
\small
    \begin{tabular}{c c c c}
        \hspace{-0.03\linewidth}\makecell{Data-Generating\\ Input Field (True) } \quad&
          \hspace{-0.03\linewidth}\makecell{Inverse Solution \\ by PDE (Reference)} & \hspace{-0.03\linewidth}\makecell{Inverse Solution \\ by FNO}&  \hspace{-0.03\linewidth}\makecell{Inverse Solution \\ by DIFNO}\\
         \includegraphics[width=0.2\linewidth]{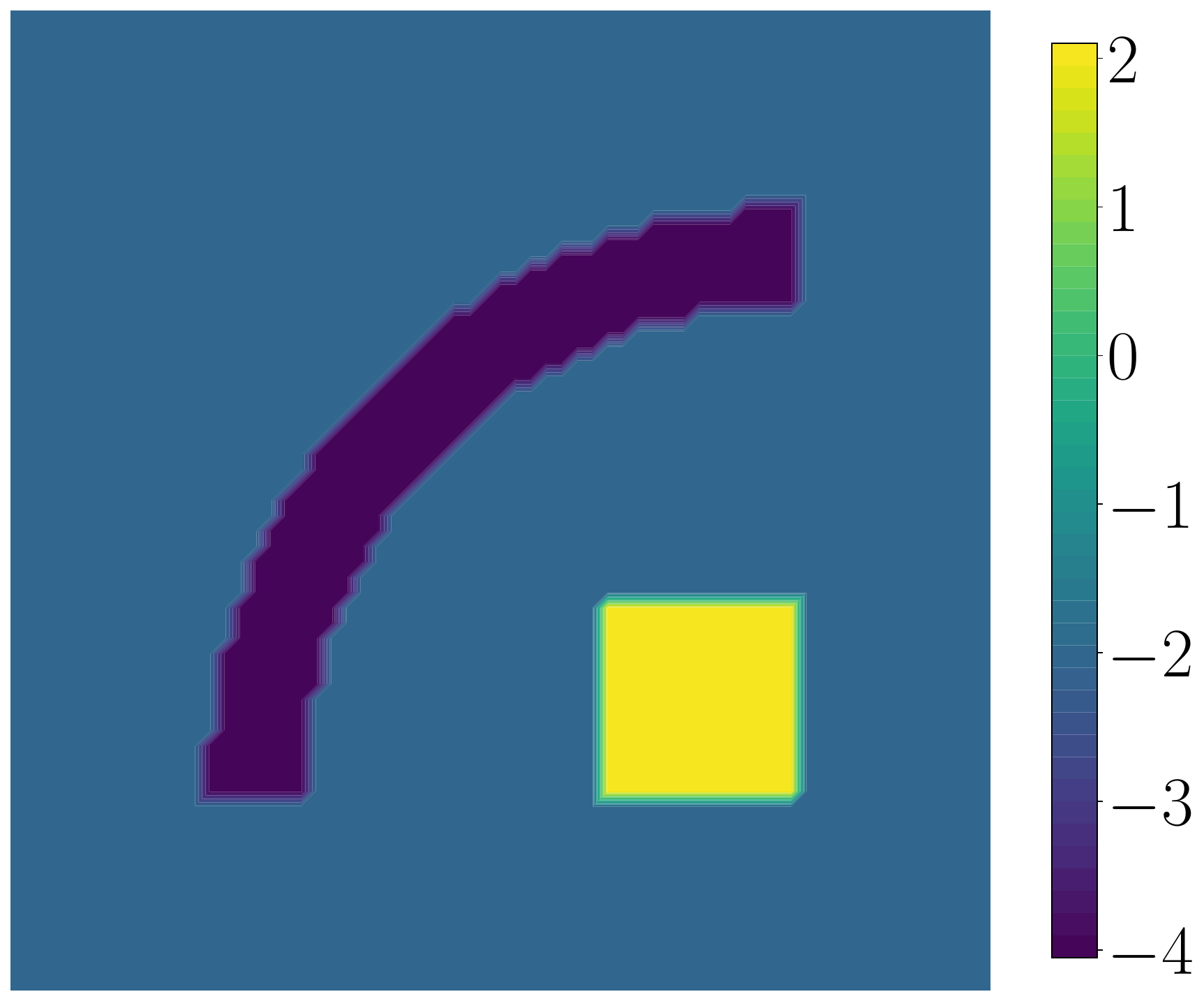} & 
         \includegraphics[width=0.2\linewidth]{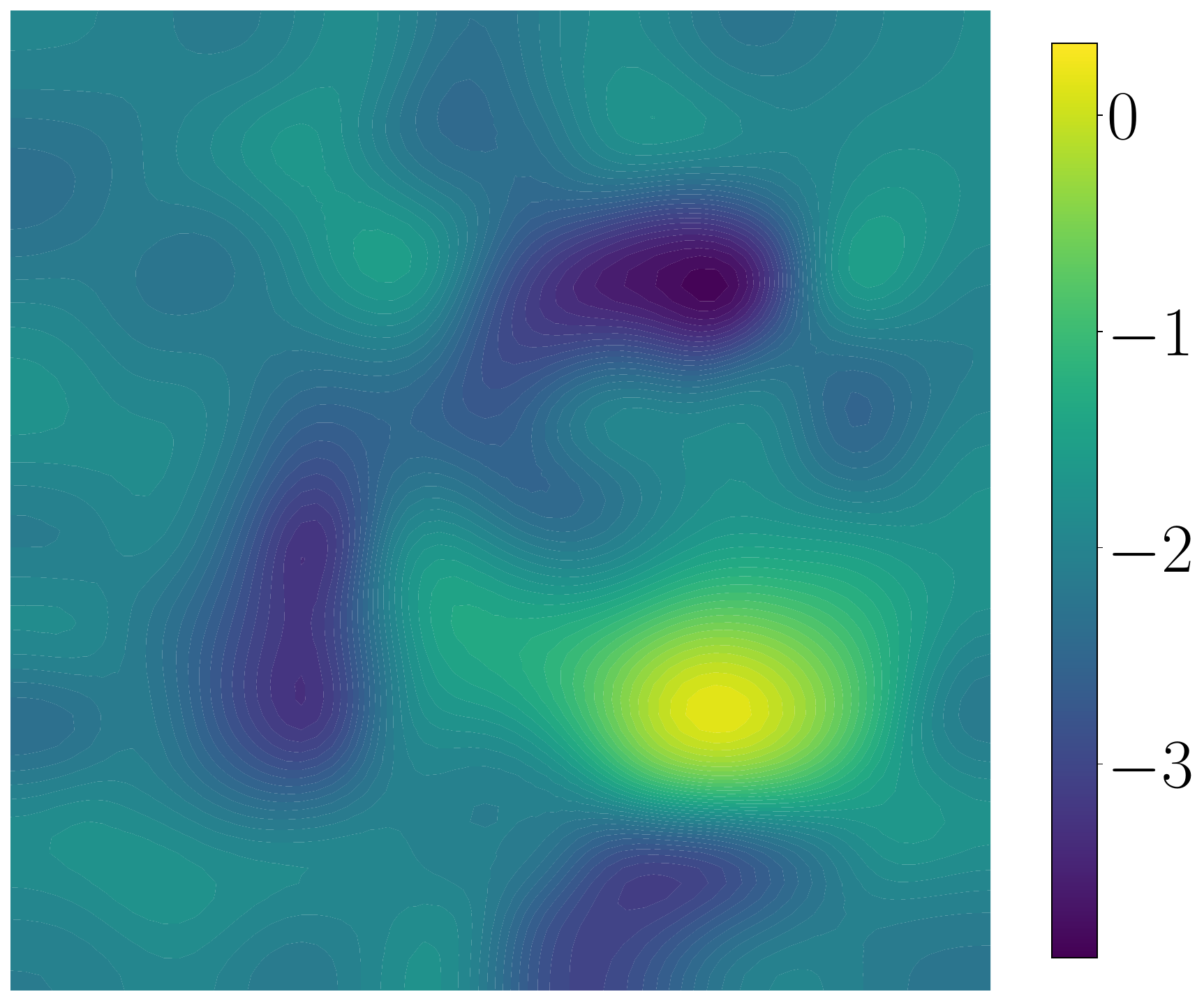} &
         \includegraphics[width=0.2\linewidth]{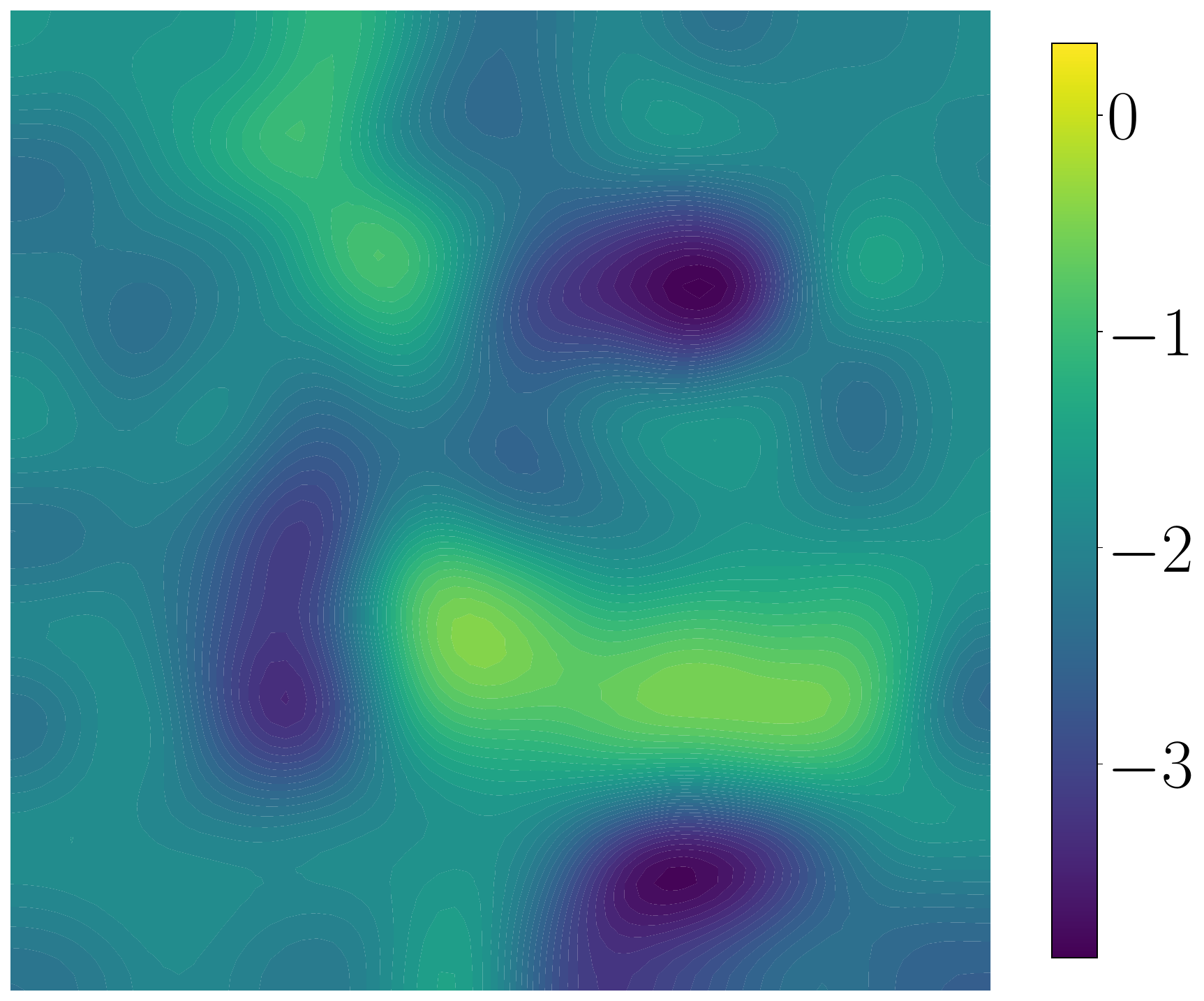} &
         \includegraphics[width=0.2\linewidth]{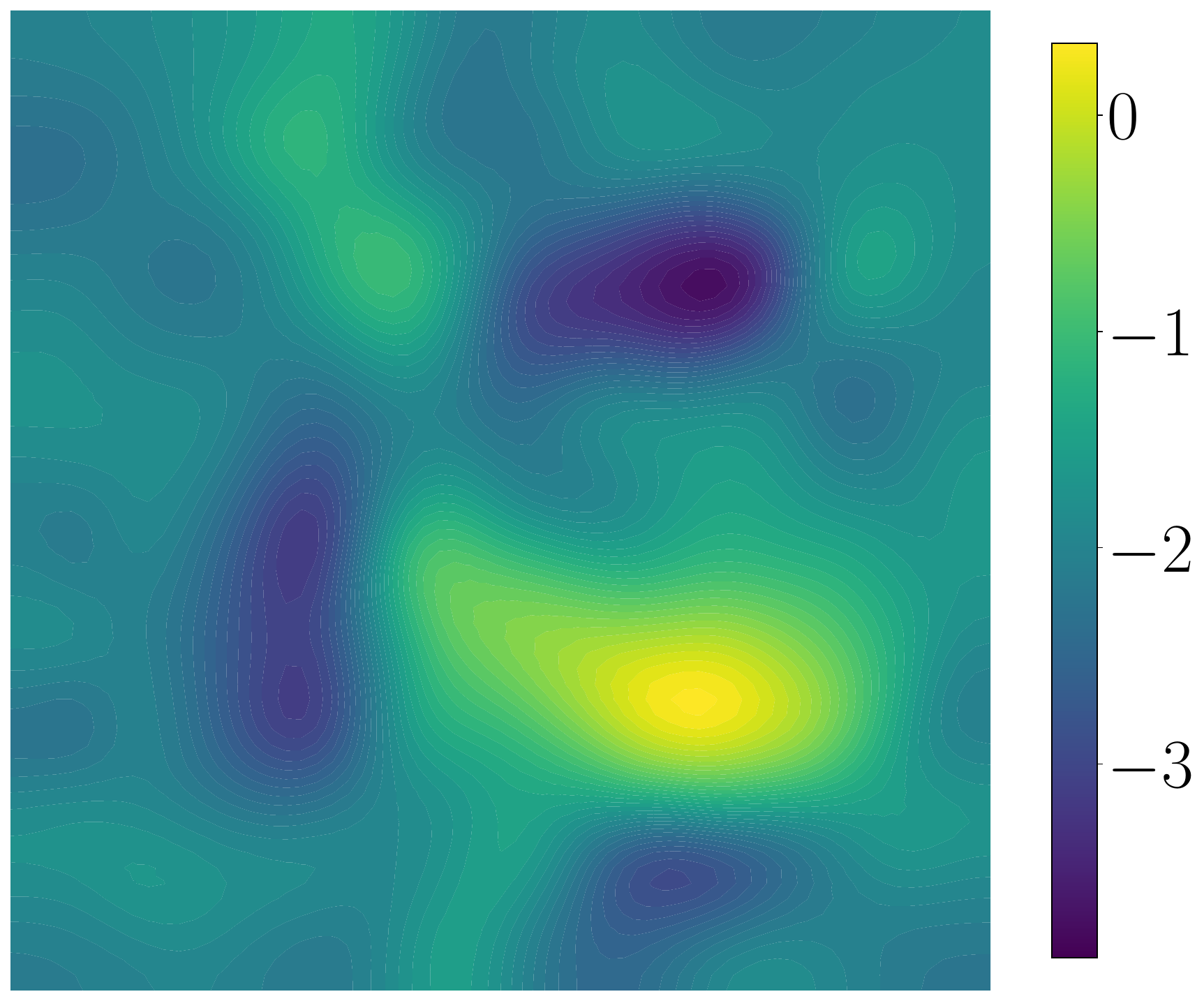}
    \end{tabular}
}
    \caption{Solutions of inverse problems based on the nonlinear diffusion--reaction equation in \Cref{fig:poisson_source_term}. The left figure is the data-generating (true) log-diffusivity field. Then, from left to right, are inverse solutions obtained by PDE, the best-performing FNO, and the best-performing DIFNO. }
    \label{fig:poisson_map}
\end{figure}

\begin{figure}[htb]
\centering
\scalebox{0.9}{    \renewcommand{\arraystretch}{1.5}
\begin{tabular}{
|>{\centering\arraybackslash} m{0.115\linewidth}||
>{\centering\arraybackslash} m{0.13\linewidth}|
>{\centering\arraybackslash} m{0.13\linewidth}|
>{\centering\arraybackslash} m{0.13\linewidth}|
>{\centering\arraybackslash} m{0.13\linewidth}|
>{\centering\arraybackslash} m{0.1\linewidth}|}\hline
     \multicolumn{6}{|c|}{\makecell{Errors in the Inverse Solutions for Nonlinear Diffusion--Reaction}}  \\\hline\hline
     &\multicolumn{4}{c|}{Number of Training Samples} & \multirow{2}{*}{\makecell{Pointwise\\Error}}\\\cline{1-5}
     & 512 & 1024 & 2048 & 4096 &  \\\hline
     FNO&  \raisebox{-0.15\linewidth}{\includegraphics[width=0.9\linewidth]{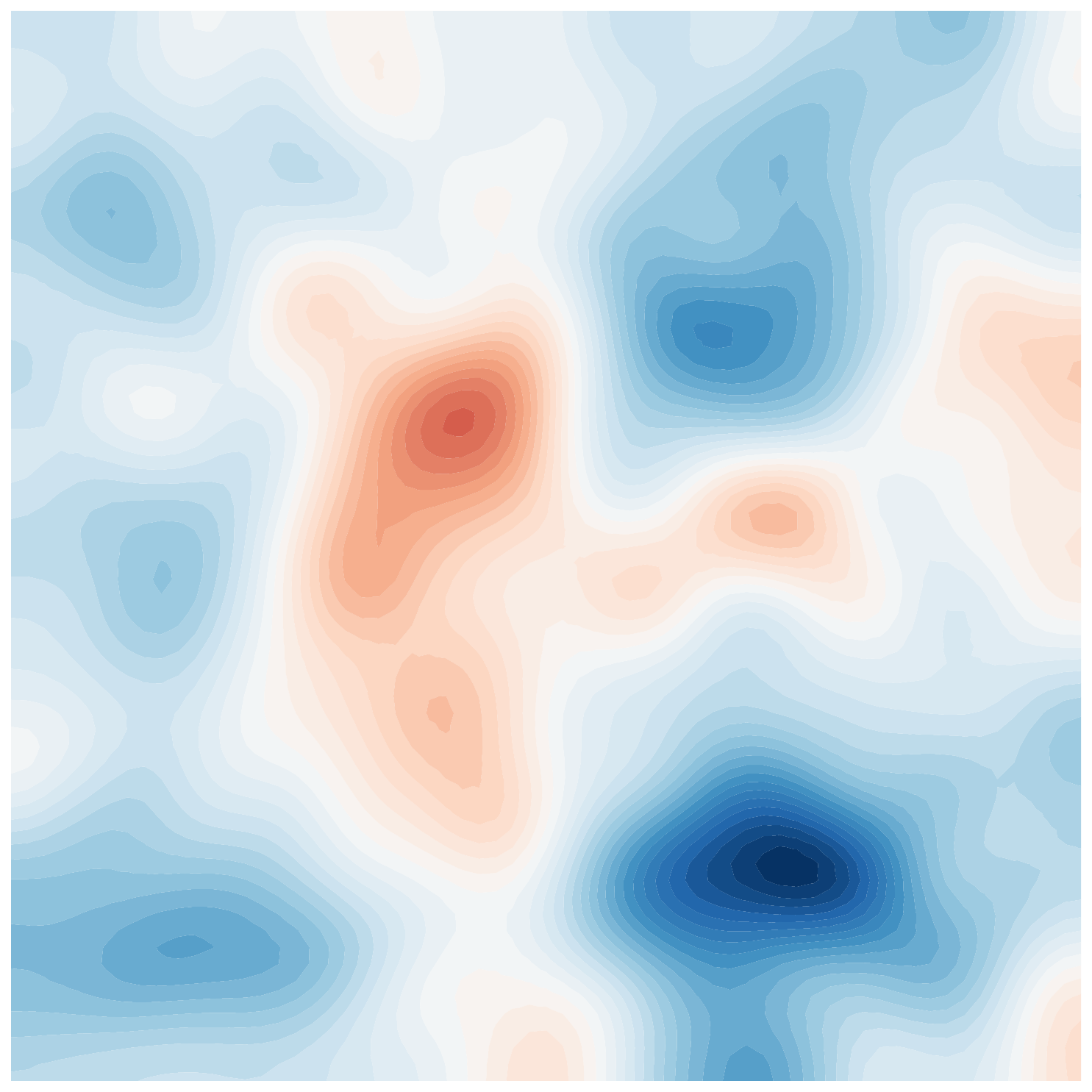}} & \raisebox{-0.15\linewidth}{\includegraphics[width=0.9\linewidth]{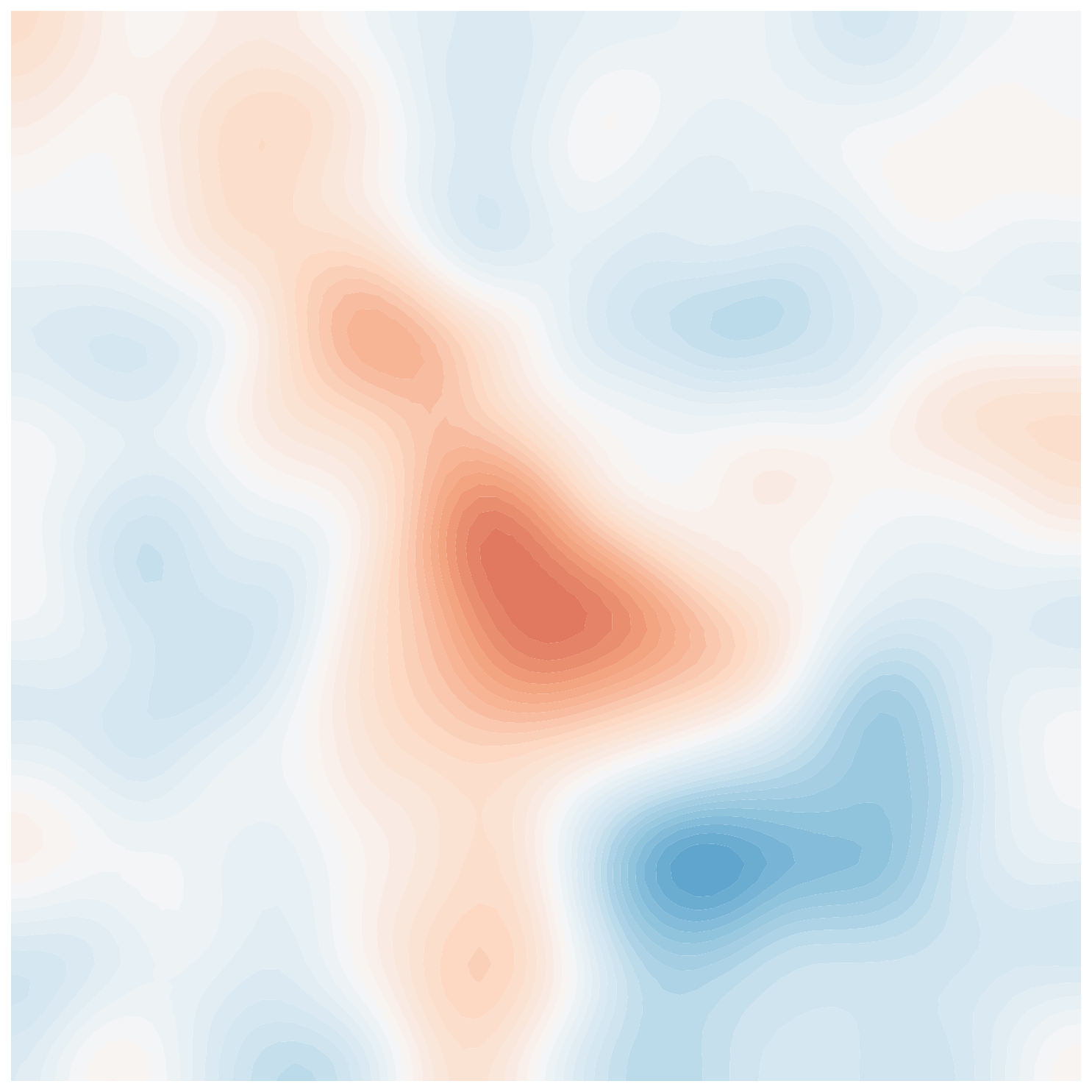}} & \raisebox{-0.15\linewidth}{\includegraphics[width=0.9\linewidth]{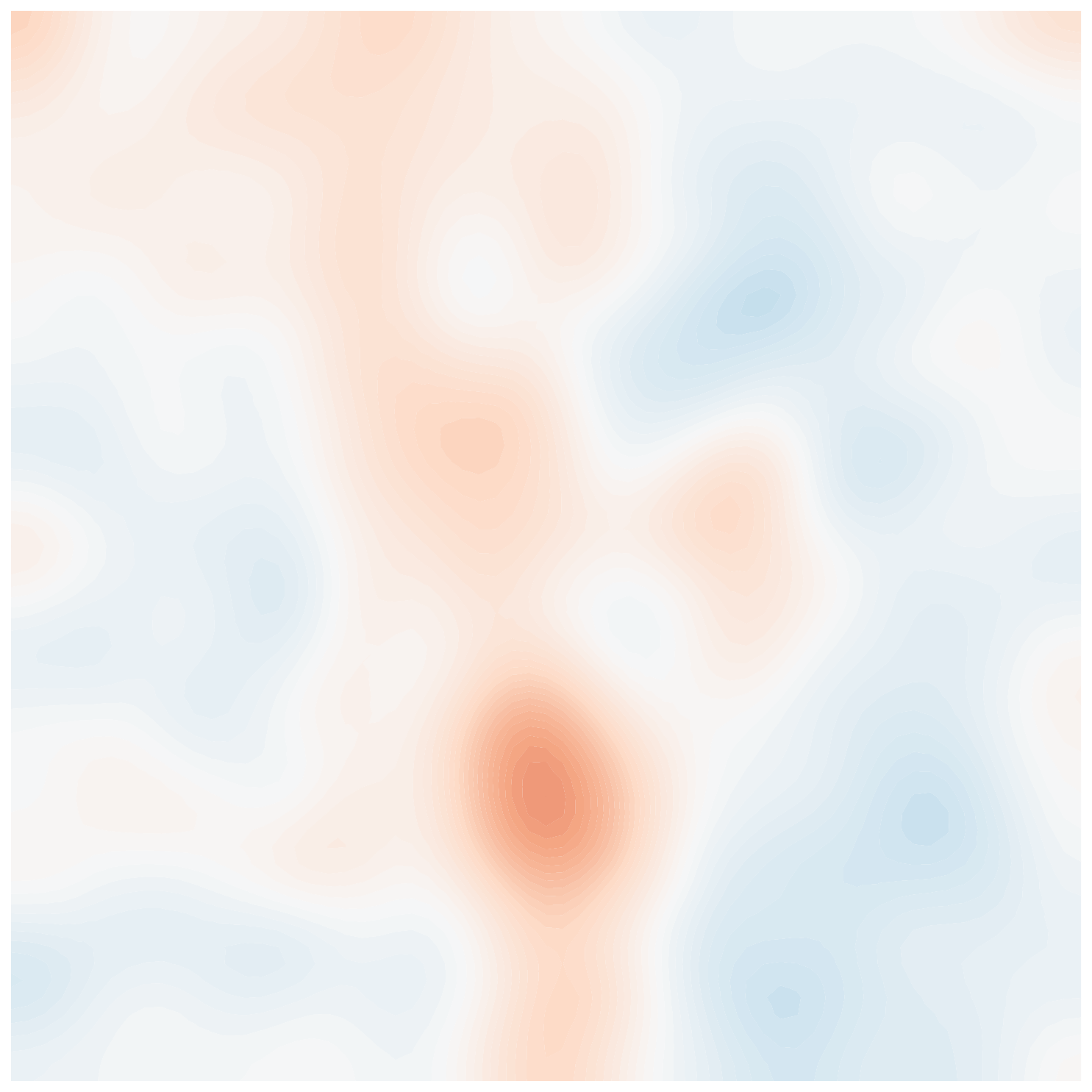}} & \raisebox{-0.15\linewidth}{\includegraphics[width=0.9\linewidth]{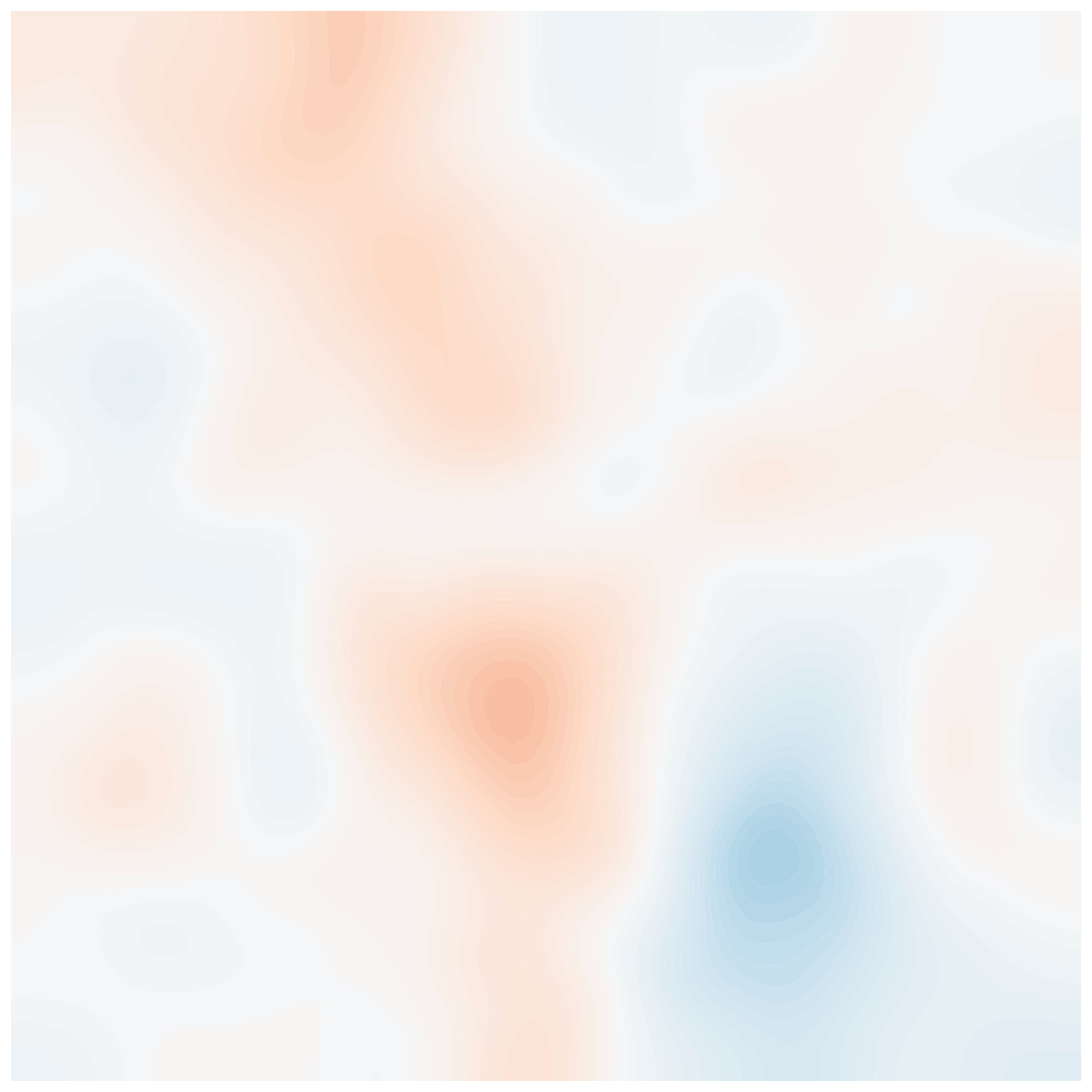}} & \multirow{5}{*}{\includegraphics[width=0.32\linewidth]{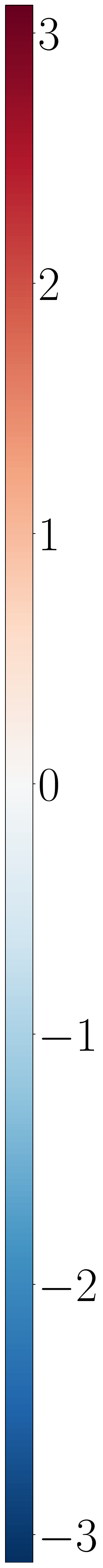}}\\\cline{1-5}
     \makecell{Relative\\Error}& \quad 42.76 \% & \quad 26.29 \% & \quad 15.51 \% & \quad 14.41 \% & \\\cline{1-5}
     \makecell{Red.-Basis\\DIFNO}&  \raisebox{-0.13\linewidth}{\includegraphics[width=0.9\linewidth]{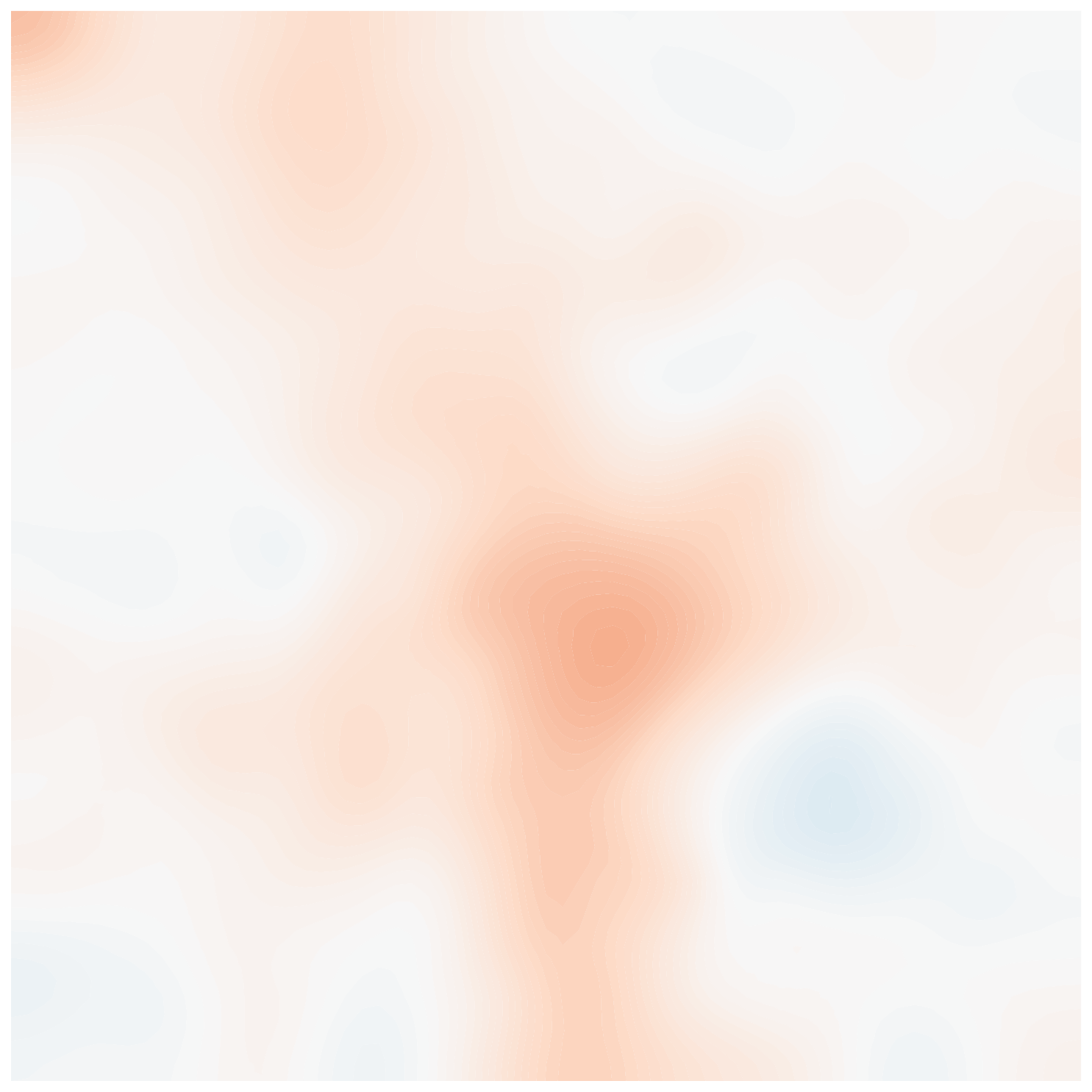}} & \raisebox{-0.13\linewidth}{\includegraphics[width=0.9\linewidth]{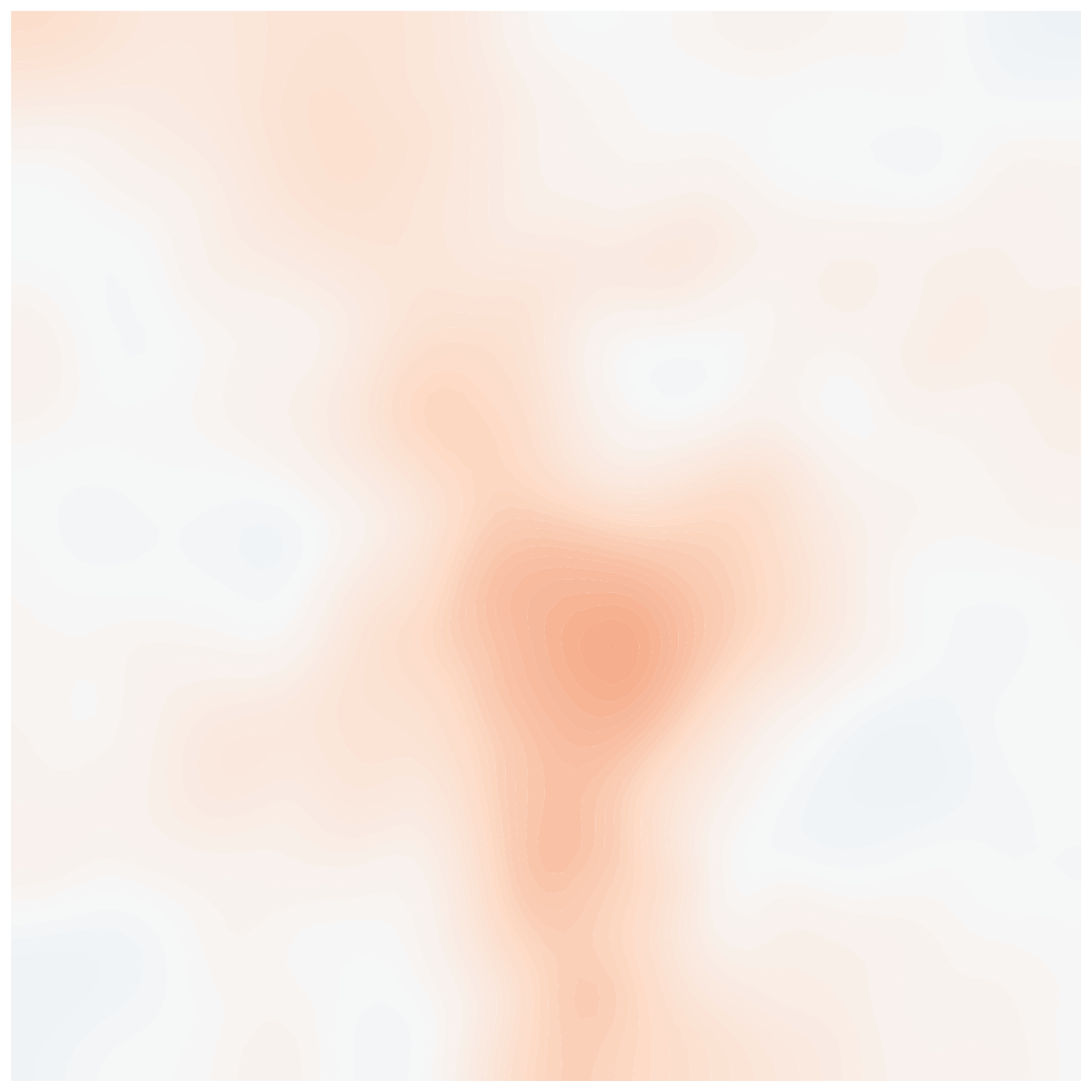}} & \raisebox{-0.13\linewidth}{\includegraphics[width=0.9\linewidth]{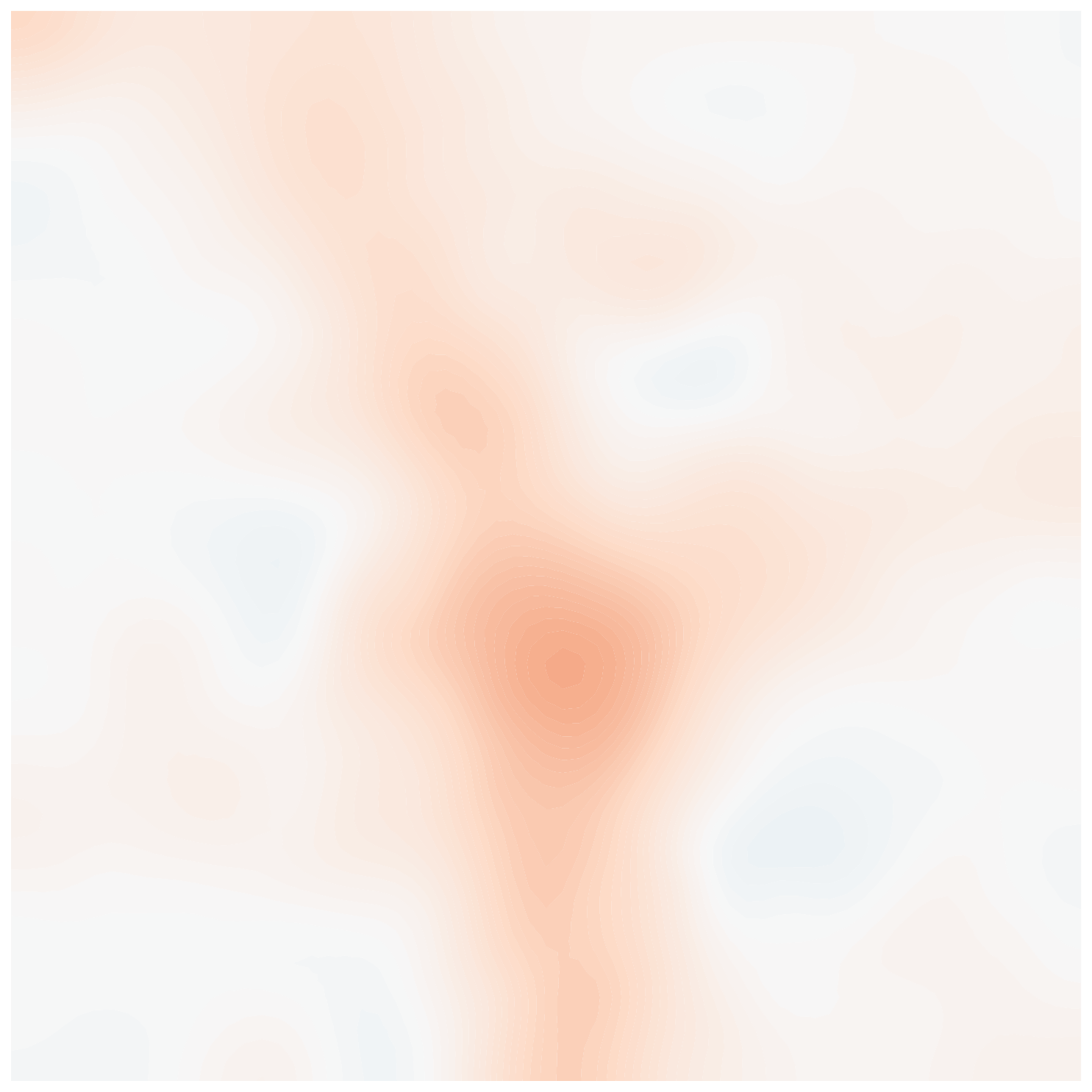}} & \raisebox{-0.13\linewidth}{\includegraphics[width=0.9\linewidth]{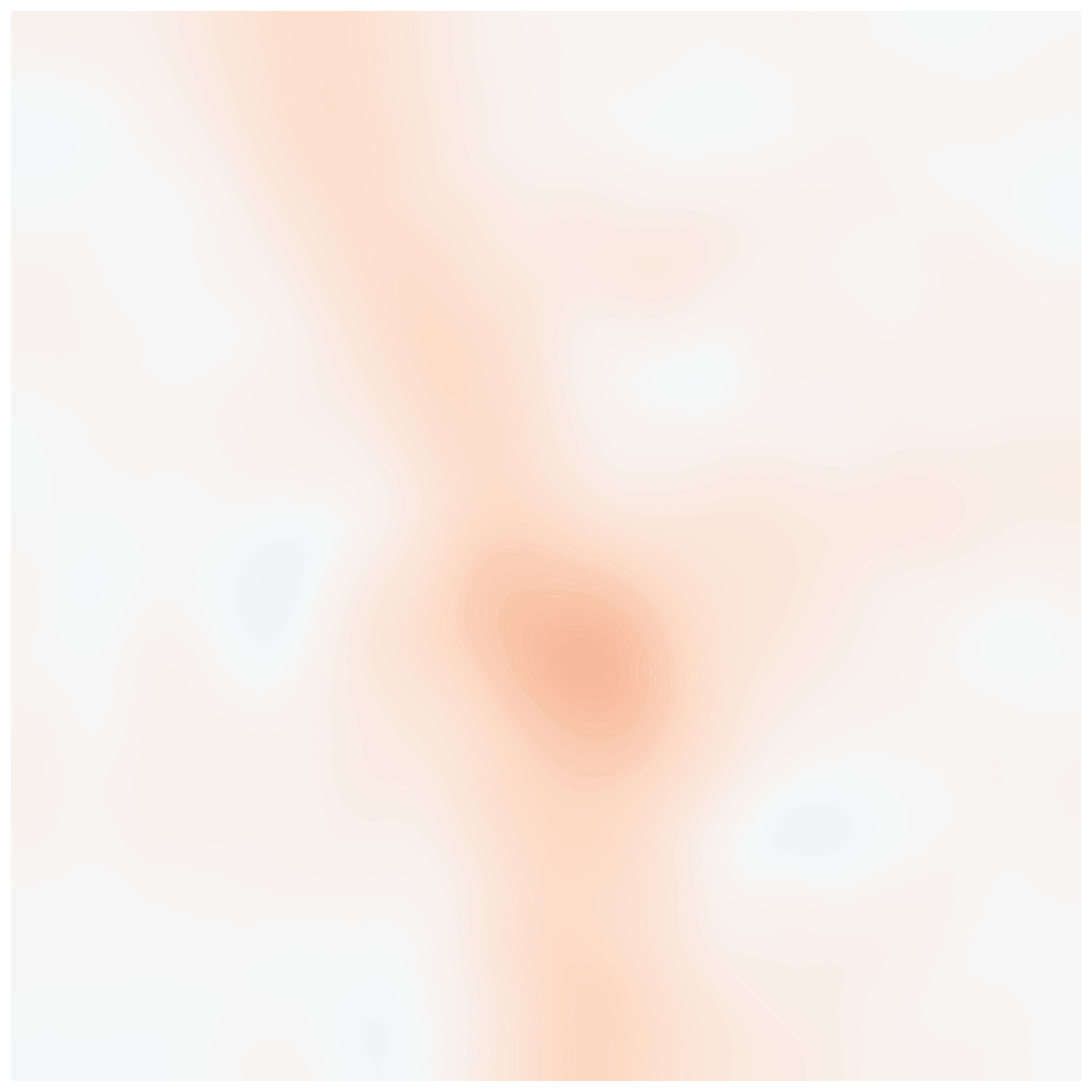}} & \\\cline{1-5}
     \makecell{Relative\\Error}&  \quad \textbf{15.06 \%} & \quad \textbf{16.12 \%} & \quad 16.41 \% & \quad \textbf{13.00 \%} &\\\cline{1-5}
     \makecell{Mixed-Res.\ \\DIFNO}&  \raisebox{-0.13\linewidth}{\includegraphics[width=0.9\linewidth]{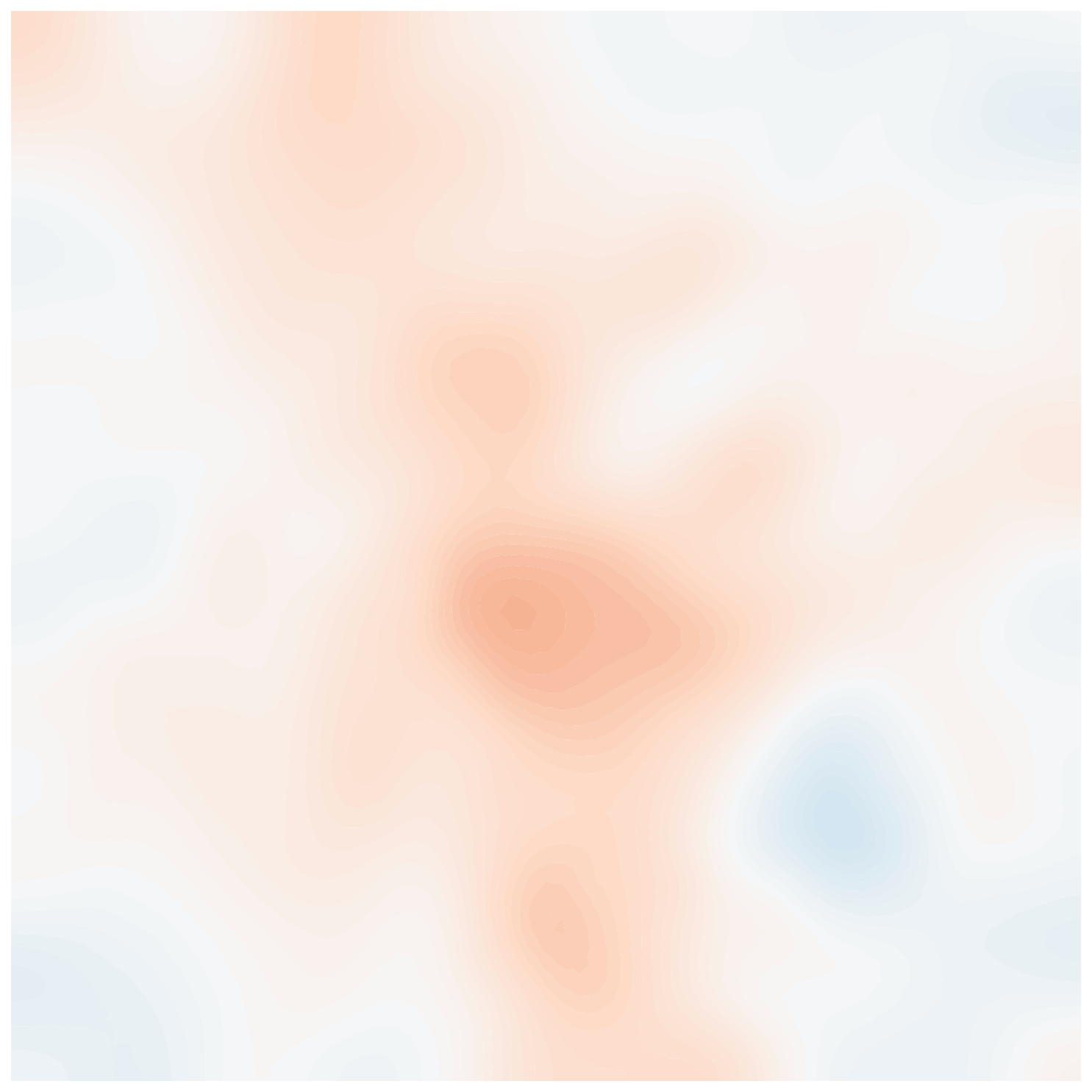}} & \raisebox{-0.13\linewidth}{\includegraphics[width=0.9\linewidth]{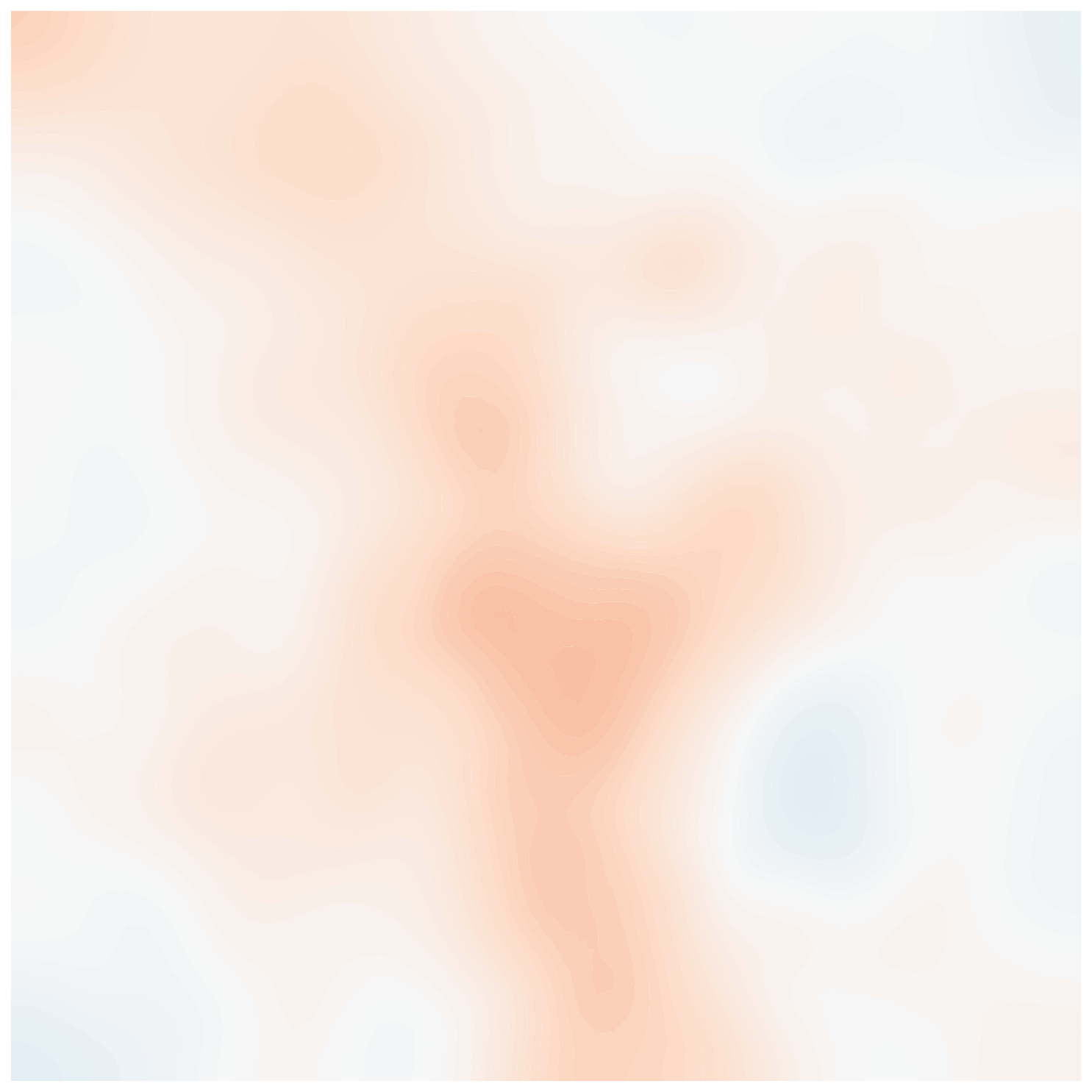}} & \raisebox{-0.13\linewidth}{\includegraphics[width=0.9\linewidth]{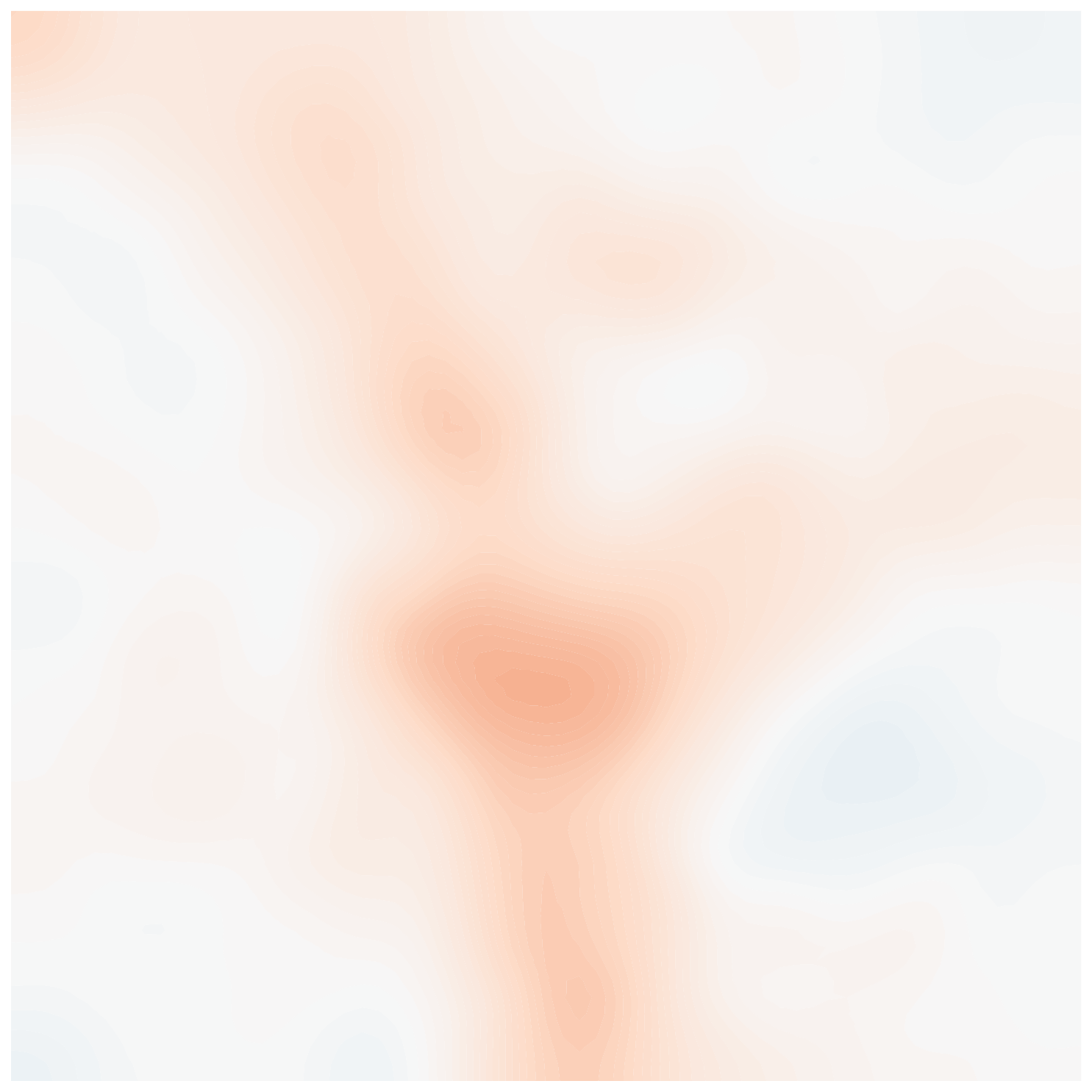}} & \raisebox{-0.13\linewidth}{\includegraphics[width=0.9\linewidth]{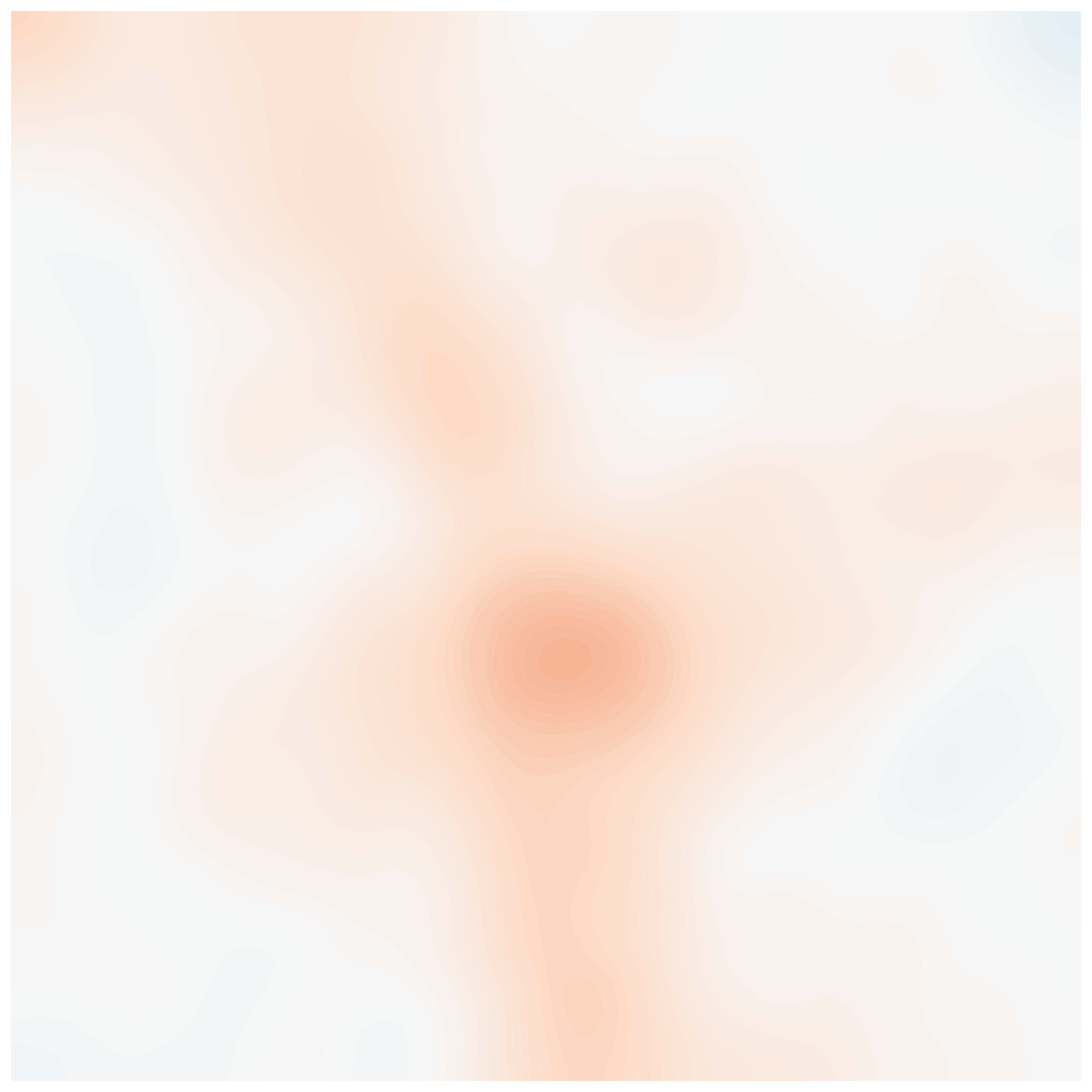}} & \\\cline{1-5}
     \makecell{Relative\\Error}& \quad 16.09 \% & \quad 16.21 \% & \quad \textbf{15.15 \%} & \quad 14.06 \% & \\\hline
     
\end{tabular}
    \renewcommand{\arraystretch}{1}
}

    \caption{Spatial pointwise errors in the inverse solutions by FNOs and DIFNOs when compared to the reference solution found by the PDE model. From top to bottom are results by the FNOs, the reduced-basis DIFNOs with DIS dimension reduction, and the mixed-resolution DIFNOs at different numbers of training samples; see their generalization error measures in \Cref{fig:poisson_train}.
    The relative spatial $L^2$ errors in the inverse solutions are shown below each plot.}
    \label{fig:poisson_map_comparison}
\end{figure}

\subsection{Helmholtz}
\paragraph{Problem Setup}For a second numerical example, we consider the Helmholtz equation, adapted from \cite{Wu2023}. This PDE is particularly difficult for operator learning due to its non-coercive ellipticity and the resulting oscillatory PDE solutions. We consider a Helmholtz problem given by:
\begin{equation}\label{eq:helmholtz}
    \begin{aligned} 
        \left(-\Delta - \exp(2a(x))\kappa^{2} \right)u(x) = f(x), \quad & x \in \Omega = [0, 3]^2, \\
        \text{PML boundary condition}, \quad &x \in \partial \Omega \setminus \Gamma_{\text{top}}, \\
        \nabla u(x) \cdot n = 0, \quad & x \in \Gamma_{\text{top}}.
    \end{aligned}
\end{equation}
where $\Gamma_{\text{top}}$ denote the top boundary, $\kappa =9.11$ is the wavenumber, and $n$ is the outward normal. The PDE solution $u$ represents the total wave field, and the input $a$ to the PDE solution operator is the logarithmic prefactor of the wavenumber, with an input distribution that has a covariance of $\cC_\cX = (12.5\cI - 0.5\Delta)^{-2}$. The source term $f$ is a point source located at $x = (0.775, 2.5)$. The perfectly matched layer (PML) boundary condition approximates a semi-infinite domain \cite{berenger1994pml}. The input and output spaces are discretized using linear triangular finite elements with 104 cells in each direction, including unit-width PML layers. We use the $L^2((-1, 4) \times (-1, 3))$ norm for the output space $\cY$. We use a $105\times 105$ high-fidelity grid for the FNOs and DIFNOs. For reduced-basis DIFNOs, we use $r_{\cX_{\delta}} = 196$ and $r_{\cY} = 392$. For mixed-resolution DIFNOs, we use a $14 \times 14$ low-fidelity grid and a $27 \times 27$ intermediate-resolution grid for the calculations.

\begin{figure}[htb]
\centering
{
\small
\begin{tabular}{c c c}
    \hspace{-0.03\linewidth}\makecell{Wavenumber \\
    Log-Prefactor Field (Input)}&\hspace{-0.03\linewidth} \hspace{-0.03\linewidth}\makecell{PDE solution (Re.)\\ with Observations} &\hspace{-0.04\linewidth}\makecell{PDE solution (Im.)\\ with Observations}\\
    \includegraphics[width=0.27\linewidth]{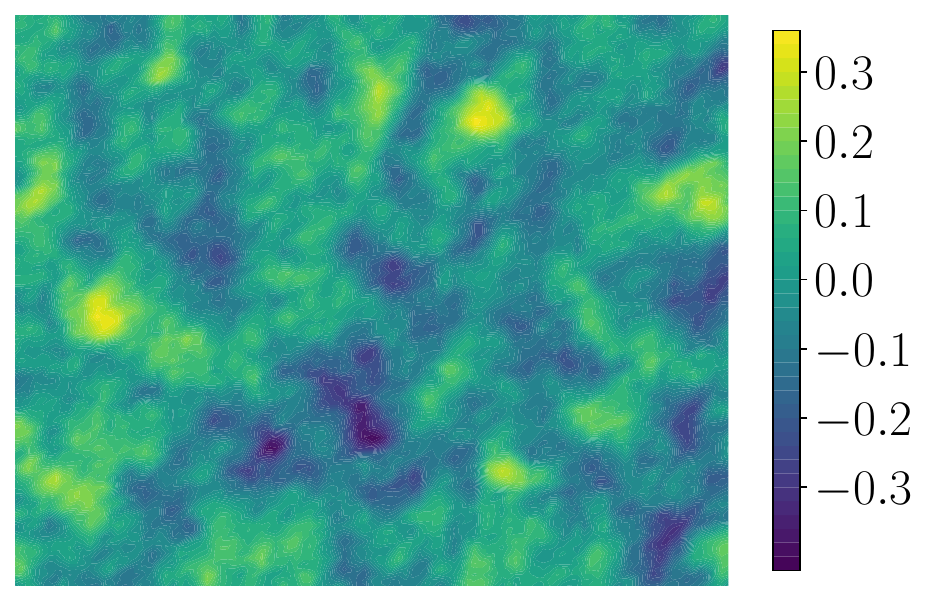}  & \includegraphics[width=0.27\linewidth]{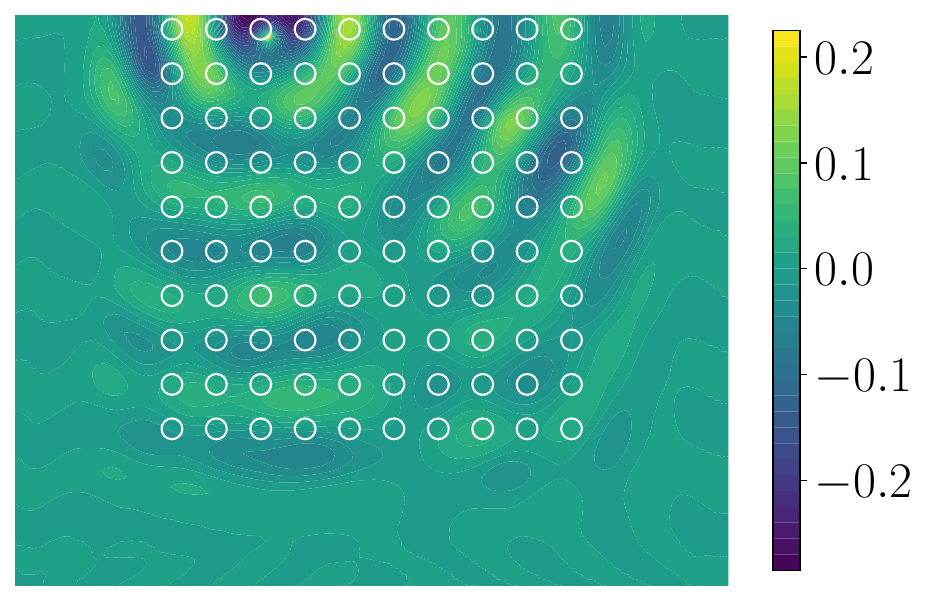} & \includegraphics[width=0.27\linewidth]{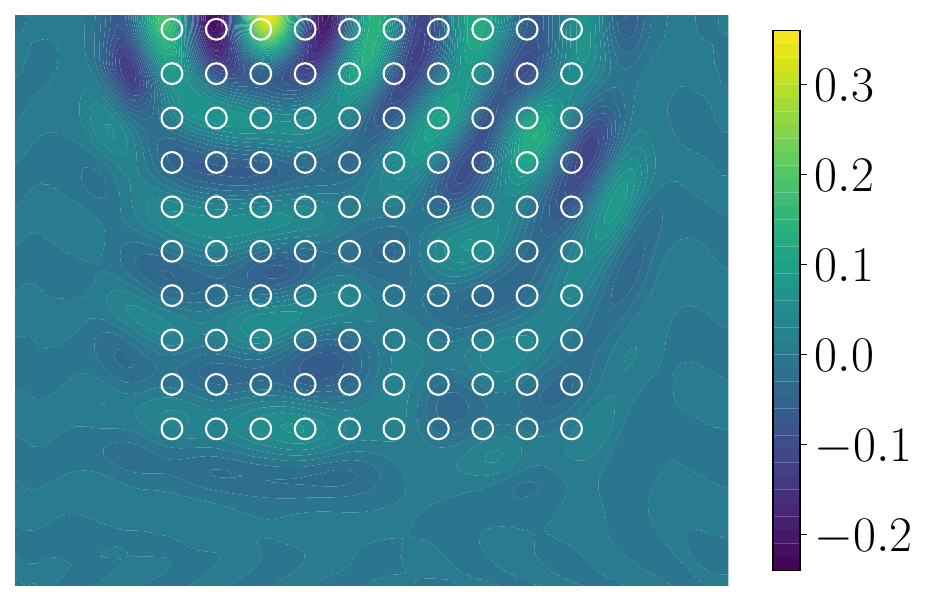}
\end{tabular}
}
    \caption{The step up of the numerical study on the Helmholtz equation. (\textit{Left}): An input sample $a^{(i)} \sim \mu$ of the wavenumber log-prefactor field. (\textit{Center}) The real part of the PDE solution at $a^{(i)}$. (\textit{Right}): The imaginary  part of the PDE solution at $a^{(i)}$. The circles indicate the locations of the pointwise observations of the PDE solution for the inverse problem.}
    \label{fig:helmholtz_param_state}
\end{figure}

\paragraph{Training Results} We train FNOs, reduced-basis DIFNOs, and mixed-resolution DIFNOs at varying training sample sizes, and their testing results are shown in \Cref{fig:poisson_train}. These results demonstrate that DIFNOs achieve substantial improvements in predicting both PDE solutions and Fr\'echet derivatives. Reduced-basis DIFNOs obtained with two different linear dimension reduction methods perform similarly, and both substantially outperform mixed-resolution DIFNOs. We reiterate that, despite inferior generalization accuracy, the mixed-resolution DIFNOs train faster.

\begin{figure}[htb]
    \centering
\small
    \renewcommand{\arraystretch}{1.2} 
    \setlength{\tabcolsep}{5pt}       

    \begin{tabular}{c c c}
        & \hspace{0.03\linewidth} \makecell{Helmholtz\\
        PDE Solution Relative Error} & \hspace{0.01\linewidth} \makecell{Helmholtz\\Fr\'echet Derivative Relative Error} \\
        
        \raisebox{3em}{\rotatebox{90}{Relative Error}} & 
        \includegraphics[width=0.42\linewidth]{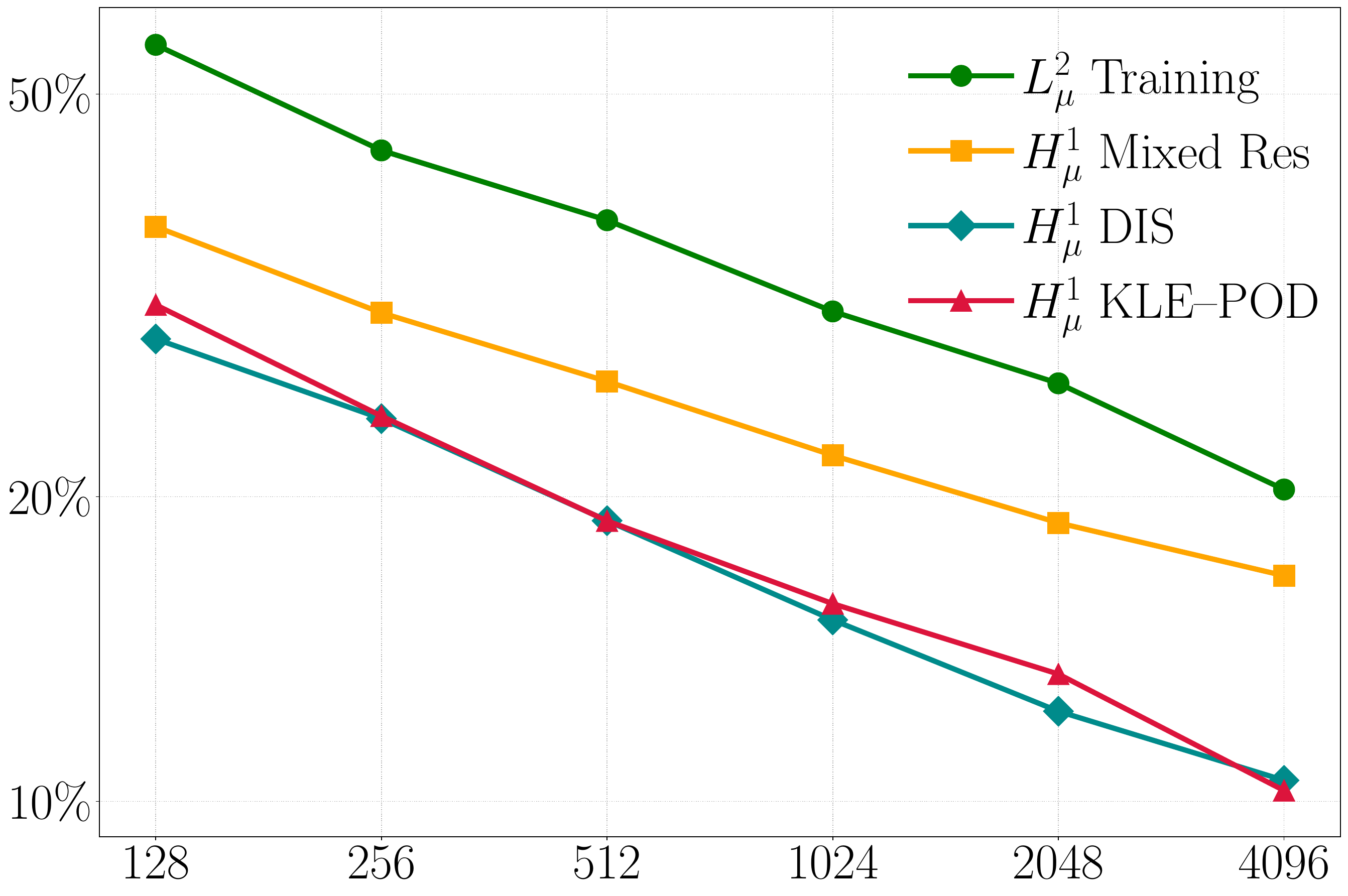} & 
        \includegraphics[width=0.42\linewidth]{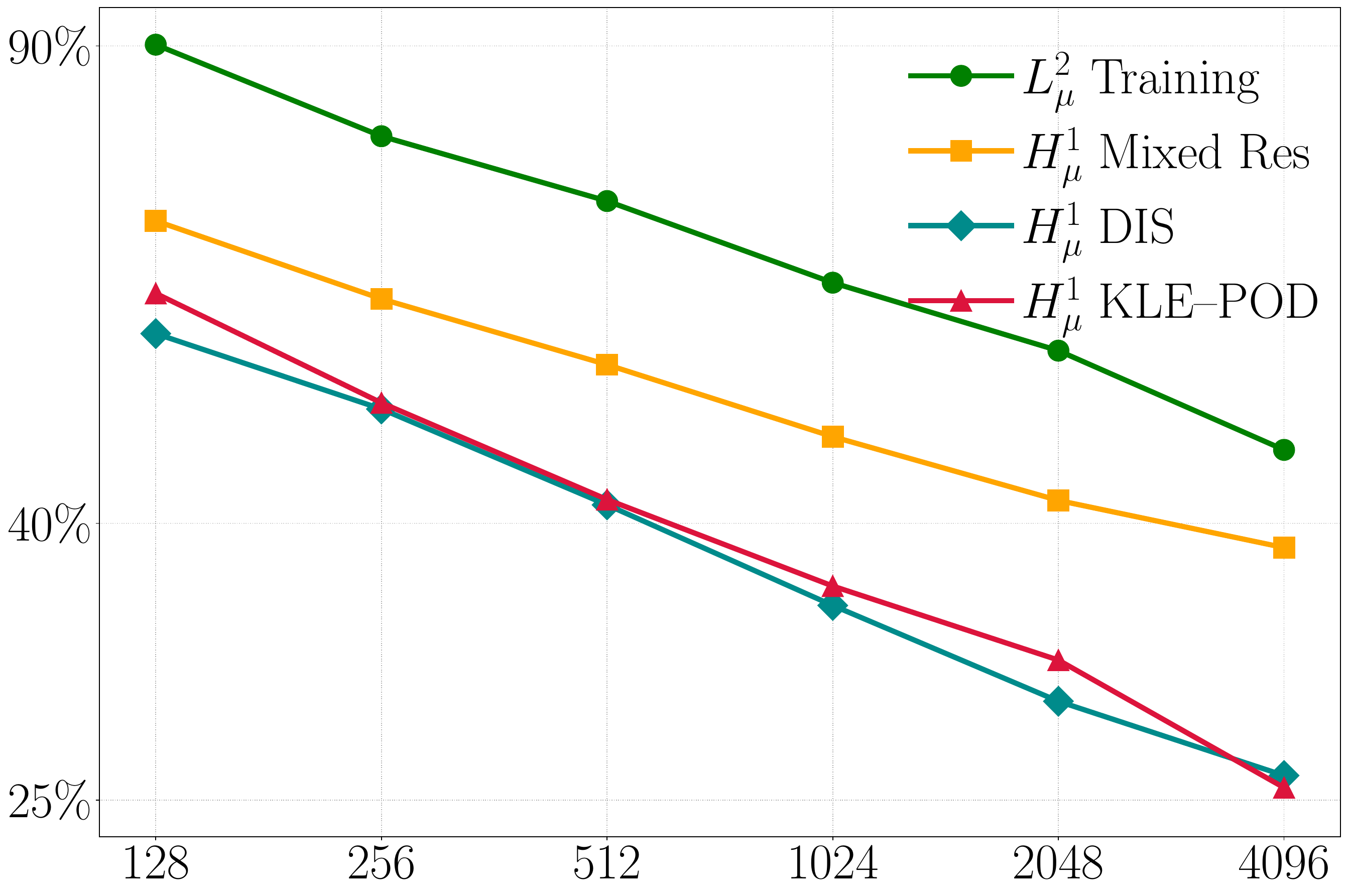} \\
        
        & \hspace{0.03\linewidth} Number of Training Samples & \hspace{0.03\linewidth} Number of Training Samples \\
    \end{tabular}

    \caption{The relative error in the PDE solutions and Fr\'echet derivative predictions for the Helmholtz equation. We train FNOs ($L^2_{\mu}$ Training), reduced-basis DIFNOs with KLE--POD dimension reduction ($H^1_{\mu}$ KLE--POD), reduced-basis DIFNOs with DIS dimension reduction ($H^1_{\mu}$ DIS), and mixed-resolution DIFNOs ($H^1_{\mu}$ Mixed Res.) on a varying number of training samples. We use an additional 128 testing samples to evaluate the relative error.}
    \label{fig:helmholtz_train}
\end{figure}

\paragraph{Inverse Problem Results} We now study the solutions of the inverse problem using the PDE model, FNO, and DIFNO. To create an out-of-distribution sample, the data-generating (true) wavenumber log-prefactor field is obtained by first sampling the input distribution $\mu$ and then applying a filtering function $g$ pointwise to the sample:
\begin{equation*}
    g(x) = \text{sign}(x) \cdot \min \left\{\frac{\sup\{|a|\} |x|}{0.25}, \sup \{|a|\} \right\}
\end{equation*}
The pointwise observations of the PDE solution at this filtered sample are extracted; see the locations of the observations in \Cref{fig:helmholtz_param_state}. The data is then obtained by adding 0.1\% white noise. The data-generating input field, the inverse solutions, and their spatial pointwise errors are shown in \Cref{fig:helmholtz_map} and \Cref{fig:helmholtz_map_comparison}. The results show similar trends as the inverse problem results for the nonlinear diffusion--reaction equation. Here, the mixed-resolution DIFNOs struggle to keep up with the reduced-basis DIFNOs, yet still substantially outperform the FNOs. The FNOs failed to capture the vanishing spatial gradient in the PML layers. Furthermore, we pay special attention to two inverse solutions produced by (i) the FNO with 4096 training samples and (ii) the reduced-basis DIFNO with 512 training samples. This pair of FNO and DIFNO performs similarly in predicting PDE solutions, but the FNO has slightly lower Fr\'echet derivative accuracy. We see that the pointwise error in the inversion solutions by this FNO is much more pronounced, which supports our theoretical analysis in 
Propositions \ref{prop:surrogate_optimization_error}, \ref{prop:surrogate_optimization_error_strongly_convex}, and \Cref{theorem:inverse_problem_ua}.

\begin{figure}[htb]
    \centering
{
    \small
   \begin{tabular}{c c c c c}
        \makecell{Data-Generating\\ Input Field (True)} \quad&
          \makecell{Inverse Solution \\ by PDE (Reference)} & \makecell{Inverse Solution \\ by FNO}&  \makecell{Inverse Solution \\ by DIFNO} &\\
         \includegraphics[width=0.20\linewidth]{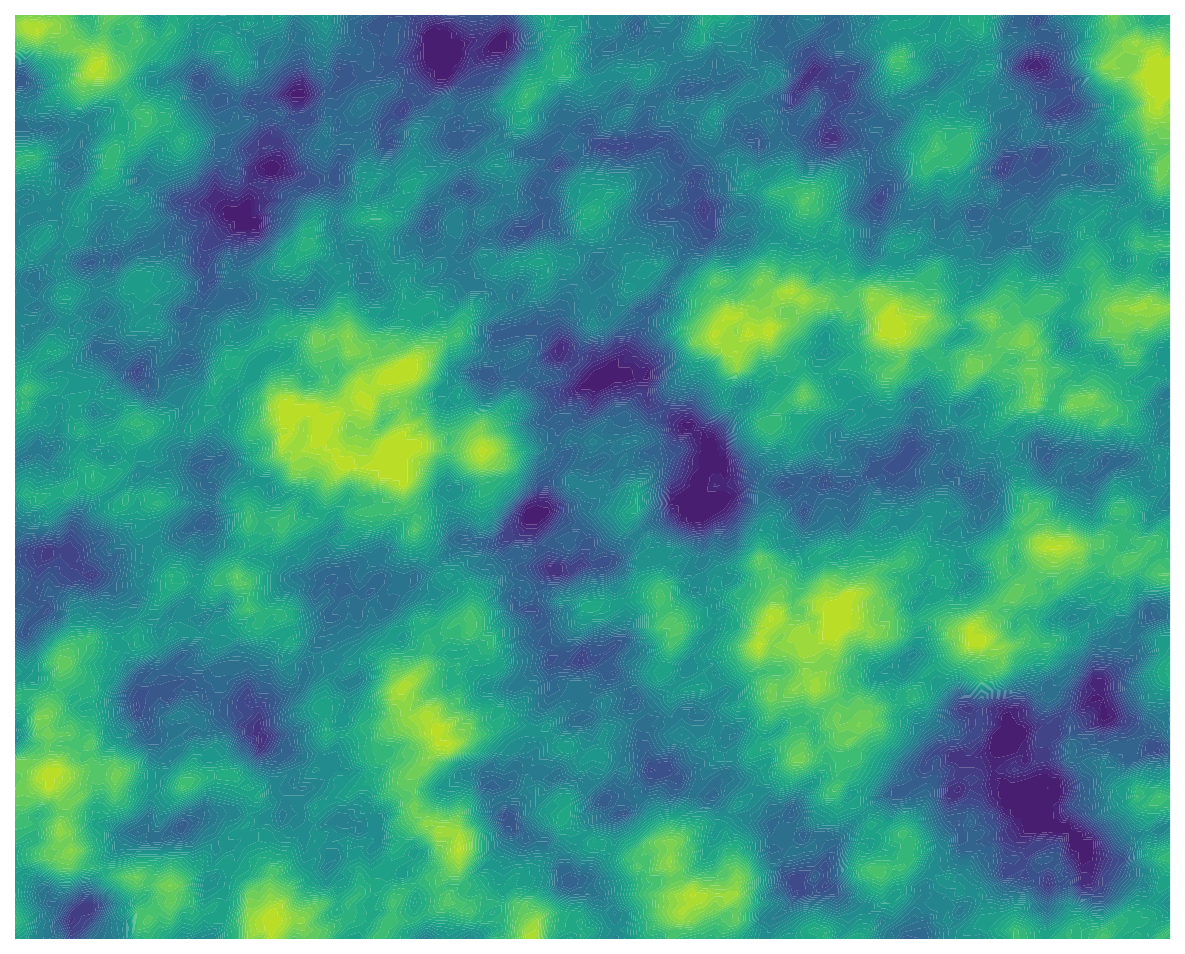} & 
         \includegraphics[width=0.20\linewidth]{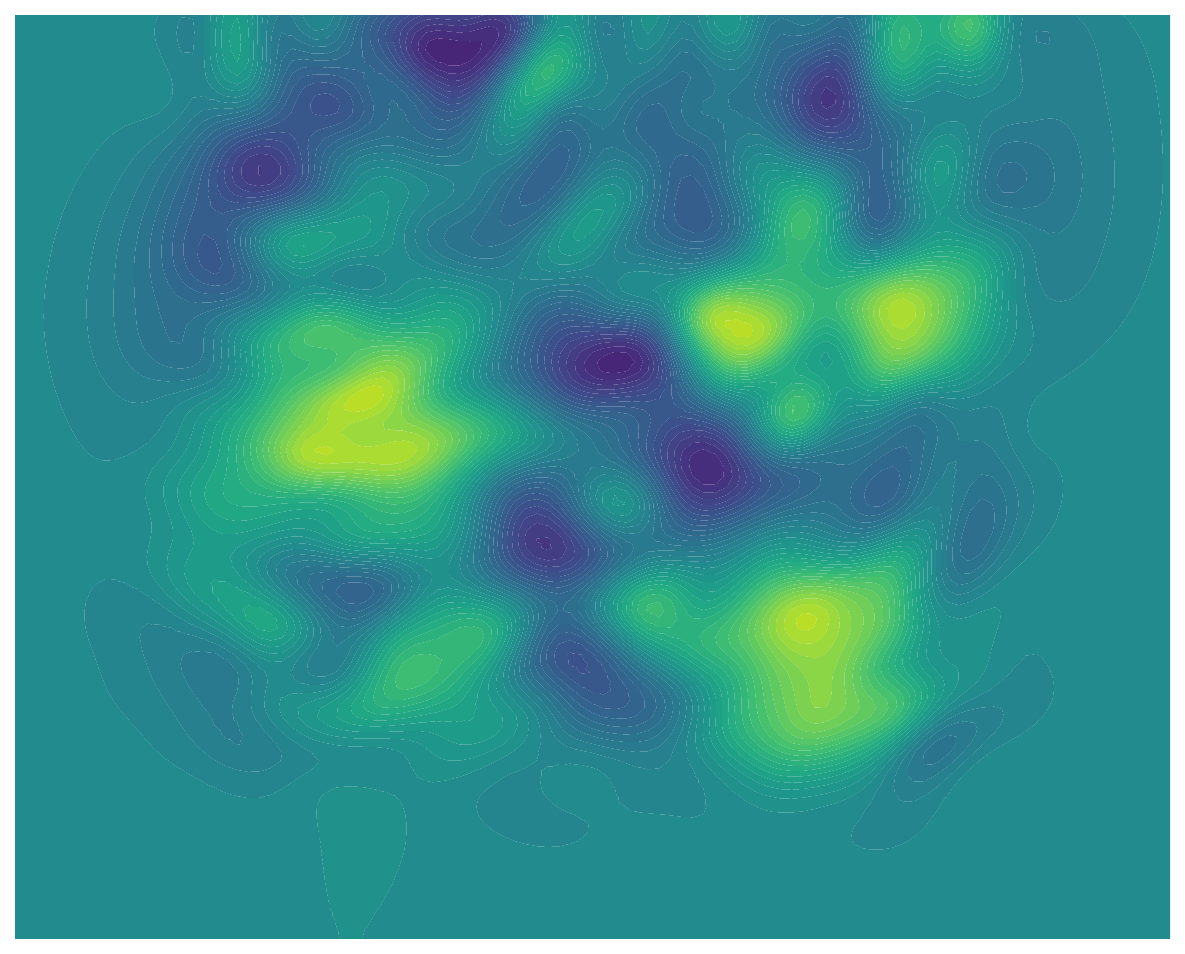} &
         \includegraphics[width=0.20\linewidth]{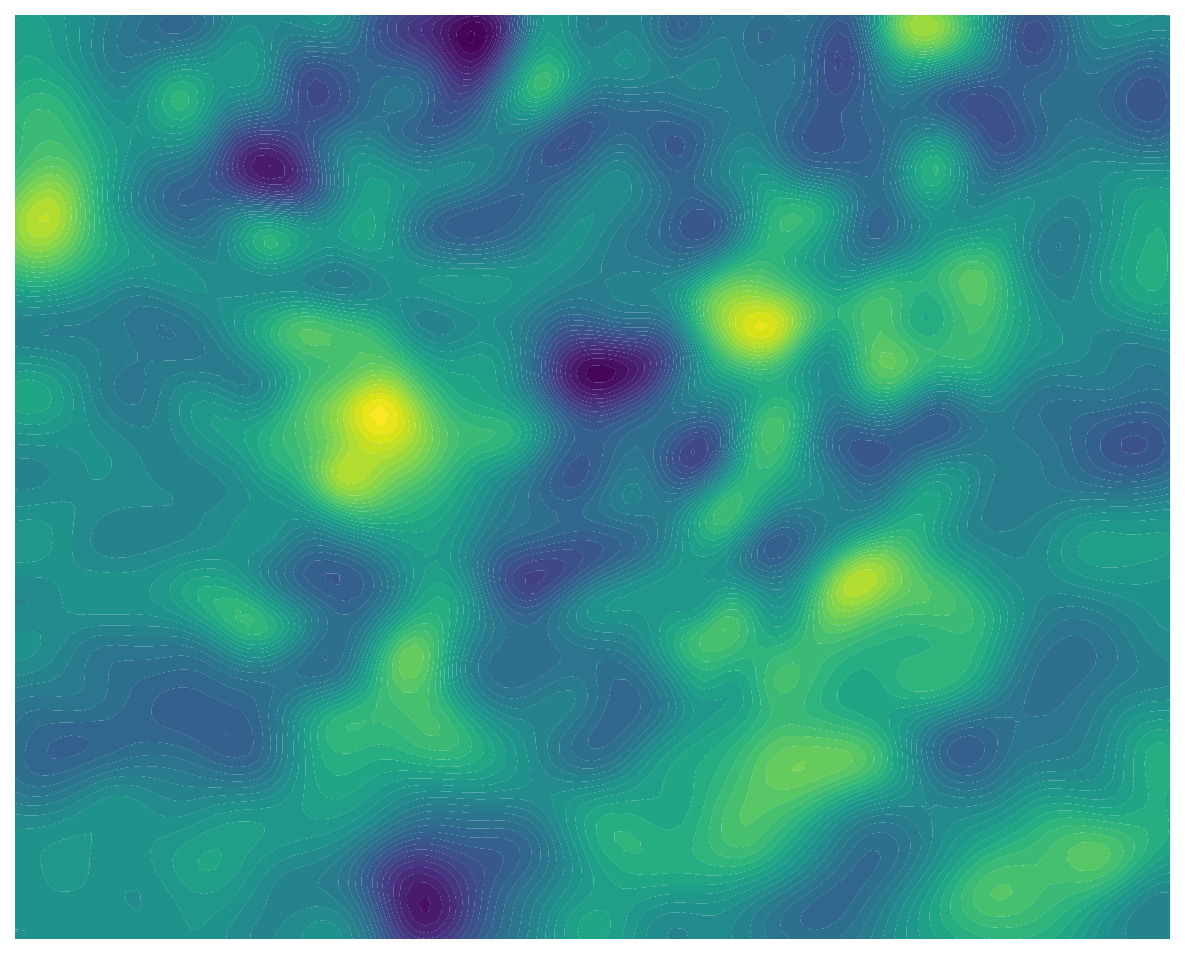} &
         \includegraphics[width=0.20\linewidth]{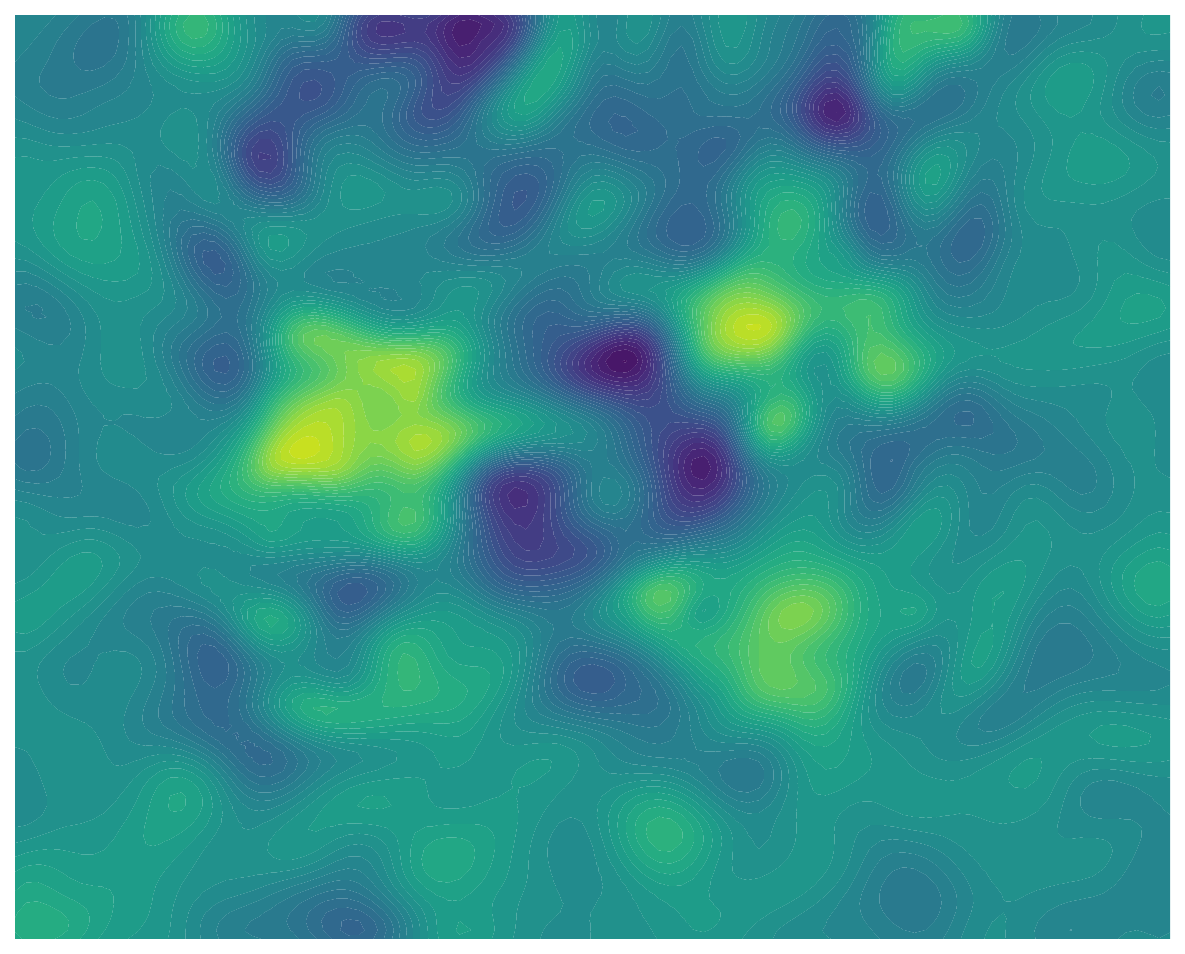} & 
         \includegraphics[width=0.04\linewidth]{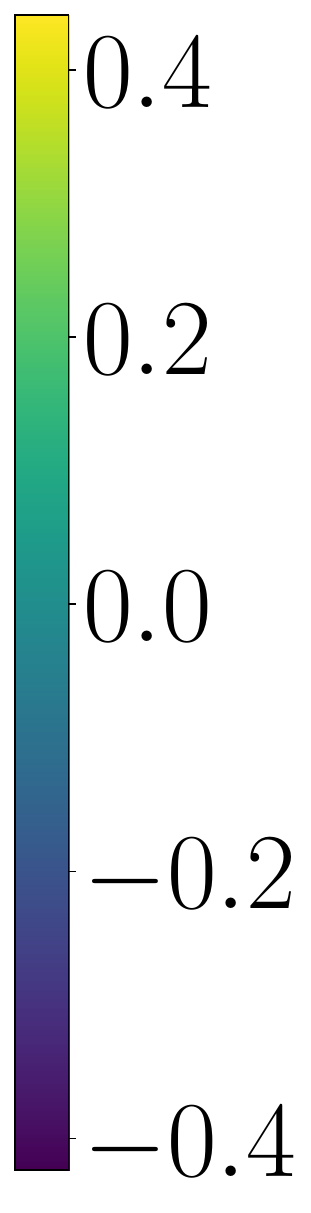}
    \end{tabular}
}
    \caption{Solutions of inverse problems based on the Helmholtz equation in \Cref{fig:helmholtz_param_state}. The left figure is the data-generating (true) wavenumber log-prefactor field. Then, from left to right, are inverse solutions obtained by PDE, the best-performing FNO, and the best-performing DIFNO.
    }
    \label{fig:helmholtz_map}
\end{figure}

\begin{figure}[htb]
\centering
\scalebox{0.9}{\renewcommand{\arraystretch}{1.5}
\begin{tabular}{
|>{\centering\arraybackslash} m{0.115\linewidth}||
>{\centering\arraybackslash} m{0.15\linewidth}|
>{\centering\arraybackslash} m{0.15\linewidth}|
>{\centering\arraybackslash} m{0.15\linewidth}|
>{\centering\arraybackslash} m{0.15\linewidth}|
>{\centering\arraybackslash} m{0.1\linewidth}|}\hline
    \multicolumn{6}{|c|}{\makecell{Errors in the Inverse Solutions for Helmholtz Equation}} \\\hline\hline
    &\multicolumn{4}{c|}{Number of Training Samples} & \multirow{2}{*}{\makecell{Pointwise\\Error}}\\\cline{1-5}
    & 512 & 1024 & 2048 & 4096 & \\\hline
    FNO& \raisebox{-0.13\linewidth}{\includegraphics[width=0.9\linewidth]{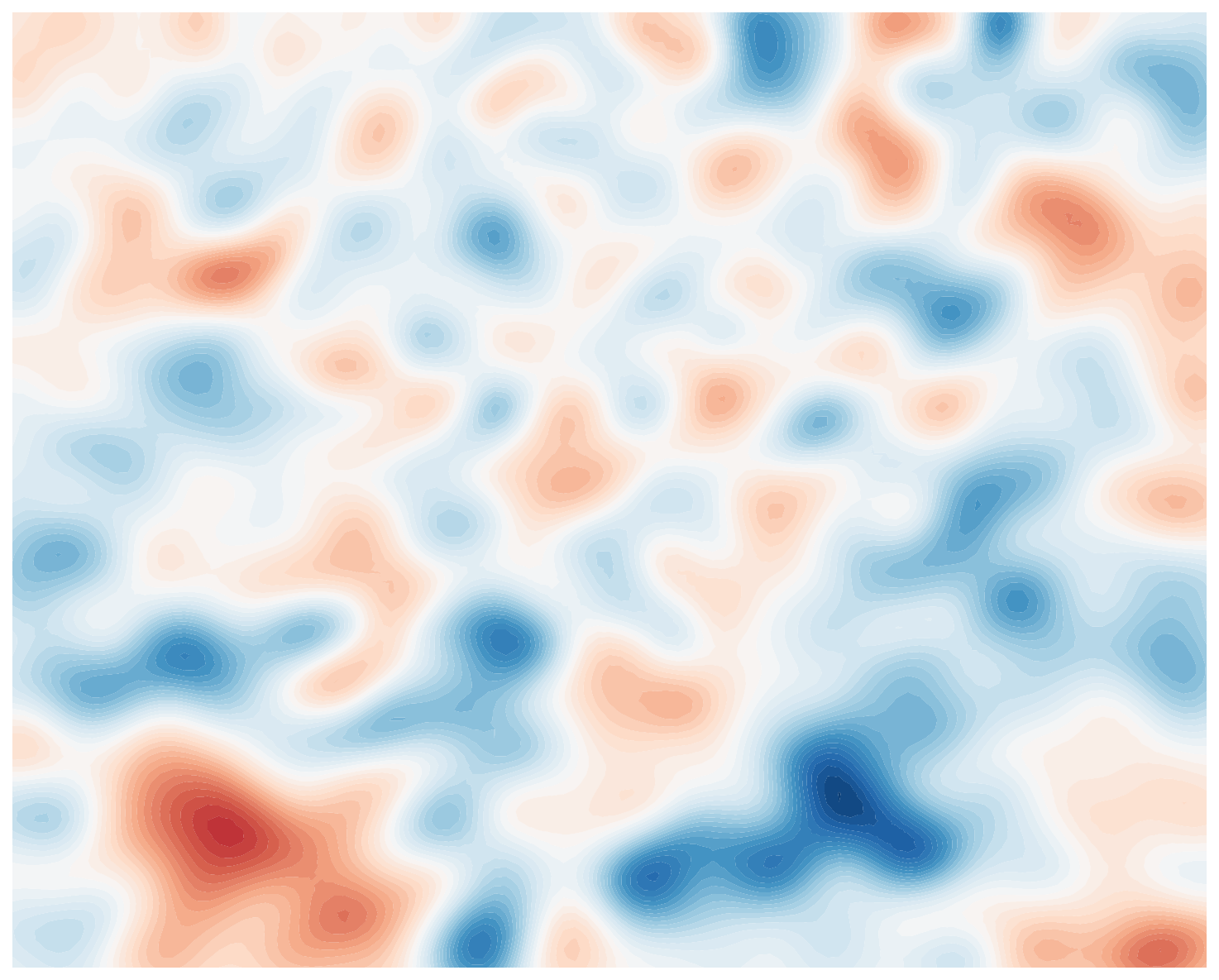}} & \raisebox{-0.13\linewidth}{\includegraphics[width=0.9\linewidth]{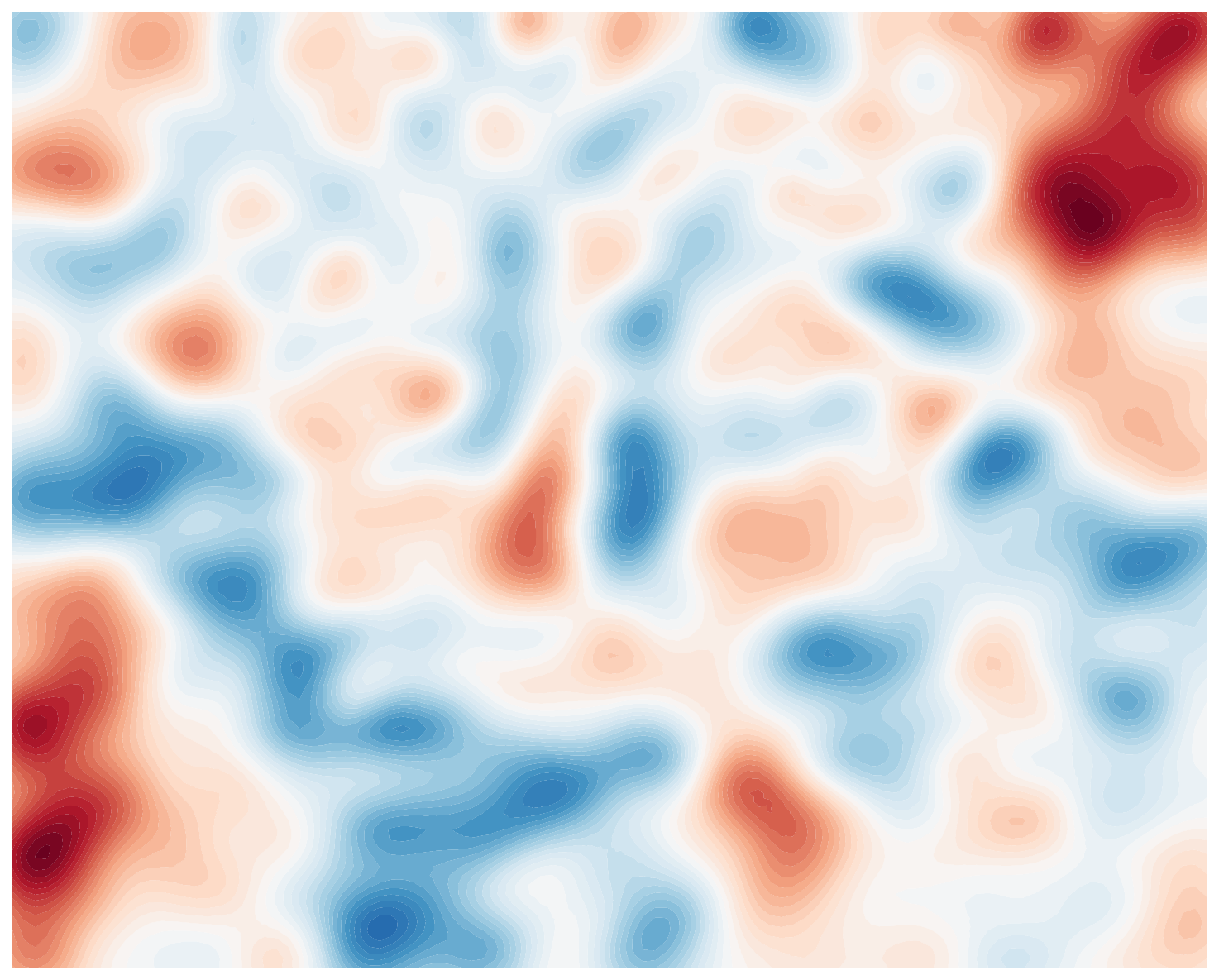}} & \raisebox{-0.13\linewidth}{\includegraphics[width=0.9\linewidth]{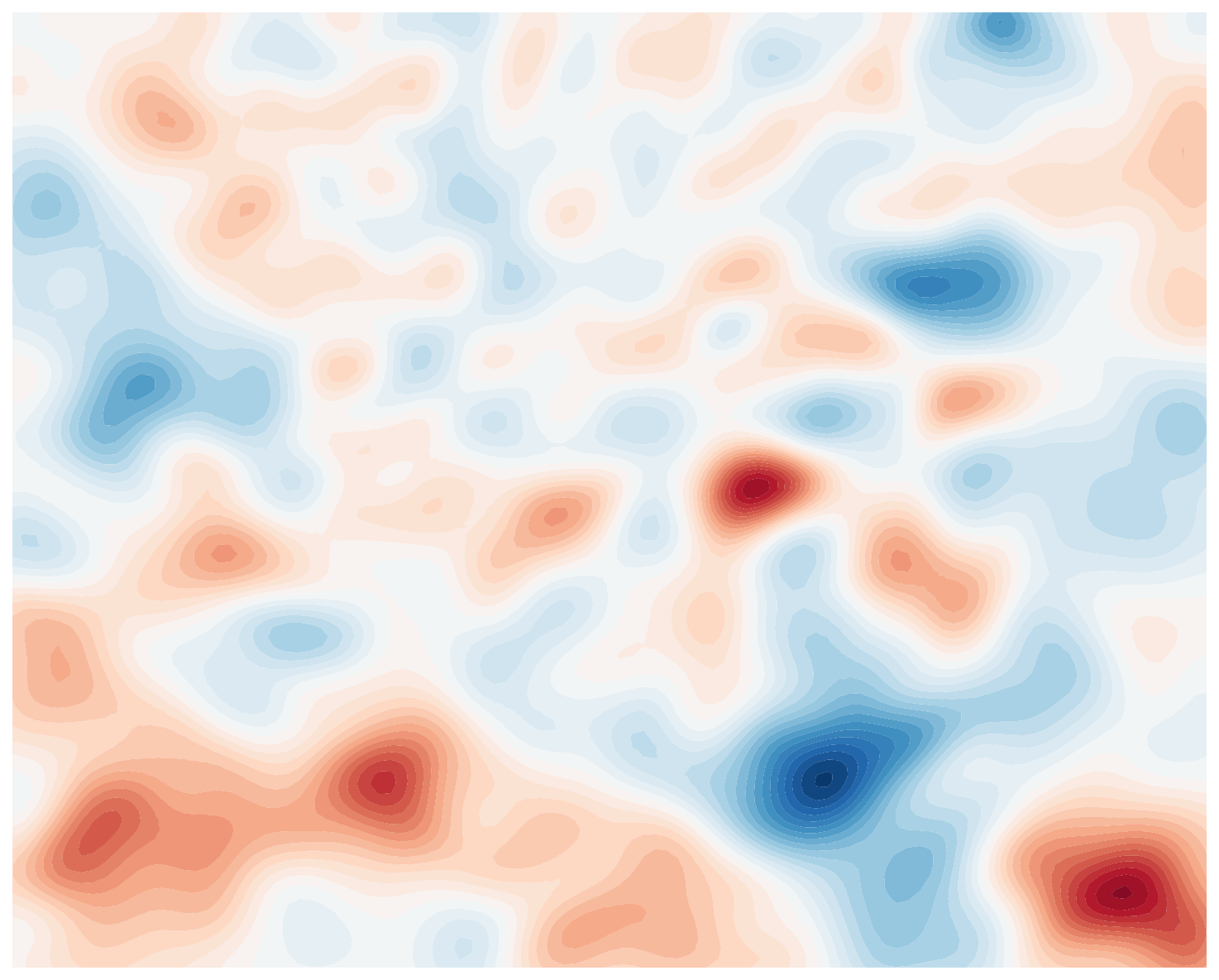}} & \raisebox{-0.13\linewidth}{\includegraphics[width=0.9\linewidth]{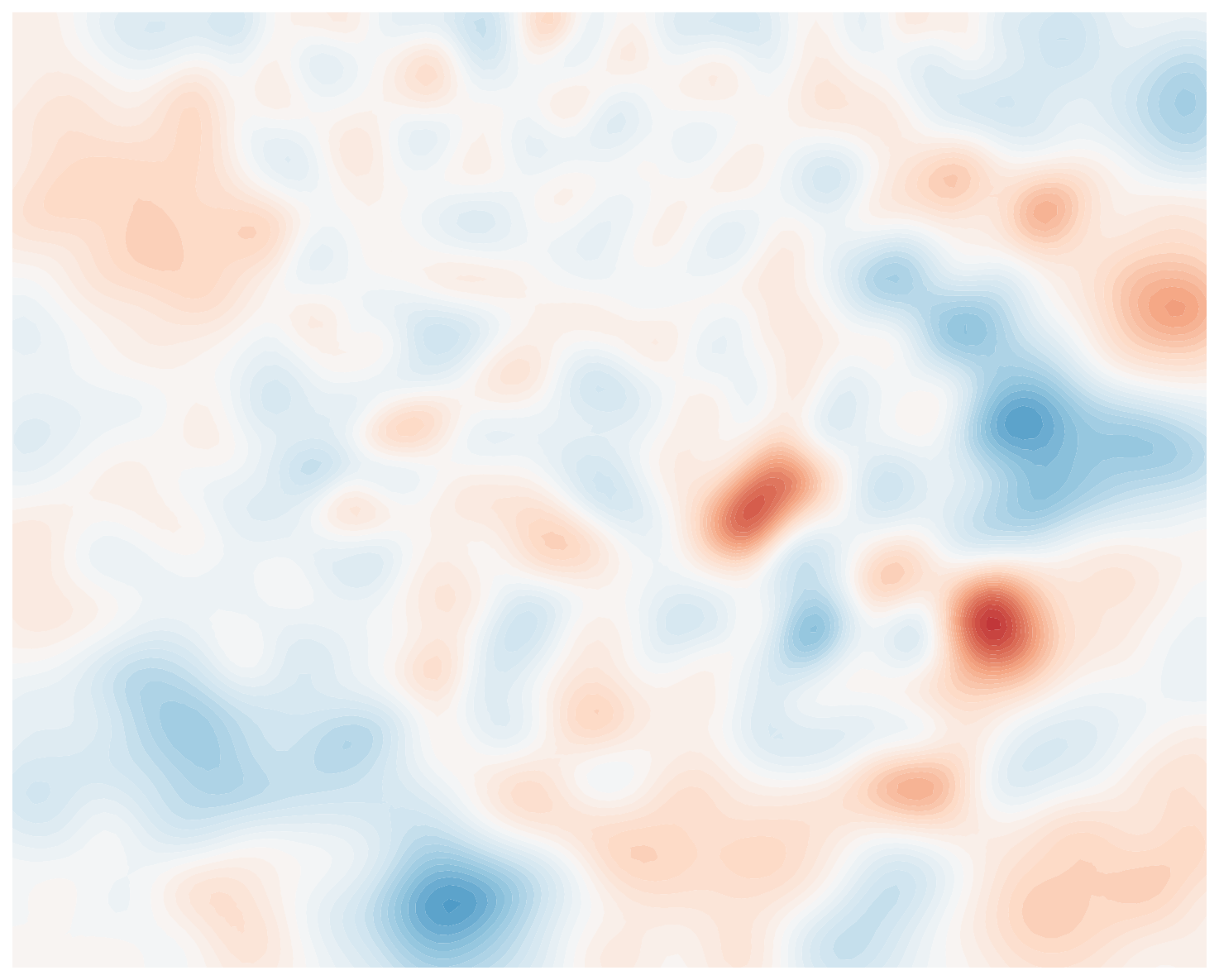}} & \multirow{5}{*}{\includegraphics[width=0.57\linewidth]{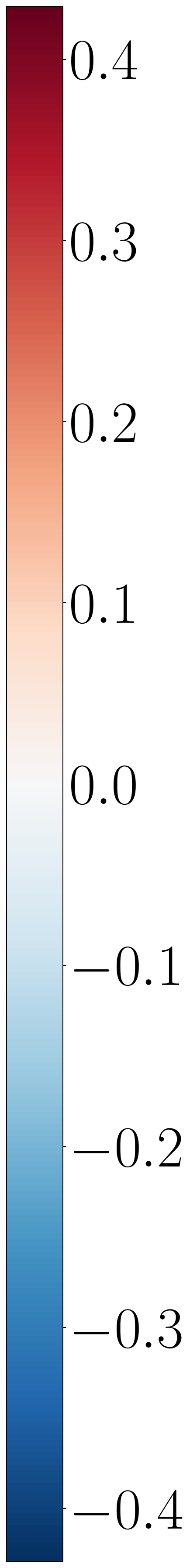}}\\\cline{1-5}
    \makecell{Relative\\Error}& \quad 163.07 \% & \quad 200.89 \% & \quad 151.51 \% & \quad 97.89 \% & \\\cline{1-5}
    \makecell{Red.-Basis\\DIFNO}& \raisebox{-0.11\linewidth}{\includegraphics[width=0.9\linewidth]{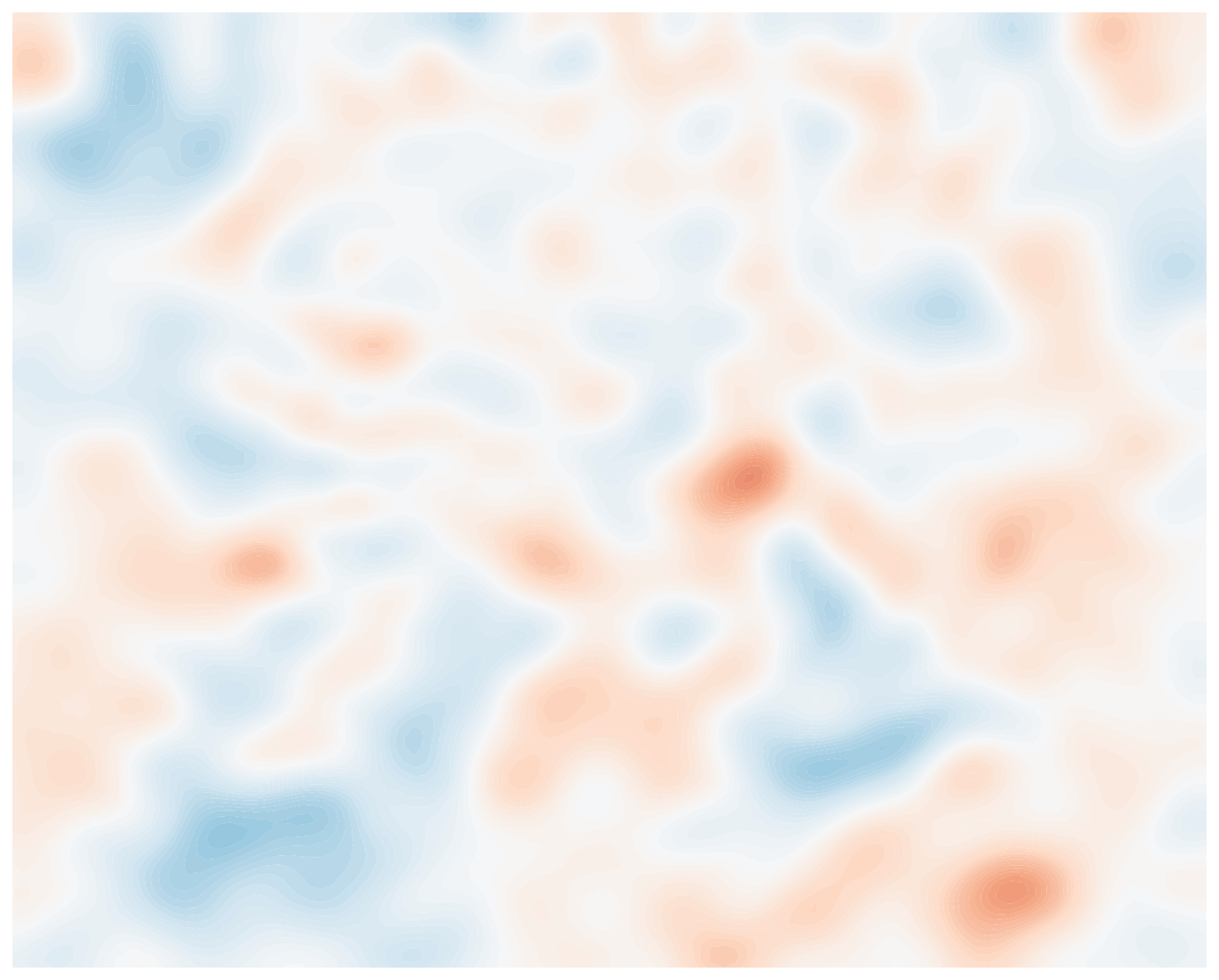}} & \raisebox{-0.11\linewidth}{\includegraphics[width=0.9\linewidth]{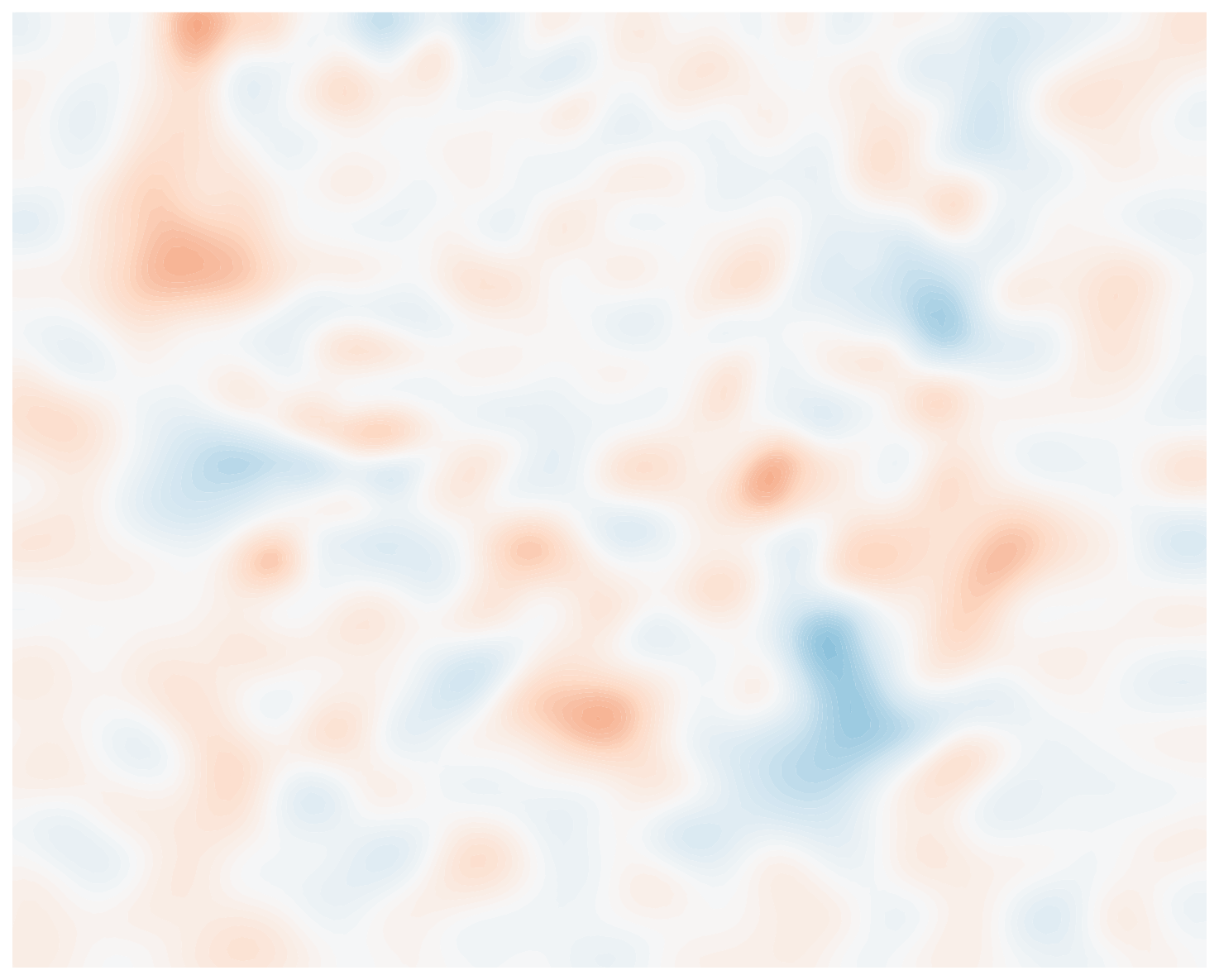}} & \raisebox{-0.11\linewidth}{\includegraphics[width=0.9\linewidth]{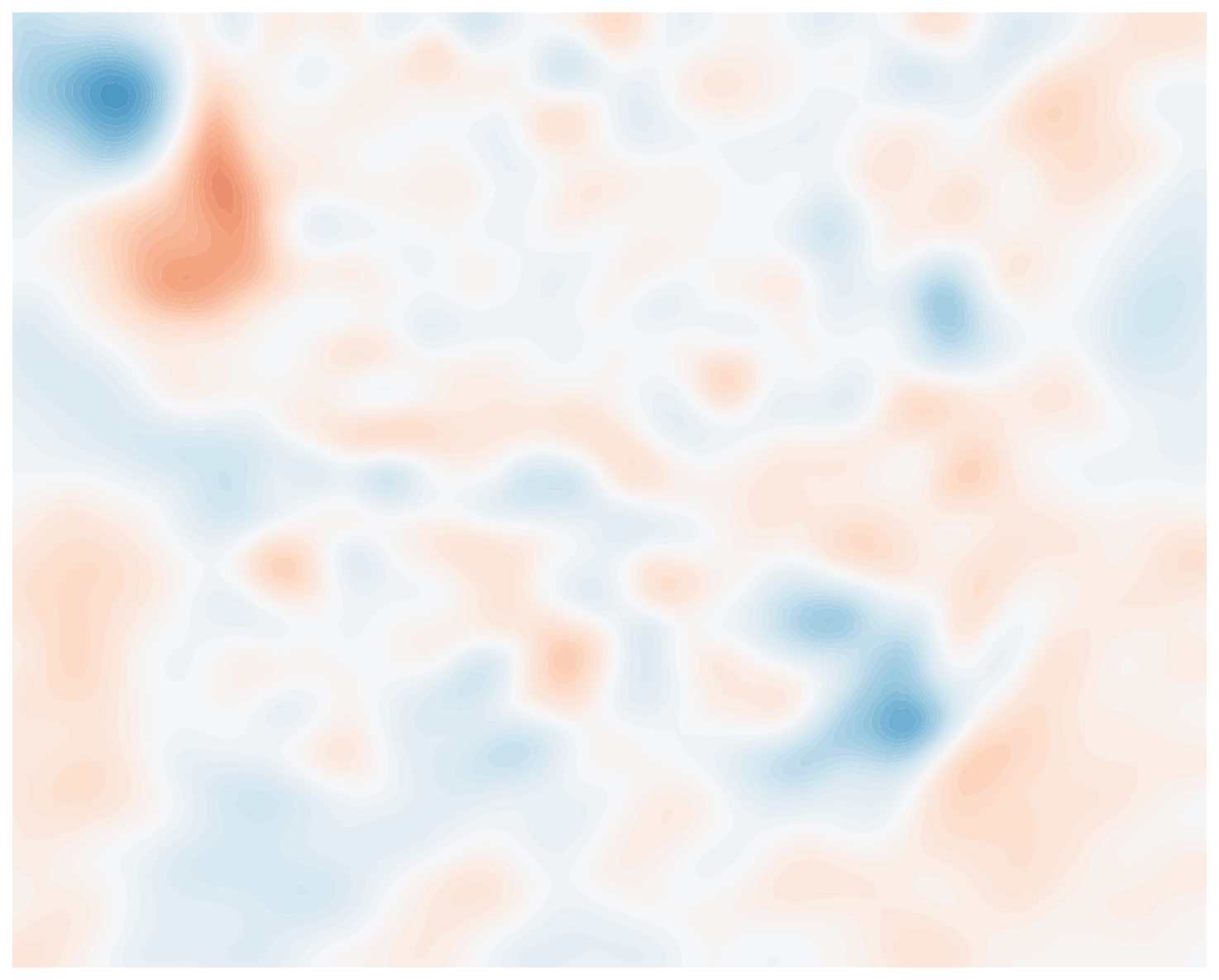}} & \raisebox{-0.11\linewidth}{\includegraphics[width=0.9\linewidth]{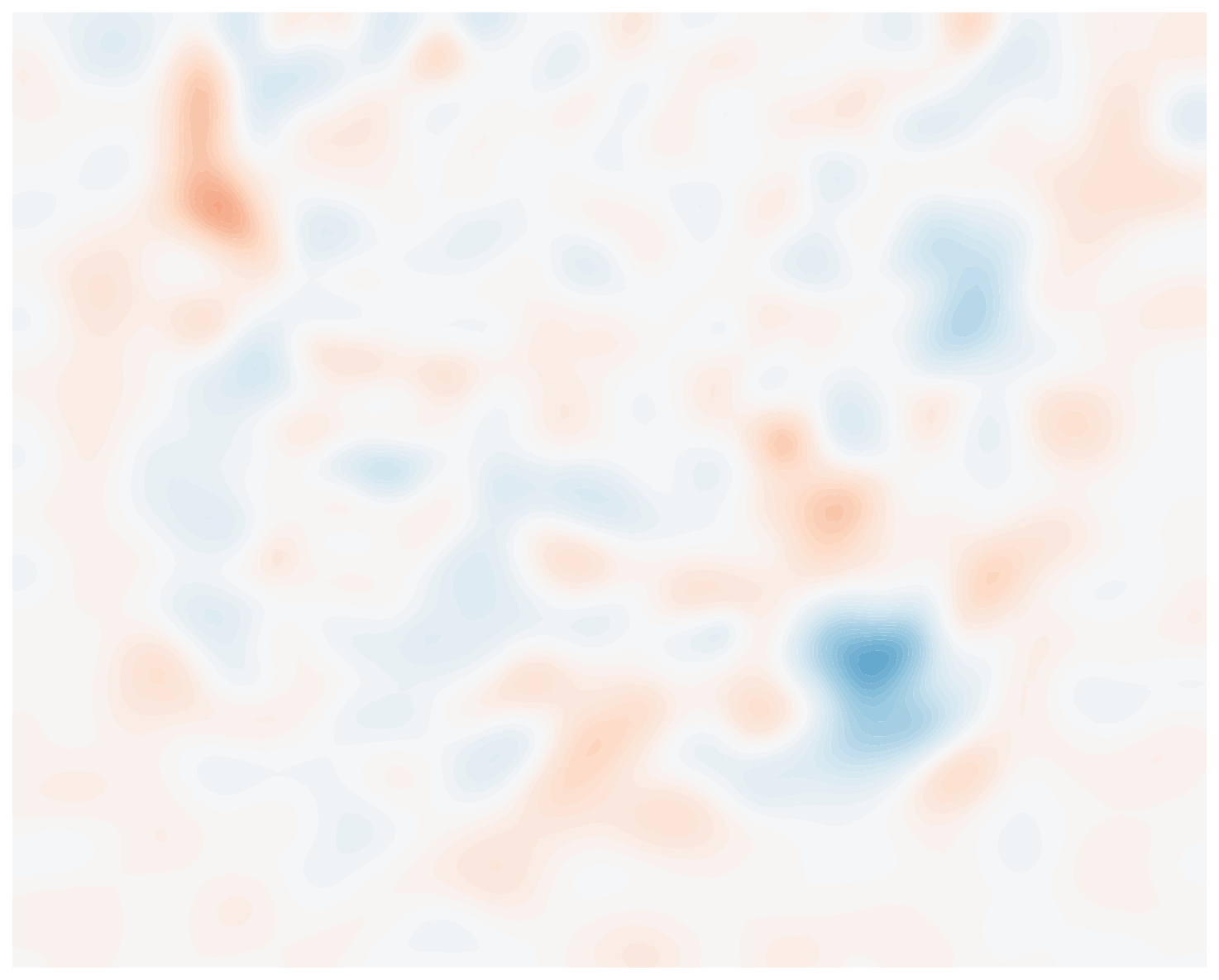}} & \\\cline{1-5}
    \makecell{Relative\\Error}& \quad \textbf{89.60 \%} & \quad \textbf{59.82 \%} & \quad \textbf{56.08 \%} & \quad \textbf{52.86 \%} &\\\cline{1-5}
    \makecell{Mixed-Res.\ \\DIFNO}& \raisebox{-0.11\linewidth}{\includegraphics[width=0.9\linewidth]{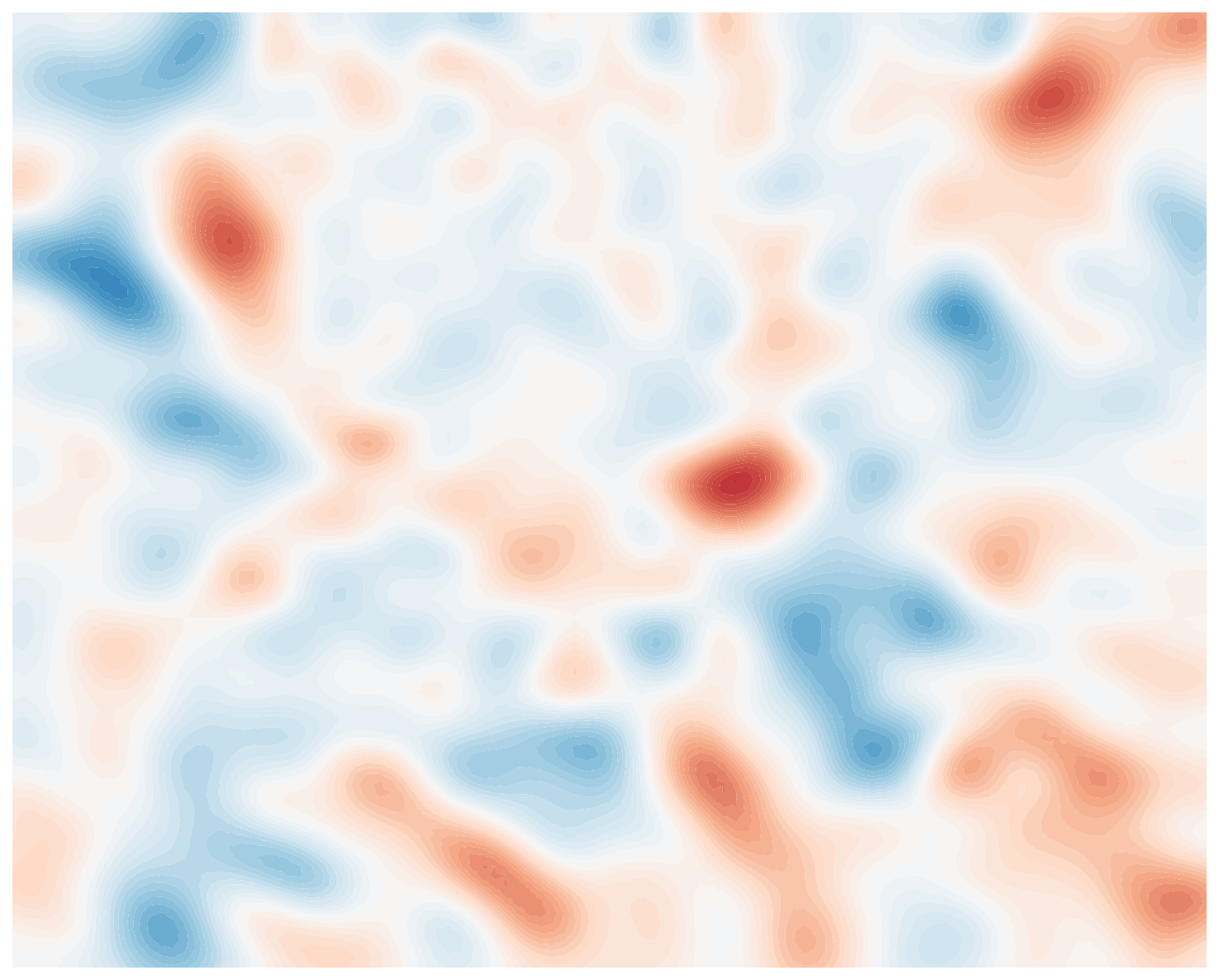}} & \raisebox{-0.11\linewidth}{\includegraphics[width=0.9\linewidth]{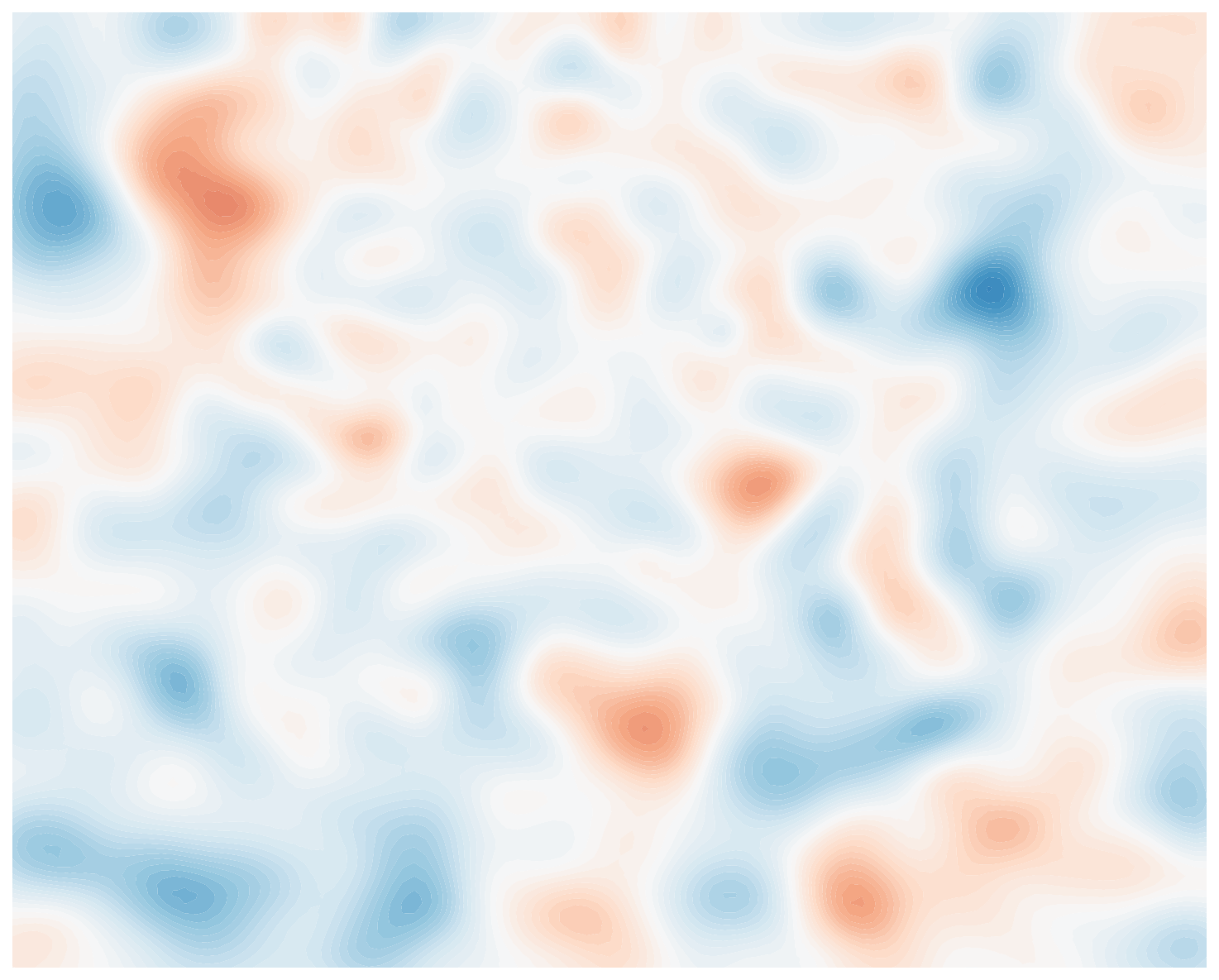}} & \raisebox{-0.11\linewidth}{\includegraphics[width=0.9\linewidth]{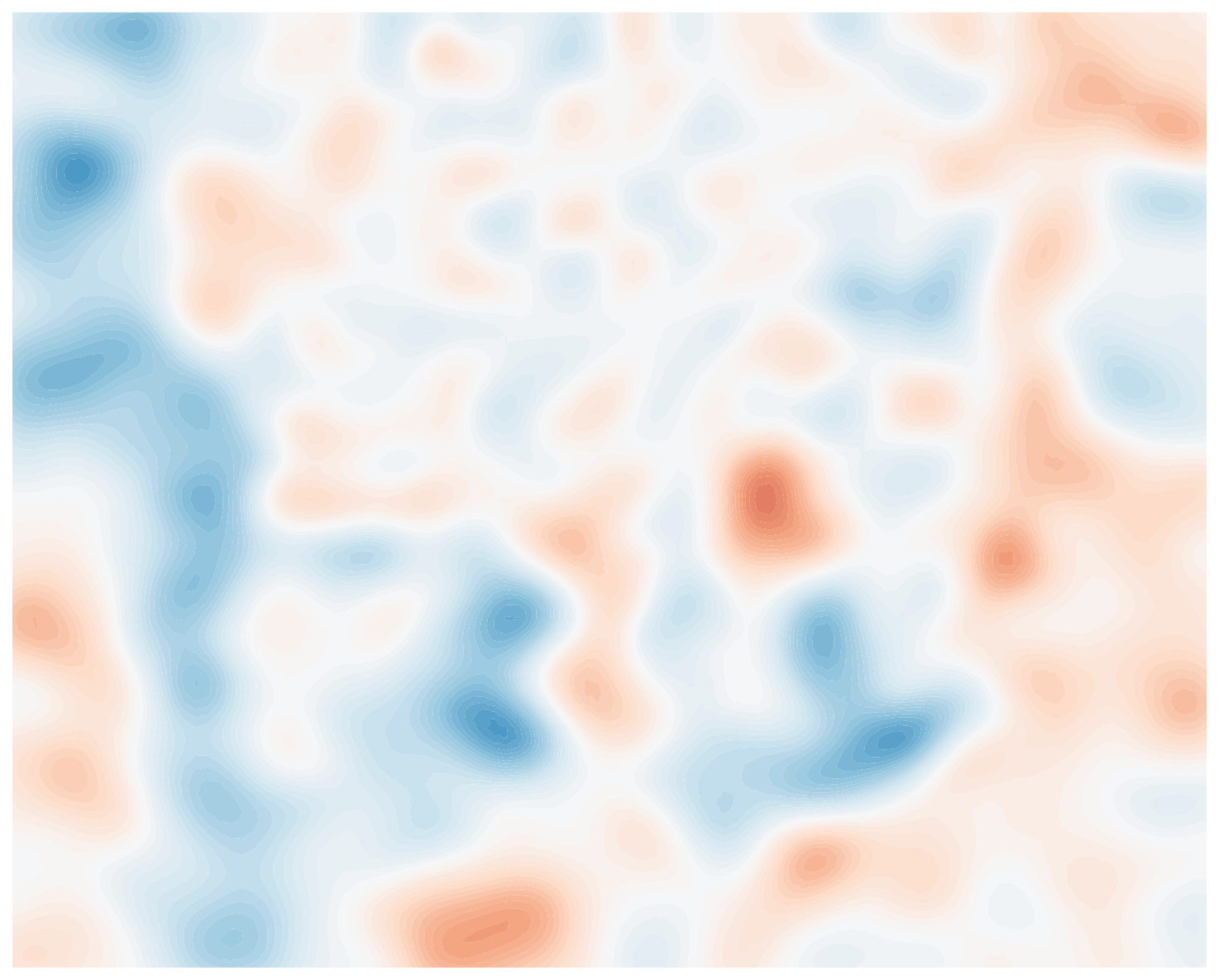}} & \raisebox{-0.11\linewidth}{\includegraphics[width=0.9\linewidth]{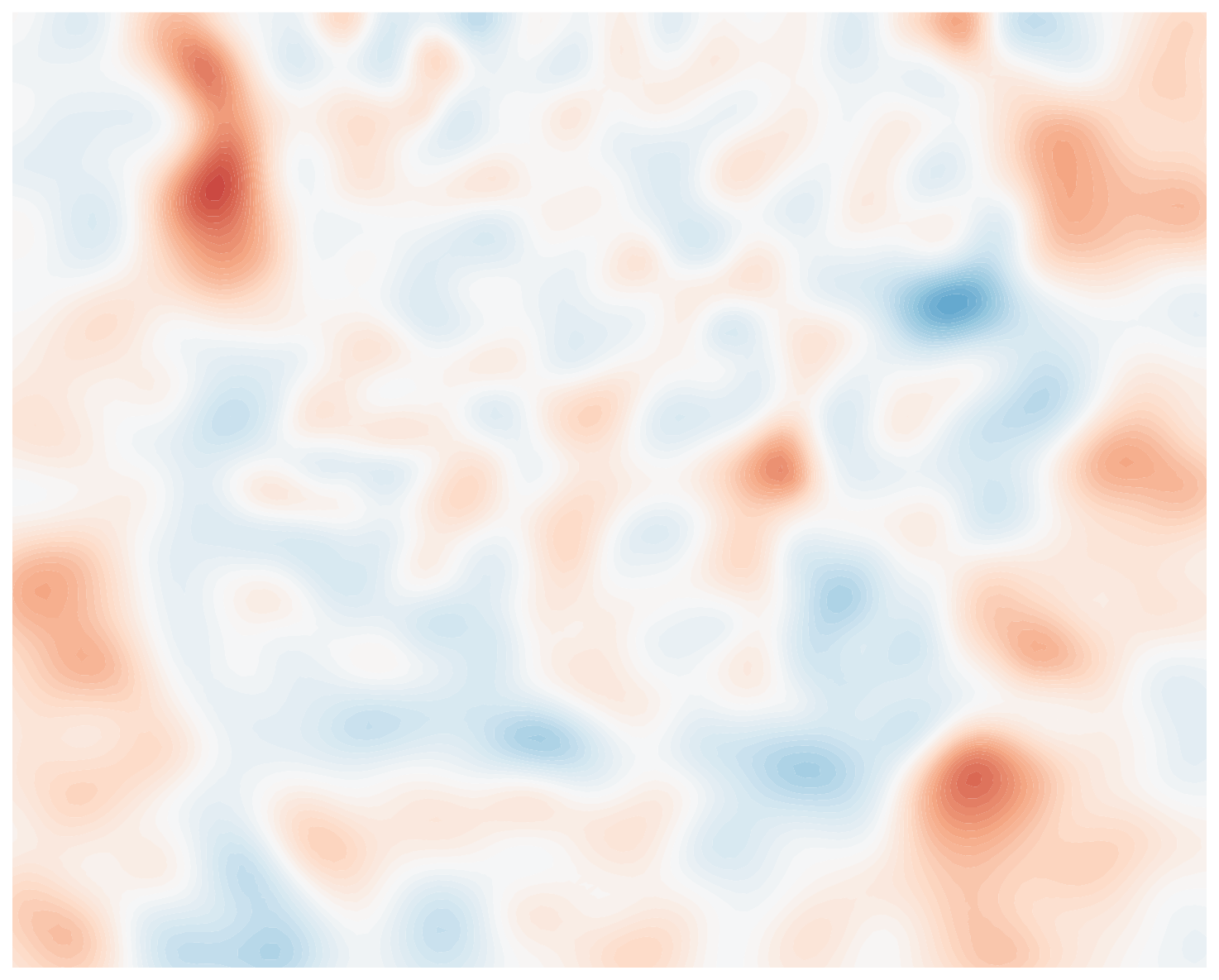}} & \\\cline{1-5}
    \makecell{Relative\\Error}& \quad 127.38 \% & \quad 107.46 \% & \quad 111.63 \% & \quad 96.60 \% & \\\hline

\end{tabular}
    \renewcommand{\arraystretch}{1}
}
    \caption{Spatial pointwise errors in the inverse solutions by FNOs and DIFNOs for the Helmholtz equation when compared to the reference solution obtained by the PDE model. The relative spatial $L^2$ errors in the inverse solutions are shown below each plot.}
    \label{fig:helmholtz_map_comparison}
\end{figure}

\subsection{Navier--Stokes}\label{subsec:navier--stokes}

\paragraph{Problem Setup}For our final numerical experiment, we consider the 2D stream function formulation of the Navier--Stokes equations defined on a periodic spatial domain $(0,1)^2$ identified with the unit torus $\bT^2$ for time $t\geq 0$:
\begin{equation}\label{eq:ns}
    \begin{aligned} 
        \partial_t \omega(x,t) &= - \cJ(\psi(\cdot, t),\omega(\cdot, t), x) +  \nu\Delta \omega(x,t) + f(x), && x\in \bT^2, t\geq 0,\\
        \omega(x,t) &= -\Delta\psi(x,t), && x\in \bT^2, t\geq 0,\\
        \omega(x,0) &= \omega_0(x), && x\in \bT^2\\
        \int_{\bT^2} \psi(x,t)\, \textrm{d}x &= 0, && t\geq 0\\
        \cJ(\psi(\cdot, t),\omega(\cdot, t), x) &\coloneqq \partial_{x_2} \psi(x,t)\partial_{x_1} \omega(x,t) - \partial_{x_1} \psi(x,t)\partial_{x_2} \omega(x,t), && x=(x_1, x_2)\\
    \end{aligned}
\end{equation}
where $\omega$ is the vorticity, and $f=\sqrt{2}\sin(2\pi(x_1+x_2) + \pi/4)/20$ is a forcing term, $\psi$ is the stream function, and $\nu=10^{-4}$ is the kinematic viscosity. The Reynolds number is estimated as 1,134 \cite{manfroi2002stability}. We consider learning the flow map $\cS_\tau$ of a normalized time interval of size $\tau=2$---the mapping from the vorticity field at time $t_0$ to $t_0+2$, i.e., the input is $a = \omega(\cdot, t_0)$ and the output is $u=\omega(\cdot, t_0+2) = \cS_2(a)$. The variation on the input $a$ originates from random vorticity initial conditions that follow a centered Gaussian measure with covariance $\cC_\cX = (10/3\cI - \Delta/30)^{-2}$. We set $t_0=23$ for the input vorticity field; thus, the data-generating distribution is the pushforward of the Gaussian measure by $\cS_{23}$. This operator learning problem is challenging as the mapping's sensitivities are difficult to learn without the derivative loss. The choice of $t_0=23$ poses additional challenges, as we find that the evolution of the vorticity field is more predictable at earlier times.

We solve the 2D Navier--Stokes equations using a pseudo-spectral method combined with a semi-implicit time-stepping scheme. Specifically, the nonlinear terms are treated with a two-step Adams--Bashforth method, while the linear terms employ Crank--Nicolson method. Both the input and output spaces are discretized on a $64 \times 64$ grid, and we define the output space $\mathcal{Y}$ using the $L^2(\mathbb{T}(0,1)^2)$ norm. 
The initial stream function $\psi_0$ is derived by solving the Poisson equation $\omega_0 = -\Delta \psi_0$. For reduced-basis DIFNOs, we use $r_{\mathcal{X}_{\delta}} = 256$ and $r_{\mathcal{Y}} = 256$. For mixed-resolution DIFNOs, we consider 4- and 2-times grid size reductions for the low-fidelity and intermediate resolutions, respectively. A representative sample of the input vorticity field and its corresponding output is visualized in Figure~\ref{fig:ns_param_state}.

\begin{figure}[htb]
    \centering
    {
    \small
    \begin{tabular}{c c}
        \hspace{-0.04\linewidth}\makecell{Vorticity Field \\
        at $t=23$ (Input)}&\hspace{-0.06\linewidth} \makecell{Vorticity Field\\ at $t=25$ (Output)\\ with Observations}\\
        \includegraphics[width=0.3\linewidth]{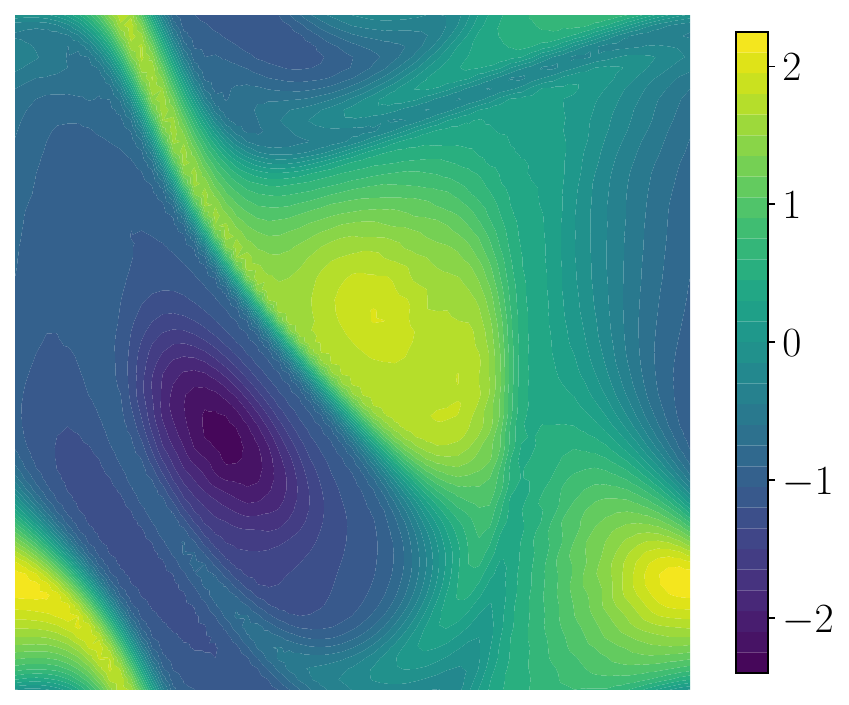}  & \includegraphics[width=0.3\linewidth]{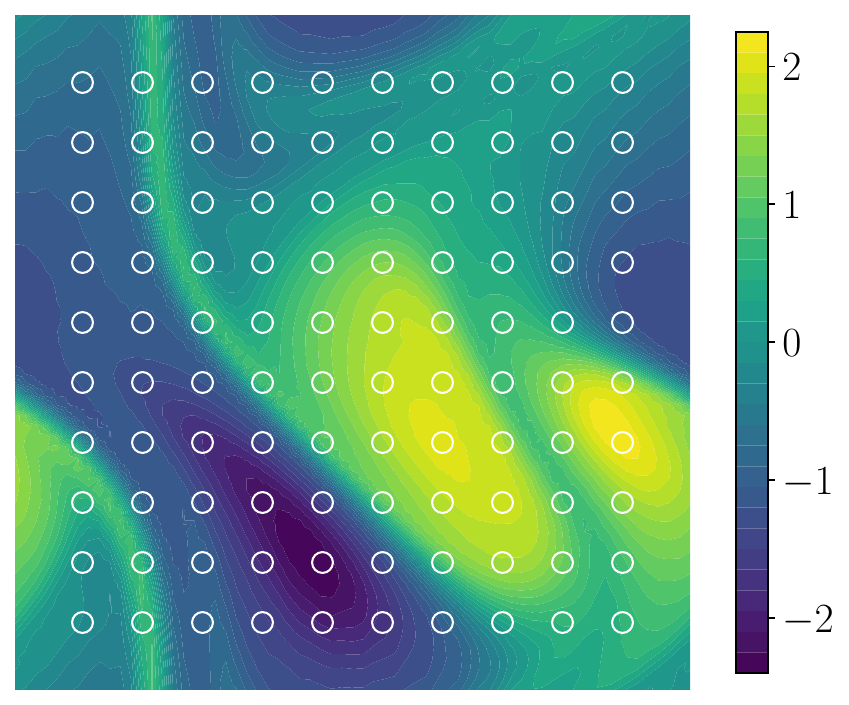}
    \end{tabular}
    }

    \caption{The step up of the numerical study on the Navier--Stokes equation. (\textit{Left}): An input sample $a^{(i)} \sim \mu$ of the vorticity field at $t=23$. (\textit{Right}) The output vorticity field at $t=25$ obtained by solving the equation starting from the input sample. The circles indicate the locations of the pointwise observations of the inverse problem output.}
    \label{fig:ns_param_state}
\end{figure}

\paragraph{Training Results}We train FNOs, reduced-basis DIFNOs, and mixed-resolution DIFNOs at varying training sample sizes, and their testing results are shown in \Cref{fig:ns_train}. The results show that the DIFNOs substantially outperform FNOs for both error measures. The advantages of DIFNOs are more significant when measured by the prediction error in the Fr\'echet derivative of this mapping. Furthermore, the empirical convergence rates of the Fr\'echet derivative as training sample size increases are substantially higher for the DIFNOs. Across the two DIFNO training techniques, we again observe that reduced-basis derivative learning is more accurate than mixed-resolution derivative learning, although the latter is less computationally expensive.

\paragraph{Inverse Problem Results} We study the solutions to the inverse problem obtained by the PDE, the trained FNOs, and the trained DIFNOs.
The PDE inverse solution here is obtained by L-BFGS optimizer implemented in PyTorch.
The data-generating (true) vorticity field input is chosen to be nine Gaussian vortices (left of \Cref{fig:ns_map}). The data is generated by solving the equation starting from this input, extracting pointwise observations as in \Cref{fig:ns_param_state}, and applying 1\% additive white noise. The inverse solutions and their spatial pointwise errors are shown in \Cref{fig:ns_map,fig:ns_map_comparison}. The solution obtained by DIFNOs matches the reference solution from the PDE more closely compares to the FNOs. Furthermore, the FNOs struggle to improve the inverse solution as the number of training samples increases, which can be attributed to the slow convergence rate of the Fr\'echet derivative approximation and support or analysis in Propositions \ref{prop:surrogate_optimization_error}, \ref{prop:surrogate_optimization_error_strongly_convex} and \Cref{theorem:inverse_problem_ua}.

\begin{figure}[htb]
    \centering
    \renewcommand{\arraystretch}{1.2} 
    \setlength{\tabcolsep}{5pt}       
\small
    \begin{tabular}{c c c}
        & \hspace{0.03\linewidth}\makecell{Navier-Stokes \\
        PDE Solution Prediction Error} & \hspace{0.03\linewidth}\makecell{Navier-Stokes \\
        Fr\'echet Derivative Prediction Error} \\
        
        \raisebox{3em}{\rotatebox{90}{Relative Error}} & 
        \includegraphics[width=0.42\linewidth]{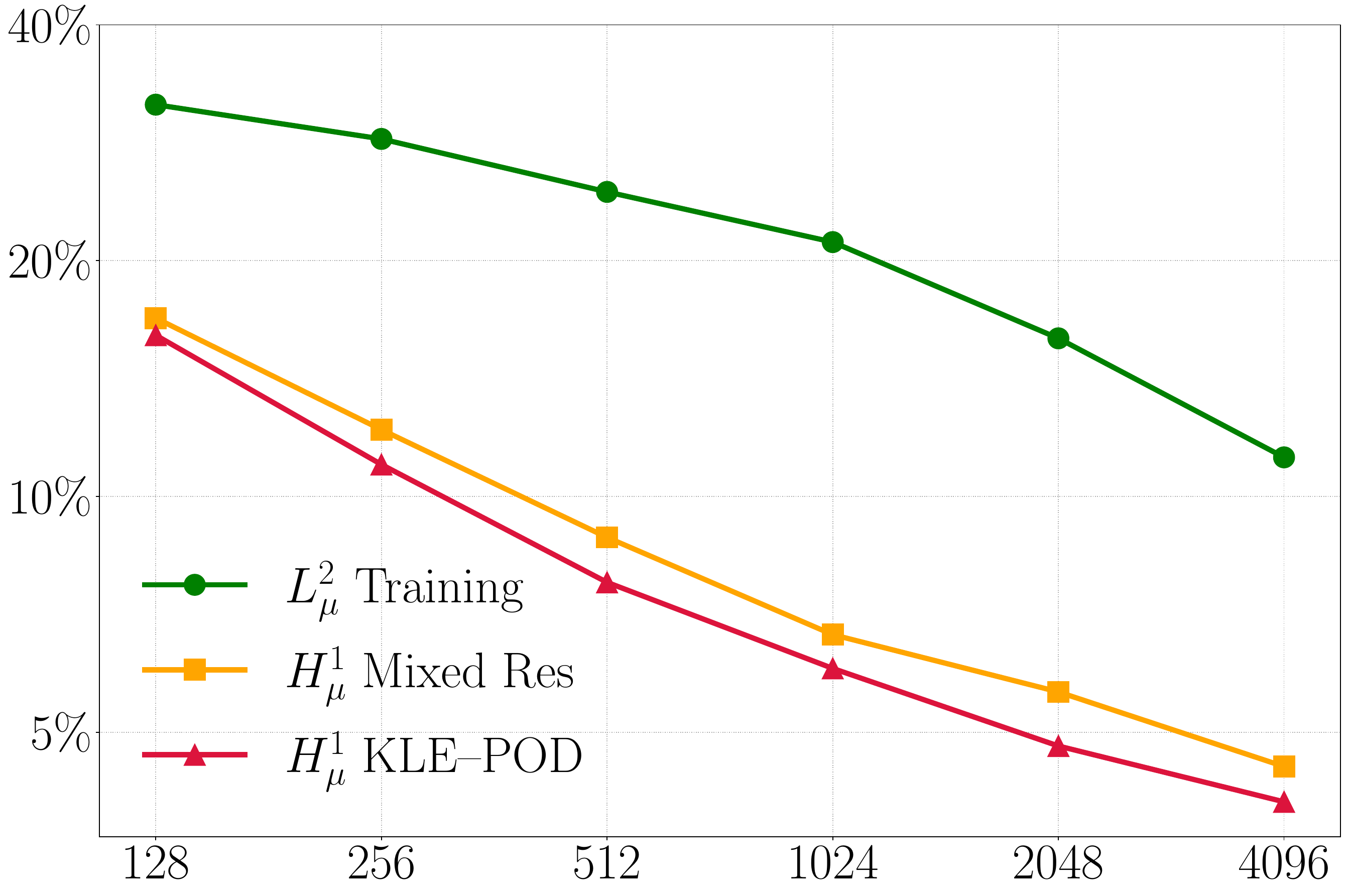} & 
        \includegraphics[width=0.42\linewidth]{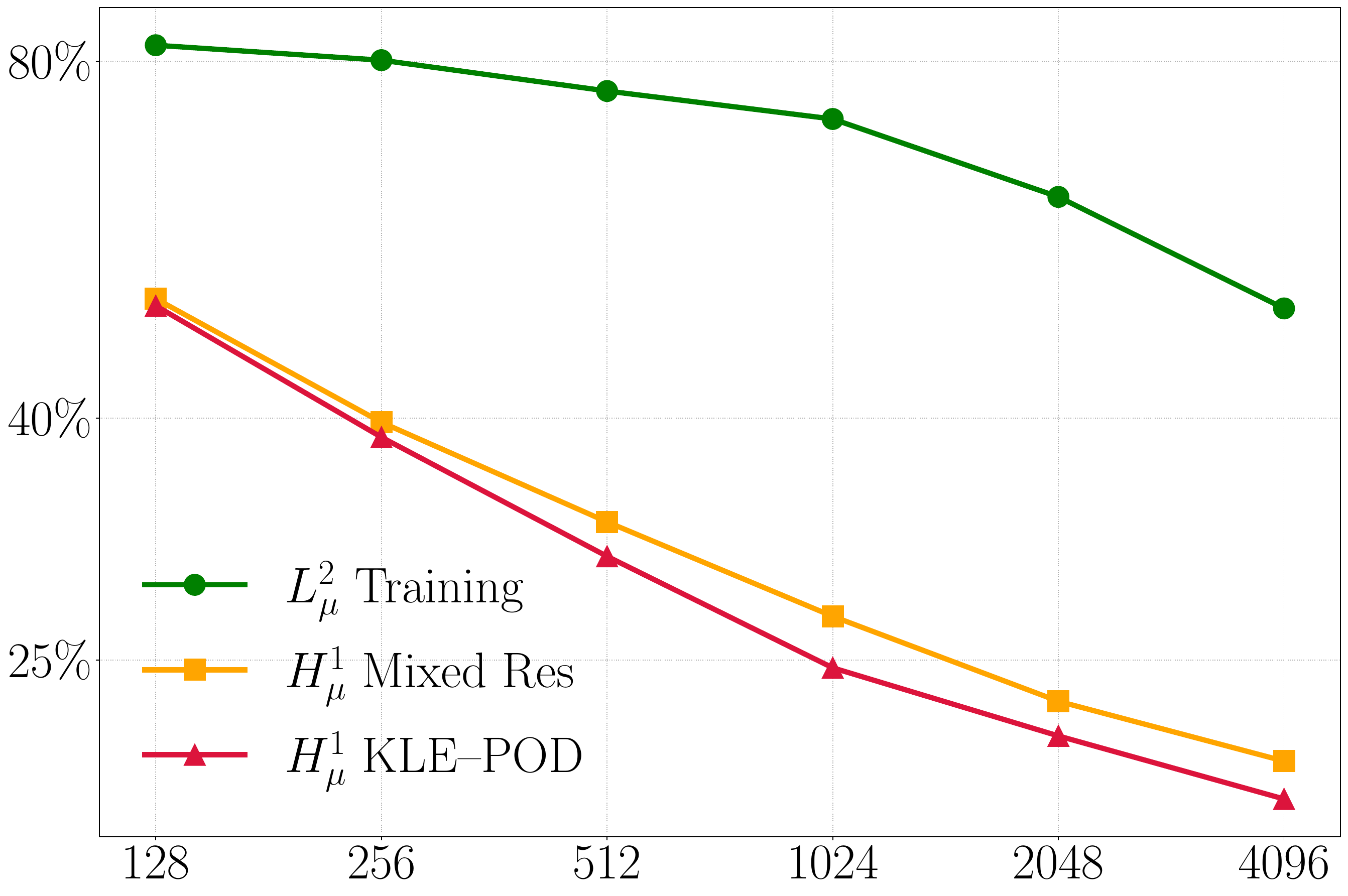} \\
        
        & \hspace{0.03\linewidth}Number of Training Samples & \hspace{0.03\linewidth}Number of Training Samples \\
    \end{tabular}

    \caption{The relative error in the PDE solutions and Fr\'echet derivative predictions for the Navier--Stokes equation. We train FNOs ($L^2_{\mu}$ Training), reduced-basis DIFNOs with KLE--POD dimension reduction ($H^1_{\mu}$ KLE--POD), and mixed-resolution DIFNOs ($H^1_{\mu}$ Mixed Res.) on a varying number of training samples. We use an additional 128 testing samples to evaluate the relative error.}
    \label{fig:ns_train}
\end{figure}

\begin{figure}[htb]
    \centering
    \small
    \begin{tabular}{c c c c c}
        \makecell{Data-Generating\\ Input Field (True)} \quad&
          \makecell{Inverse Solution \\ by PDE (Reference)} & \makecell{Inverse Solution \\ by FNO}&  \makecell{Inverse Solution \\ by DIFNO} &\\
         \includegraphics[width=0.21\linewidth]{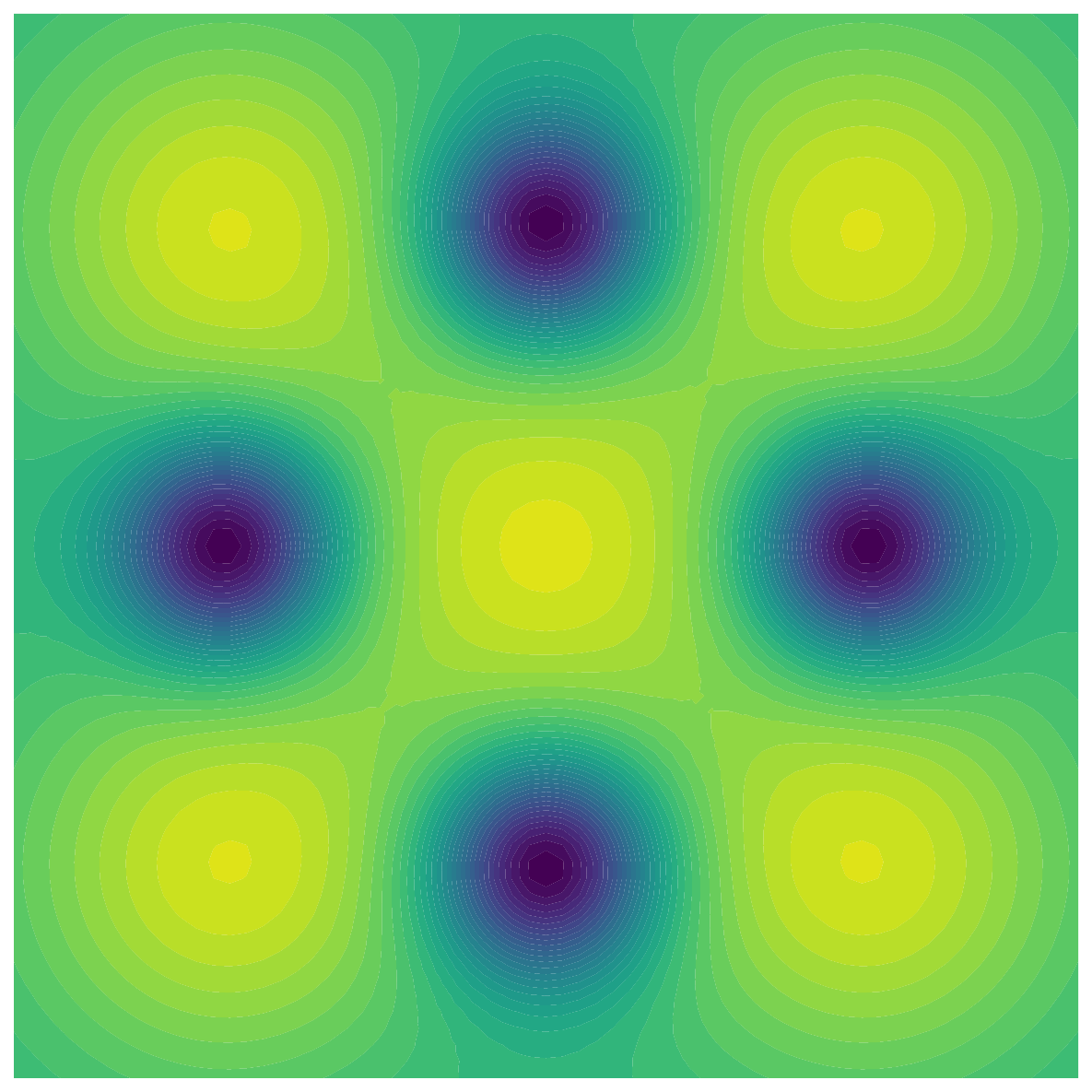} & 
         \includegraphics[width=0.21\linewidth]{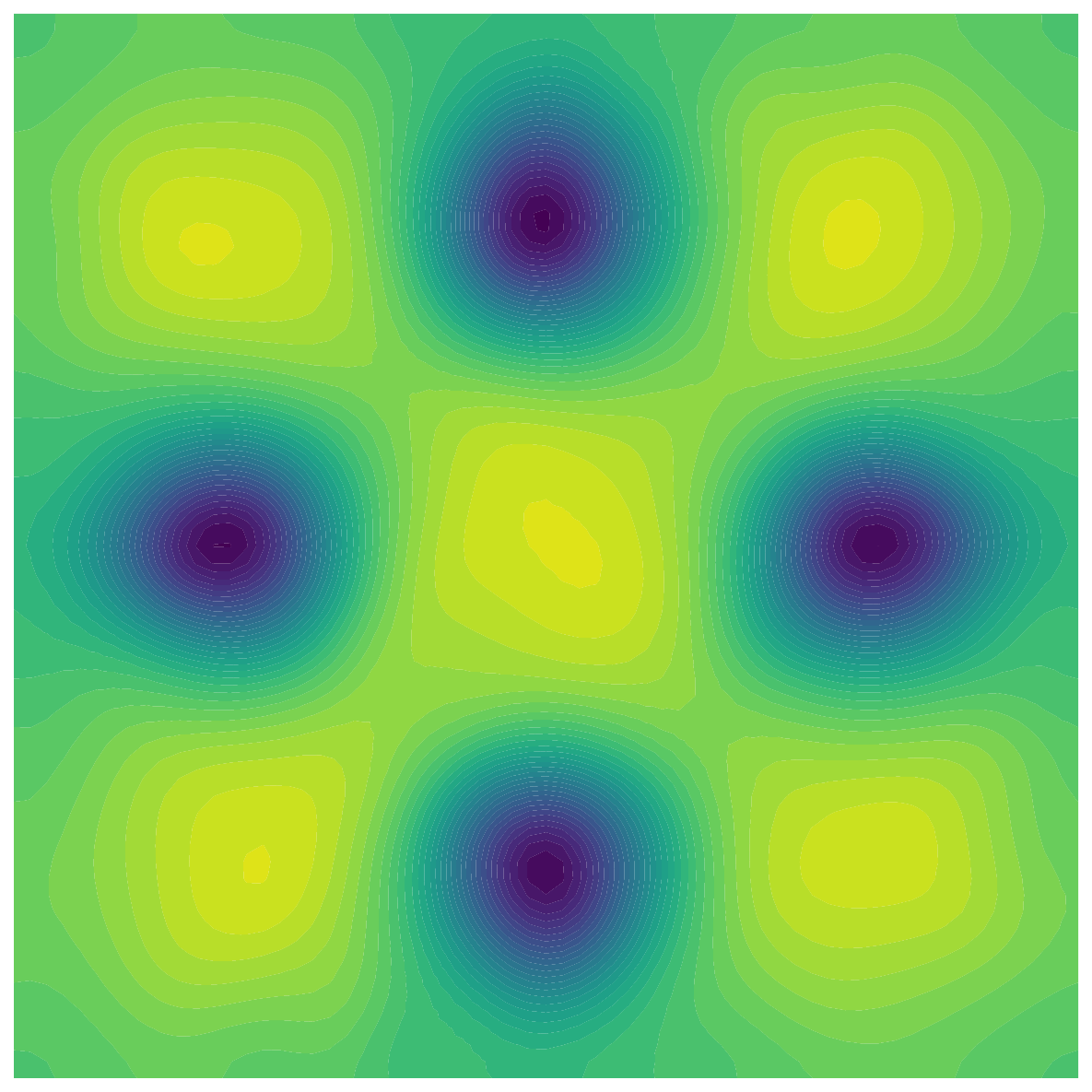} &
         \includegraphics[width=0.21\linewidth]{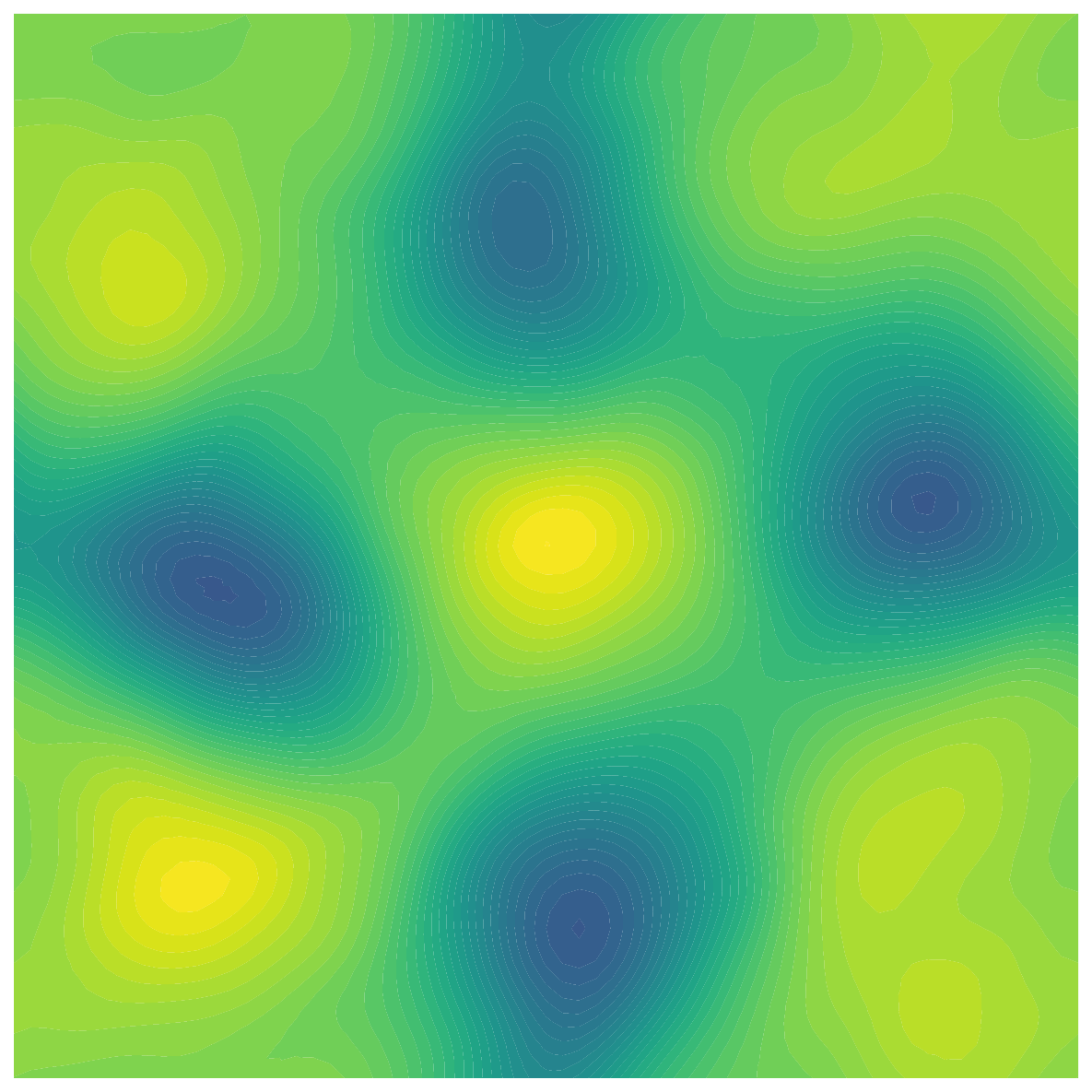} &
         \includegraphics[width=0.21\linewidth]{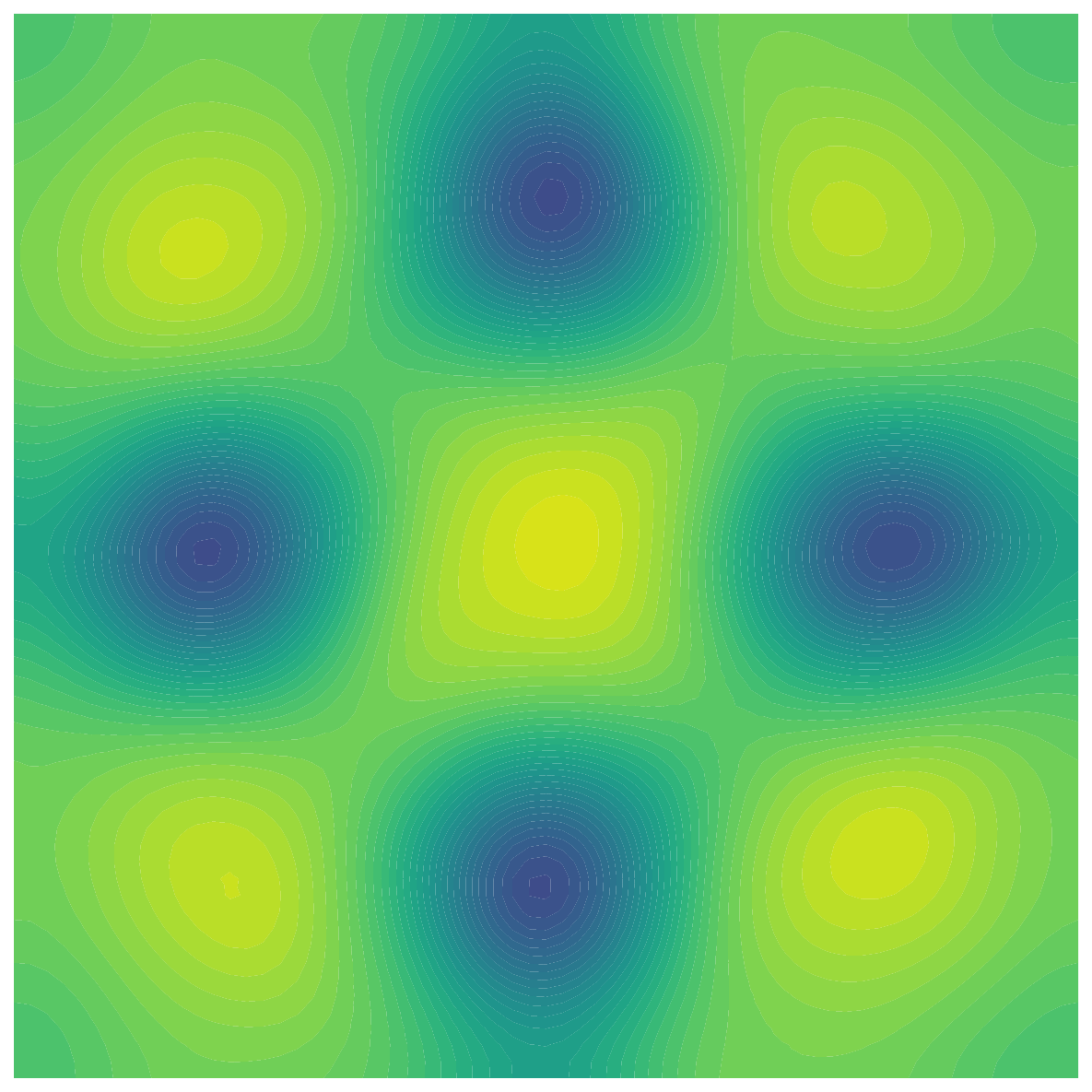} & 
         \includegraphics[width=0.04\linewidth]{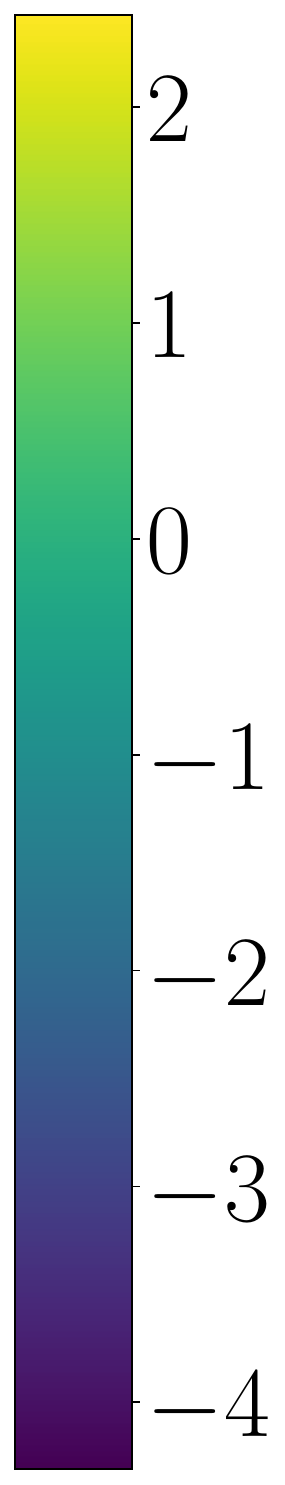}
    \end{tabular}
    \caption{Solutions of inverse problems based on the Navier--Stokes equation in \Cref{fig:ns_param_state}. The left figure is the data-generating (true) vorticity field input. Then, from left to right, are inverse solutions obtained by PDE, the best-performing FNO, and the best-performing DIFNO.
    }
    \label{fig:ns_map}
\end{figure}

\begin{figure}[htb]
\centering
\scalebox{0.9}{     \renewcommand{\arraystretch}{1.5}
\begin{tabular}{
|>{\centering\arraybackslash} m{0.115\linewidth}||
>{\centering\arraybackslash} m{0.13\linewidth}|
>{\centering\arraybackslash} m{0.13\linewidth}|
>{\centering\arraybackslash} m{0.13\linewidth}|
>{\centering\arraybackslash} m{0.13\linewidth}|
>{\centering\arraybackslash} m{0.1\linewidth}|}\hline
    \multicolumn{6}{|c|}{\makecell{Errors in the Inverse Solutions for Navier--Stokes Equation}} \\\hline\hline
    &\multicolumn{4}{c|}{Number of Training Samples} & \multirow{2}{*}{\makecell{Pointwise\\Error}}\\\cline{1-5}
    & 512 & 1024 & 2048 & 4096 & \\\hline
    FNO& \raisebox{-0.13\linewidth}{\includegraphics[width=0.9\linewidth]{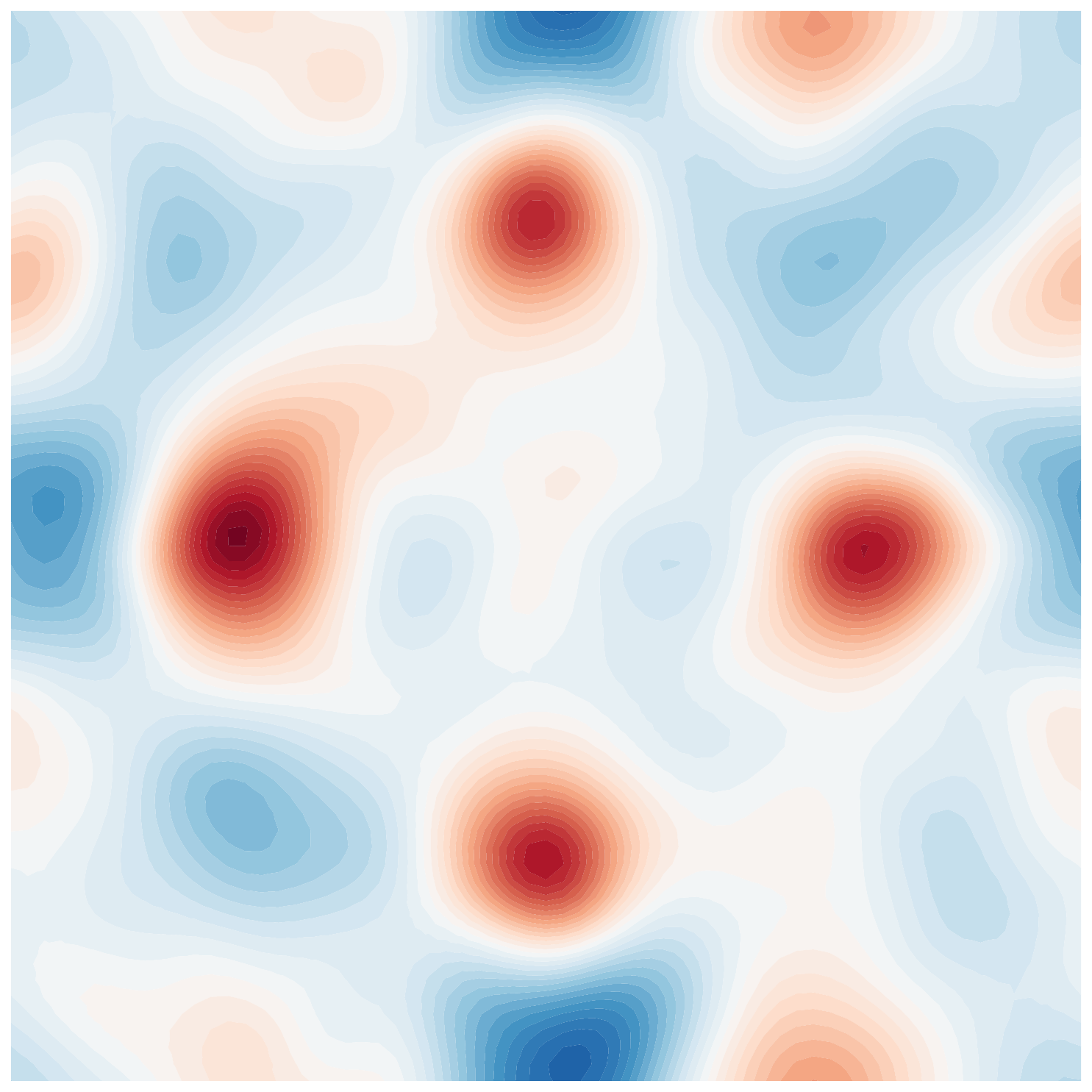}} & \raisebox{-0.13\linewidth}{\includegraphics[width=0.9\linewidth]{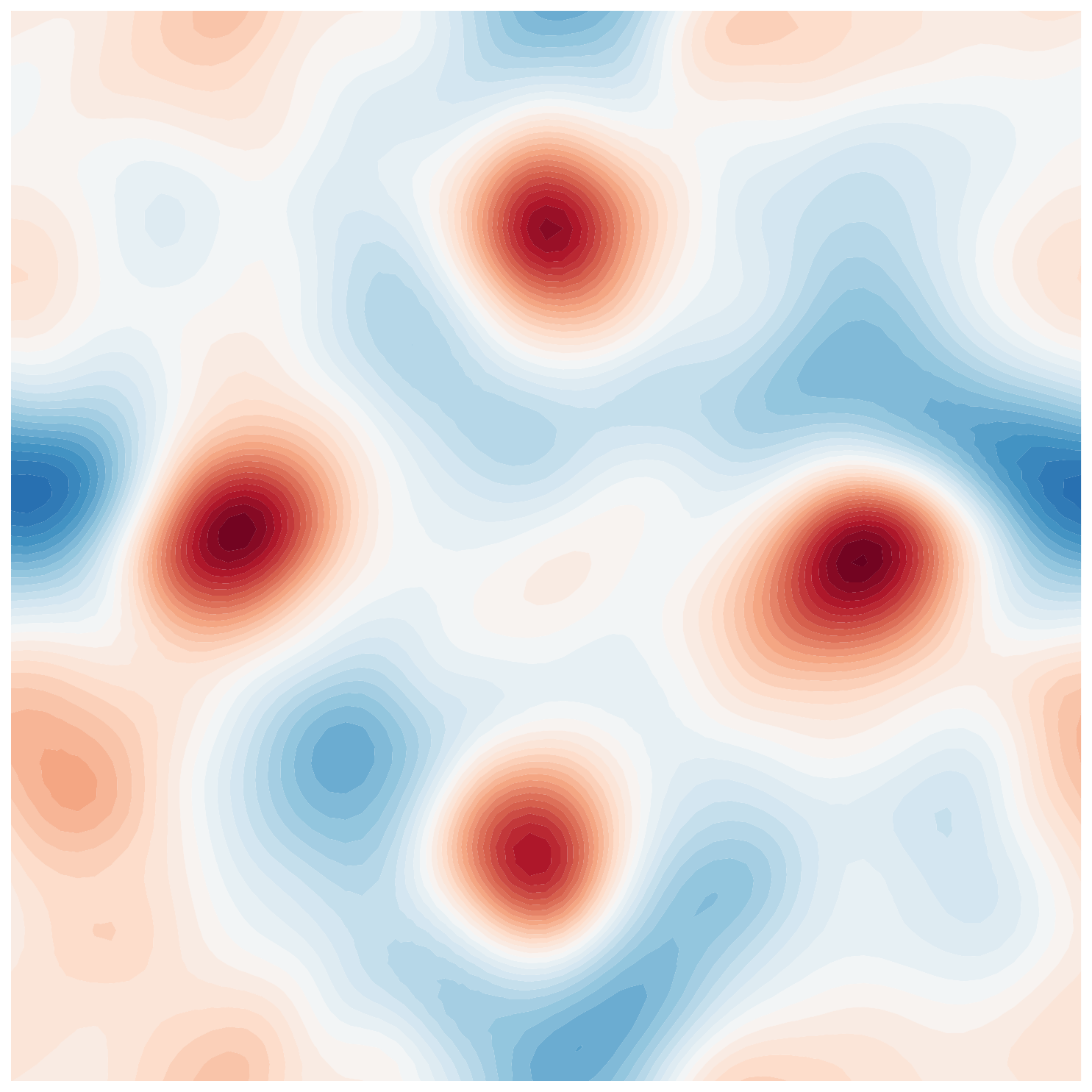}} & \raisebox{-0.13\linewidth}{\includegraphics[width=0.9\linewidth]{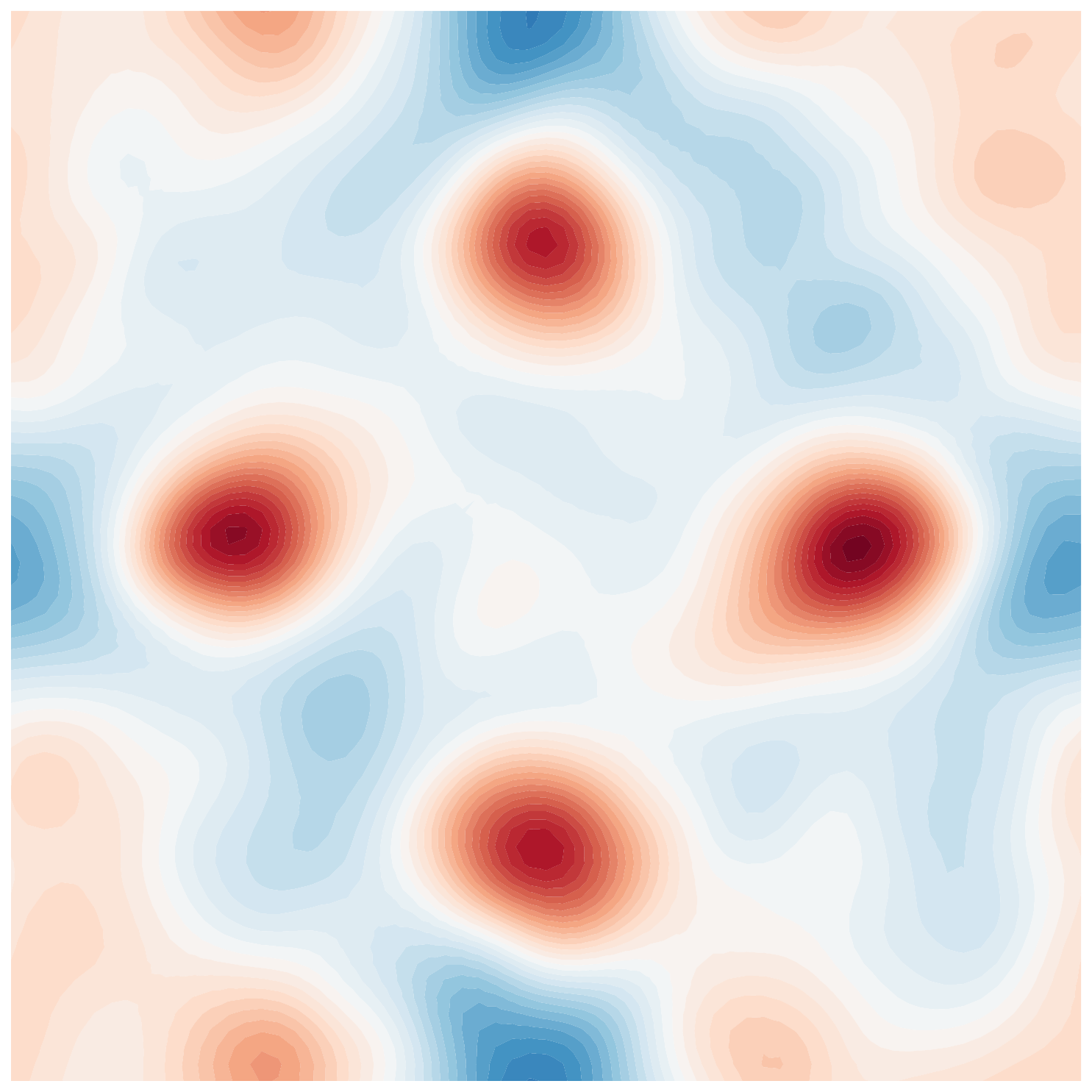}} & \raisebox{-0.13\linewidth}{\includegraphics[width=0.9\linewidth]{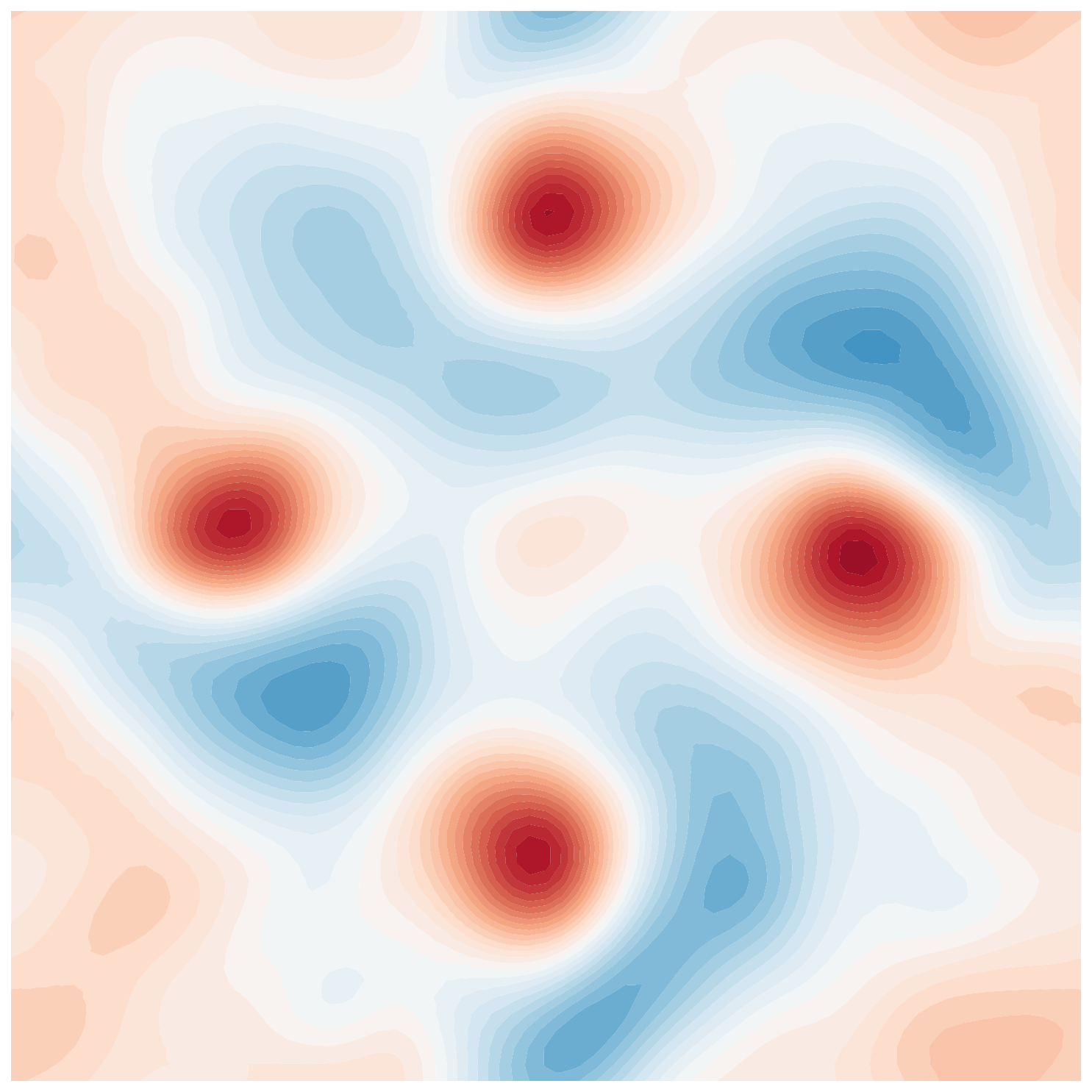}} & \multirow{5}{*}{\includegraphics[width=0.32\linewidth]{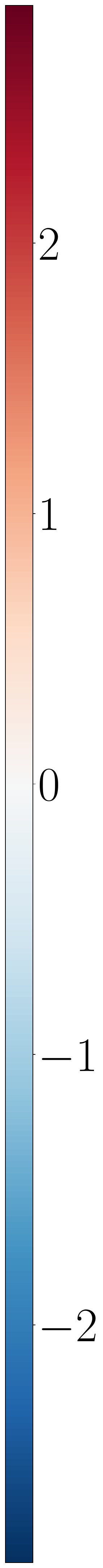}}\\\cline{1-5}
    \makecell{Relative\\Error}& \quad 52.17 \% & \quad 54.84 \% & \quad 49.65 \% & \quad 52.42 \% & \\\cline{1-5}
    \makecell{Red.-Basis\\DIFNO}& \raisebox{-0.13\linewidth}{\includegraphics[width=0.9\linewidth]{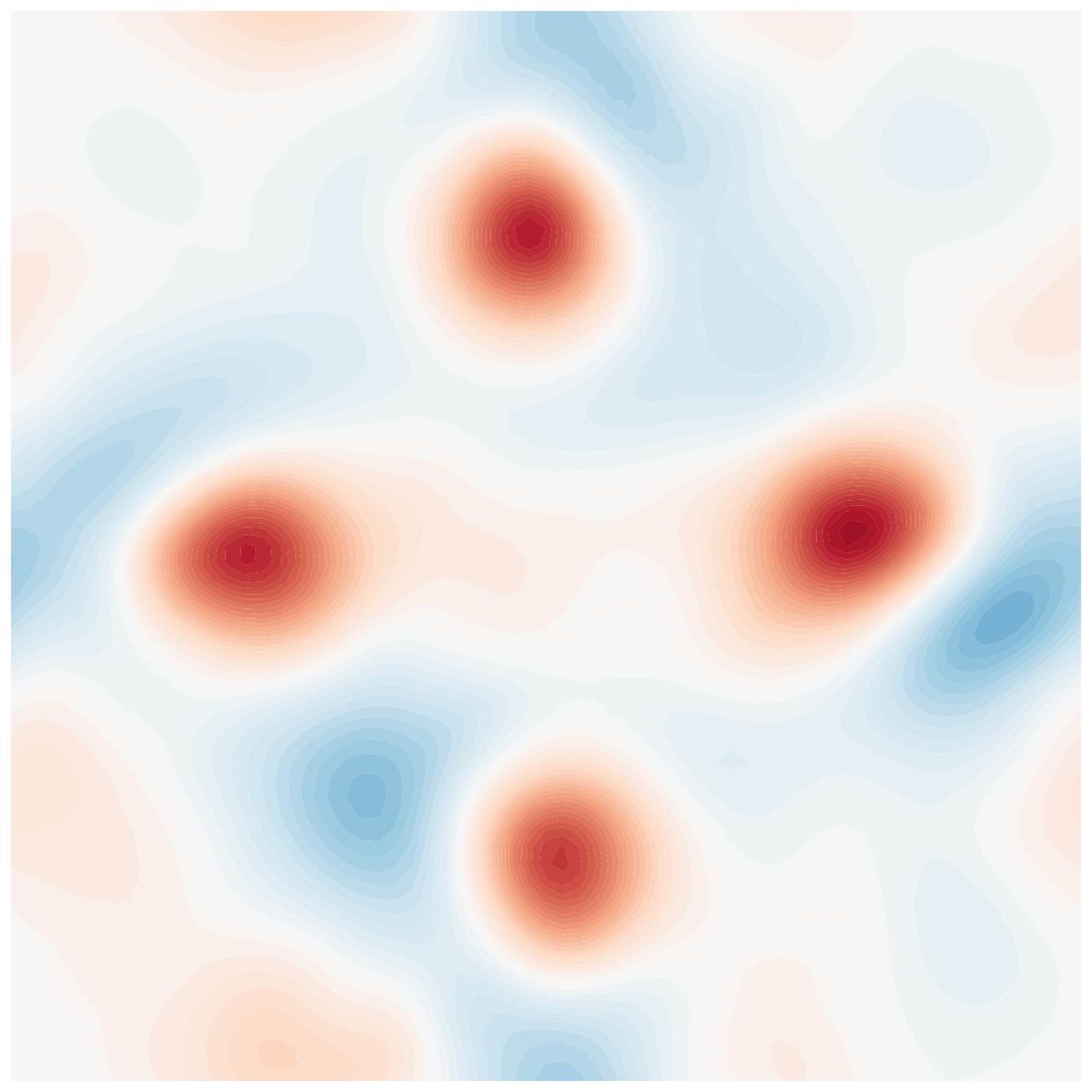}} & \raisebox{-0.13\linewidth}{\includegraphics[width=0.9\linewidth]{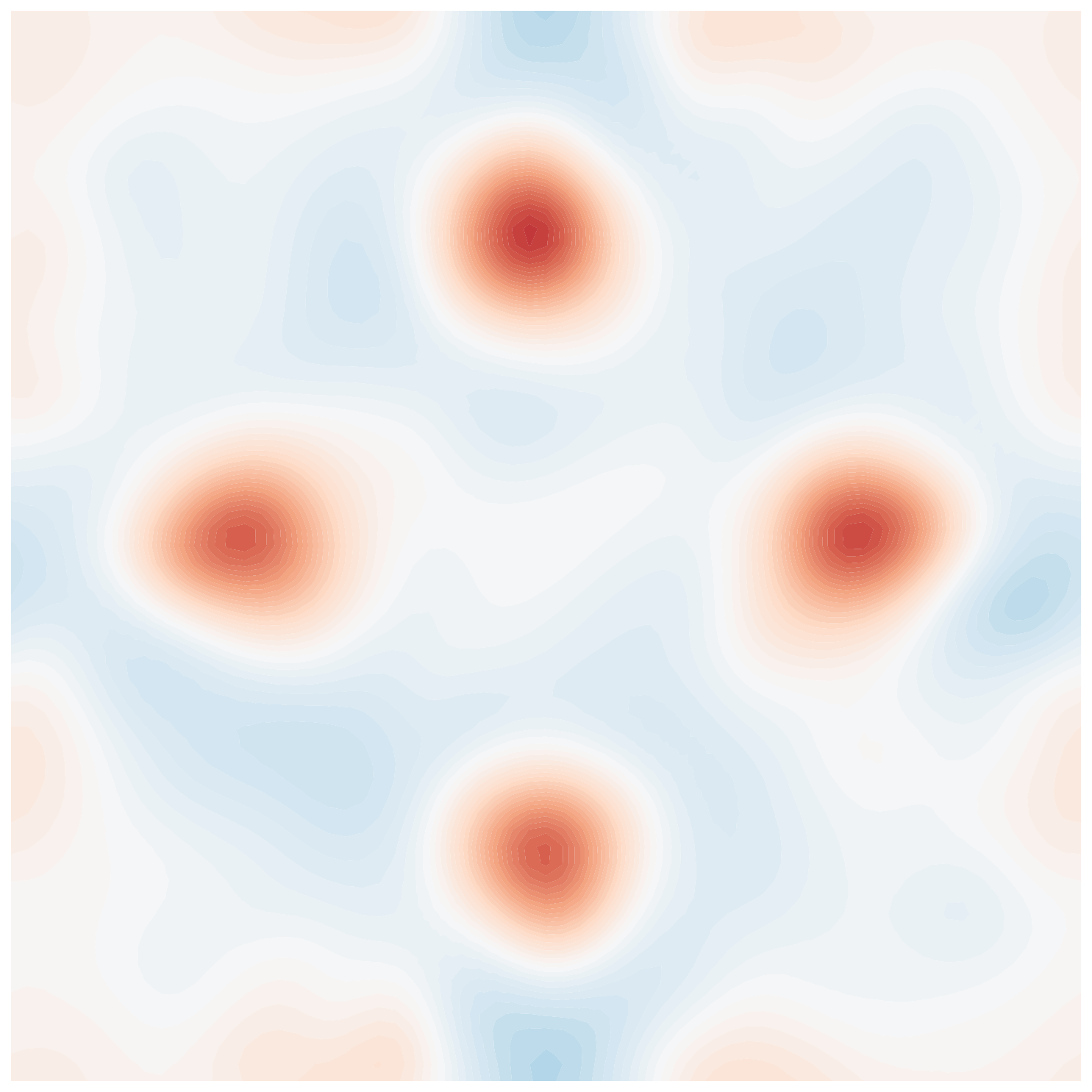}} & \raisebox{-0.13\linewidth}{\includegraphics[width=0.9\linewidth]{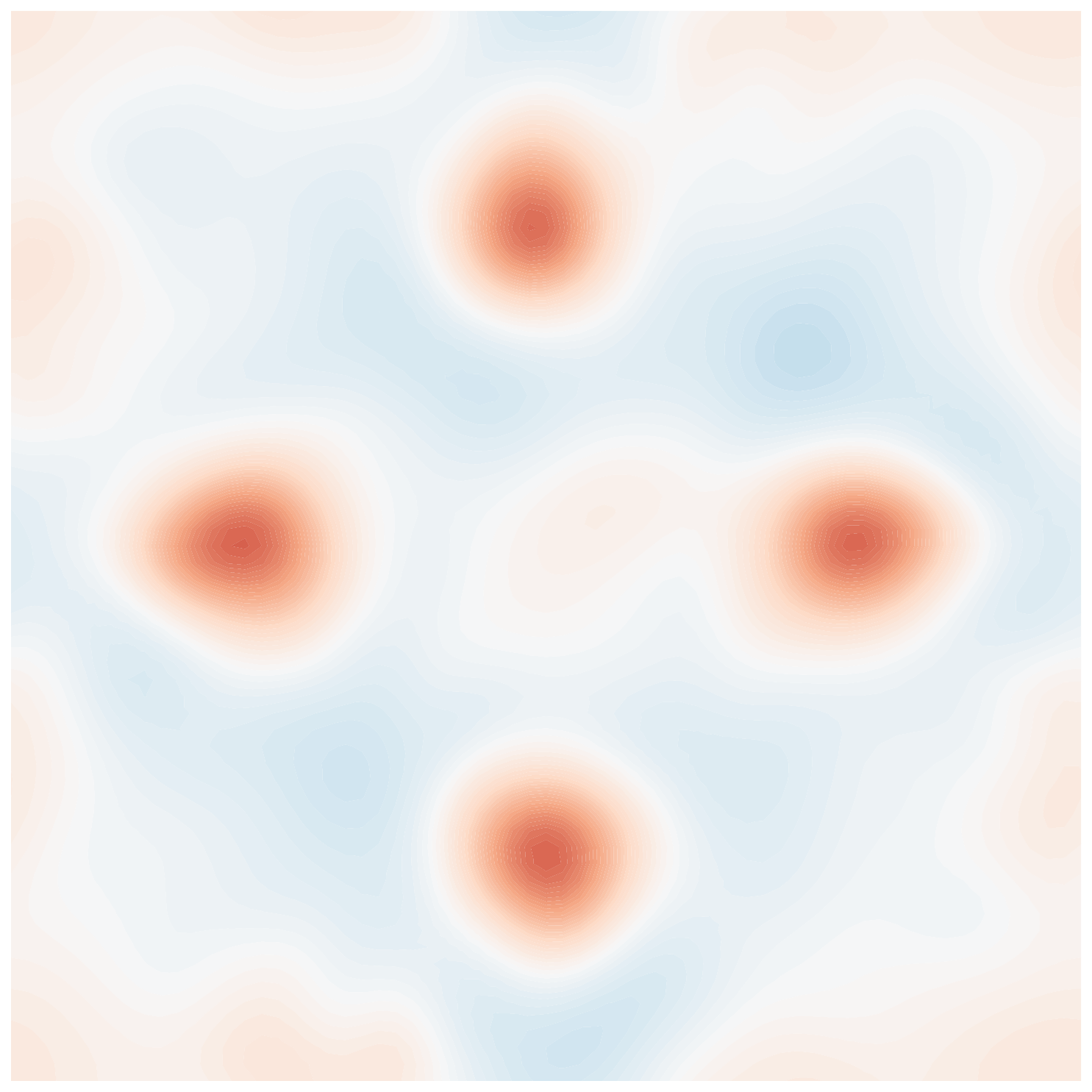}} & \raisebox{-0.13\linewidth}{\includegraphics[width=0.9\linewidth]{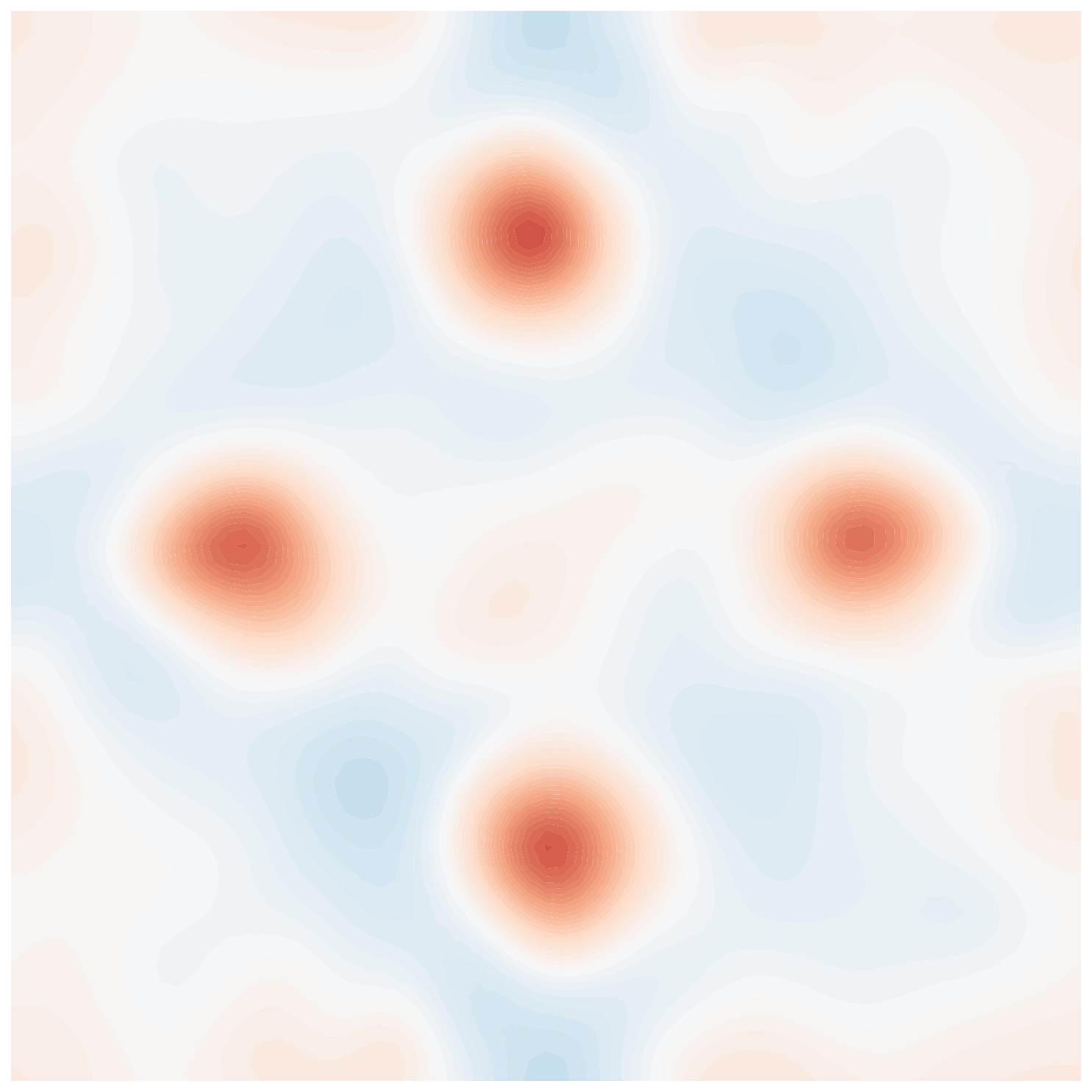}} & \\\cline{1-5}
    \makecell{Relative\\Error}& \quad 39.58 \% & \quad 29.71 \% & \quad \textbf{26.63 \%} & \quad 27.27 \% &\\\cline{1-5}
    \makecell{Mixed-Res.\ \\DIFNO}& \raisebox{-0.13\linewidth}{\includegraphics[width=0.9\linewidth]{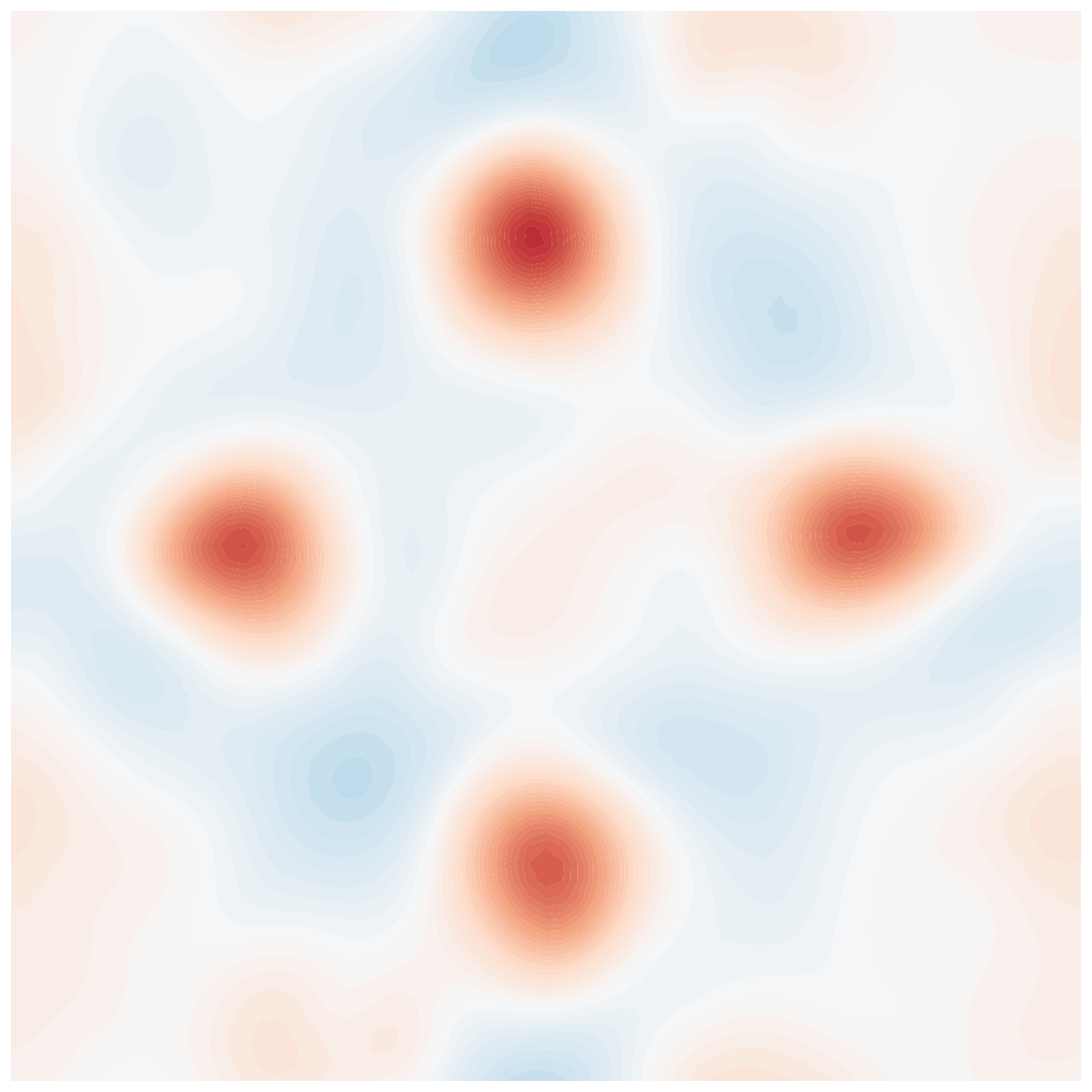}} & \raisebox{-0.13\linewidth}{\includegraphics[width=0.9\linewidth]{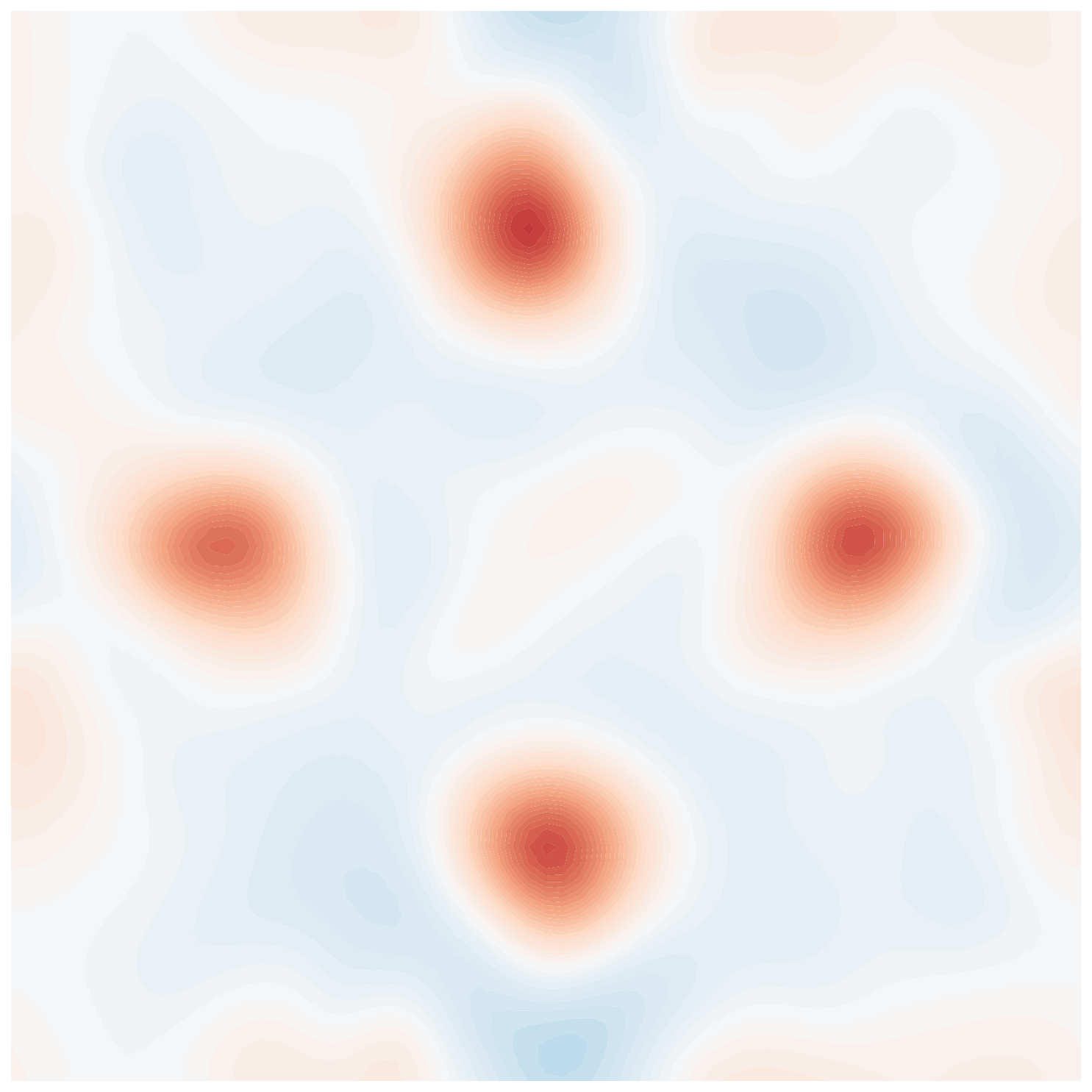}} & \raisebox{-0.13\linewidth}{\includegraphics[width=0.9\linewidth]{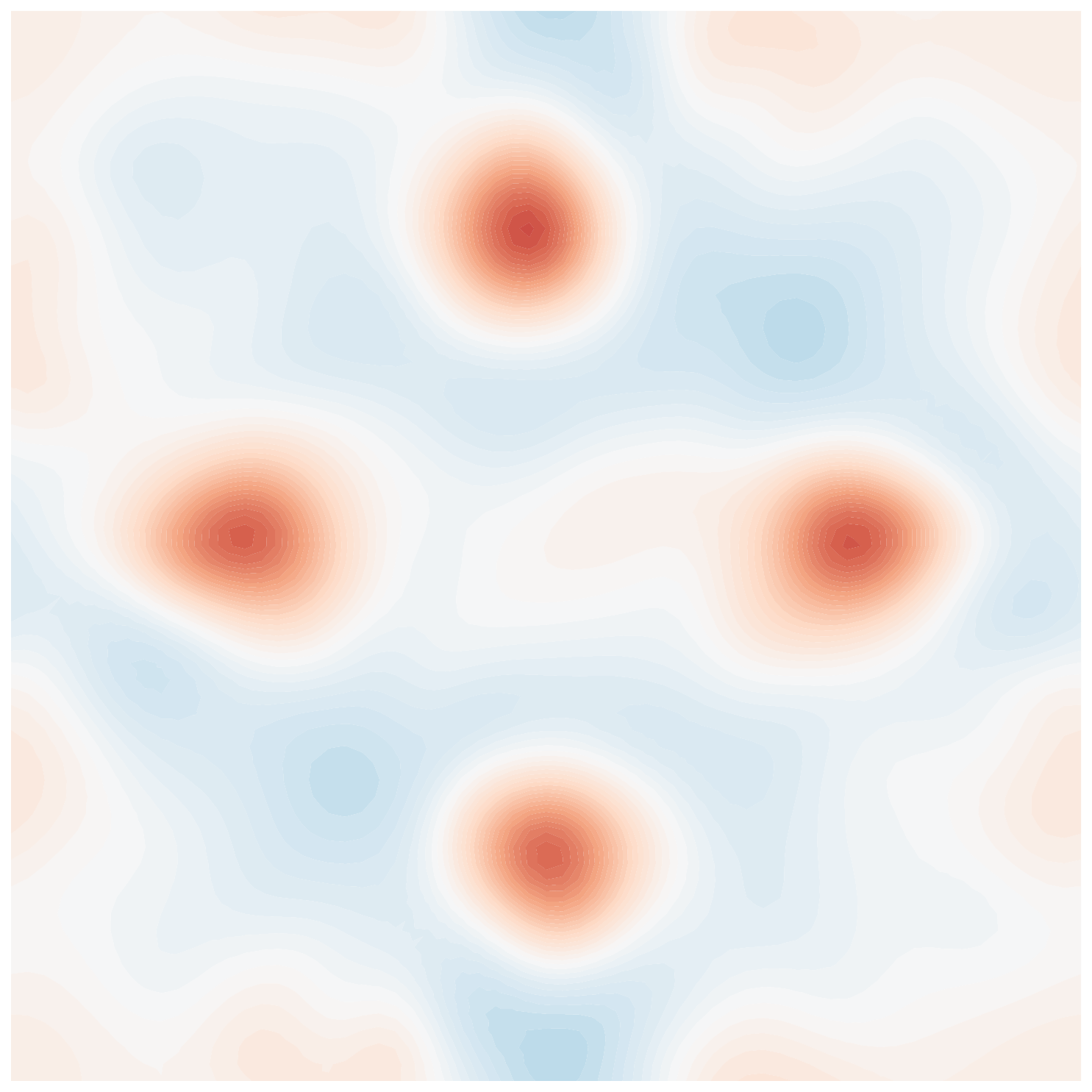}} & \raisebox{-0.13\linewidth}{\includegraphics[width=0.9\linewidth]{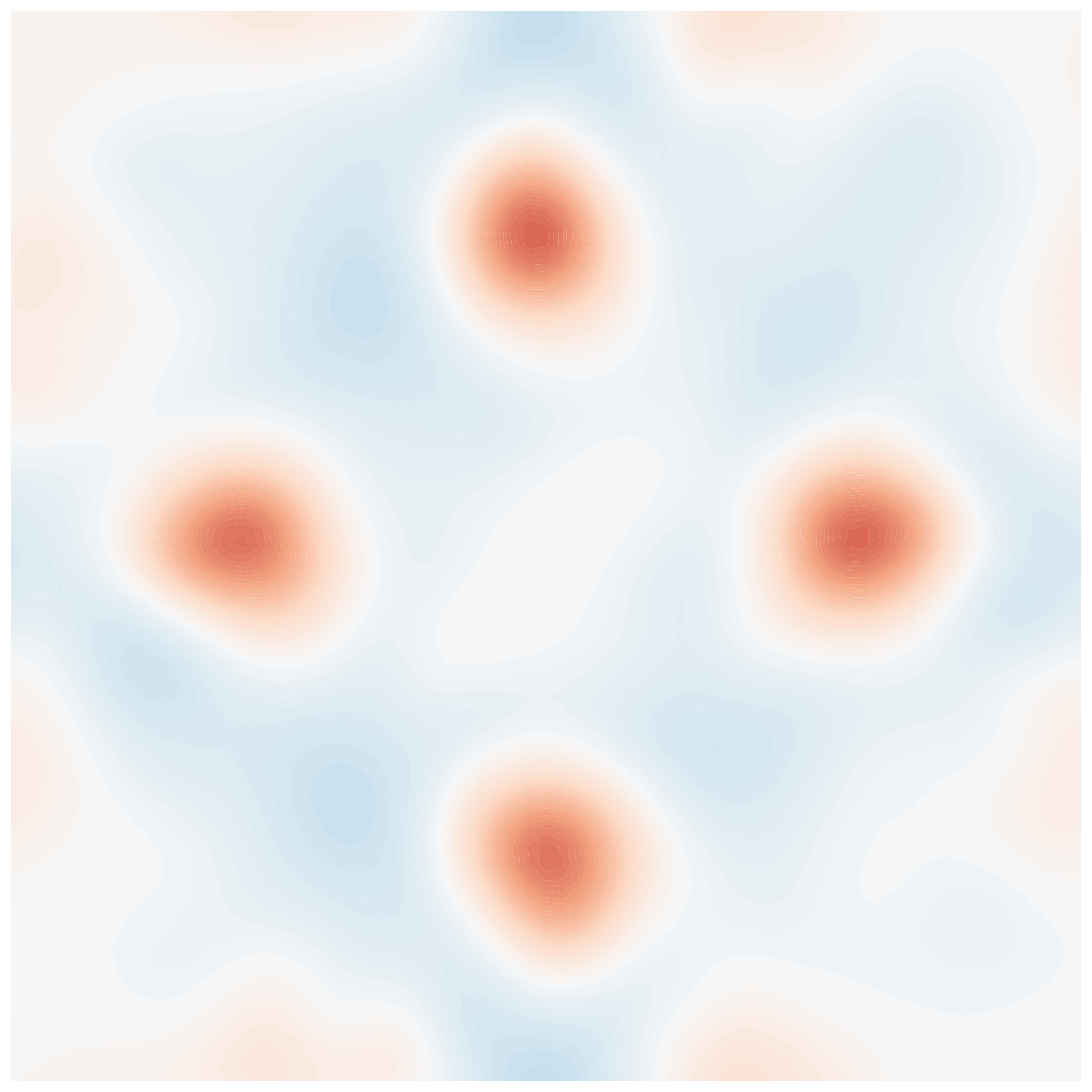}} & \\\cline{1-5}
    \makecell{Relative\\Error}& \quad \textbf{30.13 \%} & \quad \textbf{28.69 \%} & \quad 29.61 \% & \quad \textbf{26.77 \%} & \\\hline

\end{tabular}
    \renewcommand{\arraystretch}{1}
}
    \caption{Spatial pointwise errors in the inverse solutions by FNOs and DIFNOs for the Navier--Stokes equation when compared to the reference solution obtained by the PDE model. The relative spatial $L^2$ errors in the inverse solutions are shown below each plot.}
    \label{fig:ns_map_comparison}
\end{figure}

\subsection{Comparison of Intermediate Resolution Used in Mixed-Resolution Training}\label{sec:mixed_res_ablation}

As discussed in \Cref{sec:mixed_res}, the discrepancy between the high-resolution $\Lambda_H$ and low-resolution $\Lambda_L$ scales can introduce an irreducible gap in the derivative loss. To address this, we investigate the influence of an intermediate resolution, $\Lambda_M$, on mixed-resolution DINO performance. We conduct an ablation study using the nonlinear diffusion--reaction equation, maintaining the data-generation protocol and model architecture detailed in the previous section. \Cref{fig:mixed_res_ablation} presents the error metrics for models trained with intermediate resolutions of $17 \times 17$ and $33 \times 33$, alongside a comparison with the standard $L^2_\mu$ training formulation. Our results indicate that the resolution of $\Lambda_M$ correlates positively with model performance; when the gap between $\Lambda_M$ and $\Lambda_H$ is substantial, the empirical convergence rates for both error measures are significantly degraded.
\begin{figure}[htb]
    \centering
    \renewcommand{\arraystretch}{1.2} 
    \setlength{\tabcolsep}{5pt}       

    \begin{tabular}{c c c}
        & \hspace{0.03\linewidth}\makecell{Nonlinear Diffusion--Reaction \\
        PDE Solution Prediction Error} & \hspace{0.03\linewidth}\makecell{Nonlinear Diffusion--Reaction \\
        Fr\'echet Derivative Prediction Error} \\
        
        \raisebox{3em}{\rotatebox{90}{Relative Error}} & 
        \includegraphics[width=0.42\linewidth]{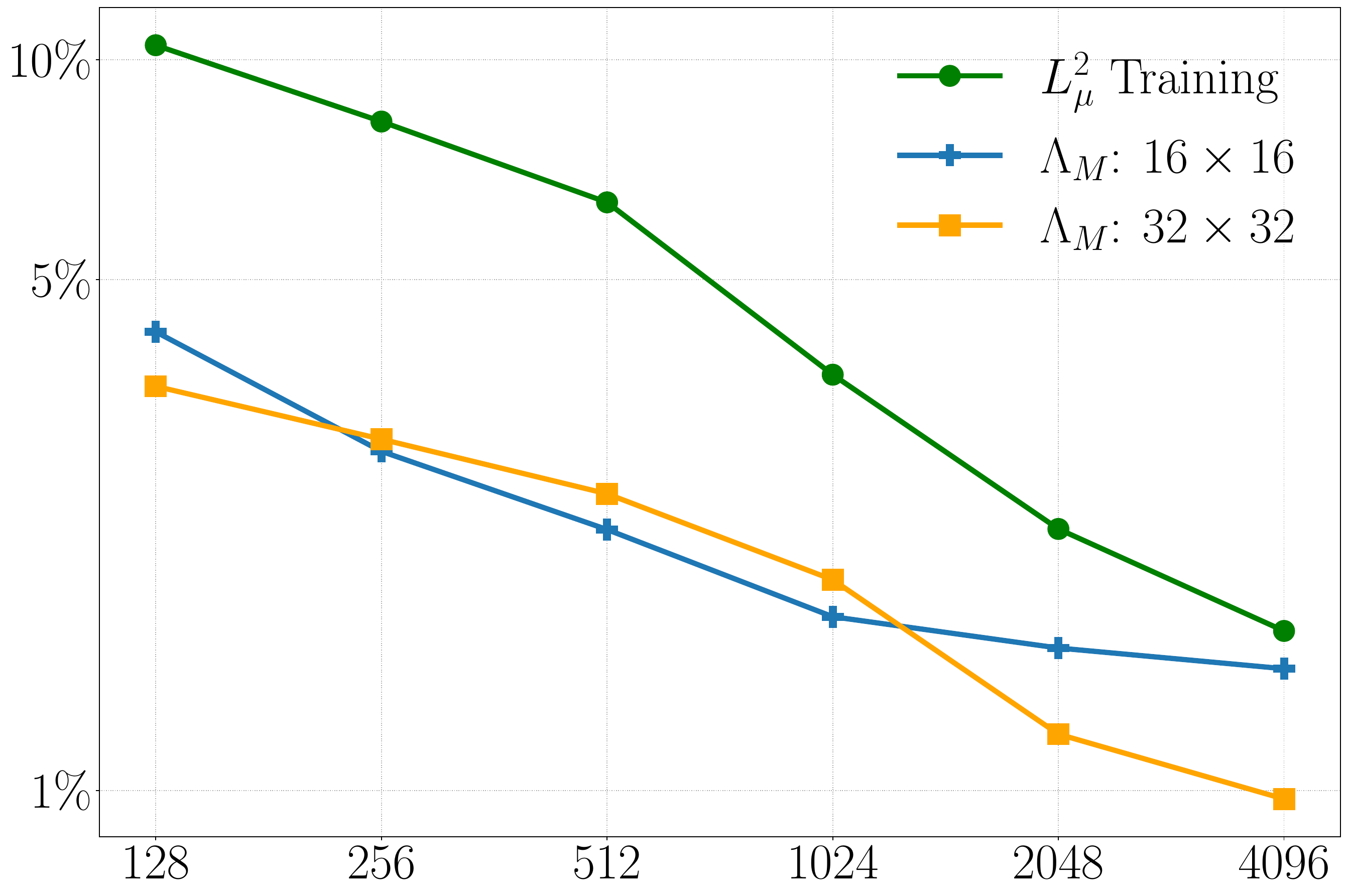} & 
        \includegraphics[width=0.42\linewidth]{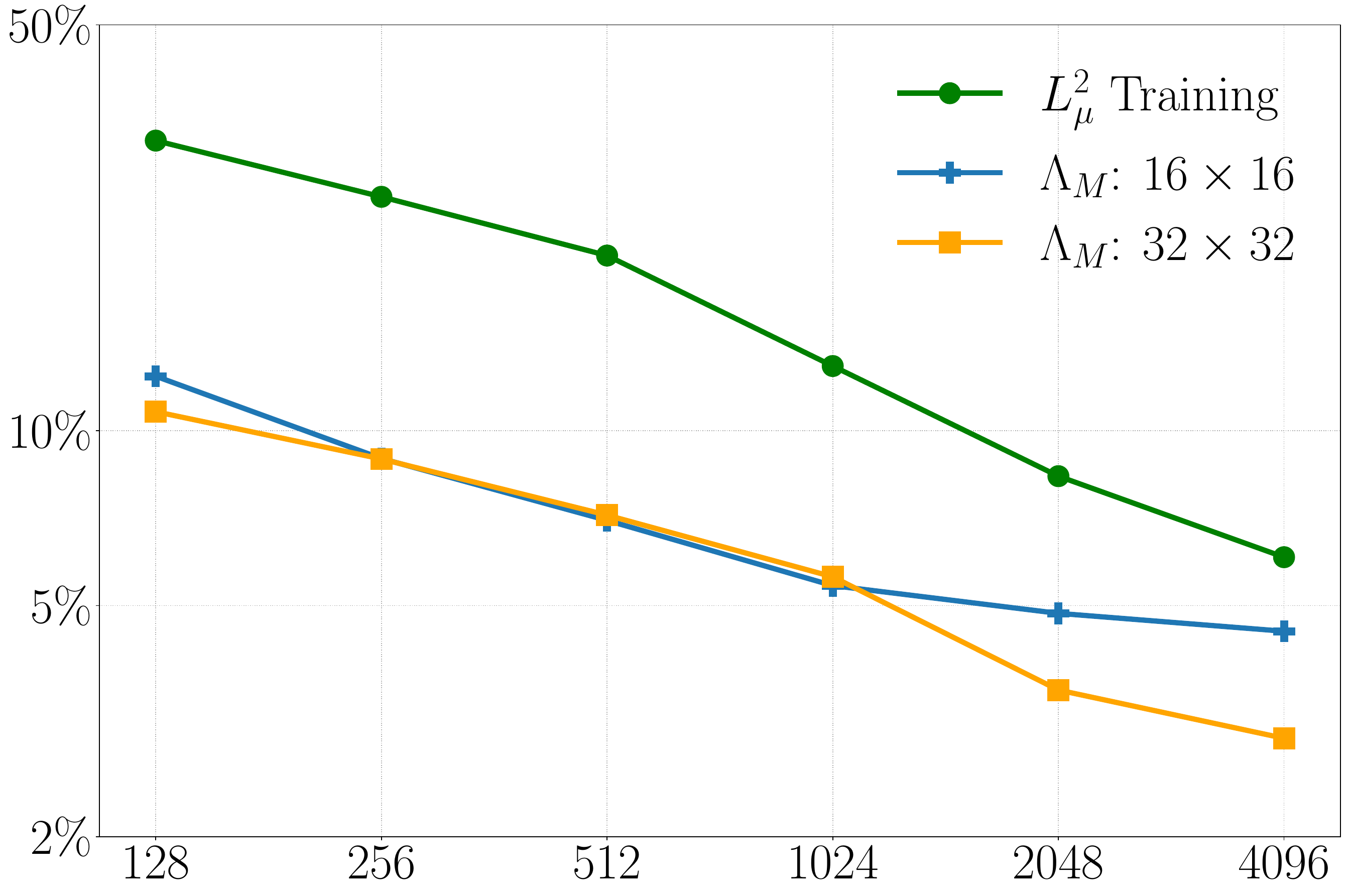} \\
        
        & \hspace{0.03\linewidth} Number of Training Samples & \hspace{0.03\linewidth} Number of Training Samples \\
    \end{tabular}

    \caption{
    Generalization error on the nonlinear diffusion--reaction problem with different intermediate resolution settings.
    The high-resolution $\Lambda_H$ is set to $64 \times 64$, and low-resolution $\Lambda_L$ is set to $16 \times 16$.
    We use two different intermediate-resolutios $\Lambda_M$: $16 \times 16$, $32 \times 32$.
    We can see that aggressive resolution reduction will create a significant performance deterioration.}
    \label{fig:mixed_res_ablation}
\end{figure}

\subsection{Comparison of Training Time}
We report the training times for a single batch corresponding to each training formulation in the benchmark problems in \Cref{tab:training_cost}, 
and the timing of the mixed-resolution ablation study in \Cref{tab:ablation_cost}.
Timing is carried out on the compute nodes of the Vista HPC system at TACC, where each compute node hosts one Nvidia GH200 Grace Hopper superchip.

\begin{table}[htb]
\centering
\caption{Training time for a single batch for different training formulations. 
Reported timings correspond to one training iteration over a batch of 32 samples.
The daggered time and the starred time is reported as the aggregated GPU time over 2 GPUs and 8 GPUs respectively due to memory constraints,
while all other times are obtained using a single GPU.
Note that the high resolution grids $\Lambda_H$ and intermediate resolution grids $\Lambda_M$ for mixed-resolution formulation are set differently for each problem.
}
\label{tab:training_cost}

\scalebox{0.95}{
\begin{tabular}{|c|c|c|c|}
\hline
& FNO & DIFNO (Mixed Res.) & DIFNO (Reduced Basis) \\
\hline
\makecell{Nonlinear Diffusion--Reaction\\$\Lambda_H: 65\times65 \quad \Lambda_M: 33 \times 33$} & 0.014s & 0.44s & $1.67\text{s}^\dagger$ \\
\hline
\makecell{Helmholtz\\$\Lambda_H: 105\times105 \quad \Lambda_M: 27\times27$} & 0.024s & 0.29s & $3.01\text{s}^*$ \\
\hline
\makecell{Navier--Stokes\\$\Lambda_H: 64 \times 64 \quad \Lambda_M: 32 \times 32$} & 0.014s & 0.51s & $1.59\text{s}^\dagger$ \\
\hline
\end{tabular}
}
\end{table}

\begin{table}[htb]
\centering
\caption{Training time for a single batch for the mixed-resolution ablation study in \Cref{sec:mixed_res_ablation}. 
The high resolution grid $\Lambda_H$ here has the resolution $65 \times 65$.}
\label{tab:ablation_cost}

\begin{tabular}{|c|c|c|}
\hline
FNO & Mixed Res. $\Lambda_M: 17 \times 17$ & Mixed Res. $\Lambda_M: 33 \times 33$ \\
\hline
0.014s & 0.19s & 0.44s  \\
\hline
\end{tabular}
\end{table}

While DIFNO generally incurs higher training costs, these costs are often justified in inverse problems due to the importance of derivative accuracy, as seen in our benchmark problems.
Moreover, the mixed-resolution strategy provide effective means to mitigate this cost.
In particular, the mixed-resolution approach introduces a cost-accuracy trade-off at the training stage by choosing the intermediate resolution $\Lambda_M$, thereby allowing one to interpolate between a high-resolution DIFNO and a low-resolution DIFNO in terms of training time and accuracy.
Finally, for more complex PDE systems, data generation typically dominates the overall surrogate construction cost, in which case the additional training expense associated with DIFNO becomes relatively minor. We further emphasize that this cost is incurred only during offline training and has no impact on the model's online cost once employed in PDE-constrained optimization problems.

\section{Conclusions}\label{section:conclusions}
In this work, we established the theoretical and computational foundations for derivative-informed neural operators (DIFNOs). This development is motivated by the need to accelerate the solution of PDE-constrained optimization problems (e.g., inverse problems, optimal design, and optimal control), in which repeated solutions of the PDE model and its sensitivities across different inputs are required. We consider using a cheap-to-evaluate surrogate of the PDE solution operator to replace the PDE model in optimization. We first show in \Cref{prop:surrogate_optimization_error,prop:surrogate_optimization_error_strongly_convex} that accurate surrogate-driven optimization solutions require accurate surrogate predictions of not only the PDE solutions but also their Fr\'echet derivatives with respect to the parametric inputs. These results provide strong support for derivative-informed operator learning over the conventional method, as the former simultaneously minimizes errors in surrogate predictions of the PDE solutions and Fr\'echet derivatives using samples. Additionally, it motivates our subsequent analysis on the approximation capability of FNOs for derivative-informed operator learning.

First, we established simultaneous universal approximation of continuously differentiable operators and their Fréchet derivatives by FNOs on compact sets (\Cref{theorem:main_ua}). Our analysis extends the strategies in \cite{kovachki2021universal} to also consider approximating the Fr\'echet derivative. This result certifies the approximation capability of FNOs for derivative-informed operator learning. This result also supplies the conditions to prove that the FNO is capable of solving inverse problems arbitrarily well (\Cref{theorem:inverse_problem_ua}).

Second, we established universal approximation results for FNOs in weighted Sobolev spaces (\Cref{theorem:main_ua_mu}). These results rely on a novel cutoff argument to extend the universal approximation on compact sets to the setting where the input is unbounded, in the sense that the input measures of the weighted Sobolev spaces have unbounded topological supports. This result holds for continuously differentiable FNOs with GELU-like activation functions (\Cref{def:activation}), provided that the input measures and the high-fidelity operators satisfy mild assumptions (\Cref{assumption:measurability}). These results provide a more practical approximation theory for DIFNOs, as we use the norm on the weighted Sobolev space as the loss function in derivative-informed operator learning.

Having laid out a solid foundation for DIFNO approximation theory, we propose dimension-reduction and multi-resolution techniques for efficient DIFNO training, whose computational and memory costs increase rapidly with grid size under a naive implementation. First, the dimension-reduction technique uses reduced bases in relevant subspaces of the input and output spaces, with grid-independent dimensionality. The reduced bases are used to compress the derivative loss, making it active only in the latent spaces, thereby alleviating the dependence of the data-generation and training costs of DIFNOs on the grid size. Furthermore, the multi-resolution technique generates compressed derivative samples at low resolution and evaluates the derivative loss at low resolution, thereby drastically reducing data generation and training costs.

Finally, we conduct numerical experiments that use DIFNOs and conventional FNOs to learn solution operators for PDEs and to solve inverse problems. The PDEs considered are nonlinear diffusion--reaction, Helmholtz, and Navier--Stokes equations. Several key observations and inductions are made from the numerical results:

\begin{enumerate}[(i)]
    \item The DIFNOs consistently achieve sample complexity more than an order of magnitude lower than that of the FNOs for learning to predict PDE solutions and Fr\'echet derivatives.
    \item The reduced-basis DIFNOs consistently achieve better sample complexity than the mixed-resolution DIFNOs, while the mixed-resolution DIFNOs are much cheaper to train. Moreover, the accuracy of the mixed-resolution DIFNO is closely related to the difference between the low and high resolutions.
    \item The DIFNOs are much more reliable than the FNOs when used to solve inverse problems, particularly at low training sample sizes.
    \item For the Navier--Stokes equations, the FNOs see improved predictions of PDE solutions as the training sample size increases, yet they struggle to improve their Fr\'echet derivative predictions compared to the DIFNOs. We observe that the inverse solutions by the FNOs also struggle to converge to the reference compared to those of DIFNOs. This result indicates that the fidelity of the surrogate Fr\'echet derivative is a critical factor in surrogate-driven solutions of inverse problems, which agrees with our theoretical results in \Cref{prop:surrogate_optimization_error,prop:surrogate_optimization_error_strongly_convex}.
\end{enumerate}

Several limitations of DIFNO remain to be studied. First, due to the structure of FNOs, evaluating their Fr\'echet derivatives remains intensive in computation and memory---even with the proposed efficient formulations, DIFNO training remains substantially more expensive than conventional FNOs. Second, many challenging operator-learning problems arise in multi-scale, multi-physics systems. While the Helmholtz and Navier-Stokes equations studied in this work are considered difficult for operator learning, it would be interesting to further examine the performance of DIFNOs in more sophisticated, realistic settings. Addressing these limitations constitutes the main direction for our future research. Additionally, the difficulty experienced by conventional FNOs in learning the Fr\'echet derivative in our examples of the Navier--Stokes equations is connected to active research on operator learning of oscillatory and chaotic systems \cite{molinaro2025generative, khodakarami2025mitigating, oommen2025integrating}. It is of interest to investigate how the derivative-informed learning formulation can address the shortfall of conventional learning in this context from both theoretical and numerical perspectives.

\section*{Acknowledgment}

Omar Ghattas, Dingcheng Luo, Thomas O'Leary-Roseberry and Boyuan Yao were supported by the National Science Foundation through awards 2324643, 2313033 and  2425922 the Deparment of Energy through award DE-SC0023171, and the Air Force Office of Scientific Research through award FA9550-24-1-0327.
Lianghao Cao was supported by a Department of
Defense Vannevar Bush Faculty Fellowship held by Andrew M. Stuart. Nikola B. Kovachki is grateful to the Nvidia Corporation for support through full-time employment.

\addcontentsline{toc}{section}{References}
\bibliographystyle{elsarticle-num-names}
\biboptions{sort,numbers,comma,compress}                 
\bibliography{main.bib}

\appendix
\renewcommand{\thesection}{\Alph{section}}
\makeatletter
\renewcommand\@seccntformat[1]{\appendixname\ \csname the#1\endcsname.\hspace{0.5em}}
\makeatother

\section{Definitions and Mathematical Preliminaries}\label{sec:notes_on_periodic_sobolev_spaces}

\subsection{Lebesgue and Sobolev Spaces}
In this work, we are primarily concerned with mappings between functions defined on the periodic torus $\bT^{d}$, which we identify with $[0, 2\pi]^d$. 
Note that functions on $\bT^d$ can be identified with $2\pi$-periodic functions on $\bR^d$.
We introduce the standard Lebesgue spaces, $L^p(\bT^d)$, whose norms are given by 
\[
    \| u \|_{L^p(\bT^d)} := \begin{cases}
        \left(\int_{\bT^d} |u(x)|^p dx \right)^{1/p} & 1 \leq p < \infty, \\
        \mathrm{esssup}_{x \in \bT^d} |u(x)| & p = \infty.
    \end{cases}
\]   

For integer values $m \in \bN$ and $u \in C^{m}(\bT^d)$, we can define the Sobolev norm
\begin{equation} \label{eq:integer_sobolev_norm}
    \|u\|_{H^m}^2 = \sum_{0 \leq |\alpha|_1 \leq m} \|\partial ^\alpha u\|_{L^2}^2,
\end{equation}
where $\alpha = (\alpha_1, \dots, \alpha_{d})$ is a multi-index, $|\alpha|_1 = \sum_{j=1}^{d}|\alpha_j|$, 
and $\partial^{\alpha}u = \partial_{x_1}^{\alpha_1} \dots \partial_{x_d}^{\alpha_d} u $.
The Sobolev spaces with integer powers, $H^{m}(\bT^d)$, are then defined as the completion of $C^{m}(\bT^d)$ under the $H^{m}$ norm. 

For non-integer powers $s \geq 0$, we can introduce the fractional Sobolev spaces $H^s(\bT^d)$ by Fourier transforms. 
Recall that for $u \in L^2(\bT^d)$, we have 
\[
    \hat{u}(k) := \cF(u)(k) = \frac{1}{(2\pi)^{d}} \int_{\bT^d} u(x) \exp(-i\langle k, x\rangle) dx, \quad k \in \bZ^d.
\]
We define the norm $\| \cdot \|_{H^s}$ as 
\[
    \| u \|_{H^s} = (2\pi)^d \sum_{k \in \bZ^d} (1 + |k|^2)^{s} |\hat{u}(k)|^2.
\]
The spaces $H^s(\bT^d)$ are then defined for $s \geq 0$ as
\begin{equation}\label{eq:real_sobolev}
    H^s(\bT^d) := \{ u \in L^2(\bT^d) : \|u \|_{H^s} < \infty \}.
\end{equation}
Note that this norm is equivalent to the previous definition for integer values $s = m$. 

In this work, we will consider mainly real-valued function spaces. Thus, we will use $L^p(\bT^d)$ and $H^s(\bT^d)$ to specifically denote $L^p(\bT^d; \bR)$ and $H^{s}(\bT^d; \bR)$.
For a real-valued function $u$, the Fourier coefficients exhibit the symmetry $\hat{u}(k) = \overline{\hat{u}(-k)}$ for all $k \in \bZ^{d}$.

\subsection{Sobolev Spaces as a Hilbert Scale}
In the real-valued setting, the Fourier expansion for $L^2(\bT^d)$ can instead be written in terms of the 
sinusoids $\cos(\langle k, x \rangle)$ and $\sin(\langle k, x \rangle)$.
That is, we consider the functions 
\begin{align}
    \psi_0(x) & := \sqrt{\frac{1}{(2\pi)^d}}, 
    \\
    \psi_{k}^{\cos}(x) & := \sqrt{\frac{2}{(2\pi)^d}} \cos(\langle k, x \rangle),
    \qquad k \in \bZ^{d}, \\
    \psi_{k}^{\sin}(x) & := \sqrt{\frac{2}{(2\pi)^d}} \sin(\langle k, x \rangle),
    \qquad k \in \bZ^{d},
\end{align}
from which take one from each pair of $(k, -k)$ to form an orthonormal basis of $L^2(\bT^d)$.
For example, we can use the index set
\[
    \zlex := \{ k \in \bZ^{d} : \exists j, \,  1 \leq j \leq d, \, \text{ such that } k_{j} > 0 \text{ and } k_i = 0, \, 1 \leq i \leq j \},
\]
such that $\zlex \cup -\zlex \cup \{0\} = \mathbb{Z}^{d}$.
In this case, the Sobolev norm can be written as 
\begin{equation}
    \|u\|_{H^{s}}^2 
    = 
    \langle u, \psi_0 \rangle_{L^2}^2 
    + 
    \sum_{k \in \zlex} (1 + |k|^2)^{s} \langle u, \psi_k^{\cos} \rangle_{L^2}^2
    + 
    \sum_{k \in \zlex} (1 + |k|^2)^{s} \langle u, \psi_k^{\sin} \rangle_{L^2}^2.
\end{equation}
It can also be verified that $\psi_0$ along with 
\begin{align}
    \psi_{s,k}^{\cos} &:= (1+|k|^2)^{-\frac{s}{2}} \psi_k^{\cos},  \qquad k \in \zlex, \\ 
    \psi_{s,k}^{\sin} &:= (1+|k|^2)^{-\frac{s}{2}} \psi_k^{\sin},  \qquad k \in \zlex,
\end{align}
form an orthonormal basis for $H^s(\bT^d)$.
This basis can be used to evaluate the Hilbert--Schmidt norm of an operator, $\cA : \cX = H^s(\bT^d) \rightarrow \cY$, i.e.,
\begin{equation}
    \|\cA\|_{\HS(\cX, \cY)}^2 = 
    \|\cA \psi_0 \|_{\cY}^2 + \sum_{k \in \zlex} 
        \|\cA \psi_{s,k}^{\cos} \|_{\cY}^2  + 
        \|\cA \psi_{s,k}^{\sin} \|_{\cY}^2.
\end{equation}

Moreover, the sinusoidal basis functions can be enumerated as $\{ \psi_j \}_{j \in \bN}$ 
such that the wavenumbers for the $j$th basis function, $k^{(j)}$, satisfy 
\begin{equation}\label{eq:weight_decay_rate}
    w_j := (1 + |k^{(j)}|^2)^{\frac{1}{2}} \sim j^{\frac{1}{d}}, \qquad j \in \bN.
\end{equation}
We note that $\{(w_j^{2}, \psi_j)\}_{j \in \bN}$ are the eigenpairs of the elliptic operator $\cA = \cI - \Delta$ defined on the domain $H^2(\bT^d)$, i.e., $\cA \psi_j = w_j^2 \psi_j$.
Thus, the family of $H^{s}(\bT^d)$ with $s \geq 0$ forms a Hilbert scale such that $H^{s}(\bT^d) \hookrightarrow H^{s'}(\bT^d)$ whenever $s' \leq s$,
and  
\[
   \|u\|_{H^s}^2 = \| \cA^\frac{s}{2} u \|_{L^2}^2 = \sum_{j \in \bN} w_j^{2s}  | \langle u, \psi_j \rangle_{L^2}|^2. 
\]
By \eqref{eq:weight_decay_rate}, we also have an equivalent definition for $H^{s}$ norm, which we write as $\cH^{s}$, given by 
\begin{equation}\label{eq:hilbert_scale_norm}
    \|u\|_{\cH^s}^2 := \sum_{j=1}^{\infty}  j^{\frac{2s}{d}} |u_j|^2 \sim \|u\|_{H^s}^2.
\end{equation}
This is generated by the equivalent inner product 
\begin{equation}\label{eq:hilbert_scale_inner_product}
    \langle u, v \rangle_{H^{s}(\bT^d)} \sim \langle u, v \rangle_{\cH^s} := \sum_{j=1}^{\infty}  j^{\frac{2s}{d}} u_j {v_j},
\end{equation}
where $u_j = \langle u, \psi_j \rangle_{L^2(\bT^d)}$
and $v_j = \langle v, \psi_j \rangle_{L^2(\bT^d)}$.

Using this equivalent definition of $H^{s}(\bT^d)$, we can show that the inclusion operator from $H^s(\bT^d)$ to $H^{s'}(\bT^d)$
is additionally Hilbert--Schmidt provided $s$ is sufficiently greater than $s'$.
\begin{proposition} \label{prop:inclusion_operator_is_hilbert_schmidt}
Let $s \geq 0$.
The inclusion operator $\cI : H^{s + \delta}(\bT^d) \rightarrow H^{s}(\bT^d)$ 
is Hilbert--Schmidt for all $\delta > d/2$.
\end{proposition}
\begin{proof}
     Note that under the inner product of $\cH^s$, 
     $\{\psi_j \}_{j=1}^{\infty}$ is an orthonormal basis when $s = 0$
     while $\{j^{-s/d} \psi_j\}_{j=1}^{\infty}$ is an orthonormal basis when $s > 0$.
     Consider now the Hilbert--Schmidt norm of the inclusion $\cI : \cH^{s+\delta} \rightarrow \cH^{s}$. 
     This can be evaluated using the basis functions $\{j^{-(s+\delta)/d} \psi_j\}_{j=1}^{\infty}$, 
     \[
        \| \cI \|_{\HS(\cH^{s+\delta}, \cH^{s})}^2
        = \sum_{j=1}^{\infty} \|j^{-(s+\delta)/d}\psi_j\|_{\cH^s}^2
        = \sum_{j=1}^{\infty} j^{2s/d} \cdot j^{-2(s+\delta)/d}
        = \sum_{j=1}^{\infty} j^{-2 \delta/d},
     \]
     which is finite when $\delta > d/2$.
\end{proof}

\subsection{Embedding Theorems}
Additionally, we will make use of two standard results from the theory of Sobolev spaces. The first is the Sobolev embedding theorem.
\begin{theorem}[Sobolev embedding theorem]
    Let $d \in \bN$ and $s > \frac{d}{2}$. Then, 
    the embedding from $H^{s+m}(\bT^d)$ to $\cont^m(\bT^d)$ is continuous for every integer $m \geq 0$.
\end{theorem}

An implication of the Sobolev embedding theorem is the fact that $H^s(\bT^d)$ forms a Banach algebra whenever the embedding into $\cont^0(\bT^d)$ holds. That is, $H^s(\bT^d)$ is closed under multiplication.
\begin{theorem}[$H^s(\bT^d)$ is a Banach algebra]\label{theorem:banach_algebra}
    Let $d \in \bN$ and $s > \frac{d}{2}$. Then, 
    there exists constant $\const_s \geq 0$ such that 
    \[ 
        \|u v\|_{H^s(\bT^d)} \leq \const_s \|u\|_{H^s(\bT^d)} \|v\|_{H^s(\bT^d)}
    \]
    for all $u, v \in H^s(\bT^d)$.
\end{theorem}

\subsection{Superposition Operators}
A superposition operator (or Nemytskii operator) is a mapping defined by composing a function $u: \bT^d \rightarrow \bR$ with another function $\sigma : \bT^d \times \bR \rightarrow \bR$. 
That is, the superposition operator $\cS$ is defined as
\[
    \cS : u \mapsto \cS(u), \quad \cS(u)(x) = \sigma(x, u(x)).
\]
We say $\cS$ is autonomous if  $\sigma(x, y) = \sigma(y)$, i.e.~$\cS(u)(x) = \sigma(u(x))$.
We encounter autonomous superposition operators on Lebesgue and Sobolev spaces when applying the activation function $v \mapsto \sigma(v)$ on functions $v \in L^p$ or $H^s$ within the FNO architecture.
Here, we will summarize the continuity and Fr\'echet differentiability properties of such operators. 
In particular, we will introduce settings sufficient for analyzing FNOs.

The simplest case is when mapping from $L^\infty$ to $L^\infty$.
We state the result for scalar-valued outputs, but note that generalization to the vector-valued case is straightforward.
The proof essentially follows that of \cite[Lemma 4.12]{troltzsch2010optimal}.
\begin{lemma}[Differentiability of superposition operators on $L^{\infty}$]\label{lemma:superposition_Linfty}
Let $\sigma \in C^2(\bR^{d_v}, \bR)$. 
Then, the superposition operator
\[
    \cS(u)(x) =  \sigma(u(x))
\]
is a continuously differentiable mapping from $L^{\infty}(\bT^d; \bR^{d_v}) \rightarrow L^{\infty}(\bT^d; \bR)$. 
Moreover, its derivative is given by 
\[
    (D\cS(u)w)(x) = \sum_{j=1}^{d_v}\partial_j \sigma(u(x)) w_j(x),
\]
where $\partial_j \sigma$ denotes the partial derivative of $\sigma$ with respect to its $j$th argument.
\end{lemma}

Stronger requirements are needed when considering compositions with functions in $H^s(\bT^d)$. In particular, sufficient regularity of the function $\sigma$ is needed to preserve the regularity of $u \in H^s$. 
The following result is essentially a special case of \cite[Theorem 4.1]{Valent88}.

\begin{lemma}[Differentiability of superposition operators on $H^m$]\label{lemma:superposition_sobolev}
    Let $m > d/2$ be an integer and let $\sigma \in C^{m+1}(\bR^{d_v}, \bR)$.
    Then, the superposition operator 
    \[ 
        \cS(u)(x) = \sigma(u(x)) 
    \]
    is a continuously differentiable mapping from $H^{m}(\bT^d; \bR^{d_v}) \rightarrow H^{m}(\bT^d; \bR)$.
    Moreover, its derivative 
    is given by
    \[
        (D\cS(u)w)(x) = \sum_{j=1}^{d_v} \partial_j \sigma(u(x)) w_j(x).
    \]
\end{lemma}

\section{Surrogate-based Optimization Results}
\subsection{Residual Errors for General Objectives}\label{sec:proof_optimization_error}
We begin with the proof of the residual error at approximate minimizers in the general case.
\begin{proof}(\Cref{prop:surrogate_optimization_error})
    At any $a \in \cX$, the error in the derivative of $f$ is bounded by 
    \begin{align*}
        & \|D f(a) - D f_{w}(a)\|_{\cX'}
        \\
        & \qquad = \|\partial_u F(\cG(a), a)D \cG(a) - \partial_u F(\widetilde{\cG}_w(a), a) D\widetilde{\cG}_w(a) 
            + \partial_a F(\cG(a), a) - \partial_a F(\widetilde{\cG}_w(a), a) \|_{\cX'} \\
        & \qquad 
        \leq \|\partial_u F(\cG(a), a)(D\cG(a)- D\widetilde{\cG}_w(a))\|_{\cX'} \\ 
        & \qquad \qquad + \| (\partial_u F(\cG(a), a) -  \partial_u F(\widetilde{\cG}_w(a), a)) D\widetilde{\cG}_w(a)\|_{\cX'}\\
        & \qquad \qquad +\|\partial_a F(\cG(a), a) - \partial_a F(\widetilde{\cG}_w(a), a)\|_{\cX'} \\
        & \qquad 
            \leq \| \partial_u F(\cG(a), a)\|_{\cY'} \| D \cG(a) - D \widetilde{\cG}_w (a) \|_{\Op(\cX,\cY)} \\
        & \qquad \qquad + \| \partial_u F(\cG(a), a) - \partial_u F(\widetilde{\cG}_w(a), a)\|_{\cY'} \| D \widetilde{\cG}_w(a)\|_{\Op(\cX,\cY)} \\ 
        & \qquad \qquad + \| \partial_a F(\cG(a), a) - \partial_a F(\widetilde{\cG}_w(a), a)\|_{\cX'}.
    \end{align*}
    Since $DF$ is Lipschitz, there exists a Lipschitz constant $L$ such that 
    \[ 
        \|\partial_u F(u_2,a_2) - \partial_u F(u_1, a_1)\|_{\cY'}
        + \|\partial_a F(u_2,a_2) - \partial_a F(u_1, a_1)\|_{\cX'}
        \leq L(\|u_2 - u_1\|_{\cY} + \|a_2 - a_2\|_{\cX}).
        \]
    Using this fact, we have the following
    \begin{align*}
        \| \partial_u F(\cG(a), a) \|_{\cY'} &\leq L ( \|\cG(a)\|_{\cY} + \|a\|_{\cX}) + \|\partial_u F(0, 0) \|_{\cY'}, \\
        \| \partial_a F(\cG(a), a) - \partial_a F(\widetilde{\cG}_w(a), a) \|_{\cY'} &\leq L \|\cG(a) - \widetilde{\cG}_w(a)\|_{\cY}, \\
        \| \partial_u F(\cG(a), a) - \partial_u F(\widetilde{\cG}_w(a), a) \|_{\cY'} &\leq L \|\cG(a) - \widetilde{\cG}_w(a)\|_{\cY}.
    \end{align*}
    
    Moreover, we have 
    \[ 
    \|D \widetilde{\cG}_w(a)\|_{\Op(\cX,\cY)} \leq \|D \cG(a)\|_{\Op(\cX,\cY)} + \|D \cG(a) - D\widetilde{\cG}_w(a)\|_{\Op(\cX,\cY)}.
    \]
    Combining the inequalities, we arrive at 
    \begin{align*}
        \|D f(a) - D f_{w}(a)\|_{\cX'} 
        & \leq L( (\|\cG(a)\|_{\cY} + \|a\|_{\cX}) + \| \partial_u F(0,0)\|_{\cX'})\, \mathscr{E}_1(a)  \\
        & 
        \quad 
        + L ( \|D \cG(a)\|_{\cY} + \mathscr{E}_1(a)) \mathscr{E}_0(a)
        + L \mathscr{E}_0(a) 
        .
    \end{align*}
    which, after collecting the like terms, yields
    \begin{align*}
        & \|D f(a) - D f_{w}(a)\|_{\cX'}   \\
        & \qquad \lesssim \left( \|\cG(a)\|_{\cY} + \|D \cG(a)\|_{\Op(\cX,\cY)} + \|a\|_{\cX} + 1 \right) 
        \left( \mathscr{E}_0(a) + \mathscr{E}_1(a) + \mathscr{E}_0(a) \cdot \mathscr{E}_1(a) \right).
    \end{align*}
    It remains to recognize that at $a^{\dagger}$, $D f_w(a^\dagger) = 0$.
    
\end{proof}

\subsection{Solution Errors for Strongly Convex Objectives}\label{sec:proof_optimization_error_strongly_convex}
The proof for the case of strongly convex $f$ in \Cref{prop:surrogate_optimization_error_strongly_convex} then follows.
\begin{proof}(of \Cref{prop:surrogate_optimization_error_strongly_convex})

     \textit{Part (i).} 
     We note that the boundedness of $\|\cG(a)\|_{\cY}$, $\|D\cG(a)\|_{\Op(\cX,\cY)}$, and $\|a\|_{\cX}$ for $a \in B_R(a^{\star})$ implies that 
    \[ 
        \sup_{a \in B_R(a^{\star})} \left(\|\cG(a)\|_{\cY} + \|D \cG(a)\|_{\Op(\cX,\cY)} + \|a\|_{\cX}\right) \leq 2M + R + \|a^{\star}\|_{\cX},
    \]
    which is a constant depending only on the operator $\cG$ and the radius $R$. 
    Thus, \Cref{prop:surrogate_optimization_error} yields 
    \begin{equation} \label{eq:bounded_set_residual_error}
        \| D f(a^{\dagger}) \|_{\cX'} \lesssim \mathscr{E}_0(a^{\dagger}) + \mathscr{E}_1(a^{\dagger}) + \mathscr{E}_0(a^{\dagger}) \cdot \mathscr{E}_1(a^{\dagger})
    \end{equation}
    for any stationary point $a^{\dagger} \in B_R(a^{\star})$.
    To translate this to a bound on the error $\|a^{\dagger} - a^{\star}\|_{\cX}$, we make use of the strong convexity of $f$.
    In particular, for $D^2f(a)(h,h) \geq \lambda \|h\|_{\cX}^2$, we have
    \[
        \|D f(a^{\dagger})\|_{\cX'} \|a^{\dagger} - a^{\star}\|_{\cX} \geq  D f(a^{\dagger})(a^{\dagger} - a^{\star})  \geq \lambda \|a^{\dagger} - a^{\star}\|_{\cX}^2,
    \]
    which, along with \eqref{eq:bounded_set_residual_error}, implies that 
    \[
        \|a^{\dagger} - a^{\star}\|_{\cX} \leq \lambda^{-1} \|Df(a^{\dagger})\|_{\cX'} 
        \lesssim
            \mathscr{E}_0(a^{\dagger}) + \mathscr{E}_1(a^{\dagger}) + \mathscr{E}_0(a^{\dagger}) \cdot \mathscr{E}_1(a^{\dagger}).
    \]

    \textit{Part (ii).} 
    Consider first the constrained minimization problem
    \begin{equation}\label{eq:constrained_ball_minimization}
        \min_{a \in B_R(a^{\star})} f_w(a).
    \end{equation}
    Since $B_R(a^{\star})$ is a closed and convex set and $f_w$ is weakly lower semicontinuous and differentiable, results from standard variational analysis imply the existence of a solution $a^{\dagger}$ to this problem with the optimality condition
    \[
        \langle D f_w(a^{\dagger}), a - a^{\dagger} \rangle_{\cX} \geq 0 \qquad \forall a \in B_R(a^{\star}).
    \]
    We can then show that, provided the surrogate error is small, $a^{\dagger}$ must lie in the interior of $B_R(a^{\star})$.
    Strong convexity of $f$ implies that for any $a$,
    \[
        f(a) \geq f(a^{\star}) + \frac{\lambda}{2} \|a - a^{\star} \|^2.
    \]
    Therefore, surrogate objective function at any point $a$ on the boundary of $B_R(a^{\star})$ satisfies
    \[
        f_w(a) \geq f_w(a^{\star}) + \frac{\lambda}{2} R^2 - |f(a) - f_w(a)| - |f (a^{\star}) - f_w(a^{\star})|.
    \]
    From this, we see that the surrogate objective function at $a$ on the boundary is strictly greater than its value at $a^{\star}$ provided that 
    \begin{equation}\label{eq:requirement_for_local_minimizer}
        |f(a) - f_w(a)| + |f(a^{\star}) - f_w(a^{\star})| < \frac{\lambda}{4}R^2.
    \end{equation}
    
    We now show how the operator surrogate error $\mathscr{E}_0(a)$ translates to the objective function error. 
    By the mean value theorem, we have for any $a$, 
    \begin{equation}
        |f(a) - f_w(a)| = |F(\cG(a), a) - F(\widetilde{\cG}_w(a), a)| \leq \|\partial_u F(\xi, a)\|_{\cY'} \|\cG(a) - \widetilde{\cG}_w(a)\|_{\cY},
    \end{equation}
    where $\xi = t \cG(a) + (1-t) \widetilde{\cG}_w(a)$ for some $t \in [0,1]$.
    Then, by the Lipschitz continuity of $DF$, we have 
    \[
        \| \partial_u F(\xi, a)\|_{\cY'} \leq L \left( \|\xi\|_{\cY} + \|a\|_{\cX} \right) + \|\partial_u F(0,0)\|_{\cY'}.
    \]
    Moreover, $\|\xi\|_{\cY} \leq \|\cG(a) - \widetilde{\cG}_w(a)\|_{\cY} + \|\cG(a)\|_{\cY}$.
    Thus, the overall bound on the objective function error is given by 
    \[
        |f(a) - f_w(a)| \leq \left( L \left(\|\cG(a) - \widetilde{\cG}_w(a)\|_{\cY} + \|\cG(a)\|_{\cY} + \|a\|_{\cX}\right) + \|\partial_u F(0,0)\|_{\cY'} \right) \| \cG(a) - \widetilde{\cG}_w(a)\|_{\cY}.
    \]
    Since $\|\cG(a)\|_{\cY}$ and $\|a\|$ are bounded for any $a \in B_R(a^{\star})$, we have 
    \begin{equation}
        |f(a) - f_w(a)| \lesssim 
            \sup_{a \in B_R(a^{\star})} \mathscr{E}_0(a) + \left(\sup_{a \in B_R(a^{\star})} \mathscr{E}_0(a) \right) ^2.
    \end{equation}
    Therefore, provided $\sup_{a \in B_R(a^{\star})} \mathscr{E}_0(a)$ is sufficiently small such that \eqref{eq:requirement_for_local_minimizer} is satisfied, any solution $a^{\dagger}$ to the problem \eqref{eq:constrained_ball_minimization} cannot lie on the boundary of $B_R(a^{\star})$. 
    Instead, $a^{\dagger}$ must belong to the interior of $B_R(a^{\star})$, where the optimality condition implies that $D f(a^{\dagger}) = 0$.

\end{proof}

\section{Universal Approximation on Compact Sets}
\label{sec:supporting_lemmas_and_ua}
In this section, we present details for the proof of \Cref{theorem:main_ua}. 
We begin by proving a few supporting results, before putting them together in \Cref{sec:proof_of_main_ua}.

\subsection{Replacement Lemma}
We begin by proving the analogy of the replacement lemma \cite[Lemma 47]{kovachki2021universal} for both the operator and its first derivative.
That is, for a sequence of operators, $\cG = \cG_n \circ \cdots \circ \cG_1$, 
if each $\cG_j$, $j = 1, \dots, n$ can be approximated arbitrarily well by an FNO $\cN_j$, $j = 1, \dots, n$, 
then its composition $\cG$ can be approximated arbitrarily well by the compositions of FNOs $\cN := \cN_n \circ \cdots \circ \cN_1$,
measured using both the operator value and its first derivative.
\begin{lemma}[Replacement Lemma in $\cont^1$]\label{lemma:composition}
    Let $\cX_0, \cX_1, \dots, \cX_{n}$ be Banach spaces where the intermediate spaces $\cX_1, \dots, \cX_{n-1}$ are finite dimensional.
    Further, let $\cG_j : \cX_{j-1} \rightarrow \cX_{j}$, $j = 1, \dots, n$, be continuously differentiable operators.
    Suppose that for any compact subsets $K_j \subset \cX_j$, $j = 0, \dots, n-1$,
    and for any $\epsilon > 0$, 
    there exist continuously differentiable FNO approximations 
    $\cN_j : \cX_{j-1} \rightarrow \cX_{j}$, $j = 1, \dots, n-1$,
    such that 
    \begin{equation}
        \max 
        \left\{ 
        \sup_{a \in K_{j-1}} 
            \| \cG_j(a) - \cN_j(a)\|_{\cX_j} ,
        \sup_{a \in K_{j-1}} 
            \| D \cG_j(a) - D \cN_j(a) \|_{\Op(\cX_{j-1}, \cX_{j})} 
        \right\}
        \leq \epsilon, 
    \end{equation}
    for $j = 1, \dots, n$. 
    Then, for any $\epsilon > 0$ and any compact subset $K_0 \subset \cX_0$, there exist FNOs $\{\cN_j \}_{j=1}^{n}$ such that 
    for the compositions
    \[
    \cG := \cG_{n} \circ \cdots \circ \cG_{2} \circ \cG_{1} \quad  \text{and} \quad 
    \cN := \cN_{n} \circ \cdots \circ \cN_2 \circ \cN_1,
    \]
    we have
    \begin{equation} \label{eq:composition_error}
        \max \left\{ 
        \sup_{a \in K_0} 
            \| \cG (a) - \cN(a)\|_{\cX_{n}} ,
        \sup_{a \in K_0} 
            \| D \cG (a) - D \cN(a) \|_{\Op(\cX_0, \cX_n)} 
        \right\}     
            \leq \epsilon.
    \end{equation}
    When $\cX_0, \dots \cX_n$ are Hilbert spaces, the same result holds with the operator norm $\Op$ in \eqref{eq:composition_error} replaced by the Hilbert--Schmidt norm $\HS$. 
\end{lemma}

\begin{proof}
    We will begin with the case of $n = 2$, with $n \geq 2$ following by induction.
    Since $K_0$ is compact, the image $\cG_1(K_0)$ is compact in $\cX_1$ and is therefore bounded as $\cX_1$ is finite dimensional. 
    Let $\tilde{R} > 0$ denote the bound on $\cG_1(K_0)$ such that $\cG_1(K_0)$ is contained in the closed ball of radius $\tilde{R}$.
    In the following, we restrict ourselves to approximations $\cN_1$ such that $\sup_{a \in K_0} \|\cG_1(a) - \cN_1(a)\|_{\cX_1} \leq 1$. 
    Given any such $\cN_1$, for $R = \tilde{R} + 1$, the set $\cG_1(K_0) \cup \cN_1(K_0) \subset {B}_R(0)$, where $B_R(0) := \{ \|a\|_{\cX_{1}} \leq R\}$ is compact in $\cX_1$.
    Then, for any $\eta \in (0,1)$, 
    let $\cN_1$ and $\cN_2$ be chosen such that 
    \begin{equation}
        \sup_{a \in K_0} \| \cG_1(a) - \cN_1(a)\|_{\cX_1} + \| D \cG_1(a) - D \cN_1(a) \|_{\Op(\cX_0, \cX_1)} \leq \eta 
    \end{equation}
    and
    \begin{equation}
        \sup_{a \in {B}_R(0)} \| \cG_2(a) - \cN_2(a)\|_{\cX_2} + \| D \cG_2(a) - D \cN_2(a) \|_{\Op(\cX_1, \cX_2)} \leq \eta
    \end{equation}
    are simultaneously satisfied.
    
    We can now begin by considering the function error $\|\cG_2 \circ \cG_1 - \cN_2 \circ \cN_1 \|_{\cX_2}$. 
    This can be decomposed as 
    \[
        \| \cG_2(\cG_1(a)) - \cN_2(\cN_1(a)) \|_{\cX_2}
        \leq 
        \| \cG_2(\cG_1(a)) - \cG_2(\cN_1(a)) \|_{\cX_2} +
        \| \cG_2(\cN_1(a)) - \cN_2(\cN_1(a)) \|_{\cX_2}.
    \]
    Since $\cG_2$ is continuous, it is uniformly continuous over the compact set ${B}_{R+1}(0) := \{\|a\|_{\cX_1} \leq R + 1\}$, such that 
    \[ 
    \| \cG_2( \cG_1(a)) - \cG_2 ( \cN_1(a)) \|_{\cX_2} \leq 
    \omega_{\cG_2,0}(\|\cG_1(a) - \cN_1(a) \|_{\cX_1}) 
    \leq \omega_{\cG_2,0}(\eta)
    \]
    On the other hand, the second term is bounded by $\eta$ such that the overall error is bounded as 
    \begin{equation}\label{eq:composition_function_error}
    \| \cG_2( \cG_1(a)) - \cG_2 ( \cN_1(a)) \|_{\cX_2} \leq \omega_{\cG_2,0}(\eta) + \eta.
    \end{equation}

    The case for the derivative error is similar. Applying the chain rule and the triangle inequality, we obtain
    \begin{align*}
    & \|D(\cG_2 \circ \cG_1)(a) - D(\cN_2 \circ \cN_1)(a)\|_{\Op(\cX_0,\cX_2)} \\
    & \quad = \| D \cG_2(\cG_1(a)) D \cG_1(a) -  D \cN_2(\cN_1(a)) D \cN_1(a)\|_{\Op(\cX_0,\cX_2)}\\
    & \quad \leq 
    \underbrace{\|D \cG_2(\cG_1(a)) D \cG_1(a) - D \cG_2(\cN_1(a)) D \cG_1(a)\|_{\Op(\cX_0,\cX_2)}}_{(\mathrm{I})}
    \\
    & \qquad + \underbrace{\|D \cG_2(\cN_1(a)) D \cG_1(a) - D \cG_2(\cN_1(a)) D \cN_1(a)\|_{\Op(\cX_0,\cX_2)}}_{(\mathrm{II})} \\
    & \qquad + \underbrace{\|D \cG_2(\cN_1(a)) D \cN_1(a) - D \cN_2(\cN_1(a)) D \cN_1(a)\|_{\Op(\cX_0,\cX_2)}}_{(\mathrm{III})}.
    \end{align*}
    We consider three different contributions separately. 
    For the first term (I), 
    we have 
    \[
    \|D \cG_2(\cG_1(a)) D \cG_1(a) - D \cG_2(\cN_1(a)) D \cG_1(a)\|_{\Op(\cX_0,\cX_2)}
    \leq M_1 
    \|D \cG_2(\cG_1(a)) - D \cG_2(\cN_1(a))\|_{\Op(\cX_1,\cX_2)},
    \] 
    where $M_1 = \sup_{a \in K_0} \|D \cG_1(a)\|_{\Op(\cX_0,\cX_1)}$.
    Similar to the previous case, we use the uniform continuity of $D\cG_2$ over ${B}_R(0)$ to obtain
    \[
   \|D \cG_2(\cG_1(a)) D \cG_1(a) - D \cG_2(\cN_1(a)) D \cG_1(a)\|_{\Op(\cX_0,\cX_2)} \leq M_1 \omega_{\cG_2, 1}(\|\cG_1(a) - \cN_1(a)\|_{\cX_1}) \leq M_1 \omega_{\cG_2,1}(\eta),
    \]
    where $\omega_{\cG_2, 1}$ is the modulus of continuity of $D \cG_2$ over ${B}_R(0)$.
    For the second term (II), we have 
    \[
    \|D \cG_2(\cN_1(a)) D \cG_1(a) - D \cG_2(\cN_1(a)) D \cN_1(a)\|_{\Op(\cX_0,\cX_2)} 
    \leq M_2
       \|D \cG_1(a) - D \cN_1(a) \|_{\Op(\cX_0,\cX_1)} 
       \leq M_2 \eta,
       \]
    where $M_2 = \sup_{v \in {B}_R(0)} \|D \cG_2(v)\|_{\Op(\cX_1,\cX_2)}$.
    Finally, for the third term (III), 
    \begin{align*}
    & \|D \cG_2(\cN_1(a)) D \cN_1(a) - D \cN_2(\cN_1(a)) D \cN_1(a)\|_{\Op(\cX_0,\cX_2)} \\
    & \quad \leq
    \|D \cG_2(\cN_1(a)) - D \cN_2(\cN_1(a)) \|_{\Op(\cX_1,\cX_2)}
    \|D \cN_1(a)\|_{\Op(\cX_0,\cX_1)} \\
    & \quad \leq 
    \|D \cG_2(\cN_1(a)) - D \cN_2(\cN_1(a)) \|_{\Op(\cX_1,\cX_2)}
    (\|D \cG_1(a)\|_{\Op(\cX_0,\cX_1)} + \| D \cG_1(a) - D \cN_1(a)\|_{\Op(\cX_0,\cX_1)}) \\
    & \quad \leq 
    \eta(M_3 + \eta),
     \end{align*}
    where $M_3 = \sup_{a \in K_0}\|D \cG_1(a)\|_{\Op(\cX_0,\cX_1)}$.
    Thus, overall, we have 
    \begin{equation}\label{eq:composition_derivative_bound}
    \|D(\cG_2 \circ \cG_1)(a) - D(\cN_2 \circ \cN_1)(a)\|_{\Op(\cX_0,\cX_2)} \leq M_1 \omega_{\cG_2,1}(\eta) + M_2 \eta + M_3\eta + \eta^2.
    \end{equation}
    It suffices pick $\cN_1$ and $\cN_2$ corresponding to a sufficiently small $\eta$ such that \eqref{eq:composition_function_error} and \eqref{eq:composition_derivative_bound} are simultaneously bounded by $\epsilon$.

    The results directly translate to the case where derivative operators are measured by their $\HS$ norms.
    In particular, since $\cX_1$ is finite dimensional, $\HS(\cX_0,\cX_1)$ and $\HS(\cX_1,\cX_2)$ are equivalent to $\Op(\cX_0,\cX_1)$ and $\Op(\cX_1,\cX_2)$.
    Moreover, for $\cA_1 \in \Op(\cX_0,\cX_1)$ and $\cA_2 \in \Op(\cX_1,\cX_2)$,
    \[
        \|\cA_2 \circ \cA_1 \|_{\HS(\cX_0,\cX_2)} 
            \lesssim \|\cA_2 \|_{\HS(\cX_1,\cX_2)} \|\cA_1 \|_{\HS(\cX_0,\cX_1)} 
            \lesssim \|\cA_2 \|_{\Op(\cX_1,\cX_2)} \|\cA_1 \|_{\Op(\cX_0,\cX_1)}.
    \]            
    Thus, all the calculations in the proof can be repeated using $\HS$ instead of $\Op$, up to equivalence constants.
\end{proof}

\subsection{Smooth Approximations of the Identity}
In proving \Cref{theorem:main_ua}, we will make use of smooth approximations of the identity operator $\cI (v) = v$ on function spaces. 
These can be constructed using FNOs with a single local layer. 
We first consider the smooth approximation of the identity function on $\bR$, i.e. $\id(x) = x$.  
We will use the form found in \cite{kovachki2021universal}.
\begin{lemma}[Smooth approximation of identity function]\label{lemma:identity_approximation_scalar}
    Let $\sigma \in \cont^{\infty}(\bR)$ be non-constant and let $x_0 \in \bR$ be such that $\sigma'(x_0) \neq 0$.
    Then, for any $k \in \bN$ and $\epsilon, R > 0$, there exists $\theta \in (0, 1]$ for which
    \begin{equation}\label{eq:identity_approx}
        \id_\theta(x) = \frac{\sigma(x_0 + \theta x) - \sigma(x_0 - \theta x) }{2 \theta \sigma'(x_0)}    
    \end{equation}
    is an approximation to the identity function $\id : \bR \rightarrow \bR$, $\id(x) = x$, such that 
    \begin{equation}
        \| \id_\theta - \id \|_{\cont^k([-R, R])} \leq \epsilon.
    \end{equation}
\end{lemma}
\begin{proof}
    The proof follows from the argument used in \cite[Lemma 6]{kovachki2021universal}.
    We consider a point $x_0 \in \bR$ such that $\sigma'(x_0) \neq 0$ and $\theta \in (0, 1]$.
    Note that $\id_\theta$ corresponds to a neural network with a single hidden layer of width two.
    Moreover, $\id_\theta \in \cont^{\infty}(\bR)$, 
    with derivatives of the form 
    \begin{equation}
        \id_\theta^{(k)}(x)= \theta^{k-1} \left( \frac{\sigma^{(k)}(x_0 + \theta x) + (-1)^{k-1} \sigma^{(k)}(x_0 - \theta x)}{2 \sigma'(x_0)} \right), \qquad k \in \bN.
    \end{equation}
    
    The approximation is constructed such that $\id_\theta (0) = 0$ and $\id_\theta'(0) = 1$, i.e., it is a good approximation of the identity function near the origin. 
    To see this, we apply Taylor's theorem,
    \[
        \id_\theta(x) = \id_\theta(0) + \id_\theta'(0) x + \frac{1}{2}\id_\theta''(\xi) x^2
        = x +  \frac{1}{2}\id_\theta''(\xi) x^2
        ,
    \]
    where $\xi \in [0, x]$. Substituting in the expression for the second derivative, we arrive at the following bound,
    \[
        | \id_\theta(x) - x | \leq \theta \left| \frac{\sigma''(x_0 + \theta \xi) - \sigma''(x_0 - \theta \xi)}{2 \sigma'(x_0)} \right| \frac{x^2}{2}.
    \] 
    Since $\theta \leq 1$, we can take 
    \[
        M_0 = M_0(\sigma, R) =  \sup_{\xi \in [-R, R]} \left| \frac{\sigma''(x_0 + \xi) - \sigma''(x_0 - \xi)}{2 \sigma'(x_0)} \right| \frac{R^2}{2},
    \] 
    arriving at 
    \begin{equation}
        |\id_\theta(x) - x| \leq M_1 \theta \qquad \forall x \in [-R, R].
    \end{equation}
    Moreover, we consider an analogous argument using the Taylor expansion for the derivative $\id_\theta'(x)$ about $x = 0$, 
    i.e.,
    \[
        \id_\theta'(x) = 1 + \id_\theta''(\xi) x = 1 + \theta \frac{\sigma''(x_0 + \theta\xi) - \sigma''(x_0 - \theta\xi)}{2 \sigma'(x_0)}x.
    \]
    Taking 
    \[
        M_1 = M_1(\sigma, b)  = 
        \sup_{\xi \in [-R, R]} 
        \left|
        \frac{\sigma''(x_0 + \xi) - \sigma''(x_0 - \xi)}{2 \sigma'(x_0)} 
        \right| R,
    \]
    we arrive at a similar result for the derivative,
    \begin{equation}
        |\id_\theta'(x) - 1| \leq M_1 \theta \qquad \forall x \in [-R, R].
    \end{equation}

    For higher order derivatives, we note that $f(x) = x$ has $f^{(k)}(x) = 0$ for any $k \geq 2$.
    On the other hand, higher order derivatives of $\id_\theta$ can be made arbitrarily small, since by the derivative expression, for $k \geq 2$, we have
    \begin{equation}
        |\id_\theta^{(k)}(x)| \leq M_{k}(\sigma, R) \theta^{k-1},
    \end{equation}
    where 
    \[
        M_{k}(\sigma, R) = \sup_{x \in [-R, R]} \left| \frac{\sigma^{k}(x_0 + \theta x) + (-1)^{k-1} \sigma^{k}(x_0 - \theta x)}{2 \sigma'(x_0)} \right|.
    \]
    Thus, the approximation error $\id_\theta (x) - x$ can be made arbitrarily small in $\cont^k([-R,R])$ for any $k \geq 0$ by reducing $\theta$ as desired.
    
\end{proof}

We can extend this to identity mappings on function spaces, $\cI(v) = v$,
by applying identity approximation \eqref{eq:identity_approx} pointwise as a superposition operator, i.e., $\cI_{\theta}(v))(x) = \id_{\theta}(v(x))$.
This yields a smooth approximation to the identity operator for the space $L^{\infty}(\bT^d)$.
\begin{lemma}[Identity approximation in $L^{\infty}$]\label{lemma:identity_approx_Linfty}
    Let $\sigma \in \cont^{\infty}(\bR)$ be non-constant. 
    Then, for any $\epsilon, R > 0$, there exists a single layer FNO $\cI_\theta$ with activation function $\sigma$ approximating the identity operator 
    $\cI : L^{\infty}(\bT^d) \rightarrow L^{\infty}(\bT^d)$, $\cI(v) = v$, such that
    \begin{equation}
        \|  \cI_\theta(v) - v \|_{L^{\infty}} \leq \epsilon
    \end{equation}
    and 
    \begin{equation}
        \|  D \cI_\theta(v) - \cI \|_{\Op(L^{\infty}, L^{\infty})} \leq \epsilon
    \end{equation}
    for all $v \in L^{\infty}(\bT^d)$ with $\|v\|_{L^{\infty}} \leq R$.
\end{lemma}
\begin{proof}
    Let $R > 0$ be given, and let $\id_\theta(x)$, $\theta = \theta(\epsilon)$, be an approximation of the identity function of the form \eqref{eq:identity_approx} such that $\|\id_\theta - \id \|_{\cont^{1}([-R,R])} \leq \epsilon$.
    Then, we define the FNO $\cI_\theta : L^{\infty}(\bT^d) \rightarrow L^{\infty}(\bT^d)$ as the single-layer FNO $\cI_\theta(v)(x) = \id_\theta(v(x))$. 
    Since $\id_\theta(x)$ has continuous derivatives of arbitrary order, 
    \Cref{lemma:identity_approx_Linfty} implies that as a superposition operator, 
    $\cI_h$ is continuously differentiable, and its derivative is given by 
    \begin{equation}\label{eq:derivative_of_nemytskii}
        (D \cI_\theta(v)w)(x) = \id_\theta'(v(x)) w(x), \quad x \in \bT^d.
    \end{equation}
    Moreover, when $\|v \|_{L^{\infty}} \leq R$, we have
    $
        | \id_\theta(v(x)) - v(x) | \leq \epsilon 
    $
    for all $x \in \bT^d$,
    such that $\|\cI_\theta(v) - v \|_{L^{\infty}} \leq \epsilon$.
    For the derivatives,
    it suffices to check that for any $w \in L^{\infty}(\bT^d)$,
    \begin{align*}
        \| D\cI_\theta(v) w - w \|_{L^{\infty}} = 
        \| \left( \id_\theta'(v(\cdot)) - 1 \right) w \|_{L^{\infty}}
        \leq \| \id_\theta'(v(\cdot)) - 1 \|_{L^{\infty}} \|w\|_{L^\infty}
        \leq \epsilon \|w\|_{L^\infty}.
    \end{align*}
    That is,  
    $
    \|D\cI_\theta - \cI \|_{ \Op(L^{\infty}, L^{\infty})} \leq \epsilon.
    $ 
\end{proof}

An analogous result holds for identity mappings between Sobolev spaces. 
For this, we will prove a case sufficient for our purposes, making strong assumptions about the orders of the Sobolev spaces.

\begin{lemma}[Identity approximation in Sobolev spaces]\label{lemma:identity_approximation_sobolev}
    Let $\sigma \in \cont^{\infty}(\bR)$ be non-constant and $m$ be an integer satisfying $m > d/2$. 
    Then, for any $\epsilon, R > 0$, there exists a single layer FNO $\cI_\theta$ with activation function $\sigma$ 
    approximating the identity (inclusion) operator $\cI : H^{2m}(\bT^d) \rightarrow H^{m}(\bT^d)$ such that 
    \begin{equation}
        \|  \cI_\theta(v) - v \|_{H^m} \leq \epsilon
    \end{equation}
    and 
    \begin{equation}
        \|  D \cI_\theta(v) - \cI \|_{\Op(H^{2m}, H^m)} \leq \epsilon
    \end{equation}
    for all $v \in H^{2m}(\bT^{d})$ with $\|v\|_{H^{2m}} \leq R$.
\end{lemma}
\begin{proof}
    For $m > d/2$, we have the continuous embedding of 
    $H^{2m}(\bT^d)$ into  $\cont^{m}(\bT^d)$ with 
    $\|v \|_{\cont^{m}} \leq \const_H \|v\|_{H^{2m}}$, $\const_{H} = \const_{H}(d,m)$.
    This allows us to work directly with continuously differentiable input functions.
    To construct $\cI_{\theta}$, we consider an arbitrary $\epsilon_{\id} < 1$ and $\id_\theta(x)$, $\theta = \theta(\epsilon_{\id})$ taking the form of \eqref{eq:identity_approx} with 
    \[\|\id_\theta - \id\|_{\cont^{m+1}([-\const_{H} R, \const_{H} R])} \leq \epsilon_{\id}.\]
    Note that by \Cref{lemma:superposition_sobolev}, as a superposition operator, $\cI_\theta : H^{2m}(\bT^d) \rightarrow H^{m}(\bT^d)$ is continuously differentiable with derivative of the form \eqref{eq:derivative_of_nemytskii}.

    \paragraph{Output Error} We first estimate the error in the outputs, as measured by the $H^m$ norm. 
    For integer values of $m$, this is equivalent to
    \[
        \| \id_\theta(v) - v \|_{H^m}^2 = \| \id_\theta \circ v - v \|_{L^2}^2 + \sum_{k=1}^{m} \sum_{|{\alpha}|_1 = k } \| \partial^{\alpha}(\id_\theta \circ v - v) \|_{L^2}^2.
    \]
    
    \begin{enumerate}[label=(\roman*)]
    \item For the ${L^2}$ component, 
    since $\|v\|_{L^{\infty}} \leq \const_{H} R $,
    we have 
    \[ \|\id_\theta \circ v - v \|_{L^2} \leq \| \id_{\theta} \circ v - v \|_{L^{\infty}} \leq \const_{2,\infty}  \epsilon_{\id},
    \] 
    where $\const_{2,\infty} = \const_{2,\infty}(d)$ is the constant for the embedding $L^{\infty}(\bT^d) \hookrightarrow L^2(\bT^d)$, i.e., $\|v\|_{L^2} \leq \const_{2,\infty} \|v\|_{L^{\infty}}$.
    \item For the first derivative, we have $\partial_{x_j} (\id_{\theta} \circ v(x)) = \id_\theta'(v(x)) \partial_{x_j} v(x) $, such that 
    \[
        \| \partial_{x_j}(\id_{\theta} \circ v) - \partial_{x_j} v \|_{L^2}
        \leq \| (\id_{\theta}'(v(\cdot)) - 1) \|_{L^{\infty}} \| \partial_{x_j} v \|_{L^2} \leq R \epsilon_{\id}.
    \]
    \item For the higher order derivatives, we recall the Fa\`a di Bruno formula \cite{Hardy06} for the $k$th derivative of a composition of functions. 
    That is, for $f : \bR^{d} \rightarrow \bR$ and $g : \bR \rightarrow \bR$, this takes the form 
    \[
    \frac{\partial^{k} (g \circ f)}{\partial x_{i_1} \dots \partial x_{i_k}} = 
    \sum_{\pi \in \Pi(k)} g^{(|\pi|)}(f) \prod_{E \in \pi} \frac{\partial^{|E|} f}{\prod_{j \in E} \partial x_{i_j}},
    \]
    where $\Pi(k)$ is the set of partitions of $\{1, \dots, k\}$ and $|\pi|, |E|$ refer to the number of elements in $\pi$ and $E$, respectively.
    Note that there is only one partition with a single element, namely $\pi_k := \{ \{1, \dots, k\}\}$,
    which yields the term 
    $
    g'(f) \frac{\partial^k f}{\partial x_{i_1} \dots \partial x_{i_k}}.
    $
    We can therefore decompose the sum as 
    \[
    \frac{\partial^{k} (g \circ f)}{\partial x_{i_1} \dots \partial x_{i_k}} = 
    g'(f) \frac{\partial^k f}{\partial x_{i_1} \dots \partial x_{i_k}} +
    \sum_{\pi \in \Pi \setminus \{\pi_k\}} g^{(|\pi|)}(f) \prod_{E \in \pi} \frac{\partial^{|E|} f}{\prod_{j \in E} \partial x_{i_j}}.
    \]
    Using $\id_\theta = g$ and $v = f$, we have 
    \[
    \frac{\partial^{k} (\id_\theta \circ v - v)}{\partial x_{i_1} \dots \partial x_{i_k}}(x) = 
    \left( \id_{\theta}'(v(x)) - 1 \right) \frac{\partial^k v}{\partial x_{i_1} \dots \partial x_{i_k}}(x) +
    \sum_{\pi \in \Pi \setminus \{ \pi_k\}} 
    \id_\theta^{(|\pi|)}(v(x)) \prod_{E \in \pi} \frac{\partial^{|E|} v}{\prod_{j \in E} \partial x_{i_j}} (x).
    \]
    In particular, 
    \[
    \left\| \left( \id_{\theta}'(v(\cdot)) -1 \right) \frac{\partial^k v}{\partial x_{i_1} \dots \partial x_{i_k}} \right\|_{L^2}
    \leq 
    \left\| \left( \id_{\theta}'(v(\cdot)) -1 \right) \right\|_{L^{\infty}}  \left\| \frac{\partial^k v}{\partial x_{i_1} \dots \partial x_{i_k}} \right\|_{L^2}
    \leq
    R \epsilon_{\id} 
    \]
    and 
    \begin{align*}
    \left\| \id_\theta^{(|\pi|)}(v(x)) \prod_{E \in \pi} \frac{\partial^{|E|} v}{\prod_{j \in E} \partial x_{i_j}} (x) \right\|_{L^2}
    &\leq \const_{2,\infty} \left\| \id_\theta^{(|\pi|)}(v(x)) \right\|_{L^{\infty}} 
    \prod_{E \in \pi} \left\| \frac{\partial^{|E|} v}{\prod_{j \in E} \partial x_{i_j}} (x) \right\|_{L^{\infty}} \\
    & \leq \const_{2,\infty} R^{|\pi|} \epsilon_{\id}.
    \end{align*}
    \end{enumerate}
    Summing over the individual contributions, we have $\|\id_\theta(v) - v \|_{H^m} \leq \const_0 \epsilon_{\id}$, where $\const_0 = \const_0(d, m, R)$ is a positive constant.

    \paragraph{Derivative Error} For the derivative of the identity operator, we have
    \begin{align*}
        \| D\cN(v) w - w \|_{H^{m}} &= 
        \| \left( \id_\theta'(v(\cdot)) - 1 \right) w \|_{H^m} \\
        & \leq \const_m \| \id_\theta'(v(\cdot)) - 1 \|_{H^m} \| w\|_{H^m} \\
        & \leq \const_m \| \id_\theta'(v(\cdot)) - 1 \|_{H^m} \| w\|_{H^{2m}},
    \end{align*}
    where we have made use of the Banach algebra property of $H^m(\bT^d)$ (\Cref{theorem:banach_algebra}) when $m > d/2$.
    Since $\|\id_{\theta} - \id\|_{\cont^{m+1}([-\const_H R, \const_H R])} \leq \epsilon_{\id}$, 
    we also have that $\|\id_{\theta}' - 1 \|_{\cont^{m}([-\const_H R, \const_H R])} \leq \epsilon_{\id}$.
    Thus, by analogous reasoning to before, we can show $\| \id_\theta'(v(\cdot)) - 1 \|_{H^m} \lesssim \epsilon_{\id}$, such that $\|D \cN(v) - \cI \|_{\Op(H^{2m}, H^m)} \leq \const_{1} \epsilon_{\id} $ for some $\const_1 = \const_1(d,m,R) $
    
    The overall approximation error can be made small by choosing a sufficiently small $\epsilon_{\id}$ for which $\id_\theta$ is constructed.

\end{proof}

\subsection{Truncation in Fourier Modes}
Here we consider the truncation of an operator $\cG$ onto the Fourier modes $|k| \leq N$, $\cG_N = \cP_N \circ \cG \circ \cP_N$.
We show that for sufficiently large $N$, 
the truncation error can be made arbitrarily small simultaneously in the operator $\| \cG(a) - \cG_N(a) \|$ and its derivative $\|D \cG(a) - D\cG_N(a)\|$.

\begin{lemma}[Truncation error]\label{lemma:truncation}
    Let $s, s', \delta \geq 0$ and let 
    $\cX = H^s(\bT^{d}; \bR^{d_a})$, 
    $\cY = H^{s'}(\bT^{d}; \bR^{d_u})$, and
    $\cX_{\delta} = H^{s+\delta}(\bT^d; \bR^{d_a})$.
    Suppose $\cG : \cX \rightarrow \cY$ is continuously differentiable and $D\cG: \cX \rightarrow \HS(\cX_{\delta}, \cY)$ is well-defined and continuous. 
    Then, for any compact set $K \subset \cX$ and $\epsilon > 0$, 
    there exists $N_{\epsilon} \in \bN$ such that 
    \begin{equation}\label{eq:function_truncation}
        \sup_{a \in K} \|\cG(a) - \cP_N \cG(\cP_N a) \|_{\cY} \leq \epsilon
    \end{equation}
    and 
    \begin{equation}\label{eq:derivative_truncation} 
    \sup_{a \in K} \|D \cG(a) - \cP_N D\cG(\cP_N a) \cP_N \|_{\HS(\cX_{\delta}, \cY)} \leq \epsilon 
    \end{equation}
    whenever $N \geq N_{\epsilon}$. 
\end{lemma}
\begin{proof}
    We begin with $d_a = d_u = 1$.
    The case of the function values (i.e.~\eqref{eq:function_truncation}) is considered in \cite{kovachki2021universal}, though we will present the proof for both \eqref{eq:function_truncation} and \eqref{eq:derivative_truncation} here for completeness. 
    Let $K$ and $\epsilon$ be given. Then, the set 
    \[
        \widetilde{K} := K \cup \bigcup\limits_{N \in \bN} \cP_N K
    \]
    is compact, where $\cP_N K $ is the image of $K$ under the Fourier projection $\cP_N$.
    In the case of the function value, the error satisfies
    \[
        \| \cG(a) - \cP_N \cG(\cP_N a) \|_{\cY} \leq 
        \| \cG(a) - \cP_N \cG(a) \|_{\cY} + \| \cG(a) - \cG(\cP_N a) \|_{\cY},
    \]
    which is obtained by applying the triangle inequality and noting that $\cP_N$ is an orthogonal projection on $H^{s}(\bT^d)$ for any $s \geq 0$.
    Since $\cG$ is continuous, it is uniformly continuous over the compact set $\widetilde{K}$, such that 
    \[
        \| \cG(a) - \cG(\cP_N a) \|_{\cY} \leq \omega_0(\| a - \cP_N a\|_{\cX}) \qquad \forall a \in {K}
    \]
    where $\omega_0 : [0, \infty] \rightarrow [0, \infty]$ is the modulus of continuity of $\cG$ satisfying $\omega_0(x) \rightarrow 0$ as $x \rightarrow 0$.
    It remains to show that one can uniformly control the projection errors $\|(I - \cP_N) a\|_{\cX}$
    and 
    $\|(I - \cP_N)\cG(a)\|_{\cY}$ over $K$.
    To see this, note that the functions $f_N : \cX \rightarrow \bR$ given by
    \[
        f_{N}(a) := \|(I - \cP_N)a\|_{\cX} =  \left( \sum_{|k|_{\infty} > N} (1 + |k|^2)^{s} |\hat{a}(k)|^2 \right)^{1/2}
    \]
    are continuous, and converge to zero pointwise and monotonically as $N \rightarrow \infty$. 
    Thus, $f_N \rightarrow 0$ uniformly over the compact set ${K}$.
    Analogous reasoning implies the uniform convergence of 
    \[
        \tilde{f}_{N}(a) := \|(I - \cP_N)\cG(a)\|_{\cY} =  \left( \sum_{|k|_{\infty} > N} (1 + |k|^2)^{s'} |\hat{u}(k)|^2 \right)^{1/2}, \qquad u = \cG(a)
    \]
    to zero over $a \in {K}$. 
    Together, these two results yields an $N_{\epsilon,0}$ for which \eqref{eq:function_truncation} holds
    whenever $N \geq N_{\epsilon,0}$.

    The case of the derivatives can be shown similarly. 
    Similar to the previous case, we can decompose the derivative truncation error as
    \begin{align*}
        & \| D \cG(a) - \cP_N D \cG(\cP_N a) \cP_N \|_{\HS(\cX_{\delta},\cY)}   \\
        & \quad \leq  \| D \cG(a) - \cP_N D \cG(a) \|_{\HS(\cX_{\delta},\cY)}
        + \| \cP_N D \cG(a) - \cP_N D \cG(a) \cP_N \|_{\HS(\cX_{\delta},\cY)} \\ 
        & \qquad + \| \cP_N D \cG(a) \cP_N - \cP_N D \cG(\cP_N a) \|_{\HS} \\
        & \quad \leq  \| (I - \cP_N) D \cG(a) \|_{\HS(\cX_{\delta},\cY)} 
        + \| D \cG(a)(I - \cP_N)\|_{\HS(\cX_{\delta},\cY)}
        + \| D \cG(a) - D \cG(\cP_N a) \|_{\HS(\cX_{\delta},\cY)},
    \end{align*}
    where we have again used the fact that $\cP_N$ is an orthogonal projection on both $\cX_{\delta} = H^{s+\delta}(\bT^d)$ and $\cY = H^{s'}(\bT^d)$.
    We begin by considering $\|D\cG(a) - D \cG(\cP_N a)\|_{\HS}$. By the uniform continuity of $\cX \ni a \mapsto D\cG(a) \in \HS(\cX_{\delta}, \cY)$ over $\widetilde{K}$, we have 
    \[
        \|D \cG(a) - D \cG(\cP_N a)\|_{\HS} \leq \omega_1(\|a - \cP_N a\|_{\cX}) \qquad \forall a \in K, 
    \]
    where $\omega_1 : [0, \infty] \rightarrow [0, \infty]$ is the modulus of continuity for $D\cG$.
    Recall that due to the monotone convergence of $\|a - \cP_N a\|_{\cX} \rightarrow 0$ as $N \rightarrow \infty$, we can make $\omega_1(\|a - \cP_N a\|_{\cX})$ arbitrarily small for all $a \in K$.

    We now proceed to bound the derivative truncation error on the input side, $\| D \cG(a)(I - \cP_N)\|_{\HS}$. 
    Recall that $\psi_0$, $\psi_{s, k}^{\cos}$, and $\psi_{s,k}^{\sin}$
    with $k \in \zlex$ forms an orthonormal basis of $H^s(\bT^d)$ for any $s \geq 0$.
    Using this basis for the Hilbert--Schmidt norm, we have
    \[
    \| D \cG(a)(I - \cP_N)\|_{\HS(\cX_{\delta},\cY)} 
    = 
    \left(
    \sum_{k \in \zlex, |k|_{\infty} > N} 
    \|D \cG(a) \psi_{s+\delta,k}^{\cos}\|_{\cY}^2 
    + \|D \cG(a) \psi_{s+\delta,k}^{\sin}\|_{\cY}^2 
    \right)^{1/2}.
    \]
    Since $D \cG(a)$ is Hilbert--Schmidt, 
    the norms $\| D \cG(a)(I - \cP_N)\|_{\HS(\cX_{\delta},\cY)}$
    converge monotonically to zero as $N \rightarrow \infty$.
    Moreover, the continuity of $D \cG$ implies that $a \mapsto \| D \cG(a)(I - \cP_N)\|_{\HS(\cX_{\delta},\cY)}$ is continuous and hence $\| D \cG(a)(I - \cP_N)\|_{\HS(\cX_{\delta},\cY)} \rightarrow 0$ uniformly in $a \in \widetilde{K}$ as $N \rightarrow \infty$.

    For the derivative truncation error on the output side, we have 
    \[
    \| (I - \cP_N) D \cG(a) \|_{\HS(\cX_{\delta}, \cY)}  = \| D \cG(a)^* (I - \cP_N) \|_{\HS(\cY, \cX_{\delta})}.
    \]
    Note that since
    \[\|D \cG(a)^* \|_{\HS(\cY, \cX_\delta)} 
        = \|D \cG(a)\|_{\HS(\cX_{\delta}, \cY)}
    ,\]
    the mapping $a \mapsto D\cG(a)^*$ is also continuous. 
    Thus, by an analogous argument, this time using $\psi_0, \psi_{s',k}^{\cos}, \psi_{s', k}^{\sin}$ as the basis for $\cY = H^{s'}(\bT^d)$, 
    we also have the uniform convergence of $\|(I - \cP_N) D \cG(a)\|_{\HS(\cX_{\delta},\cY)} \rightarrow 0$ as $N \rightarrow \infty$ over $a \in \widetilde{K}$.
    Thus, we can choose $N_{\epsilon,1}$ such that
   \[ \sup_{a \in K} \| (I - \cP_N) D \cG(a) \|_{\HS(\cX_{\delta},\cY)} 
        + \| D \cG(a)(I - \cP_N)\|_{\HS(\cX_{\delta},\cY)}
        + \| D \cG(a) - D \cG(\cP_N a) \|_{\HS(\cX_{\delta},\cY)} \leq \epsilon \]
    for all $N \geq N_{\epsilon,1}$.
    Taking $N_{\epsilon} = \max(N_{\epsilon,0}, N_{\epsilon,1})$ yields the desired result.
    Note that for $d_a, d_u > 1$, one can follow the same procedure using the sinusoidal basis for each of the $d_a$ or $d_u$ components.
\end{proof}

\subsection{Approximation of the Truncated Operator}
Next, we show that for any fixed $N$, the truncated operator $\cG_N$ can be arbitrarily well approximated by an FNO. 
\begin{lemma}\label{lemma:truncated_operator}
    Let $s, s', \delta \geq 0$ and let 
    $\cX = H^s(\bT^{d}; \bR^{d_a})$, 
    $\cY = H^{s'}(\bT^{d}; \bR^{d_u})$, and
    $\cX_{\delta} = H^{s+\delta}(\bT^d; \bR^{d_a})$.
    Suppose $\cG : \cX \rightarrow \cY$ is continuously differentiable and 
    $D\cG: \cX \rightarrow \HS(\cX_{\delta}, \cY)$ is well-defined and continuous. 
    Furthermore, for a fixed $N \in \bN$, let $\cG_{N} := \cP_N \circ \cG \circ \cP_N$.
    Then, given $\sigma \in \smootha$, for any compact set $K \subset \cX$ and $\epsilon > 0$
    there exists an FNO $\cN \in \cont^1(\cX;\cY)$ with activation function $\sigma$ such that 
    \begin{equation}
        \sup_{a \in K} \| \cG_N(a) - \cN(a) \|_{\cY} \leq \epsilon
    \end{equation}
    and 
    \begin{equation}
        \sup_{a \in K} \| D\cG_N (a) - D \cN(a) \|_{\HS(\cX_{\delta}, \cY)} \leq \epsilon.
    \end{equation}
\end{lemma}

Before proving \Cref{lemma:truncated_operator}, we prove two supporting results, \Cref{lemma:ft_approximation} and \Cref{lemma:inverse_ft_approximation}, which are the $\cont^1$ analogies to Lemma 7 and Lemma 8 of \cite{kovachki2021universal}.
That is, we show that the operators $\cF_N \circ \cP_N$ and $\cP_N \circ \real \cF_N^{-1}$ can be well approximated by FNOs in terms of both its output values and its first derivative.
For $\cF_N \circ \cP_N$, we have the following result.
\begin{lemma}[Approximation of the Fourier transform in $\cont^1$]\label{lemma:ft_approximation}
    Let $N \in \bN$ and $s \geq 0$. 
    Then, given $\sigma \in \smootha$, for any compact set $K \subset H^s(\bT^d)$ and $\epsilon > 0$, there exists an FNO 
    $\cN \in \cont^1(H^s(\bT^d) ; L^2_{0}(\bT^d;\bR^{2 K_N}))$
    with activation function $\sigma$  
    such that 
    \begin{equation}
        \sup_{a \in K} \| \cF_{N} \circ P_{N}(a) - \cN (a) \|_{L^2_{0}(\bT^d;\bR^{2 K_N})} \leq \epsilon
    \end{equation}
    and 
    \begin{equation}
        \sup_{a \in K} \| D (\cF_{N} \circ P_{N})(a) - D \cN (a) \|_{\Op(H^s(\bT^d; \bR),L^2_0(\bT^d; \bR^{2 K_N}))} \leq \epsilon.
    \end{equation}
    Note that $\cN$ maps to output functions that are constant with respect to $x \in \bT^d$. 
\end{lemma}
\begin{proof}
    We verify that the approximation strategy in \cite[Lemma 7]{kovachki2021universal} holds also for derivatives.
    Let $\cT_{l,k} : H^s(\bT^{d}; \bR) \rightarrow \bR$, $l \in \{1, 2\}$, $|k|_{\infty} \leq N$ denote each component of the mapping $\cF_{N} \circ \cP_N$, 
    such that $\real(\hat{a}_{k}) = \cT_{1,k}(\cP_N a)$ and $\imag(\hat{a}_k) = \cT_{2,k}(\cP_N a)$. 
    Note that we then identify the output space $\bR$ with the space of constant functions on $\bT^d$, i.e.~$L^2_0(\bT^d; \bR)$.
    Additionally, since $K$ is compact, there exists $R > 0$ such that $\|a\|_{H^{s}} \leq R$ for all $a \in K$.

    We demonstrate how each component $\cT_{l,k}$ can be approximated by an FNO $\cN_{l,k}$.
    We focus on the case of $l = 1$ (i.e., consider the real component), 
    such that $\cT_{l,k}$ is represented by the following sequence of mappings,
    \begin{equation*}
        a(x) 
        \underset{\cS_1}{\longrightarrow}
        \begin{pmatrix}
            \cP_N a(x) \\ 
            \cos(\langle k, x \rangle)
        \end{pmatrix}
        \underset{\cS_2}{\longrightarrow}
            \cP_N a(x) \cos(\langle k, x \rangle)
        \underset{\cS_3}{\longrightarrow}
           \int \cP_N a(x) \cos(\langle k, x\rangle) dx,
    \end{equation*}
    which maps between the spaces 
    \begin{equation*}
        H^s(\bT^d; \bR)  
        \underset{\cS_1}{\longrightarrow}
        L^2_N(\bT^d; \bR^2) \hookrightarrow L^\infty(\bT^d; \bR^2) 
        \underset{\cS_2}{\longrightarrow}
        L^{\infty}(\bT^d; \bR^2) 
        \underset{\cS_3}{\longrightarrow}
        \bR \cong L^2_{0}(\bT^{d}; \bR).
    \end{equation*}
    Analogous constructions hold for $l = 2$, 
    where $\cos(\langle k, x\rangle)$ is replaced by $\sin(\langle k, x \rangle)$.

    The operation $\cS_{1}$ is represented by a linear FNO (an FNO without the activation function on the final layer), 
    $\cN_1 = \cL_1 \circ \cR_1$, 
    consisting of a lifting layer
    \[\cR_1(a)(x) = (a(x), 0) \]
    followed by the linear layer 
    \[\cL_{1} (v)(x) = W v(x) + b(x) + (\cF^{-1} P_1 \odot \cF v)(x)
    \]
    where $W \equiv 0$ and $b(x) = (0, \cos(\langle k, x\rangle))$,
    and $P_1(k) = I_{2 \times 2}$ if $|k|_{\infty} \leq N$ and is zero otherwise.

    The subsequent mapping, $\cS_2 : L^2_{N}(\bT^d; \bR^{2}) \rightarrow L^{\infty}(\bT^d)$, is given by the operation 
    $\cS_2(v_1, v_2)(x) = g_{\times}(v_1(x), v_2(x)) := v_1(x) v_2(x)$.
    Note that, by assumption, for any $a \in K$
    its projection $\cP_N a \in L^2_N (\bT^d)$ satisfies $\|\cP_N a\|_{L^2} \leq R$.
    Moreover, since $L_N^2(\bT^d) \subset \cont^{\infty}(\bT^d)$ is finite dimensional, its norm is equivalent to the $L^\infty$ norm, i.e., $\|v\|_{L^{\infty}} \leq \const_{N,1} \|v\|_{L^2}$ for all $v \in L^2_N(\bT^d)$.
    This means we have $\|\cP_N a\|_{L^\infty} \leq \const_{N,1} R$,
    such that the output of $\cS_1$ is such that $\cS_1(a)(x) = (\cP_N a(x), \cos(\langle k, x \rangle)) \in [-\const_{N,1}R, \const_{N,1}R] \times [-1, 1]$ for any $x \in \bT^d$.
    
    Consider an arbitrary $\epsilon_{\times} \leq 1$. We can choose to approximate $g_{\times}$ by a neural network $\tilde{g}_\times$ such that 
    \[ \| g_{\times} - \tilde{g}_{\times}  \|_{\cont^1 \left([-\const_{N,1}R, \const_{N,1}R] \times [-1, 1] \right)} \leq \epsilon_{\times}. \]
    We recognize that $\cN_2(v)(x) = \tilde{g}(v_1(x), v_2(x))$ 
    is an FNO with only local layers.
    Moreover, the composition $\cN_2 \circ \cL_1 \circ \cR_1$ is also an FNO, 
    since the output of $\cL_1$ is now passed through an activation function in the first layer of $\tilde{g}_{\times}$.
    As a superposition operator, $\cN_2 : L^{\infty}(\bT^d; \bR^2) \rightarrow  L^{\infty}(\bT^d; \bR)$ is continuously differentiable.
    Its output approximation error is given by 
    for all $v = \cS_1(a)$, $\|a\|_{L^2(\bT^d)} \leq R$.
    Similarly, for the derivatives,
    \begin{align*}
        \| D \cS_2(v)w - D \cN_2(v)w\|_{L^{\infty}}
        &= 
     \| \partial_{1} g_{\times} (v_1(\cdot), v_2(\cdot)) - \partial_{1} \tilde{g}_{\times}(v_1(\cdot), v_2(\cdot)) \|_{L^{\infty}}
        \; \|w_1\|_{L^{\infty}} \\
    & \qquad + \| \partial_{2} g_{\times} (v_1(\cdot), v_2(\cdot)) - \partial_{2} \tilde{g}_{\times}(v_1(\cdot), v_2(\cdot)) \|_{L^{\infty}} 
        \; \|w_2\|_{L^{\infty}}
    \\ & \leq 2 \epsilon_{\times} \|w\|_{L^{\infty}}.
    \end{align*}
    This implies 
    $\|D \cS_{2}(v) - D \cN_2(v) \|_{\Op(L^{\infty}(\bT^d; \bR^2), L^{\infty}(\bT^d; \bR))} \leq 2 \epsilon_{\times}$.

    Finally, we have $\cS_3 : L^{\infty}(\bT^d; \bR) \ni v \rightarrow \int v dx \in  \bR \cong L^2_0(\bT^d; \bR)$,
    such that $\cS_3$ is a continuous linear functional,
    and $|\int v dx| \leq |\bT^d| \|v\|_{L^{\infty}}$.
    This is represented by a linear Fourier layer, i.e.~$\cS_3 \equiv \cL_3$, where
    \[
    \cL_3 : L^{\infty}(\bT^d) \rightarrow L^2_0(\bT^d), \quad \cL_3(v) = \cF^{-1} P_{0} \odot \cF v,
    \]
    using $P_0(k) = 1$ if $k = 0$ and $0$ otherwise. 
    
    Under this construction, 
    for any $a \in L^2(\bT^d; \bR), \|a\|_{L^2} \leq R$, 
    the difference between $\cT_{1,k}$ and the approximation 
    $\cL_3 \circ \cN_2 \circ \cL_1 \circ \cR_1 \equiv \cS_3 \circ \cN_2 \circ \cS_1$ is given by 
    \begin{align*}
        \| \cT_{1,k}(a) - \cS_3 \circ \cN_2 \circ \cS_1 (a) \|_{L^2} 
        &\leq \| \cS_3 \|_{\Op(L^{\infty}(\bT^d; \bR), L^2_0(\bT^d;\bR))}  \| \cS_2(\cS_1(a)) - \cN_2(\cS_1(a)) \|_{L^{\infty}}  \\
        &\leq \| \cS_3 \|_{\Op(L^{\infty}(\bT^d; \bR), L^2_0(\bT^d;\bR))}  \epsilon_{\times}  \\
        & =\const_0  \epsilon_{\times},
    \end{align*}
    where $\const_0 = \const_0(d) := \| \cS_3 \|_{\Op(L^{\infty}(\bT^d; \bR), L^2_0(\bT^d;\bR))}.$
    For the derivative, we have 
    \begin{align*}
        & \|D \cT_{1,k}(a) - D (\cS_3 \circ \cN_2 \circ \cS_1)(a) \|_{\Op(H^s(\bT^d; \bR), L^2_0(\bT^d; \bR))} \\
        & \quad = \| \cS_3 (D \cS_2(\cS_1(a)) - D\cN_{2}(\cS_1(a))) D \cS_1 \|_{\Op(H^s(\bT^d; \bR), L^2_0(\bT^d; \bR))} \\        
        & \quad \leq \const_0
            \| D \cS_2(\cS_1(a)) - D\cN_{2}(\cS_1(a)) \|_{\Op(L^2_N(\bT^d; \bR^{2}), L^\infty(\bT^d; \bR))} 
            \| D \cS_1 \|_{\HS(H^s(\bT^d; \bR), L^2_N(\bT^d; \bR^{2}))} \\        
        & \quad \leq \const_0 \const_{N,2}
            \| D \cS_2(\cS_1(a)) - D\cN_{2}(\cS_1(a)) \|_{\Op(L^\infty(\bT^d; \bR^{2}), L^\infty(\bT^d; \bR))} 
            \| D \cS_1 \|_{\HS(H^s(\bT^d; \bR), L^2_N(\bT^d; \bR^{2}))} \\        
        & \quad \leq \const_1 \epsilon_{\times},
    \end{align*}
    where $\const_{N,2}$ is the equivalence constant for $L_N^2(\bT^d;\bR^2)$ and $L^{\infty}(\bT^d;\bR^2)$, i.e.,
    $\|v\|_{L^{\infty}} \leq \const_{N,2} \|v\|_{L^2}$ for all $v \in L^2_{N}(\bT^d;\bR^2)$ 
    and $\const_1 = \const_1(N,d) := 2 \const_0 \const_{N,2} \| D \cS_1 \|_{\HS(L^2(\bT^d; \bR), L^2_N(\bT^d; \bR^{2}))}$.
    Thus, $\cT_{1,k}$ can be approximated by $\cS_3 \circ \cN_2 \circ \cS_1$ to arbitrary function and derivative accuracy. 

    Since this is not yet an FNO, we compose this with the single-layer FNO approximation of the identity operator $\cI_{\theta}$ given in \Cref{lemma:identity_approx_Linfty}, 
    such that the overall composition $\cN_{1,k} = \cI_{\theta} \circ \cS_3 \circ \cN_2 \circ \cS_1$ is an FNO.
    In particular, by \Cref{lemma:identity_approx_Linfty}, $\cI_{\theta}$ can be made arbitrarily accurate over $L^{\infty}$-bounded inputs.
    Moreover, since the output of $\cS_3 \circ \cN_2 \circ \cS_1$ belongs to a finite-dimensional space, 
    the Replacement Lemma \ref{lemma:composition} ensures the existence of $\cI_\theta$ and $\cN_2$ (using appropriate $\theta, \epsilon_{\times}$)
    such that for any $\epsilon_{1,k}$, 
    $\cN_{1,k}$ satisfies
    \[ \| \cT_{1,k}(a) - \cN_{1,k}  (a) \|_{L^2_0(\bT^d;\bR)} \leq \epsilon_{1,k}
    \qquad 
       \|D \cT_{1,k}(a) - D \cN_{1,k} (a) \|_{\Op(H^s(\bT^d), L^2_0(\bT^d; \bR))} \leq \epsilon_{1,k}.
    \]

    Finally the overall FNO approximation $\cN$ is the concatenation of each component $\cN_{l,k}$, constructed with an appropriately chosen $\epsilon_{l,k}$ such that the overall error is bounded by $\epsilon.$
    Note that for each component $\cN_{l,k}$ to have the same depth, the same $\cN_2$ (i.e., $\tilde{g}_{\times}$ should be used, with $\epsilon_{\times}$ chosen such that all $\epsilon_{l,k}$ are simultaneously controlled.

\end{proof}

Similarly, for the inverse Fourier transform $\cP_N \circ \real\cF_N^{-1}$, we have the following result.
\begin{lemma}[Approximation of the inverse Fourier transform in $\cont^1$]\label{lemma:inverse_ft_approximation}
    Let $N \in \bN$ and $s \geq 0$. 
    Then, given $\sigma \in \smootha$, for any $\epsilon > 0$ and $R > 0$, there exists an FNO 
    $\cN \in \cont^1( L^2_0(\bT^d; \bR^{2K_N}) ; H^s(\bT^d))$ 
    with activation function $\sigma$ such that 
    \begin{equation}\label{eq:inverse_fourier_approx}
        \| \cP_N \circ \real \cF_N^{-1} (\hat{w}) - \cN (\hat{w}) \|_{H^s(\bT^{d})} \leq \epsilon
    \end{equation}
    and 
    \begin{equation}\label{eq:inverse_fourier_derivative_approx}
        \|D (\cP_N \circ \real \cF_N^{-1})(\hat{w}) - D \cN(\hat{w}) \|_{\Op(L^2_0(\bT^d; \bR^{2K_N}), H^s(\bT^d))} \leq \epsilon
    \end{equation}
    for all $\hat{w} \in L^2_0(\bT^d; \bR^{2K_N})$ with $\|\hat{w}\|_{L^2} \leq R$.
\end{lemma}

\begin{proof}
    Consider first the operator $\real \cF_N^{-1}$ such that for $\hat{w} = \{\hat{w}_{l,k}\}_{l \in \{1, 2\}, |k|_{\infty} \leq N}$, we have
    \[
        \real \cF_{N}^{-1} \hat{w} = \sum_{|k|_{\infty} \leq N} \hat{w}_{1,k} \cos(\langle k, x \rangle)
        - \hat{w}_{2,k} \sin(\langle k, x \rangle).
    \]
    We note that $\real \cF_{N}^{-1}$ can be represented by the composition $\cQ \circ \cT$,
    where
    $\cT : L^2_0(\bT^d; \bR^{2K_N}) \rightarrow L^2(\bT^d; \bR^{2K_N})$ 
    is the concatenation of $\cT_{l,k}$, $l \in \{1,2\}$, $|k|_{\infty} \leq N$, defined by
    \begin{align*}
        \cT_{1,k}(\hat{w}) = \hat{w}_{1,k} \cos(\langle k, x \rangle), 
        \quad 
        \cT_{2,k}(\hat{w}) = \hat{w}_{2,k} \sin(\langle k, x \rangle), 
    \end{align*}
    while $\cQ : L^2(\bT^d; \bR^{2 K_N}) \rightarrow L^2(\bT^d; \bR)$ is the continuous linear operator 
    \[
    \cQ ((v_{1,k}, v_{2,k})_{|k| \leq N}) = \sum_{|k|\leq N} v_{1,k}(x) - v_{2,k}(x).
    \]
    For any $\epsilon_{\cT} > 0$, we first show that there exists an FNO approximation of $\cT$ to $\epsilon_{\cT}$ accuracy, 
    \[
        \| \cT(\hat{w}) - \cN(\hat{w}) \|_{L^2(\bT^d; \bR^{2 K_N})} \leq \epsilon_{\cT} 
    \]
    and 
    \[
        \| D \cT(\hat{w}) - D \cN(\hat{w}) \|_{\Op(L^2_0(\bT^d; \bR^{2 K_N}), L^2(\bT^d; \bR^{2 K_N}))} \leq \epsilon_{\cT}, 
    \]
    for all $\|\hat{w}\|_{L^2_0(\bT^d; \bR^{2K_N})} \leq R$. 
    Then, $\cQ \circ \cN$ is also an FNO, and 
    \[
        \| \cQ \circ \cT(\hat{w}) - \cQ \circ \cN(\hat{w}) \|_{L^2(\bT^d; \bR^{2 K_N})} 
        \leq \| \cQ \|_{\Op(L^2(\bT^d; \bR^{2K_N}, L^2(\bT^d; \bR))} \epsilon_{\cT}
    \]
    and 
    \[
        \| D (\cQ \circ \cT(\hat{w})) - D (\cQ \circ \cN(\hat{w})) \|_{\Op(L^2_0(\bT^d; \bR^{2 K_N}), L^2(\bT^d; \bR^{2 K_N}))} 
        \leq \| \cQ \|_{\Op(L^2(\bT^d; \bR^{2K_N}), L^2(\bT^d; \bR))} \epsilon_{\cT}.
    \]

    To this end, 
    we can consider each component of $\cT_{l,k}(w) = \cS_2 \circ \cS_1(\hat{w}_{l,k})$ 
    as the composition of two operators, $\cS_1$ and $\cS_2$.
    We focus on the $l = 1$ case, where this is given by 
    \begin{equation*}
        \hat{w}_{1,k}
        \underset{\cS_1}{\longrightarrow}
        \begin{pmatrix}
            \hat{w}_{1,k} \\ 
            \cos(\langle k, x \rangle)
        \end{pmatrix}
        \underset{\cS_2}{\longrightarrow}
            \hat{w}_{1,k} \cos(\langle k, x \rangle),
    \end{equation*}
    and simply note that for $l = 2$, $\cos(\langle k, x \rangle)$ is replaced by $\sin( \langle k, x \rangle)$.
    The operators map between the spaces 
    \[
        \bR \cong L^2_0(\bT^d; \bR)
            \underset{\cS_1}{\longrightarrow}
        L^2_{N}(\bT^d; \bR^{2})
            \underset{\cS_2}{\longrightarrow}
        L^2(\bT^d; \bR^{2}).
    \]
    It suffices to approximate $\cT_{l,k} = \cS_2 \circ \cS_1$ by an FNO for each $l,k$ and concatenating the result. 
    
    First, $\cS_1 = \cL_1 \circ \cR_1$ is exactly represented by a linear FNO with a lifting layer
    \[
        \cR_1 (\hat{w}) = (\hat{w}_{1, k}, 0)
    \]
    followed by a linear layer
    \[
        \cL_1(v)(x) = Wv(x) + b(x)
    \]
    where $W \equiv 0 $ and $b(x) = (0, \cos (\langle k, x \rangle)).$
 
    The subsequent mapping, $\cS_2 : L^2_{N}(\bT^d; \bR^{2}) \rightarrow L^{2}(\bT^d; \bR)$, 
    is given by the operation 
    $\cS_2(v_1, v_2)(x) = g_{\times}(v_1(x), v_2(x)) := v_1(x) v_2(x)$.
    This is essentially identical to $\cS_{2}$ used in the proof of the previous lemma. 
    For completeness, we will write the procedure for the construction of its FNO approximation $\cN_2$, but we will not explicitly track the norm-equivalence constants that arise in the proof since it is largely analogous to the previous proof.

    Note that, by assumption, for any $\|\hat{w}\|_{L^2_0(\bT^d; \bR^{2K_N})} \leq R$, 
    we have $|\hat{w}_{l,k}| \leq \const_{1} R$ for some constant $\const_{1} = \const_{1}(N,d)$.
    Thus, the output of $\cS_1$ is such that $\cS_1(\hat{w}_{1,k})(x) = (\hat{w}_{1,k}, \cos(\langle k, x \rangle) \in [-\const_{1}R, \const_{1}R] \times [-1, 1]$ for any $x \in \bT^d$.
    For any $\epsilon_{\times} \leq 1$, we can choose to approximate $g_{\times}$ by a neural network $\tilde{g}_{\times}$ such that 
    \[ \| g_{\times} - \tilde{g}_{\times} \|_{\cont^1 \left([-\const_1R, \const_1 R] \times [-1, 1] \right)} \leq \epsilon_{\times}. \]
    This defines an FNO $\cN_2(v)(x) = \tilde{g}_{\times}(v_1(x), v_2(x))$,
    and the composition 
    $\cN_2 \circ \cL_1 \circ \cR_1 \equiv \cN_2 \circ \cS_1$ 
    is also an FNO.
    Using the same arguments as in the previous lemma, we have 
    \[
        \| \cS_2(v)  - \cN_2(v) \|_{L^{\infty}} 
        \leq \|g_{\times}(v_1(\cdot), v_2(\cdot)) - \tilde{g}_{\times}(v_1(\cdot), v_2(\cdot))\|_{L^{\infty}}
        \leq \epsilon_{\times}
    \]
    and 
    \[ 
    \|D \cS_{2}(v) - D \cN_2(v) \|_{\Op(L^{\infty}(\bT^d; \bR^2), L^{\infty}(\bT^d; \bR))} \leq 2 \epsilon_{\times}
    \]
    for any $v = \cS_1(\hat{w}_{l,k})$ where  $\|\hat{w}\|_{L^2_0(\bT^d; \bR^{2K_N})} \leq R$.
    Moreover, since $L^2_{N}(\bT^d; \bR^2) \hookrightarrow L^{\infty}(\bT^d; \bR^2)$ and 
    $ L^{\infty}(\bT^d) \hookrightarrow L^2(\bT^d)$, 
    we can rewrite these results in terms of the $L^2_{N}$ and $L^2$ norms as 
    \begin{align*}
        \| \cS_2(v)  - \cN_2(v) \|_{L^{2}(\bT^d; \bR)} \leq \const_{2} \epsilon_{\times},
        \qquad 
        \|D \cS_{2}(v) - D \cN_2(v) \|_{\Op(L^{2}_{N}(\bT^d; \bR^2), L^{2}(\bT^d; \bR))} 
            \leq \const_3 \epsilon_{\times},
    \end{align*}
    where $\const_2 = \const_2(N,d)$ and $\const_3 = \const_3(N,d)$ are positive constants depending on the norms of the inclusion operators. 
    The error of the composition is therefore bounded by
    \[
        \| \cS_2 \circ \cS_1 (\hat{w}_{1,k})  - \cN_2 \circ \cS_1 (\hat{w}_{1,k}) \|_{L^{2}} \leq \const_{2} \epsilon_{\times},
    \]
    and 
    \begin{align*}
        & \| D(\cS_2 \circ \cS_1) (\hat{w}_{1,k})  - D (\cN_2 \circ \cS_1) (\hat{w}_{1,k}) \|_{\Op(L^{2}_{N}(\bT^d; \bR^2), L^{2}(\bT^d; \bR))}  \\
        & \quad  \leq 
            \| D \cS_2(\cS_1(\hat{w}_{1,k})) - D \cN_2(\cS_1(\hat{w}_{1,k})) \|_{\Op(L^{2}_{N}(\bT^d; \bR^2), L^{2}(\bT^d; \bR))}  
            \cdot 
            \| D \cS_1 \|_{\Op(L^{2}_{0}(\bT^d; \bR), L^{2}_{N}(\bT^d; \bR^2))} \\ 
        & \quad \leq \const_4 \epsilon_{\times},
    \end{align*}
    where $\const_4 = \const_4(N,d)$.
    
    Thus, concatenating each component and appropriately choosing $\epsilon_{\times}$ ensures that $\cT$, and therefore $\cQ \circ \cT$, 
    can be approximated to arbitrary accuracy by an FNO.
    That is, for any $\epsilon_{\cT} > 0$, there exists an FNO $\cQ \circ \cN_{\cT}$ such that 
    \begin{equation}\label{eq:inverse_ft_function_error_before_projection}
        \| \cQ \circ \cT (\hat{w})  - \cQ \circ \cN_{\cT} (\hat{w}) \|_{L^{2}(\bT^d;\bR)} \leq \epsilon_{\cT}
    \end{equation}
    and 
    \begin{equation}\label{eq:inverse_ft_derivative_error_before_projection}
        \| D(\cQ \circ \cT)(\hat{w}) - D(\cQ \circ \cN_{\cT}(\hat{w}) \|_{\Op(L^{2}_{0}(\bT^d; \bR^{2K_N}), L^{2}(\bT^d; \bR))} \leq \epsilon_{\cT}
    \end{equation}
    for all $\|\hat{w}\|_{L^2_0(\bT^d;\bR^{2K_N})} \leq R$.

    We now consider the final composition with the projection operator $\cP_N : L^2(\bT^d; \bR) = L^2(\bT^d) \rightarrow L^2_N(\bT^d) \hookrightarrow H^s(\bT^d)$.
    To approximate this using a continuously differentiable FNO, we consider any $m \in \bN$ such that $m > \max(s, d/2)$.
    Then, we view the final projection as the sequence of operations 
    \[
        L^2(\bT^d) \underset{\cP_N}{\rightarrow} L^2_{N}(\bT^d) 
            \hookrightarrow H_{N}^{2m}(\bT^d)
            \underset{\cI}{\rightarrow}
                H^{m}(\bT^d)
            \hookrightarrow H^{s}(\bT^d)
    \]
    where $\cI(v) = v$ refers to the identity (inclusion) operator. 
    Our approximation of $\cP_N$ then takes the form $\cN_{\cP_N} := \cI_{\theta} \circ \cP_N$
    using $\cI_{\theta}:H^{2m}(\bT^d) \rightarrow H^{m}(\bT^d)$ from \Cref{lemma:identity_approximation_sobolev}, and maps between the spaces
    \[
        L^2(\bT^d) \underset{\cP_N}{\rightarrow} L^2_{N}(\bT^d) 
            \hookrightarrow H_{N}^{2m}(\bT^d)
            \underset{\cI_{\theta}}{\rightarrow}
                H^{m}(\bT^d)
            \hookrightarrow H^{s}(\bT^d).
    \]
    Here, $\cI_{\theta}$ is continuously differentiable as a superposition operator due to our choice of $m$.

    Since the projection is bounded with $\|\cP_N\|_{\Op(L^2(\bT^d), H_{N}^{2m}(\bT^d))} = \const_5(N,d) < \infty$,
    our previous results
    \eqref{eq:inverse_ft_function_error_before_projection} and \eqref{eq:inverse_ft_derivative_error_before_projection}
    additionally imply that the mapping 
    \[ \cP_N \circ \cQ \circ \cT : L^2_0(\bT^d; \bR^{2K_N}) \rightarrow H_{N}^{2m}(\bT^d)\]
    can be approximated by the linear FNO 
    \[ \cP_N \circ \cQ \circ \cN_{\cT}: L^2_0(\bT^d; \bR^{2K_N}) \rightarrow H_{N}^{2m}(\bT^d)\] 
    to arbitrary accuracy for all $\|\hat{w}\|_{L^2_0(\bT^d; \bR^{2K_N})} \leq R$.
    On the other hand, \Cref{lemma:identity_approximation_sobolev} implies the existence of a $\theta > 0$ such that 
    $\cI_{\theta} : H^{2m}(\bT^d) \rightarrow H^m(\bT^d)$ 
    approximates $\cI$ to arbitrary accuracy over any bounded set. 
    Since the intermediate space $H^{2m}_N(\bT^d)$ is finite dimensional, 
    \Cref{lemma:composition} yields for any $\epsilon > 0$ the existence of $\cN_{\cT}$ and $\cI_{\theta}$ such that the FNO $\cN := \cN_{\cP_N} \circ \cQ \circ \cN_{\cT}$
    satisfies
    \begin{equation}\label{eq:inverse_ft_function_error_before_projection}
        \| \cP_N \circ \cQ \circ \cT (\hat{w})  - \cN_{\cP_N} \circ \cQ \circ \cN_{\cT} (\hat{w}) \|_{H^{m}(\bT^d; \bR)} \leq \epsilon,
    \end{equation}
    and 
    \begin{equation}\label{eq:inverse_ft_derivative_error_before_projection}
        \| D(\cP_N \circ \cQ \circ \cT)(\hat{w}) - D(\cN_{\cP_N} \circ \cQ \circ \cN_{\cT})(\hat{w}) \|_{\Op(L^{2}_{N}(\bT^d; \bR^{2K_N}), H^{m}(\bT^d; \bR))} 
        \leq \epsilon,
    \end{equation}
    for all $\|\hat{w}\|_{L^2_0(\bT^d;\bR^{2K_N})} \leq R$.
    Finally, error bounds in terms of the $H^s$ norm can then be obtained using the continuous embedding $H^{m}(\bT^d) \hookrightarrow H^{s}(\bT^d)$ for all $s \leq m$.

\end{proof}

We can now proceed with the proof of \Cref{lemma:truncated_operator}.
\begin{proof}(of \Cref{lemma:truncated_operator}).
    We focus on the case $d_a = d_u = 1$. The proof for $d_a, d_u \geq 1$ is analogous.
    We consider the representation of $\cG_N := \cP_N \circ \cG \circ \cP_N$ as 
    \[
        \cG_N = \cP_N \circ \real \cF_{N}^{-1} \circ \widehat{G}_N \circ \cF_N \circ \cP_N,
    \]
    where $\widehat{G}_N = \cF_N \circ \cG \circ \cF_N^{-1} : \bR^{2K_N} \rightarrow \bR^{2K_N}$.
    To make the connection with FNOs, 
    we interpret $\widehat{G}_N$ along with $\cF_{N}$ and $\cF_{N}^{-1}$ as mappings to/from constant vector-valued functions, 
    i.e.~$\widehat{\cG}_N : L^2_0(\bT^d; \bR^{2K_N}) \rightarrow L^2_0(\bT^d;\bR^{2K_N})$, $\widehat{\cG}_N(v)(x) = \widehat{G}_N(v(x))$, 
    $\cF_N : L^2(\bT^d; \bR) \rightarrow L^2_0(\bT^d; \bR^{2K_N})$, and
    $\real \cF_N^{-1} : L^2_0(\bT^d; \bR^{2K_N}) \rightarrow L^2_0(\bT^d; \bR)$.
    In particular, for any $\epsilon_G, R_G > 0$, universal approximation theorems on $\bR^{2K_N}$ yields the existence of a neural network $\tilde{G}$ such that 
    \[
        \| \widehat{G}_N - \widetilde{G}\|_{\cont^1([-R_G, R_G]^{2K_N})} \leq \epsilon_G.
    \]
    This translates to the following: for any $\epsilon_G, R_G$, there exists a continuously differentiable FNO $\cN_{N} : L^2_0(\bT^d; \bR^{2K_N}) \rightarrow L^2_0(\bT^d; \bR^{2K_N})$, $\cN(w)(x) = \widetilde{G}(w(x))$, such that 
    \[
        \| \widehat{\cG}_N(w) - \cN_N(w) \|_{L^2_0(\bT^d;\bR^{2K_N})} \leq \epsilon_G, \quad 
        \| D\widehat{\cG}_N(w) - D\cN_N(w) \|_{\Op
            (L^2_0(\bT^d;\bR^{2K_N}), L^2_0(\bT^d;\bR^{2K_N}))
        } \leq \epsilon_G
    \]
    for all $\|w\|_{L^2_0(\bT^d;\bR^{2K_N})} \leq R_G$. Additionally, 
    \begin{itemize}
        \item \Cref{lemma:ft_approximation} yields an approximation $\cN_{\mathrm{FT}}$ of the Fourier transform $\cF_N \circ \cP_N : H^s(\bT^d; \bR) \rightarrow L^2_0(\bT^d; \bR^{2K_N})$
        for arbitrary function/derivative accuracy over bounded sets
        \item \Cref{lemma:inverse_ft_approximation} yields an approximation $\cN_{\mathrm{IFT}}$ of the inverse Fourier transform $\cP_N \circ \real \cF_N^{-1} : L^2_0(\bT^d; \bR^{2K_N}) \rightarrow H^{s'}(\bT^d; \bR)$
        for arbitrary function/derivative accuracy over bounded sets
    \end{itemize}
    Since the intermediate spaces are finite dimensional, \Cref{lemma:composition} (with the Hilbert--Schmidt norm) 
    implies that for any compact $K$ and $\epsilon > 0$, 
    there exists an FNO $\cN = \cN_{\mathrm{IFT}} \circ \cN_{N} \circ \cN_\mathrm{FT}$ for which the desired error bounds hold.
    For $d_a, d_u \geq 1$, 
    we can approximate the Fourier transforms for each of the $d_a$ input components and the inverse Fourier transform for each of the $d_u$ output components, 
    before approximating the mapping between the Fourier coefficients ($2K_N \cdot d_a$ input coefficients and $2 K_N \cdot d_u $ output coefficients). 
\end{proof}

\subsection{Proof of the Main Result}\label{sec:proof_of_main_ua}
We are now in the position to formalize the proof of \Cref{theorem:main_ua}.
\begin{proof}(of \Cref{theorem:main_ua})
    The error can be decomposed into a truncation error and an approximation error
    for both the function value and the derivatives, i.e.,
    \[ 
        \| \cG(a) - \cN(a) \|_{\cY} \leq 
            \| \cG(a) - \cG_N(a) \|_{\cY} + 
            \| \cG_N(a) - \cN(a) \|_{\cY} .
    \]
    and 
    \[ 
        \| D \cG(a) - D \cN(a) \|_{\HS(\cX_{\delta}, \cY)} \leq
            \| D\cG(a) - D\cG_N(a) \|_{\HS(\cX_{\delta}, \cY)} + 
            \| D\cG_N(a) - D\cN(a) \|_{\HS(\cX_{\delta}, \cY)}.
    \]
    By \Cref{lemma:truncation}, we can find $N$ such that the truncation errors are bounded as
    \[ 
        \sup_{a \in K} \| \cG(a) - \cG_N(a) \|_{\cY} \leq \epsilon/2 , \qquad 
        \sup_{a \in K} \| D\cG(a) - D\cG_N(a) \|_{\HS(\cX_{\delta}, \cY)} \leq \epsilon / 2. 
    \]
    Subsequently, by \Cref{lemma:truncated_operator}, we can find $\cN$ approximating $\cG_N$ for the chosen $N$ such that the approximation errors are bounded as 
    \[ \sup_{a \in K} \| \cG_N(a) - \cN(a) \|_{\cY} \leq \epsilon/2 , \qquad 
     \sup_{a \in K} \| D\cG_N(a) - D\cN(a) \|_{\HS(\cX_{\delta}, \cY)} \leq \epsilon / 2. \]
    Combining the two results yields the desired error bound.
\end{proof}

\subsection{Proof of the Inverse Problem Result}
\label{subsec:proof_inverse_problem_ua}
\Cref{prop:surrogate_optimization_error_strongly_convex} can be combined with the universal approximation on compact sets (\Cref{theorem:main_ua}) 
to prove \Cref{theorem:inverse_problem_ua}.

\begin{proof}(of \Cref{theorem:inverse_problem_ua})
    We begin by showing a corollary of the universal approximation over bounded sets in $\cX_{\delta}$, making use of the compact embedding of $\cX_{\delta}$ into $\cX$. 
    Recall that for $\cX_{\delta} = H^{s+\delta}(\bT^d;\bR^{d_a})$ and $\cX = H^{s}(\bT^d;\bR^{d_a})$, the inclusion operator $\cI : \cX_{\delta} \rightarrow \cX$ is compact. 
    Therefore, for the closed ball $B_R(a^{\star}) \subset \cX_{\delta}$, the closure of its image $K = \overline{\cI(B_R(a^{\star}k))} \subset \cX$ is compact.
    Moreover, by \Cref{prop:inclusion_operator_is_hilbert_schmidt}, $\cX \ni a \mapsto D \cG(a) \in \HS(\cX_{\delta}, \cY)$ is well-defined and continuous.
    Then, \Cref{theorem:main_ua} yields the existence of an FNO $\cN : \cX \rightarrow \cY$ that approximates $\cG$ over $K$, such that 
    \[
        \sup_{a \in K } \|\cG(a) - \cN(a)\|_{\cY} \leq \epsilon, \qquad 
        \sup_{a \in K} \|D \cG(a) - D \cN(a)\|_{\HS(\cX_{\delta}, \cY)}.
    \]
    When restricted to the space $\cX_{\delta}$, 
    we have 
    \[
        \sup_{a \in B_R(a^{\star})} \| (\cG \circ \cI) (a) - (\cN \circ \cI)(a)\|_{\cY} \leq 
        \sup_{a \in K} \|\cG(a) - \cN(a)\|_{\cY} 
        \leq \epsilon
    \]
    and 
    \[
        \sup_{a \in B_R(a^{\star})} \| D (\cG \circ \cI) (a) - D (\cN \circ \cI)(a)\|_{\HS(\cX_{\delta}, \cY)} \leq 
        \sup_{a \in K} \|D \cG(a) - D \cN(a)\|_{\HS(\cX_{\delta}, \cY)}
        \leq \epsilon.
    \]
    Since $\|\cA\|_{\Op(\cX_{\delta}, \cY)} \leq \|\cA\|_{\HS(\cX_{\delta}, \cY)}$, the operator $\cN \circ \cI$, which can be identified with an FNO on $\cX_{\delta}$, satisfies the desired error bounds.

    The compactness of the embedding also implies that $\cG : \cX_{\delta} \rightarrow \cY$ and $D \cG : \cX_{\delta} \rightarrow \Op(\cX_{\delta}, \cY) $ are bounded. That is,
    \[
        \sup_{a \in B_R(a^{\star})} \|\cG(a)\|_{\cY}
        \leq 
        \sup_{a \in K} \|\cG(a)\|_{\cY} < \infty
    \]
    and 
    \[
        \sup_{a \in B_R(a^{\star})} \|D \cG(a)\|_{\Op(\cX_{\delta}, \cY)}
        \lesssim
        \sup_{a \in K} \|D \cG(a)\|_{\Op(\cX, \cY)}
        < \infty,
    \]
    by the continuity of $\cG, D \cG$ on $\cX$ and the compactness of $K$.
    
    Next, we show that the cost function \eqref{eq:cost_function_inverse_problem} is weakly lower semicontinuous.
    Here, the regularization term $\frac{\beta}{2} \|a\|_{\cX_{\delta}}^2$ is continuous and convex, which implies that it is weakly lower semicontinuous.
    On the other hand, the data misfit functional
    \[
        f_{\mathrm{data}}(a) = a \mapsto \frac{1}{2} \| \Gamma_{\noise}^{-1/2} \left( \cH (\cG(a)) - \yobs \right)\|^2,
    \]
    is a continuous from $\cX$ to $\bR$. 
    When restricted to $a \in \cX_{\delta}$, 
    the data misfit is given by $f_{\mathrm{data}} \circ \cI$.
    Since $\cI$ is a compact operator, 
    the resulting data misfit is also weakly lower semicontinuous.

    Moreover, we note that for \eqref{eq:cost_function_inverse_problem}, the function 
    \[
        F(u,a) = \frac{1}{2} \left\| \Gamma_{\noise}^{-1/2}(\cH(u) - \yobs) \right\|^2 + \frac{\beta}{2} \|a\|_{\cX_{\delta}}^2
    \]
    has Lipschitz continuous derivatives as it is constructed from the sum of two continuous quadratic functionals.

    Finally, given the constant $\const$ in Part (i) and tolerance $\epsilon_0$ in Part (ii) of \Cref{prop:surrogate_optimization_error_strongly_convex},
    we find the corresponding FNO $\cN$ with error tolerance $\epsilon_{\cN} \leq \epsilon_0$, such that simultaneously
    \[
        \sup_{a \in B_R(a^{\star})} \| \cG(a) - \cN(a) \|_{\cY} \leq \epsilon_{\cN} \leq \epsilon_0, 
        \qquad
        \sup_{a \in B_R(a^{\star})} \| D \cG(a) - D \cN(a) \|_{\Op(\cX_{\delta},\cY)} \leq \epsilon_{\cN} \leq \epsilon_0,
    \]
    and 
    $\const (\epsilon_{\cN}  + \epsilon_{\cN} + \epsilon_{\cN}^2) \leq \epsilon$. 
    The above results, along with the strong convexity assumptions, thus allow us to apply \Cref{prop:surrogate_optimization_error_strongly_convex},
    using $\cX_{\delta}$ as the underlying Hilbert space instead of $\cX$. 
    This asserts both the existence of the minimizer and the error bound on $\|a^{\dagger} - a^{\star}\|_{\cX_{\delta}}$.
\end{proof}
\section{Extension to Unbounded Domains}
\label{sec:unbounded_domains}

We now develop the proof for \Cref{theorem:main_ua_mu},
which extends the universal approximation result to the $W^{1,2}_{\mu, \delta}$
under \Cref{assumption:measurability}
and using specifically activation functions $\sigma \in \srla$ as defined in \Cref{def:activation}.

To do so, we begin by developing several useful constructions with $\srla$ activation functions. 
Then, similar to the proof of \Cref{theorem:main_ua}, 
we first approximate $\cG$ by $\cG_N = \cP_N \circ \cG \circ \cP_N$, 
and then approximate $\cG_N$ by an FNO $\cN$. 
We show that both approximations can be made arbitrarily small in the $W^{1,2}_{\mu,\delta} $ norm.
These results are then put together in \Cref{sec:proof_of_main_ua_mu} 
to prove \Cref{theorem:main_ua_mu}.

\subsection{Useful Function Constructions}
Activation functions $\sigma \in \srla$ satisfy typical properties of GELU. Namely, $\sigma(x) \approx x$ as $x \rightarrow \infty$ and $\sigma(x) \approx 0$ as $x \rightarrow -\infty$.
This is summarized in the following proposition.
\begin{proposition}\label{proposition:properties_of_gelu}
    Let $\sigma \in \srla$. 
    Then, the following hold:
    \begin{enumerate}[label=(\roman*)]
        \item $\sigma(x) - x \rightarrow 0$ and $\sigma'(x) \rightarrow 1$ as $x \rightarrow \infty$;
        \item $\sigma(x) \rightarrow 0$ and $\sigma'(x) \rightarrow 0$ as $x \rightarrow -\infty$.
    \end{enumerate}
\end{proposition}
\begin{proof}
    Recall that $\sigma(x) = x \sigmoid(x)$, and therefore is in $\cont^{\infty}(\bR)$ since both $x$ and $\sigmoid$ are in $\cont^{\infty}(\bR)$.
    Moreover, we can evaluate its limits via L'Hopital's rule. 
    For the case of $x \rightarrow \infty$, we have 
    \[
        \lim_{x \rightarrow \infty} \sigma(x) - x 
        = \lim_{x \rightarrow \infty} x(\sigmoid(x) - 1)
        = \lim_{x \rightarrow \infty} - \frac{\sigmoid'(x)}{x^{-2}} 
        = -\lim_{x \rightarrow \infty} x^2 \sigmoid'(x) = 0,
    \]
    and 
    \[
        \lim_{x \rightarrow \infty} \sigma'(x) = \lim_{x \rightarrow \infty} x \sigmoid'(x) + \sigmoid(x) = 1,
    \]
    where we have used the fact that $\Phi(x) \rightarrow 1 $, $x^2\sigmoid(x) \rightarrow 0$, and $ x \sigmoid'(x) \rightarrow 0$ as $x \rightarrow \infty$.
    Similarly, for $x \rightarrow -\infty$, we have 
    \[
        \lim_{x \rightarrow -\infty} \sigma(x)  
        = \lim_{x \rightarrow -\infty} x \sigmoid(x) 
        = \lim_{x \rightarrow -\infty} - \frac{\sigmoid'(x)}{x^{-2}} 
        = -\lim_{x \rightarrow -\infty} x^2 \sigmoid'(x) 
        = 0, 
    \]
    and 
    \[
        \lim_{x \rightarrow -\infty} \sigma'(x) 
        = \lim_{x \rightarrow -\infty} x \sigmoid'(x) + \sigmoid(x) 
        = 0,
    \]
    since $\Phi(x) \rightarrow 0$ as $x \rightarrow -\infty$.
\end{proof}

Using $\srla$ activation functions, we can construct several useful operations out of simple feedforward neural networks, which we introduce below.
All of these functions can be identified with local FNO layers when viewed as superposition operators,
and are differentiable when the input space is $L^{\infty}(\bT^d)$.
The first is a clipping function that approximates the identity function within an interval and is bounded outside of this interval.

\begin{lemma}[Clipping Function]\label{lemma:clipping}
Let $\sigma \in \srla$. 
For $\theta, b > 0$, define
\begin{equation}\label{eq:clipping}
   \clip(x;\theta, b) 
    :=  \frac{\sigma(\theta(x+b)) - \sigma(\theta(x-b))}{\theta} - b.
\end{equation}
Then, for any $\epsilon, R > 0$, there exist $\theta, b > 0$ such that 
the following hold:
\begin{enumerate}[label=(\roman*)]
\item $\clip(x; \theta, b)$ is an approximation to the identity function on $[-R, R]$, i.e., 
        \[ \|\clip(\cdot; \theta,b) - \id \|_{\cont^1([-R, R])} \leq \epsilon. \]
\item There exists $M = M(\sigma, \theta, b) > 0$ such that $|\clip(x; \theta, b)| \leq M$ and $|\clip'(x; \theta, b)| \leq M$ for all $x \in \bR$.
\end{enumerate}
\end{lemma}

\begin{proof}
    For conciseness, we will write $\clip(x) = \clip(x; \theta, b)$. First, we recognize that
    \[
        \clip(x) = (x+b) \sigmoid(\theta(x + b)) - (x - b) \sigmoid(\theta(x - b)) - b,
    \]
    which implies that 
    \[
        \clip(0) = b \sigmoid(\theta b) + b \sigmoid(-\theta b) - b 
        = 0,
    \]
    due to the assumption that $\sigmoid(x) + \sigmoid(-x) = 1$.
    Additionally, its derivative is given by
    \[
        \clip'(x) = \sigma'(\theta(x + b)) - \sigma'(\theta(x - b)).
    \]
    By the mean value theorem, for any $x \in [-R, R]$, we have
    \[
        \clip(x) - x = (\clip'(\xi(x)) - 1)x,
    \]
    where $\xi(x) \in [0, x] \subset [-R, R]$,
    such that the function error satisfies
    \[
       \sup_{x \in [-R, R]} |\clip(x) - x| \leq R \sup_{x \in [-R, R]} |\clip'(x) - 1|.
    \]
    On the other hand, the derivative error is simply 
    \[
       \sup_{x \in [-R, R]} |\clip'(x) - 1|.
    \]
    Thus, it suffices to select $\theta$ and $b$ such that 
    \[
       \sup_{x \in [-R, R]} |\clip'(x) - 1| \leq \max\{\epsilon, \epsilon/R \}.
    \]

    To this end, we take $b = 2R$, such that whenever $x \in [-R, R]$, we have 
    $\theta(x + b) \geq \theta R$ and $\theta(x - b) \leq -\theta R$.
    By the limiting behavior of $\sigma'$ given in \Cref{proposition:properties_of_gelu}, for any $\epsilon_1, \epsilon_2 > 0$,
    there exist $T_1 > 0$ such that $|\sigma'(t) - 1| \leq \epsilon_1$ whenever $t \geq T_1$
    and $T_2 < 0$ such that $|\sigma'(t)| \leq \epsilon_2$ whenever $t \leq T_2$.
    Therefore, taking $\epsilon_1 = \epsilon_2 = \max\{\epsilon, \epsilon/R\}/2$,
    and choosing $\theta$ such that 
    \[
         \theta R \geq \max \{ |T_1|, |T_2| \},
    \]
    yields 
    \[
        | \clip'(x) - 1| \leq | \sigma'(\theta(x + b)) - 1|  + | \sigma'(\theta(x - b)) | \leq \epsilon,
    \]
    since $\theta(x + b) \geq T_1$ and $\theta(x - b) \leq T_2$. 
    It then follows that $\|\clip - \id\|_{\cont^1([-R,R])} \leq \epsilon.$

    To see boundedness, we note that $\clip(x)$ can be written as 
    \[
        \clip(x) = \frac{1}{\theta} \int_{\theta(x - b)}^{\theta(x+b)} \sigma'(t) dt - b.
    \]
    This means 
    \[
        | \clip(x) | \leq \frac{\theta(x+b) - \theta(x-b)}{\theta} \sup_{t \in \bR} |\sigma'(t)| + b 
        = \frac{2b}{\theta} \sup_{t \in \bR} |\sigma'(t)| + |b|, 
    \]
    where $\sup_{t \in \bR} |\sigma'(t)| < \infty$ since $\sigma'$ is continuous and has bounded limits at $\pm \infty$.
    On the other hand, 
    \[
        | \clip'(x) | \leq |\sigma'(\theta(x + b)) - \sigma'(\theta(x - b))| \leq 2 \sup_{t \in \bR} |\sigma'(t)|.
    \]
    Thus, for $b = 2R$ and our choice of $\theta$, we can take 
    \[
        M(\sigma, \theta, b) = \max  \left\{\frac{2 b}{\theta} \sup_{t \in \bR} |\sigma'(t)| + b , 2 \sup_{t \in \bR} |\sigma'(t)| \right\},
    \]
    noting that the $2 b/\theta \sup_{t \in \bR} |\sigma'(t)|$ term can be made arbitrarily small by choosing $\theta$ arbitrarily large without sacrificing the approximation error $\|\clip - \id\|_{\cont^1([-R,R])}$.
\end{proof}

We can also construct a cutoff function, $\fcutoff(x)$, 
such that $\fcutoff(x) \approx 1$ for small inputs and $\fcutoff(x) \approx 0$ for large inputs.

\begin{lemma}[Cutoff Function]\label{lemma:cutoff}
Let $\sigma \in \srla$. 
For $\theta, b > 0$, define
\begin{equation}\label{eq:cutoff}
    \fcutoff(x; \theta, b) :=  \frac{\sigma(\theta(x-b-1)) - \sigma(\theta(x-\theta))}{\theta}+1
\end{equation}
Then, for any $\epsilon > 0$ and $R > 1$, 
there exist $\theta, b > 0$ such that the following hold:
\begin{enumerate}
    \item $|\fcutoff(x; \theta, b) - 1| \leq \epsilon$ \label{item:cutoff_small}
    and $|\fcutoff'(x; \theta, b)| \leq \epsilon$ for $0 \leq x \leq R$;
    \item $|\fcutoff(x; \theta, b)| \leq \epsilon$ and $|\fcutoff'(x; \theta, b)| \leq \epsilon$ for $x \geq 4R$. \label{item:cutoff_large}
\end{enumerate}
Moreover, there exists a constant $M = M(\sigma)$ such that 
given any $\theta, b > 0$, 
it holds that 
$|\fcutoff(x;\theta, b)| \leq M$
and $|\fcutoff'(x; \theta, b)| \leq M$ for all $x \in \bR$.
\end{lemma}

\begin{proof}
Take $b = 2R$. 
We begin by noting that the derivative of $\fcutoff(x) = \fcutoff(x;\theta, b)$ is given by
\[
    \fcutoff'(x) = \sigma'(\theta(x-b-1)) - \sigma'(\theta(x-b)).
\]
Suppose $x \in [0, R]$. 
Then, 
the arguments $x - b - 1 \leq - (R +1)$ and 
$x - b \leq -R$. Thus, $\theta(x-b-1)$ and $\theta(x-b)$ are both negative.
From the limiting behavior of $\sigma$ given in \Cref{proposition:properties_of_gelu}, for any $\epsilon_1$
there exist $T_1, T_1' < 0$ such that $|\sigma(t)| \leq \epsilon_1$
whenever $t \leq T_1$
and $|\sigma'(t)| \leq \epsilon_1$ 
whenever $t \leq T_1'$.
Thus, if $\theta$ is chosen such that 
$\theta R \geq \max \{|T_1|, |T_1'|, 1\}$,
we have 
\[
    | \fcutoff(x) - 1 | \leq |\sigma(\theta(x-b-1) - \sigma(\theta(x-b))| 
    \leq 2 \epsilon_1, 
\]
and simultaneously 
\[
    |\fcutoff'(x)| 
        = |\sigma'(\theta(x-b-1)) - \sigma'(\theta(x-b))| \leq 2 \epsilon_1.
\]

Suppose $x \geq 4R$. 
Note that we can write $\fcutoff$ as 
\[
    \fcutoff(x) = \frac{1}{\theta} \int_{\theta(x-b-1)}^{\theta(x-b)} 1 - \sigma'(t) dt ,
\]
which implies that 
\[
    |\fcutoff(x)| 
    \leq 
    \frac{1}{\theta} \int_{\theta(x-b-1)}^{\theta(x-b)} |1 - \sigma'(t)| dt
    \leq 
    \sup_{\theta(x-b-1) \leq t \leq \theta(x-b)} | 1 - \sigma'(t)| .
\]
Here, $x - b - 1 \geq R$ 
and $x - b \geq 2R$.
By analogous reasoning to the previous case, 
there exist $T_2 > 0$ such that 
$|\sigma'(t) - 1| \leq \epsilon_2$
whenever $t \geq T_2$. 
If $\theta$ is chosen such that $\theta R \geq T_2$, we have 
$\theta(x - b - 1) \geq T_2$ and $\theta(x - b) \geq T_2$ such that 
$ |\fcutoff(x)| \leq \epsilon_2. $
Analogously, we have 
\[
    |\fcutoff'(x)| 
        = |\sigma'(\theta(x-b-1)) - \sigma'(\theta(x-b))| 
        \leq |\sigma'(\theta(x-b-1)) - 1| + |1 - \sigma'(\theta(x-b))| 
        \leq 2 \epsilon_2
\]
for the derivative.
Thus, \Cref{item:cutoff_small,item:cutoff_large} can be obtained by choosing $\epsilon_1 = \epsilon_2 = \epsilon/2$, and subsequently choosing 
$\theta \geq \max \{ |T_1|, |T_1'|, |T_2|, 1\}/R$.

To see the global bounds, we note that 
\[
    |\fcutoff(x)| 
    \leq 
    \frac{1}{\theta} \int_{\theta(x-b-1)}^{\theta(x-b)} |1 - \sigma'(t)| dt
    \leq 
    \sup_{t \in \bR} | 1 - \sigma'(t)|
    \leq 
    \sup_{t \in \bR}  1 + |\sigma'(t)|
\]
and 
\[
    |\fcutoff'(x)| \leq 2 \sup_{t \in \bR} |\sigma'(t)|.
\]
Thus, we can take $M = 2 \sup_{t \in \bR} |\sigma'(t)| + 1 < \infty$.
Importantly, $M$ does not depend on the values of $\theta$ or $b$.
\end{proof}

Finally, we will also use a global approximation to the absolute value function.
\begin{lemma}[Absolute Value Function]\label{lemma:absval}
   Let $\sigma \in \srla$. For $\theta > 0$, define
   \begin{equation}\label{eq:magnitude}
        \absval(x; \theta) := \frac{\sigma(\theta x)+\sigma(-\theta x)}{\theta}
   \end{equation}
   Then, for any $\epsilon > 0$, there exists $\theta > 0$ such that $\absval(x) \geq 0$ and $|\absval(x; \theta) - |x|| \leq \epsilon$ for all $x \in \bR$.
   Moreover, there exists a constant $M = M(\sigma)$ such that given any $\theta > 0$,
   $|\absval'(x; \theta)| \leq M$ for all $x \in \bR$.
\end{lemma}
\begin{proof}
    We will write $\absval(x) = \absval(x;\theta)$.
    Since $\absval(x) = \absval(-x)$, we only need to consider the case $x \geq 0$.
    Firstly, we have
    \begin{equation*}
        \absval(x) = x \sigmoid(\theta x) - x \sigmoid(-\theta x) \geq 0 \qquad \forall x \geq 0
    \end{equation*}
    due to the monotonicity of $\sigmoid$.
    
    To see the approximation error, we note that for $x \geq 0$, we have 
    \[
        \absval(x) - |x| = \absval(x) - x = x \left( \sigmoid(\theta x) - \sigmoid(-\theta x) - 1 \right).
    \]
    Using the symmetry of the sigmoid, i.e., $\sigmoid(x) = 1 - \sigmoid(-x)$, 
    we can write this as 
    \[
        \absval(x) - x = - 2 x \sigmoid(-\theta x) = - \frac{2}{\theta} \sigma(- \theta x). 
    \]
    Since $\sigma(t) \rightarrow 0$ as $t \rightarrow -\infty$, $\sigma(t)$ is bounded on $(-\infty, 0]$. Therefore, we can choose $\theta$ sufficiently large such that 
    \[
        \theta \geq \frac{2 \sup_{t \leq 0} |\sigma(t)|}{\epsilon},
    \]
    for which 
    \[
        | \absval(x) - x | \leq \frac{2}{\theta} \sup_{t \leq 0} | \sigma(t) | \leq \epsilon.
    \]
    
    Finally, for the derivative, we recognize that 
    \[
        | \absval'(x) | = | \sigma'(\theta x) - \sigma'(\theta x) | \leq 2 \sup_{t \in \bR} |\sigma'(t) | = M,
    \]
    where $M = M(\sigma) = 2 \sup_{t \in \bR} |\sigma'(t)| < \infty$ does not depend on the choice of $\theta$.
\end{proof}

We can use the approximate absolute value function and the cutoff function to construct a functional acting on $L^2_{N}(\bT^d)$ that cuts off the output in terms of the $L^2$ norm of the input. 
\begin{lemma}[Cutoff Function on $L^2_{N}$]
    \label{lemma:cutoff_l2}
    Let $\sigma \in \srla$ and $N \in \bN$ be fixed. There exist constants $\const = \const(N,d,d_a) > 1$ and $M = M(\sigma, N, d, d_a)$ such that 
    for any $\epsilon > 0$ and $R > 1$, 
    there exists an FNO
    $\funcut \in \cont^1( L^2_{N}(\bT^d;\bR^{d_a}) ; \bR)$ 
    with activation function $\sigma$ such that the following hold:
    \begin{enumerate}[label=(\roman*)]
        \item $| \funcut (v) - 1| \leq \epsilon$ and $\|D \funcut (v) \|_{\Op(L^2_{N}, \bR)} \leq \epsilon$ whenever $\|v\|_{L^2_{N}} \leq R$;
        \item $| \funcut (v)| \leq \epsilon$ and $\|D \funcut (v) \|_{\Op(L^2_{N}, \bR)} \leq \epsilon$ whenever $\|v\|_{L^2_{N}} \geq \const R$; 
        \item $|\funcut(v)| \leq M $ and $\| D \funcut(v)\|_{\Op(L^2_{N}, \bR)} \leq M $ for all $v \in L^2_{N}(\bT^d;\bR^{d_a})$.
    \end{enumerate}
\end{lemma}
\begin{proof}
    We consider an FNO of the form
    \begin{equation}\label{eq:cutoff_fno}
        \funcut (v)(x) := 
        \fcutoff\left( 
            \int_{\bT^d} \sum_{i=1}^{d_a}\absval(v_i(x)) dx
        \right)
    \end{equation}
    using $\absval(x) = \absval(x;\theta_{\mathrm{abs}}, b_{\mathrm{abs}})$ 
    and $\fcutoff(x) = \fcutoff(x;\theta_{\cutoff}, b_{\cutoff})$ from \Cref{lemma:absval,lemma:cutoff}, respectively, which depend only on the choice of the activation function $\sigma$.
    We recognize that when viewed as an FNO, $\funcut $ maps to the space of constant functions. Moreover, it is composed of three operations
    \begin{enumerate}[label=(\roman*)]
        \item A local nonlinear operation $v \mapsto \sum_{i=1}^{d_a}\absval(v_i(\cdot))$ .
        \item A nonlocal linear operation $v \mapsto \int_{\bT^d} v \dx$, which can be represented using the Fourier transform $\cF_N^{-1} (P_0 \cF_{N})$, where $P_0(0) = 1$ and $P_0(k) = 0$ otherwise. 
        \item A local nonlinear operation $v \mapsto \fcutoff(v(\cdot))$.
    \end{enumerate}
    Together, these can be viewed as an FNO with two hidden layers.
    In particular, for any $\epsilon_1, \epsilon_2>0$ and $R_1 > 0$, we can construct $\absval$ such that 
    \[
        | \absval(x) - x | \leq \frac{\epsilon_1 }{d_a \cdot |\bT^d|}
        \quad \text{and} \quad
        | \absval'(x) | \leq M_{\tabsval}(\sigma)
        \quad \forall x \in \bR,
    \]
    and $\fcutoff$ such that
    \[
        \begin{cases}
            |\fcutoff(x) - 1| \leq \epsilon_2 \quad \text{and} \quad |\fcutoff'(x)|\leq \epsilon_2 & 0 \leq x \leq R_1, \\
            |\fcutoff(x)| \leq \epsilon_2 \quad \text{and} \quad |\fcutoff'(x)|\leq \epsilon_2 & x \geq 4R_1, \\
            |\fcutoff(x)| \leq M_{\tcutoff} \quad \text{and} \quad |\fcutoff'(x)| \leq M_{\tcutoff} & 
            \forall x \in \bR, 
        \end{cases} 
    \]
    where $M_{\tcutoff} = M_{\tcutoff}(\sigma)$ and $M_{\tabsval} = M_{\tabsval}(\sigma)$ are the global bounds on $\fcutoff$ and $\absval$ given by \Cref{lemma:cutoff,lemma:absval}, respectively.
    
    Our goal is to select $\const$, along with $\epsilon_1, \epsilon_2$ and $R_1$ such that the desired properties hold. 
    To this end, let 
    \[
        \cN(v) : L^2_{N} \ni v \rightarrow \int_{\bT^d} \sum_{i=1}^{d_a} \absval(v_i(x)) dx,
    \]
    which approximates $\|v\|_{L^1}$ in the sense that $\cN(v) \geq 0$ and 
    \[
        \left| \cN(v) - \|v\|_{L^1} \right| \leq \int_{\bT^d} \sum_{i=1}^{d_a} \left| \absval(v_i(x)) - |v_i(x)| \right| dx \leq \epsilon_1.
    \]
    Moreover, let $\const_{1,2} = \const_{1,2}(N,d, d_a)$ and 
    $\const_{2,1} = \const_{2,1}(N,d, d_a)$ 
    be the equivalence constants for the $L^1(\bT^d;\bR^{d_a})$ and $L^2(\bT^d;\bR^{d_a})$ norms on $L^2_{N}(\bT^d;\bR^{d_a})$, i.e.,
    \[
        \const_{2,1} \|v\|_{L^2} \leq \| v \|_{L^1} \leq \const_{1,2} \|v\|_{L^2}.
    \]
    We observe that $\cN(v)$, which approximates $\|v\|_{L^1}$, provides an estimate of $\|v\|_{L^2_N}$ up to approximation errors and equivalence constants. That is,
    \[
        |\cN(v)| \leq \|v\|_{L^1} + \epsilon_1 \leq \const_{1,2} \|v\|_{L^2} + \epsilon_1, 
    \]
    and 
    \[
        |\cN(v)| \geq \|v\|_{L^1} - \epsilon_1 \geq \const_{2,1}\|v\|_{L^2} - \epsilon_1.
    \]
    
    Our strategy is to select the approximation error $\epsilon_1$, the radius $R_1$, and the constant $\const$, such that the output of $\cN$ falls in the appropriate ranges of $\fcutoff$
    after accounting for the approximation errors and the equivalence constants.
    To do so, we first select $\epsilon_1 \leq \min\{\const_{1,2}/2, \const_{2,1}/2, 1\}$, from which $\absval$ is constructed.
    Under the setting of $R > 1$, for any $\const_1 > 1/2$, we have
    \[
        \begin{cases}
            |\cN(v)| \leq 3\const_{1,2} R/2 & \|v\|_{L^2} \leq R, \\ 
            |\cN(v)| \geq (\const_1-1/2) \const_{2,1} R & \|v\|_{L^2} \geq \const_1R, \\ 
        \end{cases}
    \]
    Thus, we can take 
    \[
    R_1 = 3 \const_{1,2}(N,d,d_a)R/2, \qquad 
    \const_1 = \const_1(N,d, d_a) = \frac{6 \const_{1,2}(N,d, d_a)}{\const_{2,1}(N,d,d_a)} + \frac{1}{2},
    \]
    noting that by definition, $\frac{\const_{1,2}(N,d, d_a)}{\const_{2,1}(N,d, d_a)} \geq 1$.
    These choices lead to $\cN(v)$ taking the following values:
    \[
        \begin{cases}
            |\cN(v)| \leq R_1 & \|v\|_{L^2} \leq R, \\ 
            |\cN(v)| \geq 4 R_1 & \|v\|_{L^2} \geq \const_1 R. \\ 
        \end{cases}
    \]
    
    Composing this $\cN$ with the chosen form of $\fcutoff$, we have $\funcut(v) = \fcutoff(\cN(v))$ satisfying 
    \[
        \begin{cases}
            |\funcut(v) - 1| \leq  \epsilon_2 & \|v\|_{L^2} \leq R, \\ 
            |\funcut(v)| \leq \epsilon_2 & \|v\|_{L^2} \geq \const_1 R, \\ 
            |\funcut(v)| \leq M_{\tcutoff}(\sigma) & \forall v \in L^2_{N}(\bT^d). \\ 
        \end{cases}
    \]
    For the derivatives, we begin by noting that for any $v, w \in L^2_{N}(\bT^d;\bR^{d_a})$,
    \[
         D \funcut (v) w  = \fcutoff'(\cN(v)) \int_{\bT^d} \absval'(v(x)) w(x) dx, 
    \]
    such that 
    \[
        | D \funcut (v) w | \leq  | \fcutoff'(\cN(v))| \cdot \| \absval(v(\cdot)) \|_{L^{\infty}} \cdot \|w\|_{L^1} 
        \leq | \fcutoff'(\cN(v))| \cdot M_{\tabsval} \cdot \const_{1,2} \|w\|_{L^2}, 
    \]
    Therefore, the derivatives have the bounds
    \[
        \| D \funcut (v) \|_{\Op(L^2_{N}, \bR)} \leq  M_{\tabsval} \const_{1,2} | \fcutoff'(\cN(v))|,
    \]
    which again can be split into the following cases based on the outputs of $\cN(v)$:
     \[
        \begin{cases}
            \| D \funcut (v) \|_{\Op(L^2_{N}, \bR)}
            \leq M_{\tabsval} \const_{1,2} \cdot \epsilon_2 
                & \|v\|_{L^2} \leq R, \\ 
            \| D \funcut (v) \|_{\Op(L^2_{N}, \bR)}
            \leq M_{\tabsval} \const_{1,2} \cdot \epsilon_2  
                & \|v\|_{L^2} \geq \const_1 R, \\ 
            \| D \funcut (v) \|_{\Op(L^2_{N}, \bR)} \leq M_{\tabsval} \const_{1,2}  M_{\tcutoff} 
            & \forall v \in L^2_{N}.
        \end{cases}
    \]
    Combining the results of the function value and the derivatives, 
    we can choose 
    \[
        \epsilon_2 = \min \left\{ \epsilon, \frac{\epsilon}{M_{\tabsval} \const_{1,2}}  \right\}
    \]
    from which $\fcutoff$ is constructed.
    The lemma then holds with 
    \[
        \const(N,d,d_a) = \frac{6 \const_{1,2}(N,d,d_a)}{\const_{2,1}(N,d,d_a)} + \frac{1}{2}, 
        \quad M(\sigma, N, d, d_a)= M_{\tcutoff}(\sigma) \max \{1, M_{\tabsval}(\sigma) \const_{1,2}(N, d, d_a)\}.
    \]
    
\end{proof}

\Cref{lemma:cutoff_l2} can be translated into an analogous cutoff functional on $H^s$ with respect to the $L^2$ norm of the Fourier projection of $v \in H^s$.
\begin{corollary}
    \label{corollary:cutoff_hs}
    Let $\sigma \in \srla$ and $N \in \bN$ be fixed. There exist constants $\const = \const(N,d,d_a)$ and $M = M(\sigma, N, d, d_a)$ such that 
    for any $\epsilon > 0$ and $R > 1$, there exists an FNO 
    $\cN \in \cont^1(H^{s}(\bT^d; \bR^{d_a}) ; \bR)$ 
    with activation function $\sigma$ such that the following hold:
    \begin{enumerate}[label=(\roman*)]
        \item $| \cN (v)| \leq 1 + \epsilon$ and $\|D \cN (v) \|_{\HS(H^s(\bT^d), \bR)} \leq \epsilon$ whenever $\|\cP_N v\|_{L^2} \leq R$;
        \item $| \cN (v)| \leq \epsilon$ and $\|D \cN (v) \|_{\HS(H^{s}(\bT^d), \bR)} \leq \epsilon$ whenever $\|\cP_N v\|_{L^2} \geq \const R$;
        \item $|\cN(v)| \leq M$ and $\| D \cN(v)\|_{\HS(H^{s}(\bT^d), \bR)} \leq M$ for all $v \in H^{s}(\bT^d;\bR^{d_a})$.
    \end{enumerate}

\end{corollary}
\begin{proof}
    We consider the cutoff function $\funcut : L^2_N \rightarrow \bR$ satisfying \Cref{lemma:cutoff_l2} with the given $\epsilon, N, R$,
    and take $\cN := \funcut \circ \cP_N$,
    where $\cP_N : H^s(\bT^d;\bR^{d_a}) \rightarrow L^2_{N}(\bT^d;\bR^{d_a})$.
    Then, the bounds for $|\cN(v)| = |\funcut(\cP_N v)|$ 
    are a direct consequence of applying \Cref{lemma:cutoff_l2} for the inputs $\cP_N v \in L^2_{N}(\bT^d;\bR^{d_a})$.
    
    For the derivatives, we begin by noting that since the output space is one-dimensional, $\Op(H^s(\bT^d;\bR^{d_a}), \bR) = \HS(H^s(\bT^d;\bR^{d_a}), \bR)$, and their norms are equal.
    Moreover,
    since $\|\cP_N v\|_{L^2} \leq \|v\|_{H^s}$, we have
    \begin{align*}
        \| D\cN(v) \|_{\Op(H^s{(\bT^d;\bR^{d_a}), \bR)}} 
        & = 
        \| D \funcut(\cP_N v) \cP_N \|_{\Op(H^s(\bT^d;\bR^{d_a}), \bR)} \\
        & \leq 
        \| D \funcut(\cP_N v) \|_{\Op(L^2_N(\bT^d;\bR^{d_a}), \bR)} 
        \| \cP_N \|_{\Op(H^s(\bT^d;\bR^{d_a}), L^2_N(\bT^d;\bR^{d_a}))} \\
        & \leq 
        \| D \funcut(\cP_N v) \|_{\Op(L^2_N(\bT^d;\bR^{d_a}), \bR)},
    \end{align*}
    which implies that the derivative bounds from \Cref{lemma:cutoff_l2} also hold.
\end{proof}

\subsection{Truncation Errors}
Similar to \Cref{lemma:truncation}, we show that under locally Lipschitz assumptions with appropriate moment conditions, 
the truncation of the operator $\cG$ to a finite collection of Fourier modes $|k|_{\infty} \leq N$ converges to $\cG$ in the $W^{1,2}_{\mu, \delta}$ norm. 
\begin{lemma}\label{lemma:truncation_mu}
    Let $s, s', \delta \geq 0$, and let $\cX := H^s(\bT^d; \bR^{d_a})$, $\cY := H^{s'}(\bT^d; \bR^{d_u})$, and $\cX_{\delta} := H^{s + \delta}(\bT^d; \bR^{d_a})$.
    Suppose $\mu$ is a Borel probability measure on $\cX$ and $\cG \in \cont^{1}(\cX;\cY)$, which together satisfy \Cref{assumption:measurability}.
    Then, for any $\epsilon > 0$, there exists $N_{\epsilon} \in \bN$ such that
    \begin{equation}\label{eq:truncation_error_mu}
        \| \cG - \cP_N \circ \cG \circ \cP_N \|_{L^2_{\mu}(\cX; \cY)}^2 
        + 
        \| D \cG - D (\cP_N \circ \cG \circ \cP_N) \|_{L^2_{\mu}(\cX; \HS(\cX_{\delta}; \cY))}^2 
        \leq \epsilon^2 
    \end{equation}
    whenever $N \geq N_{\epsilon}$.
\end{lemma}
\begin{proof}
    We will present the proof for $d_a = d_u = 1$. The cases for $d_a, d_u \geq 1$ are analogous.

    \paragraph{Output Error}
    Starting with errors in the output values, for any $N \in \bN$, we have 
    \[
        \|\cG(a) - \cP_N \cG(\cP_N a) \|_{\cY} \leq 
        \|\cG(a) - \cP_N \cG(a) \|_{\cY} + 
        \|\cP_N \cG(a) - \cP_N \cG(\cP_N a) \|_{\cY}
    \]
    such that 
    \[
        \| \cG - \cP_N \circ \cG \circ \cP_N \|_{L^2_{\mu}(\cX; \cY)}^2 
        \leq 
        2 \int_{\cX} \| \cG(a) - \cP_N \cG(a)\|_{\cY}^2 d \mu(a)
        + 2 \int_{\cX} \| \cP_N \cG(a) - \cP_N \cG(\cP_N a)\|_{\cY}^2 d \mu(a).
    \]
    For $\cY = H^{s'}(\bT^d)$, we have 
    \[
        \| \cG(a) - \cP_N \cG(a)\|_{\cY}^2 = (2\pi)^{d} \sum_{|k|_{\infty} > N} (1 + |k|^2)^{s} \left| \cF(\cG(a))(k) \right|^2,
    \]
    such that the first integral can be written as  
    \[
        \int_{\cX} \| \cG(a) - \cP_N \cG(a) \|_{\cY}^2 d \mu(a) 
        = \sum_{|k|_{\infty} > N} | \cF(\cG(a))(k)|^2.
    \]
    Since $\cG \in L^2_{\mu}(\cX;\cY)$, $\bE_{a \sim \mu}[\|\cG(a)\|_{\cY}^2]$ is finite. Moreover, we have by Fubini's theorem
    \[
        \bE_{a \sim \mu}[\|\cG(a)\|_{\cY}^2]
        = \int_{\cX} \sum_{k \in \bZ^{d}} (1 + |k|^2)^{s} |\cF(\cG(a))(k)|^2 d\mu(a)
        = \sum_{k \in \bZ^{d}} \int_{\cX} (1 + |k|^2)^{s} |\cF(\cG(a))(k)|^2 d\mu(a).
    \]
    Finiteness of this sum suggests that the trailing part converges monotonically to zero, i.e.,
    \begin{equation}\label{eq:function_trailing_converge}
        \sum_{|k|_{\infty} \geq N} \int_{\cX} (1 + |k|^2)^{s} |\cF(\cG(a))(k)|^2 d\mu(a) \rightarrow 0 \qquad \text{as } N \rightarrow \infty.
    \end{equation}

    For the second integral, we have 
    \begin{align*}
        \int_{\cX} \| \cP_N \cG(a) - \cP_N \cG(\cP_N a)\|_{\cY}^2 d \mu(a)
        & \leq 
        \int_{\cX} \| \cG(a) - \cG(\cP_N a)\|_{\cY}^2 d \mu(a) \\
        & \leq 
        \int_{\cX} L_0(a, \cP_N a)^2 \| a - \cP_N a\|^2_{\cX} d \mu (a).
    \end{align*}
    Applying the Cauchy--Schwarz inquality, we have 
    \[
        \int_{\cX} L_0(a, \cP_N a)^2 \| a - \cP_N a\|^2_{\cX} d \mu (a) 
        \leq 
        \left( \int_{\cX} L_0(a, \cP_N a)^4 d \mu(a) \right)^{\!\!\frac{1}{2}}
        \left( \int_{\cX} \| a - \cP_N a\|^4_{\cX} d \mu (a)  \right)^{\!\!\frac{1}{2}}.
    \]
    By \Cref{assumption:measurability}, we have $L_0(a, \cP_N a) \leq L_0(a, a)$ and
    \[
        \left( \int_{\cX} L_0(a, \cP_N a)^4 d \mu(a) \right)^{\!\!\frac{1}{2}} 
        \leq \left( \int_{\cX} L_0(a, a)^4 d \mu(a) \right)^{\!\!\frac{1}{2}} 
        = M_0 < \infty.
    \]
    For the second integral $\int_{\cX} \|a - \cP_N a\|_{\cX}^4$, we have $\|a - \cP_N a\|_{\cX}^4 \rightarrow 0$ for all $a \in \cX$. 
    Moreover, $\|a - \cP_N a\|_{\cX}^4 \leq \|a\|^4_{\cX}$, which is integrable due to the moment assumption in \Cref{assumption:measurability}.
    The dominated convergence theorem then implies that
    \[
        \left( \int_{\cX} 
            \|a - \cP_N a\|_{\cX}^4 d\mu(a)
        \right)^{\!\!\frac{1}{2}} 
        \qquad \text{as } N \rightarrow \infty,
    \]
    where this convergence is monotonic. 
    Thus, the upper bound on the error $\int_{\cX} \| \cP_N \cG(a) - \cP_N \cG(a \cP_Na)\|_{\cY}^2 d\mu(a)$ also converges monotonically to zero as $N \rightarrow \infty$.

    \paragraph{Derivative Error} 
    We now move to derivative accuracy, which can be shown in a similar manner. 
    The triangle inequality implies 
    \begin{align*}
        \|D \cG(a) - D(\cP_N \circ \cG \circ \cP_N)(a) \|_{\HS(\cX_{\delta}; \cY)}^2 
        &\leq 
        3 \|D \cG(a) - \cP_N D \cG(a) \|_{\HS(\cX_{\delta}; \cY)}^2 \\
        & \quad + 3 \|\cP_N D \cG(a) - \cP_N D \cG(a) \cP_N \|_{\HS(\cX_{\delta}; \cY)}^2 \\
        & \quad + 3 \|\cP_N D \cG(a) \cP_N - \cP_N D \cG(\cP_N a) \cP_N \|_{\HS(\cX_{\delta};\cY)}^2,
    \end{align*}
    In the first term, we first note that 
    \[
    \| D \cG(a) - \cP_N D \cG(a) \|_{\HS(\cX_{\delta}; \cY)}^2
        = 
        \| D \cG(a)^* - D \cG(a)^* \cP_N \|_{\HS(\cY; \cX_{\delta})}^2.
    \]
    We can then evaluate the Hilbert--Schmidt norm using the sinusoidal basis, such that 
    \[
        \| D \cG(a)^* - D \cG(a)^* \cP_N \|_{\HS(\cY; \cX_{\delta})}^2
        = \sum_{k \in \zlex, |k|_{\infty} > N} 
        \| D \cG(a)^* \varphi_{k, s'}^{\cos} \|_{\cX_{\delta}}^2 +
        \| D \cG(a)^* \varphi_{k, s'}^{\sin} \|_{\cX_{\delta}}^2.
    \]
    Similar to the projection error of $a$, the fact that $D \cG \in L^2_{\mu}(\cX; \HS(\cX_{\delta}, \cY))$
    implies that 
    \[
        \int_{\cX} \| D \cG(a) - \cP_N D \cG(a)\|_{\HS(\cX_{\delta}; \cY)}^2 d \mu(a) 
        = \int_{\cX} \sum_{k \in \zlex, |k|_{\infty} > N} 
        \| D \cG(a)^* \varphi_{k, s'}^{\cos} \|_{\cX_{\delta}}^2 +
        \| D \cG(a)^* \varphi_{k, s'}^{\sin} \|_{\cX_{\delta}}^2 d\mu(a)
    \]
    converges monotonically to zero as $N \rightarrow \infty$.
    For the second term, we can use an analogous argument to show that 
    \begin{align*}
        & \int_{\cX} \| \cP_N D \cG(a) - \cP_N D \cG(a) \cP_N \|_{\HS(\cX_{\delta}; \cY)}^2 d \mu(a)  \\
        & \qquad \leq 
        \int_{\cX} \sum_{k \in \zlex, |k|_{\infty} > N} 
        \| D \cG(a)^* \varphi_{k, s'}^{\cos} \|_{\cX_{\delta}}^2 +
        \| D \cG(a)^* \varphi_{k, s'}^{\sin} \|_{\cX_{\delta}}^2 d\mu(a),
    \end{align*}
    which converges monotonically to zero as $N \rightarrow \infty$.    
    
    It remains to consider the term involving
    $
    \| \cP_N D \cG(a) \cP_N - \cP_N D \cG(\cP_N a) \cP_N \|_{\HS(\cX_{\delta}; \cY)}.
    $
    We use the same argument based on the local Lipschitz constant. 
    That is, using the local Lipschitz constant, along with the fact that $\cP_N$ is an orthogonal projection, we have
    \begin{align*}
        &\int_{\cX} 
        \| \cP_N D \cG(a) \cP_N - \cP_N D \cG(\cP_N a) \cP_N \|_{\HS(\cX_{\delta}; \cY)}^2
        d \mu(a) \\
        & \qquad \leq
        \left( \int_{\cX} L_1(a, P_N a)^4 d\mu(a) \right)^{\!\!\frac{1}{2}}
        \left( \int_{\cX} \| a - \cP_N a \|_{\cX}^4 d \mu(a) \right)^{\!\!\frac{1}{2}} \!\! \\
        & \qquad \leq
        M_1 
        \left( \int_{\cX} \| a - \cP_N a \|_{\cX}^4 d \mu(a) \right)^{\!\!\frac{1}{2}} \!\!.
    \end{align*}
    This upper bound again converges monotonically to zero as $N \rightarrow \infty$ due to the moment assumption on $\mu$.

    \paragraph{Combined Output and Derivative Error}
    To obtain the result in \eqref{eq:truncation_error_mu}, it suffices to note that 
    since the bounds on the errors all decay monotonically with respect to $N$,
    $N_{\epsilon}$ can be selected such that both the output and the derivative errors are simultaneously small for all $N \geq N_{\epsilon}$.
\end{proof}

\subsection{Approximation of the Truncated Operator}
Next, we show that the truncated operator $\cG_N = \cP_N \circ \cG_N \circ \cP_N$ can be approximated in $W^{1,2}_{\mu, \delta}$.
To do so, we construct the FNO $\cN$ approximation to be accurate on a sufficiently large ball $B_R$, and then apply a multiplicative cutoff function such that $\cN$ and $D\cN$ are made small outside of $B_R$. In the following sections, we show that this can indeed be accomplished using the FNO architecture. 

We begin by showing that we can approximate the mapping to the Fourier coefficients, $G_N = \cF_N \circ \cG_N$, to arbitrary accuracy within $B_R$ via an FNO with bounded output values and derivatives.
\begin{lemma}[Bounded Coefficient Approximation]
    \label{lemma:coefficient_approximation_bounded}
    Let $s, s', \delta \geq 0$,
    and let $\cX := H^s(\bT^d; \bR^{d_a})$,
    $\cY := H^{s'}(\bT^d; \bR^{d_a})$,
    and $\cX_{\delta} := H^{s+\delta}(\bT^d; \bR^{d_a})$.
    Suppose $\cG \in \cont^{1}(\cX;\cY)$.
    Furthermore, for a fixed $N \in \bN$, let $\cZ_{N} := L^2_0(\bT^d; \bR^{2K_N})$ and define $G_N : \cX \rightarrow \cZ_{N}$ by
    \begin{equation}
        G_{N} := \cF_N \circ \cG_N = \cF_N \circ \cP_N \circ \cG \circ \cP_N.
    \end{equation}
    
    Then, given $\sigma \in \srla$, for any $\epsilon, R > 0$, there exists an FNO 
    $\cN \in \cont^1( \cX ; \cZ_{N})$ 
    with activation function $\sigma$ approximating $G_N$ such that simultaneously
    \begin{equation}\label{eq:coefficient_output_error}
        \|G_N(a) - \cN (a)\|_{\cZ_{N}} \leq \epsilon
    \end{equation} 
    and
    \begin{equation}\label{eq:coefficient_derivative_error}
        \|D G_N(a) - D \cN(a)\|_{\HS(\cX_{\delta}, \cZ_{N})} \leq \epsilon
    \end{equation} 
    for any $a \in \cX$ with $\|a\|_{\cX} \leq R$. Moreover, $\cN$ is bounded and has bounded derivatives, i.e., there exists an upper bound $M = M(\epsilon, R)$ such that 
    \begin{equation}
        \|\cN (a) \|_{\cZ_{N}} \leq M, \quad \|D \cN(a)\|_{\HS(\cX_{\delta}, \cZ_{N})} \leq M 
    \end{equation}
    for any $a \in \cX$.
    
    \end{lemma}
    \begin{proof}
    This follows from the procedure used to prove the universal approximation result over compact sets, \Cref{theorem:main_ua}.
    Given the radius, $R$, we consider the ball $B_R := B_R(0) = \{\|a\|_{\cX} \leq R\}$
    whose Fourier projection $\cP_N(B_R)$ is compact in $\cX$.
    We note that $G_N = \widehat{\cG}_N \circ \cF_N \circ \cP_N$, where $\widehat{\cG}_N = \cF_N \circ \cG \circ \real \cF_{N}^{-1}$.
    Recall that by \Cref{lemma:ft_approximation}, we can find an FNO $\cN_1$ approximating $\cF_{N} \circ \cP_N$, and by the finite-dimensional universal approximation theorem in Sobolev norms, we can find $\cN_2$ approximating $\widehat{\cG}_N$. These can be combined using \Cref{lemma:composition} to obtain $\widetilde{\cN} = \cN_2 \circ \cN_1$ such that 
    \[
        \| G_{N}(a) - \widetilde{\cN}(a) \|_{\cZ_{N}} \leq \epsilon / 2, \quad \| D G_{N}(a) -  D \widetilde{\cN}(a) \|_{\HS(\cX_{\delta}, \cZ_{N})} \leq \epsilon / 2
    \]
    for any $a \in \cP_N(B_R)$.
    
    This approximant $\widetilde{\cN}$ does not necessarily have uniformly bounded outputs over all of $\cX$. 
    To address this, we begin by recognizing that since $\cP_N(B_R)$ is compact, there exists $M_0 > 0$ such that 
    \[
        \sup_{a \in \cP_N(B_R)}  \|\widetilde{\cN}(a)\|_{L^\infty} \leq M_0,
    \]
    and $M_1 > 0$ such that 
    \[
        \sup_{a \in \cP_N(B_R)}  \|D \widetilde{\cN}(a)\|_{\HS(\cX_{\delta}, \cZ_{N})}
        \leq M_1.
    \]
    Thus, we choose a clipping function $\clip$ from \Cref{lemma:clipping} such that 
    \[
        \| \clip - \id\|_{\cont^1([-M_0, M_0])} \leq \frac{\epsilon}{2 \max \{ 1, \const_{2,\infty}, M_1 \}},
    \]
    where $\const_{2,\infty} = \const_{2,\infty}(N, d)$ is the norm equivalence constant satisfying
    $\|v\|_{L^2} \leq \const_{2,\infty} \|v\|_{L^{\infty}}$ for $v \in L^{\infty}(\bT^d; \bR^{2K_N})$.
    We can then take $\cN = \clip \circ \widetilde{\cN}$, where $\clip$ is applied componentwise to the outputs of $\widetilde{\cN}$, i.e.,
    \[
        [ \cN(a) ]_{i} = \clip( [ \widetilde{\cN}(a)]_i ), \quad i = 1, \dots, 2K_N,
    \]
    whose derivative takes the form 
    \[
        D [\cN (a)]_i = \clip'(\widetilde{\cN}_i(a)) D[\widetilde{\cN}(a)]_i, \qquad i = 1, \dots, 2K_N.
    \]
    We will write this as 
    \[
        D \cN(a) = \clip'( \widetilde{\cN}(a) ) D \widetilde{\cN}(a),
    \]
    where the operator $\clip'( \widetilde{\cN}(a)) \in \cL(\cZ_{N}, \cZ_{N})$ can be associated with the diagonal $2K_N \times 2K_N$ matrix whose diagonal values are
    \[
        [\clip'(\widetilde{\cN}(a)]_{ii} = \clip'([\widetilde{\cN}(a)]_i), \qquad i = 1, \dots, 2K_N.
    \]

    Under this construction, the output error using the clipped FNO $\cN$ satisfies
    \begin{align*}
        \|G_N(a) - \cN(a)\|_{L^2}
        \leq 
        \| G_N(a) - \widetilde{\cN}(a)\|_{L^2}
        + 
        \| \widetilde{\cN}(a) - \clip(\widetilde{\cN}(a))\|_{L^2},
    \end{align*}
    where for $a \in \cP_N(B_R)$, we have $\| G_N(a) - \widetilde{\cN}(a)\|_{L^2} \leq \epsilon/2$ and 
    \[
         \| \widetilde{\cN}(a) - \clip(\widetilde{\cN}(a))\|_{L^2} 
         \leq
         \const_{2,\infty} \| \widetilde{\cN}(a) - \clip(\widetilde{\cN}(a))\|_{L^{\infty}} 
         \leq \epsilon/2.
    \]
    Thus, $\|G_N(a) - \cN(a)\|_{L^2} \leq \epsilon$
    whenever $a \in \cP_N(B_R)$.
    Similarly, for the derivatives, we have 
    \[
        \|D G_N(a) - D \cN(a) \|_{\HS(\cX_{\delta}, \cZ_{N})} 
        \leq 
        \|D G_N(a) - D \widetilde{\cN}(a) \|_{\HS(\cX_{\delta}, \cZ_{N})} 
        + 
        \|\clip'(\widetilde{\cN}(a)) D \widetilde{\cN}(a) - D \widetilde{\cN}(a) \|_{\HS(\cX_{\delta}, \cZ_{N})}.
    \]
    Again, for $a \in \cP_N(B_R)$, we have 
    $ \|D G_N(a) - D \widetilde{\cN}(a) \|_{\HS(\cX_{\delta}, \cZ_{N})} \leq \epsilon/2 $,
    while 
    \[
        \|\clip'(\widetilde{\cN}(a)) D \widetilde{\cN}(a) - D \widetilde{\cN}(a) \|_{\HS(\cX_{\delta}, \cZ_{N})}
        \leq \| \clip'(\widetilde{\cN}(a)) - I_{2K_N} \|_{\Op(\cZ_{N}, \cZ_{N})}
        \|D \widetilde{\cN}(a) \|_{\HS(\cX_{\delta}, \cZ_{N})} 
        \leq \epsilon/2,
    \]
    where $I_{2K_N}$ denotes the identity matrix on $\bR^{2K_N}$.
    The second inequality follows from the fact that $|\clip'([\cN(a)]_{i}) - 1| \leq \epsilon/(2M_1)$ 
    and $|D \widetilde{\cN}(a) \|_{\HS(\cX_{\delta}, \cZ_{N})} \leq M_1$ for all $a \in \cP_N(B_R)$.
    Thus, both \eqref{eq:coefficient_output_error} and \eqref{eq:coefficient_derivative_error} are satisfied for $a \in \cP_N(B_R)$.
    Moreover, since both $G_N$ and $\cN$ consists first of the projection operator $\cP_N$, i.e.,
    \[
    G_N (a) = G_N( \cP_N a), \qquad \cN(a) = \cN(\cP_N a),
    \]
    the error bounds \eqref{eq:coefficient_output_error} and \eqref{eq:coefficient_derivative_error} additionally hold for any $a \in B_R$.
   
    Finally, to see boundedness of $\cN$, we recall that by construction, $\clip$ has bounded outputs and hence there exists a bound $M_{\mathrm{clip},0}$ such that $\|\cN(a)\|_{\cZ_{N}} \leq M_{\mathrm{clip},0} $.
    To bound the derivative of $\cN$, we note that $\cN$ is an FNO of the form 
    \[
        \cN(a) = \cN \circ \cP_N (a) = \cQ \circ \cL_{L} \circ \cdots \circ \cL_1 \circ \cR \circ \cP_N(a),
    \]
    where each layer can be viewed as an operator $\cL_{\ell} : L^{\infty}(\bT^d;\bR^{d_{\ell-1}}) \rightarrow L^{\infty}(\bT^d;\bR^{d_{\ell}})$ taking the form
    \[
        \cL_{\ell}(v) = \sigma( W_{\ell} v + b_{\ell} + \cK_{\ell} v ).
    \]
    The operator norms of the layers are uniformly bounded for all $a \in \cX$,
    \[
        \| D \cL_{\ell}(v) \|_{\Op(L^{\infty}, L^{\infty})} \leq \|\sigma'(W_{\ell}v + b_{\ell} + \cK_{\ell} v) \|_{L^\infty}
        (\| W_{\ell} \|_{\Op(L^{\infty}, L^{\infty})} + \| \cK_{\ell} \|_{\Op(L^{\infty}, L^{\infty})} )
        < \infty 
    \]
    since $|\sigma'(x)| < \infty$ for all $x \in \bR$.
    Moreover, 
    $\cP_N$, $\cR$, and $\cQ$ 
    can be viewed as operators from the spaces 
    $\cP_N \in \HS(\cX_{\delta}, L^2_N(\bT^d; \bR^{d_a}))$,
    $\cR \in \Op(L^2_N(\bT^d; \bR^{d_a}), L^{\infty}(\bT^d; \bR^{d_0}))$, 
    and  
    $\cQ \in \Op(L^{\infty}(\bT^d; \bR^{d_L}), \cZ_{N})$, 
    which all have bounded norms. 
    Thus, there exists $M_1 >0 $ such that $\|D \cN(a)\|_{\HS(\cX_{\delta}; \cZ_{N})} < M_1$ for all $a \in \cX$.
\end{proof}

Using the bounded approximation of $G_N$ along with the cutoff functional in \Cref{corollary:cutoff_hs},
we can construct a family of operators that approximate $G_N$ in such a way that their growth is uniformly dominated by $G_N$ and its derivative.

\begin{lemma}[Coefficient approximation with Bounded Growth]
    \label{lemma:coefficient_approximation_bounded_growth}
    Consider the setting of \Cref{lemma:coefficient_approximation_bounded}. Then, there is a constant $\const = \const(\sigma, N, d, d_a)$ such that 
    for any $\epsilon > 0$ and $R > 1$, there exists an FNO 
    $\cN \in \cont^1(\cX ; \cZ_{N})$ 
    with activation function $\sigma$ such that 
    \begin{align}
        \|G_N(a) - \cN(a)\|_{\cZ_{N}} &\leq \epsilon, \\
        \|D G_N(a) - D\cN(a)\|_{\HS(\cX_{\delta}, \cZ_{N})} &\leq \epsilon,
    \end{align}
    whenever $\|a\|_{\cX} \leq R$, and 
    \begin{align}
        \label{eq:coefficient_function_growth}
        \| \cN(a) \|_{\cZ_{N}} &\leq \const ( \|G_N(a)\|_{\cZ_{N}} + 1), \\
        \label{eq:coefficient_derivative_growth}
        \| D \cN(a) \|_{\HS(\cX_{\delta}, \cZ_{N})} &\leq \const ( \|G_N(a)\|_{\cZ_{N}} + \|D G_N(a) \|_{\HS(\cX_{\delta}, \cZ_{N})} + 1) ,
    \end{align}
    for all $a \in \cX$.
    Moreover, there exists a bound $M = M(\epsilon)$ such that 
    \begin{equation}
        \label{eq:coefficient_bound_with_cutoff}
        \| \cN(a)\|_{\cZ_{N}} \leq M, \qquad \|D \cN(a) \|_{\HS(\cX_{\delta}, \cZ_{N})} \leq M, 
    \end{equation}
    again for all $a \in \cX$.
\end{lemma}
\begin{proof}
    We consider the radius $R_1 = \const_{\cutoff} R$, where $\const_{\cutoff} = \const_{\cutoff}(\sigma, N,d,d_a)$ is the multiplier for the upper radius for a cutoff function from \Cref{corollary:cutoff_hs}.
    By \Cref{lemma:coefficient_approximation_bounded}, we first approximate $G_N$ over the ball of radius $R_1$ such that
    \[
        \|G_N(a) - \cN_{G}(a)\|_{\cZ_{N}} \leq \epsilon_G, \qquad 
        \|D G_N(a) - D \cN_{G}(a)\|_{\HS(\cX_{\delta}, \cZ_{N})} \leq \epsilon_G, 
    \]
    for any $\| a \|_{\cX} \leq R_1$,
    where $\epsilon_{G}$ can be made arbitrarily small,
    and simultaneously, 
    \[
        \|\cN_{G}(a)\|_{\cZ_{N}} \leq M_{G}, \qquad 
        \|D \cN_{G}(a)\|_{\HS(\cX_{\delta}, \cZ_{N})} \leq M_{G}, 
    \]
    for all $a \in \cX$, 
    where $M_{G} = M_{G}(\epsilon_G)$.

    We then consider the cutoff function from \Cref{corollary:cutoff_hs}, $\cN_{\cutoff} : \cX \rightarrow \bR$ with inner radius $R$ and outer radius $R_1 = \const_{\cutoff} R$, such that for any $\epsilon_{\cutoff} > 0$, there exists $\cN_{\cutoff}$ satisfying
    \[
        \begin{cases}
            |\cN_{\cutoff}(a) - 1| \leq \epsilon_{\cutoff}, \quad 
            \| D \cN_{\cutoff} (v) \|_{\HS(\cX_{\delta}, \bR)}
            \leq \epsilon_{\cutoff}
                & \|\cP_N v\|_{L^2} \leq R, \\ 
            |\cN_{\cutoff}(a)| \leq \epsilon_{\cutoff}, \quad 
            \| D \cN_{\cutoff} (v) \|_{\HS(\cX_{\delta}, \bR)}
            \leq \epsilon_{\cutoff}
                & \|\cP_N v\|_{L^2} \geq R_1, \\ 
            |\cN_{\cutoff}(a)| \leq M_{\cutoff}, \quad 
            \| D \cN_{\cutoff} (v) \|_{\HS(\cX_{\delta}, \bR)}
            \leq M_{\cutoff}
            & \forall v \in \cX,
        \end{cases}
    \]
    where $M_{\cutoff} = M_{\cutoff}(\sigma, N, d)$
    is a bound independent of $\epsilon_{\cutoff}$.
    
    We can control the growth of the function value and derivatives of $\cN_G(a)$ by multiplication with the cutoff function, i.e., by taking $\cT(a) := \cN_{\cutoff}(a) \cdot \cN_G(a)$. 
    To see this, we consider the following cases.
    \paragraph{Case 1: $\|\cP_N a\|_{L^2} \leq R$}
    For the errors in function values, we have the bound
    \begin{align*}
        \|G_N(a) - \cT(a) \|_{\cZ_{N}}  
        & = \|G_N(a) - \cN_{\cutoff}(a) \cN_{G}(a) \|_{\cZ_{N}}  \\ 
        & \leq 
        \|G_N(a) - \cN_{G}(a) \|_{\cZ_{N}} 
        + \|(1 - \cN_{\cutoff}(a)) \cN_{G}(a) \|_{\cZ_{N}}  \\
        & \leq \epsilon_G + \epsilon_{\cutoff} M_{G}.
    \end{align*}
    For the derivatives, we have 
    \[
     D \cT(a) 
     = D(\cN_{\cutoff}(a) \cN_{G}(a)) 
     = 
     \cN_{\cutoff}(a) D \cN_{G}(a) + 
      \cN_{G}(a)D\cN_{\cutoff}(a), 
    \]
    where we are using the notation
    \[ 
      \left( \cN_{\cutoff}(a) D \cN_{G}(a)  \right) w =
      \cN_{\cutoff}(a) \cdot \left( D \cN_{G}(a)) w\right)
    \]
    and
    \[ 
      \left( \cN_{G}(a) D\cN_{\cutoff}(a) \right) w =
      \cN_{G}(a) \cdot \left( D\cN_{\cutoff}(a) w \right), 
    \]
    for all $w \in \cX_{\delta}$.
    Thus, we can bound the error in the derivatives as follows
    \begin{align*}
        \|D G_{N} - D(\cN_{\cutoff}(a) \cN_{G}(a))\|_{\HS(\cX_{\delta}, \cZ_{N})}
        & \leq
        \|D G_{N} -  D\cN_{G}(a)\|_{\HS(\cX_{\delta}, \cZ_{N})} \\
        & \qquad 
        + \|(1 - \cN_{\cutoff}(a)) D\cN_{G}(a)\|_{\HS(\cX_{\delta}, \cZ_{N})} \\
        & \qquad 
        + \|\cN_{G}(a) D \cN_{\cutoff}(a)\|_{\HS(\cX_{\delta}, \cZ_{N})}
    \end{align*}
    Here it can be checked that     
    \[
        \|(1 - \cN_{\cutoff}(a)) D\cN_{G}(a)\|_{\HS(\cX_{\delta}, \cZ_{N})} 
        \leq 
        | 1 - \cN_{\cutoff}(a)| \|D \cN_{G}(a) \|_{\HS(\cX_{\delta}, \cZ_{N})}
    \]
    and 
    \[
       \|\cN_{G}(a) D \cN_{\cutoff}(a)\|_{\HS(\cX_{\delta}, \cZ_{N})}
       \leq 
       \| \cN_{G}(a)\|_{\cZ_{N}} \|D \cN_{\cutoff}(a)\|_{\HS(\cX_{\delta}, \bR)}.
    \]
    We therefore have 
    \[
    \|D G_{N} - D\cT(a)\|_{\HS(\cX_{\delta}, \cZ_{N})}
    \leq \epsilon_G + \epsilon_{\cutoff} M_{G} + \epsilon_{\cutoff} M_{G} 
    = \epsilon_G + 2 \epsilon_{\cutoff} M_{G} .
    \]

    \paragraph{Case 2: $\|\cP_N a\|_{L^2} \leq R_1$}
    Here we consider upper bounds on $\cT(a)$ and $D \cT(a)$.
    In particular,
    \[
        \|\cT(a)\|_{\cZ_{N}} 
        = \| \cN_{\cutoff}(a) \cN_{G}(a)\|_{\cZ_{N}}
        \leq 
        | \cN_{\cutoff}(a)| \| \cN_{G}(a)\|_{\cZ_{N}}.
    \]
    Since we have constructed $\cN_{G}$ to be accurate over the set $\cP_N(B_{R_1})$, we have 
    \[ \| \cN_{G}(a)\|_{\cZ_{N}} \leq 
    \|G_N(a) - \cN_{G}(a)\|_{\cZ_{N}}
    + \|G_N(a)\|_{\cZ_{N}}
    \leq 
    \epsilon_G
    + \|G_N(a)\|_{\cZ_{N}},
    \]
    which, along with the bound $|\cN_{\cutoff}(a)| \leq M_{\cutoff}$, yields
    \[
        \| \cT(a)\|_{\cZ_{N}}
        \leq 
        M_{\cutoff}(\|G_N(a)\|_{\cZ_{N}} + \epsilon_G).
    \]
    Similarly, 
    \begin{align*}
    \|D \cT(a) \|_{\HS(\cX_{\delta}, \cZ_{N})} 
    & =
    \| D(\cN_{\cutoff}(a) \cN_{G}(a))\|_{\HS(\cX_{\delta}, \cZ_{N})}  \\
    & \leq 
    \| \cN_{\cutoff}(a) D \cN_{G}(a)\|_{\HS(\cX_{\delta}, \cZ_{N})} 
    +
    \| \cN_{G}(a) D \cN_{\cutoff}(a) \|_{\HS(\cX_{\delta}, \cZ_{N})} 
    \\
    & \leq | \cN_{\cutoff}(a)| \| D \cN_{G}(a)\|_{\HS(\cX_{\delta}, \cZ_{N})} 
    +
    \| \cN_{G}(a) \|_{\cZ_{N}} \| D \cN_{\cutoff}(a) \|_{\HS(\cX_{\delta}, \bR)} \\
    & \leq 
    M_{\cutoff}(\|D G_{N}(a)\|_{\HS(\cX_{\delta}, \cZ_{N})} + \epsilon_G ) 
    + 
    M_{\cutoff}(\|G_{N}(a)\|_{\cZ_{N}} + \epsilon_G) \\
    & = 
    M_{\cutoff} \|G_{N}(a)\|_{\cZ_{N}}
    + M_{\cutoff} \|D G_{N}(a)\|_{\HS(\cX_{\delta}, \cZ_{N})}
    + 2 \epsilon_G M_{\cutoff}.
    \end{align*}
    
    \paragraph{Case 3: $\|\cP_N a \|_{L^2} \geq R_1$} 
    We can again derive upper bounds for $\cT(a)$ and $D \cT(a)$. 
    Making use of the value of the cutoff function, we have 
    \[
        \|\cT(a)\|_{\cZ_{N}} = 
        \| \cN_{\cutoff}(a) \cN_{G}(a)\|_{\cZ_{N}} \leq 
        \epsilon_{\cutoff} M_{G},
    \]
    and 
    \begin{align*}
        \|D \cT(a) \|_{\HS(\cX_{\delta}, \cZ_{N})} 
        & =
        \| D(\cN_{\cutoff}(a) \cN_{G}(a))\|_{\HS(\cX_{\delta}, \cZ_{N})}  \\
        & \leq 
        \| \cN_{\cutoff}(a) D \cN_{G}(a)\|_{\HS(\cX_{\delta}, \cZ_{N})} 
        +
        \| \cN_{G}(a) D \cN_{\cutoff}(a) \|_{\HS(\cX_{\delta}, \cZ_{N})}  \\
        & \leq 2 \epsilon_{\cutoff} M_{G}.
    \end{align*}
    Combining the three cases, we have 
    \begin{enumerate}[label=(\roman*)]
        \item For $\|\cP_N a\|_{L^2} \leq R$,
        \begin{align*}
            \|G_{N}(a) - \cT (a)\|_{\cZ_{N}} &\leq \epsilon_G + \epsilon_{\cutoff} M_{G} \\ 
            \|D G_{N}(a) - D \cT (a)\|_{\HS(\cX_{\delta},\cZ_{N})} &\leq \epsilon_{G} + 2 \epsilon_{\cutoff} M_{G} 
        \end{align*}
        \item For any $a \in \cX$
        \begin{align*}
            \|\cT (a)\|_{\cZ_{N}} &\leq M_{\cutoff} \left( \|G_N\|_{\cZ_{N}} + \epsilon_G + \frac{\epsilon_{\cutoff} M_G}{M_{\cutoff}} \right) \\
            \|D \cT (a)\|_{\HS(\cX_{\delta},\cZ_{N})} &\leq 
                M_{\cutoff} \left( \|G_{N}(a)\|_{\cZ_{N}} 
                + \|D G_{N}(a)\|_{\HS(\cX_{\delta}, \cZ_{N})}
                + 2 \epsilon_G 
                + \frac{2 \epsilon_{\cutoff} M_{G}}{M_{\cutoff}}
                \right) .
        \end{align*}
    \end{enumerate}
    Note that $M_{\cutoff}$ depends only on $\sigma, N, d, d_a$ and not the specific instance of $\epsilon_{\cutoff}$, 
    while $M_{G}$ does depend on $\epsilon_G$.
    Thus, we first choose $\epsilon_G$ such that 
    \[ \epsilon_G \leq \min \left\{ \frac{\epsilon}{4}, \frac{1}{8} \right\} \]
    and then choose $\epsilon_{\cutoff}$ such that 
    \[
        \epsilon_{\cutoff} \leq \min \left\{ \frac{\epsilon}{8 M_G},
        \frac{M_{\cutoff}}{8M_G} \right\}
    \]
    With this construction, we have 
    \begin{enumerate}[label=(\roman*)]
        \item For $\|\cP_N a\|_{L^2} \leq R$,
        \begin{align*}
            \|G_{N}(a) - \cT (a)\|_{\cZ_{N}} &\leq \epsilon/2, \\ 
            \|D G_{N}(a) - D \cT (a)\|_{\HS(\cX_{\delta},\cZ_{N})} &\leq \epsilon/2.
        \end{align*}
        \item For any $a \in \cX$,
        \begin{align*}
            \|\cT (a)\|_{\cZ_{N}} &\leq 
            M_{\cutoff} ( \|G_N\|_{\cZ_{N}} + 1/2 ), \\ 
            \|D \cT (a)\|_{\HS(\cX_{\delta},\cZ_{N})} &\leq 
               M_{\cutoff} \left( \|G_N\|_{\cZ_{N}} + \|D G_N \|_{\HS(\cX_{\delta}, \cZ_{N})} + 1/2 \right). 
        \end{align*}
    \end{enumerate}
    Furthermore, since $\|\cP_N a\|_{L^2} \leq \|a\|_{\cX}$, the error bounds also hold for $\|a\|_{\cX} \leq R$.
    
    Although $\cN_{\cutoff}$ and $\cN_{G}(a)$ are both FNOs mapping to the space of constant functions, 
    their product $\cT(a)$ is not yet an FNO. 
    Instead, we consider a neural network $\tilde{g}_{\times} : \bR^2 \rightarrow \bR$
    that approximates the multiplication function 
    ${g}_{\times} : (x_1, x_2) \mapsto x_1 \cdot x_2$
    over the range of $\cN_{\cutoff}$ and $\cN_{G}$.
    That is, 
    \[
        \| g_{\times} - \tilde{g}_{\times} \|_{\cont^1([-M_{\cutoff}, M_{\cutoff}] \times [-M_G, M_G])} \leq \epsilon_{\times},
    \]
    where $\epsilon_{\times} > 0$ can be made arbitrarily small.
    We then take $\cN(a) := \tilde{g}_{\times}(\cN_{\cutoff}(a), \cN_{G}(a))$,
    where the operation is understood to be componentwise for every $[\cN_G(a)]_i$, $i = 1, \dots, 2K_N$.
    
    Therefore, $\cN(a)$ approximates $\cT(a)$ with output errors 
    \[
        \|  \cN(a) - \cT(a) \|_{\cZ_{N}} \leq
        \const_{2, \infty} \epsilon_{\times},
    \]
    where $\const_{2, \infty}$ is again the norm equivalence constant such that  
    $\|v\|_{L^2} \leq \const_{2,\infty} \|v\|_{L^{\infty}}$ for $v \in \cZ_{N} = L^2_0(\bT^d;\bR^{2K_N})$.
    For the derivatives, we have 
    \begin{align*}
       &\| D \cN(a) - D \cT(a) \|_{\HS(\cX_{\delta}, \cZ_{N})}  \\
       & \quad \leq 
       \| 
           \left( \partial_1 \tilde{g}_{\times}(\cN_{\cutoff}(a), \cN_{G}(a))
            - \partial_1 {g}_{\times}(\cN_{\cutoff}(a), \cN_{G}(a))
           \right) 
           D \cN_{\cutoff}(a) 
        \|_{\HS(\cX_{\delta}, \cZ_{N})}  \\
        & \qquad 
        + 
       \| 
           \left( \partial_2 \tilde{g}_{\times}(\cN_{\cutoff}(a), \cN_{G}(a))
            - \partial_2 {g}_{\times}(\cN_{\cutoff}(a), \cN_{G}(a))
           \right) 
           D \cN_{G}(a) 
        \|_{\HS(\cX_{\delta}, \cZ_{N})}.
    \end{align*}
    Here, the first term is simply
    \begin{align*}
        & \| 
           \left( \partial_1 \tilde{g}_{\times}(\cN_{\cutoff}(a), \cN_{G}(a))
            - \partial_1 {g}_{\times}(\cN_{\cutoff}(a), \cN_{G}(a))
           \right) 
           D \cN_{\cutoff}(a) 
        \|_{\HS(\cX_{\delta}, \cZ_{N})} \\
        & \qquad \leq 
        \| \partial_1 \tilde{g}_{\times}(\cN_{\cutoff}(a), \cN_{G}(a))
            - \partial_1 {g}_{\times}(\cN_{\cutoff}(a), \cN_{G}(a))
        \|_{\cZ_{N}}
        \| D \cN_{\cutoff}(a)\|_{\HS(\cX_{\delta}, \bR)} \\
        & \qquad \leq 
        \const_{2, \infty} \epsilon_\times M_{\cutoff},
    \end{align*}
    The second term yields
    \begin{align*}
        & \| 
           \left( \partial_2 \tilde{g}_{\times}(\cN_{\cutoff}(a), \cN_{G}(a))
            - \partial_2 {g}_{\times}(\cN_{\cutoff}(a), \cN_{G}(a))
           \right) 
           D \cN_{G}(a) 
        \|_{\HS(\cX_{\delta}, \cZ_{N})} \\
        & \qquad \leq 
        \| \partial_2 \tilde{g}_{\times}(\cN_{\cutoff}(a), \cN_{G}(a))
            - \partial_2 {g}_{\times}(\cN_{\cutoff}(a), \cN_{G}(a))
            \|_{\Op(\cZ_{N}, \cZ_{N})}
        \|
         D \cN_{G}(a) 
        \|_{\HS(\cX_{\delta}, \cZ_{N})}.
    \end{align*}
    We note that since $\tilde{g}_{\times}(\cN_{\cutoff}(a), \cN_{G}(a))$ is applied componentwise to $\cN_{G}(a)$, 
    $\partial_2 \tilde{g}_{\times}(\cN_{\cutoff}(a), \cN_{G}(a))$
    is associated with a $2K_N \times 2K_N$ diagonal matrix whose diagonal values are
    \[ 
    [\partial_2 \tilde{g}_{\times}(\cN_{\cutoff}(a), \cN_{G}(a))]_{ii}
        = \partial_2 \tilde{g}_{\times}(\cN_{\cutoff}(a), [\cN_{G}(a)]_i), 
        \qquad i = 1, \dots, 2K_N.
    \]
    The analogous is true of $\partial_2 g_{\times}(\cN_{\cutoff}(a), \cN_{G}(a))$.
    Therefore, the $\Op(\cZ_{N}, \cZ_{N})$ norm of their difference is given by
    \begin{align*}
        & \| \partial_2 \tilde{g}_{\times}(\cN_{\cutoff}(a), \cN_{G}(a))
            - \partial_2 {g}_{\times}(\cN_{\cutoff}(a), \cN_{G}(a))
            \|_{\Op(\cZ_{N}, \cZ_{N})} \\
        & \qquad = 
        \max_{i = 1, \dots, 2K_N} | \partial_2 \tilde{g}_{\times}(\cN_{\cutoff}(a), [\cN_{G}(a)]_i) - \partial_2 {g}_{\times}(\cN_{\cutoff}(a), [\cN_{G}(a)]_i)| \\
        & \qquad \leq \epsilon_{\times}.
    \end{align*}
    That is, we have
    \[
       \| D \cN(a) - D \cT(a) \|_{\HS(\cX_{\delta}, \cZ_{N})} 
       \leq \const_{2,\infty} M_{\cutoff} \epsilon_{\times} + M_G \epsilon_{\times}.
    \] 
    
    We then let $\tilde{\epsilon} \leq \min \{ \epsilon, M_{\cutoff} \}$ and take 
    \[ 
        \epsilon_{\times} \leq \min \left\{\frac{\const_{2, \infty} \tilde{\epsilon}}{2}, \frac{\const_{2, \infty} M_{\cutoff} \tilde{\epsilon}}{4}, \frac{M_G \tilde{\epsilon}}{4} \right\}, 
    \]
    such that for any $a \in \cX$,
    \[
        \|\cN(a) - \cT(a) \|_{\cZ_{N}} \leq 
        \min \left\{ \frac{\epsilon}{2}, \frac{M_{\cutoff}}{2} \right\} 
        \qquad 
        \|D \cN(a) - D \cT(a) \|_{\HS(\cX_{\delta}, \cZ_{N})} \leq 
        \min \left\{ \frac{\epsilon}{2}, \frac{M_{\cutoff}}{2} \right\} .
    \]
    Thus, overall, we have $\cN(a)$ satisfying the following.
    \begin{enumerate}[label=(\roman*)]
        \item For $\|a\|_{\cX} \leq R$, 
        \[
            \|G_N(a) - \cN(a) \|_{\cZ_{N}} \leq 
            \|G_N(a) - \cT(a) \|_{\cZ_{N}} + \|\cT(a) - \cN(a) \|_{\cZ_{N}} 
            \leq \epsilon
        \]
        and 
        \begin{align*}
            \|D G_N(a) - D \cN(a) \|_{\HS(\cX_{\delta}, \cZ_{N})} 
            &\leq 
            \|D G_N(a) - D \cT(a) \|_{\HS(\cX_{\delta}, \cZ_{N})}  \\
                & \quad + \|D \cT(a) - D \cN(a) \|_{\HS(\cX_{\delta}, \cZ_{N})} \\
            &\leq \epsilon.
        \end{align*}
        \item For any $a \in \cX$, 
        \begin{align*}
            \|\cN(a)\|_{\cZ_{N}} 
            & \leq  \|\cT(a)\|_{\cZ_{N}} + \|\cT(a) - \cN(a)\|_{\cZ_{N}}  \\ 
            & \leq M_{\cutoff} (\|G_N(a)\|_{\cZ_{N}} + 1/2) + M_{\cutoff}/2  \\
            & \leq M_{\cutoff} (\|G_N(a)\|_{\cZ_{N}} + 1)
        \end{align*}
        and 
        \begin{align*}
            \|D \cN(a)\|_{\HS(\cX_{\delta}, \cZ_{N})} 
            & \leq 
            \|D \cT(a)\|_{\HS(\cX_{\delta}, \cZ_{N})}
                + \|D \cT(a) - D \cN(a)\|_{\HS(\cX_{\delta}, \cZ_{N})} 
            \\
            & \leq 
            M_{\cutoff} (\|G_N(a)\|_{\cZ_{N}} + \|D G_N(a)\|_{\HS(\cX_{\delta}, \cZ_{N})} + 1/2) + M_{\cutoff}/2  \\ 
            & \leq 
            M_{\cutoff} (\|G_N(a)\|_{\cZ_{N}} + \|D G_N(a)\|_{\HS(\cX_{\delta}, \cZ_{N})} + 1).
        \end{align*}
    \end{enumerate}
    This shows that \eqref{eq:coefficient_function_growth} and \eqref{eq:coefficient_derivative_growth} hold with $\const = M_{\cutoff}(\sigma, N, d, d_a)$.

    Finally, to see \eqref{eq:coefficient_bound_with_cutoff}, we note that both $\cN_{\cutoff}(a)$ and $\cN_{G}(a)$ take bounded values in $\bR$ and $\cZ_{N}$ respectively, and $\tilde{g}_{\times}$ is a continuous function, so that 
    $\tilde{g}_{\times}(\cN_{\cutoff}(a), \cN_{G}(a))$ takes bounded values in $\cZ_{N}$.
    For the derivatives, we recognize that as in the proof of \Cref{lemma:coefficient_approximation_bounded}
    $\cN(a)$,  is an FNO of the form 
    \[
        \cN(a) = \cQ \circ \cL_{L} \circ \dots \circ \cL_1 \circ \cR \circ \cP_N (a),
    \]
    where $\cQ$, $\cL_{\ell}$, $\cR$, and $\cP_N$ are each uniformly bounded for $a \in \cX$.
    Thus, $\|D \cN(a)\|_{\HS(\cX_{\delta}, \cZ_{N})}$ is uniformly bounded.
    
\end{proof}

We can turn this into an approximation of the truncated operator $\cG_N : \cX \rightarrow \cY$ by composition with an approximation of the inverse Fourier transform from \Cref{lemma:inverse_ft_approximation}.
\begin{lemma}[Truncated Operator Approximation with Bounded Growth]
    \label{corollary:operator_approximation_bounded_growth}
    Let $s, s', \delta \geq 0$ 
    and let $\cX := H^s(\bT^d; \bR^{d_a})$,
    $\cY := H^{s'}(\bT^d; \bR^{d_a})$,
    and $\cX_{\delta} := H^{s+\delta}(\bT^d; \bR^{d_a})$,
    Suppose $\cG \in \cont^{1}(\cX;\cY)$.
    Furthermore, for a fixed $N \in \bN$, let $\cG_N := \cP_N \circ \cG \circ \cP_N$.
    Then, given $\sigma \in \srla$, there exists $\const = \const(\sigma, N, d, d_a)$ such that 
    for any $\epsilon > 0$ and $R > 1$, there exists an FNO 
    $\cN \in \cont^1(\cX ; \cY)$ 
    with activation function $\sigma$ such that 
    \begin{align}
        \|\cG_N(a) - \cN(a)\|_{\cY} &\leq \epsilon, \\
        \|D \cG_N(a) - D\cN(a)\|_{\HS(\cX_{\delta}, \cY)} &\leq \epsilon,
    \end{align}
    whenever $\|a\|_{\cX} \leq R$. Moreover, 
    \begin{align}
        \| \cN(a) \|_{\cY} &\leq \const ( \|\cG_N(a)\|_{\cY} + 1) \\
        \| D \cN(a) \|_{\HS(\cX_{\delta}, \cY)} &\leq \const ( \|\cG_N(a)\|_{\cY} + \|D \cG_N(a) \|_{\HS(\cX_{\delta}, \cY)} + 1) 
    \end{align}
    for all $a \in \cX$.
\end{lemma}
\begin{proof}
    We consider the approximation $\cN_{G}$ to $G_N = \cF_{N} \cG_{N}$ from \Cref{lemma:coefficient_approximation_bounded_growth} satisfying 
    \begin{enumerate}[label=(\roman*)]

    \item For $\|a\|_{\cX} \leq R$,
    \begin{align*}
        \|G_N(a) - \cN_G(a)\|_{\cZ_{N}} &\leq \epsilon_G, \\
        \|D G_N(a) - D\cN_G(a)\|_{\HS(\cX_{\delta}, \cZ_{N})} &\leq \epsilon_G.
    \end{align*}
    \item For all $a \in \cX$,
    \begin{align*}
        \| \cN_G(a) \|_{\cZ_{N}} &\leq \const_G ( \|G_N(a)\|_{\cZ_{N}} + 1) \\
        \| D \cN_G(a) \|_{\HS(\cX_{\delta}, \cZ_{N})} &\leq \const_G ( \|G_N(a)\|_{\cZ_{N}} + \|D G_N(a) \|_{\HS(\cX_{\delta}, \cZ_{N})} + 1) 
    \end{align*}
    \item For all $a \in \cX$,
    \begin{equation*}
        \| \cN_G(a)\|_{\cZ_{N}} \leq M_G, \qquad \|D \cN_G(a) \|_{\HS(\cX_{\delta}, \cZ_{N})} \leq M_G.
    \end{equation*}
    \end{enumerate}
    Here $M_G$ depends on $\epsilon_G$ and $R$ while $\const_G$ does not.

    By \Cref{lemma:inverse_ft_approximation}, for any $\epsilon_{\ift} > 0 $, we can construct an approximation $\cN_{\ift}$ of the inverse Fourier transform, $\real \cF_{N}^{-1}$, such that 
    \[
        \| \real \cF_{N}^{-1}v - \cN_{\ift}(v) \|_{\cY} \leq \epsilon_{\ift}, 
        \qquad 
        \| \real \cF_{N}^{-1}v - D \cN_{\ift}(v) \|_{\Op(\cZ_{N}, \cY)} \leq \epsilon_{\ift} 
    \]
    whenever $\|v\|_{\cZ_{N}} \leq M_G$. 

    We then have the following cases.

    \paragraph{Case 1: $\|a\|_{\cX} \leq R$}
    The output errors are bounded as 
    \begin{align*}
        \|\cG_{N}(a) - \cN(a) \|_{\cY}
        & \leq 
        \|\real \cF_{N}^{-1} G_N(a) - \real \cF_{N}^{-1} \cN_{G}(a) \|_{\cY}
        +
        \|\real \cF_{N}^{-1} \cN_{G}(a) - \cN_{\ift}(\cN_G(a)) \|_{\cY} \\
        & \leq 
        \|\real \cF_{N}^{-1} \|_{\Op(\cZ_{N}, \cY)} 
            \|G_N(a) - \cN_{G}(a) \|_{\cZ_{N}}
        +
        \|\real \cF_{N}^{-1} \cN_{G}(a) - \cN_{\ift}(\cN_G(a)) \|_{\cY}
        \\
        & \leq \|\real \cF_{N}^{-1} \|_{\Op(\cZ_{N}, \cY)} \epsilon_G + \epsilon_{\ift}
    \end{align*}
    and the derivative errors are bounded as 
    \begin{align*}
        \|D \cG_{N}(a) - D \cN(a) \|_{\HS(\cX_{\delta}, \cY)}
        &\leq 
            \|\real \cF_{N}^{-1} D G_{N}(a) - \real \cF_{N}^{-1} D \cN_{G}(a) \|_{\HS(\cX_{\delta}, \cY)} \\
        & \qquad + 
            \|\real \cF_{N}^{-1} D \cN_{G}(a) - D_{\cN_{\ift}}(\cN_{G}(a)) D \cN_{G}(a) \|_{\HS(\cX_{\delta}, \cY)} \\ 
        & \leq 
            \|\real \cF_{N}^{-1} \|_{\Op(\cZ_{N}, \cY)}
            \|D G_{N}(a) - D \cN_{G}(a) \|_{\HS(\cX_{\delta}, \cZ_{N})} \\
        & \qquad +  
            \|\real \cF_{N}^{-1} - D \cN_{\ift}(\cN_{G}(a)) \|_{\Op(\cZ_{N}, \cY)}
            \|D \cN_{G}(a) \|_{\HS(\cX_{\delta}, \cZ_{N})} \\ 
        & \leq \|\real \cF_{N}^{-1}\|_{\Op(\cZ_{N}, \cY)} \epsilon_G + M_G \epsilon_{\ift}.
    \end{align*}

    \paragraph{Case 2: Any $a \in \cX$}
    The output values are bounded as follows
    \begin{align*}
    \|\cN(a)\|_{\cY} 
        & \leq \|  \real \cF_{N}^{-1} \cN_{G}(a) - \cN_{\ift}(\cN_{G}(a)) \|_{\cY} 
            + \| \real \cF_{N}^{-1} \cN_G(a)\|_{\cZ_{N}} \\
        & \leq \epsilon_{\ift} + \|\real \cF_{N}^{-1} \|_{\Op(\cZ_{N}, \cY)} \| \cN_{G}(a)\|_{\cZ_{N}} \\
        & \leq \epsilon_{\ift} + \|\real \cF_{N}^{-1} \|_{\Op(\cZ_{N}, \cY)} 
            \const_G
            \left( 
                \|G_N(a)\|_{\cZ_{N}} + 1     
            \right),
    \end{align*}
    and the derivatives are bounded as follows
    \begin{align*}
        \|D \cN(a)\|_{\HS(\cX_{\delta}, \cY)}
        & \leq 
            \| D \cN_{\ift}(\cN_{G}(a)) \|_{\Op(\cZ_{N},\cY)}
            \| D \cN_{G}(a) \|_{\HS(\cX_{\delta}, \cZ_{N})} \\
        & \leq 
            \left(
            \| \real \cF_{N}^{-1} - D \cN_{\ift}(\cN_{G}(a)) \|_{\Op(\cZ_{N}, \cY)}
            +
            \| \real \cF_{N}^{-1} \|_{\Op(\cZ_{N}, \cY)}
            \right)
            \| D \cN_G(a)\|_{\HS(\cX_{\delta}, \cZ_{N})} \\
        & \leq
            \left( 
                \epsilon_{\ift} + \|\real \cF_{N}^{-1} \|_{\Op(\cZ_{N}, \cY)}
           \right) 
           \const_G \left(
            \|G_{N}(a)\|_{\cZ_{N}} + \|D G_{N}(a)\|_{\HS(\cX_{\delta}, \cZ_{N})} + 1
           \right).
    \end{align*}
    Moreover, since $G_{N} = \cF_{N} \cG_N$, we have 
    \[
        \|G_{N}(a)\|_{\cZ_{N}} \leq \|\cF_{N}\|_{\Op(\cY, \cZ_{N})} \|\cG_N(a)\|_{\cY}
    \]
    and 
    \[
        \|D G_{N}(a)\|_{\HS(\cX_{\delta}, \cZ_{N})} \leq \|\cF_{N}\|_{\Op(\cY, \cZ_{N})} \|D \cG_N (a)\|_{\HS(\cX_{\delta}, \cY)}.
    \]
    Thus, we can bound the growth of $\cN$ by
    \[
        \|\cN(a)\|_{\cY} \leq \epsilon_{\ift} + \const_G \|\real \cF_{N}^{-1}\|_{\Op(\cZ_{N}, \cY)}
        \left( 
            \|\cF_{N}\|_{\Op(\cY, \cZ_{N})} \| \cG_{N}(a) \|_{\cY} + 1  
        \right)
    \]
    and its derivative by 
    \begin{align*}
        \|D \cN(a)\|_{\HS(\cX_{\delta}, \cY)} 
        & \leq 
           \left( 
                \epsilon_{\ift} + \|\real \cF_{N}^{-1} \|_{\Op(\cZ_{N}, \cY)}
           \right) \\
           & \quad \times \const_G 
           \left(
            \|\cF_{N}\|_{\Op(\cY,\cZ_{N})}\|\cG_{N}(a)\|_{\cY} 
                + \|\cF_{N}\|_{\Op(\cY, \cZ_{N})} \|D \cG_{N}(a)\|_{\HS(\cX_{\delta}, \cY)} + 1
           \right).
    \end{align*}

    \paragraph{Overall Bounds} To obtain the desired result, we choose
    \[
    \epsilon_{G} \leq \frac{\epsilon}{2 \|\real \cF_{N}^{-1}\|_{\Op(\cZ_{N}, \cY)}}
    \]
    and 
    \[
        \epsilon_{\ift} \leq \min \left\{ 
            \frac{\epsilon}{2},
            \frac{\epsilon}{2 M_G},
            \|\real \cF_{N}^{-1}\|_{\Op(\cZ_{N}, \cY)},
            \const_{G} \|\real \cF_{N}^{-1}\|_{\Op(\cZ_{N}, \cY)}
        \right\}.
    \]
    With this choice of $\epsilon_G, \epsilon_{\ift}$, we have the following.
    \begin{enumerate}[label=(\roman*)]
       \item For $\|a\|_{\cX} \leq R$, 
       \begin{align*}
            \| \cG_N(a) - \cN(a) \|_{\cY} &\leq \epsilon \\ 
            \| D \cG_N(a) - D \cN(a) \|_{\HS(\cX_{\delta}, \cY)} &\leq \epsilon.
       \end{align*}
       \item For any $a \in \cX$,
       \begin{align*}
            \| \cN(a) \|_{\cY} &\leq 
            \const \left( 
                \|\cG_{N}(a) \|_{\cY} + 1
            \right), \\ 
            \|  D \cN(a) \|_{\HS(\cX_{\delta}, \cY)} &\leq 
            \const \left(  \|\cG_{N}(a)\|_{\cY} + \| D \cG(a)\|_{\HS(\cX_{\delta}, \cY)} + 1
            \right),
        \end{align*}
        with 
        \[\const = \const(\sigma, N, d, d_a) = 2 \const_G(\sigma, N, d, d_a) \| \real \cF_{N}^{-1} \|_{\Op(\cZ_{N}, \cY)} \max \{ 
            \|\cF_{N} \|_{\Op(\cY, \cZ_{N})}, 1\}\]
        as the constant, noting that it does not depend on $\epsilon$ or $R$.
    
   \end{enumerate}

\end{proof}

We are now in the position to construct an $W^{1,2}_{\mu,\delta}$ approximation of the truncated operator $\cG_{N}$.
\begin{lemma}[Approximation of the Truncated Operator in $W^{1,2}_{\mu,\delta}$]
    \label{lemma:truncation_approx_mu}
    Let $s, s', \delta \geq 0$ and let $\cX = H^s(\bT^d; \bR^{d_a})$, $\cY = H^{s'}(\bT^d; \bR^{d_u})$, and $\cX_{\delta} = H^{s + \delta}(\bT^d; \bR^{d_a})$.
    Suppose $\mu$ is a Borel probability measure on $\cX$ and $\cG \in \cont^{1}(\cX;\cY)$ is such that $\cG \in L^2_{\mu}(\cX; \cY)$
    and $D \cG \in L^2_{\mu}(\cX; \HS(\cX_{\delta}, \cY))$.
    Then, given $\sigma \in \srla$, for any $\epsilon > 0$,
    there exists an FNO 
    $\cN \in \cont^1( \cX ; \cY)$ 
    with activation function $\sigma$
    such that
    \begin{equation}
     \| \cG_N - \cN \|_{L^2_{\mu}(\cX; \cY)}^2 + \| D \cG_N - D \cN \|_{L^2_{\mu}(\cX; \HS(\cX_{\delta}, \cY))}^2 \leq \epsilon^2,
    \end{equation}
    where $\cG_{N} = \cP_N \circ \cG \circ \cP_{N}$.
\end{lemma}
\begin{proof}
    By \Cref{lemma:coefficient_approximation_bounded_growth}, for any $\epsilon_{\cG} > 0, R > 1$, we can construct an FNO $\cN$ such that 
    \begin{align*}
        \|\cG_N(a) - \cN(a)\|_{\cY} &\leq \epsilon_{\cG}, \\
        \|D \cG_N(a) - D \cN(a)\|_{\HS(\cX_{\delta}, \cY)} &\leq \epsilon_{\cG},
    \end{align*}
    whenever $\|a\|_{\cX} \leq R$, and 
    \begin{align*}
        \|\cN(a)\|_{\cY} &\leq \const_{\cG} ( \| \cG_{N} \|_{\cY} + 1 ), \\
        \|D \cN(a)\|_{\HS(\cX_{\delta}, \cY)} 
        &\leq \const_{\cG}(
            \|\cG_{N}(a)\|_{\cY}  
            + \|D \cG_N (a)\|_{\HS(\cX_{\delta}, \cY)} 
            + 1 
            ),
    \end{align*}    
    for all $a \in \cX$.
    We can also write this latter bound in terms of the squares of the norms, 
    \[
        \|\cN(a)\|_{\cY}^2 \leq \const^2  ( \| \cG_{N}(a) \|_{\cY}^2 + 1 )
    \]
    and 
    \[
    \|D \cN(a)\|_{\HS(\cX_{\delta}, \cY)}^2 \leq 
        \const^2(
        \|\cG_{N}(a)\|_{\cY}^2
        + \|D \cG_N (a)\|_{\HS(\cX_{\delta}, \cY)}^2
        + 1 
        )
    \]
    where the inequalities hold with $\const^2 = 3 \const_{\cG}^2$, the value of which does not depend on $\epsilon_{\cG}$ or $R$.

    Given such an approximation, its $L^2_{\mu}$ error can be decomposed into 
    \[
        \int_{\cX} \|\cG_{N}(a) - \cN(a)\|_{\cY}^2 d\mu(a) 
        =
        \int_{B_R} \|\cG_{N}(a) - \cN(a)\|_{\cY}^2 d\mu(a) 
        + \int_{B_R^c} \|\cG_{N}(a) - \cN(a)\|_{\cY}^2 d\mu(a),
    \]
    where $B_R = B_R(0) = \{ \|a\|_{\cX} \leq R \}$ and $B_R^c = \cX \setminus B_R$.
    The bounds on $\cN$ yield 
    \[
        \int_{\cX} \|\cG_{N}(a) - \cN(a)\|_{\cY}^2 d\mu(a) 
        \leq 
        \epsilon_{\cG}^2 
        + 2 \int_{B_R^c} 
        \left( 
            \|\cG_{N}(a)\|_{\cY}^2 + \const^2(\|\cG_{N}(a)\|_{\cY}^2 + 1)
        \right)
        d\mu(a).
    \]
    Similarly, for the derivatives, we have 
    \begin{align*}
        & \int_{\cX} \|D \cG_{N}(a) - D \cN(a)\|_{\HS(\cX_{\delta}, \cY)}^2 d\mu(a)  \\
        & \qquad \leq \epsilon_{\cG}^2 
        + 
        2 \int_{B_R^c} \left(
            \|D \cG_{N}(a)\|_{\HS(\cX_{\delta},\cY)}^2 
            + 
            \const^2 ( 
            \|\cG_{N}(a)\|_{\cY}^2 
            + \|D \cG_{N}(a)\|_{\HS(\cX_{\delta},\cY)}^2 
            + 1 
            )
        \right)
        d \mu(a).
    \end{align*}
    Since, $\cG_{N} \in L^2_{\mu}(\cX; \cY)$ 
    and $D \cG_{N} \in L^2_{\mu}(\cX; \HS(\cX_{\delta}, \cY))$,
    the dominanted convergence theorem implies that there exists $R$ such that simultaneously 
    \[
    2 \int_{B_R^c} 
        \left( 
            \|\cG_{N}(a)\|_{\cY}^2 + \const^2(\|\cG_{N}(a)\|_{\cY}^2 + 1)
        \right)
        d\mu(a)
     \leq \epsilon^2 /4
    \]
    and
    \[
        2 \int_{B_R^c} \left(
        \|D \cG_{N}(a)\|_{\HS(\cX_{\delta},\cY)}^2 
        + 
        \const^2 ( 
        \|\cG_{N}(a)\|_{\cY}^2 
        + \|D \cG_{N}(a)\|_{\HS(\cX_{\delta},\cY)}^2 
        + 1 
        )
        \right)
        d \mu(a) \leq \epsilon^2/4. 
    \]
    We then choose this $R$ as the radius for \Cref{lemma:coefficient_approximation_bounded_growth},
    and construct $\cN$ to have error $\epsilon_{\cG} \leq \epsilon/2$ within $B_R$.
    With our choices of $R$ and $\epsilon_{\cG}$, we have 
    \[
        \int_{\cX} \|\cG_{N}(a) - \cN(a)\|_{\cY}^2 d \mu(a) 
        + 
        \int_{\cX} \|D \cG_{N}(a) - D \cN(a)\|_{\HS(\cX_{\delta}, \cY)}^2 d \mu(a) 
        \leq \epsilon^2.
    \]
\end{proof}

\subsection{Proof of the $W^{1,2}_{\mu, \delta}$ Theorem}\label{sec:proof_of_main_ua_mu}
Finally, by combining the truncation error and the approximation result for the truncated operator, we arrive at the universal approximation result in $W^{1,2}_{\mu,\delta}$.
\begin{proof}(of \Cref{theorem:main_ua_mu})
    Since $\cG$ satisfies \Cref{assumption:measurability}, by \Cref{lemma:truncation_mu}
    there exists $N \in \bN$ such that 
    \[
        \| \cG - \cG_{N} \|_{L^2_{\mu}(\cX; \cY)}^2
        + 
        \| D \cG - D \cG_{N} \|_{L^2_{\mu}(\cX; \HS(\cX_{\delta}, \cY))}^2
        \leq \epsilon^2/4.
    \]
    Then, by \Cref{lemma:truncation_approx_mu}, there exists an FNO $\cN$
    such that 
    \[
        \|  \cG_{N} - \cN\|_{L^2_{\mu}(\cX; \cY)}^2
        + 
        \| D \cG_{N} - \cN \|_{L^2_{\mu}(\cX; \HS(\cX_{\delta}, \cY))}^2
        \leq \epsilon^2/4.
    \]
    The triangle inequality then yields
    \begin{align*}
        & \| \cG - \cN \|_{L^2_{\mu}(\cX; \cY)}^2 + \| D \cG - D \cN \|_{L^2_{\mu}(\cX; \HS(\cX_{\delta}, \cY))}^2 \\
        & \qquad \leq  
        2 \left( \| \cG - \cG_{N} \|_{L^2_{\mu}(\cX; \cY)}^2
        +
        \|  \cG_{N} - \cN\|_{L^2_{\mu}(\cX; \cY)}^2
        \right) 
        \\ 
        & \qquad \qquad + 2 \left( 
        \| D \cG - D \cG_{N} \|_{L^2_{\mu}(\cX; \HS(\cX_{\delta}, \cY))}^2
        +
        \| D \cG_{N} - \cN \|_{L^2_{\mu}(\cX; \HS(\cX_{\delta}, \cY))}^2
        \right)  \\
        & \qquad \leq \epsilon^2.
    \end{align*}
    
\end{proof}

\end{document}